\documentclass[11pt]{article}
\usepackage[utf8]{inputenc}
\usepackage[T1]{fontenc}
\usepackage[english]{babel}
\usepackage{csquotes}
\usepackage{amsmath,amssymb,amsfonts}
\usepackage{float}
\usepackage{cancel}
\usepackage{amsthm}
\usepackage{geometry}
\usepackage[svgnames]{xcolor}
\usepackage{mathtools}
\usepackage{bm}
\usepackage{booktabs}
\usepackage{subcaption}
\usepackage{rotating}
\usepackage{siunitx}
\usepackage{hyperref}
\usepackage[shortlabels]{enumitem}
\setlist[enumerate,1]{label=\roman*)}
\usepackage[export]{adjustbox} % für valign
\mathtoolsset{showonlyrefs}

\usepackage{algorithm2e}
\RestyleAlgo{ruled}
\makeatletter
\newcommand{\nosemic}{\renewcommand{\@endalgocfline}{\relax}}% Drop semi-colon ;
\newcommand{\dosemic}{\renewcommand{\@endalgocfline}{\algocf@endline}}% Reinstate semi-colon ;
\makeatother

\newtheorem{theorem}{Theorem}
\newtheorem{lemma}[theorem]{Lemma}
\newtheorem{proposition}[theorem]{Proposition}
\newtheorem{remark}[theorem]{Remark}

\newtheorem{corollary}[theorem]{Corollary}
\theoremstyle{definition}
\newtheorem{example}[theorem]{Example}
\newtheorem{assumption}{Assumption}

\newcommand{\R}{\mathbb{R}}
\newcommand{\N}{\mathbb{N}}
\renewcommand{\d}{\mathop{}\!\mathrm{d}}
\renewcommand{\P}{\mathcal{P}}
\renewcommand{\H}{\mathcal{H}}
\renewcommand{\div}{\mathrm{div}}
\newcommand{\C}{\mathcal{C}}
\DeclareMathOperator{\NN}{\mathcal N}
\newcommand{\M}{\mathcal M}

\newcommand{\F}{\mathcal F}
\newcommand{\zb}{\bm}
\DeclareMathOperator{\TV}{TV}

\newcommand{\KL}{\mathrm{KL}}
\DeclareMathOperator{\emb}{\rm emb}

\DeclareMathOperator*{\argmin}{arg\,min}
\DeclareMathOperator*{\argmax}{arg\,max}
\DeclareMathOperator{\dom}{dom}
\DeclareMathOperator{\id}{id}
\DeclareMathOperator{\Lip}{Lip}
\DeclareMathOperator{\prox}{prox}
\DeclareMathOperator{\supp}{supp}
\newcommand{\tT}{\mathrm{T}}
\newcommand{\opt}{\mathrm{opt}}
\usepackage{dsfont}
\DeclareMathOperator{\1}{\mathds{1}}
\DeclareMathOperator{\sgn}{sgn}
\DeclareMathOperator{\E}{\mathbb E}
\DeclareMathOperator{\ran}{ran}
\newcommand{\weakly}{\rightharpoonup}\normalfont
\DeclareMathOperator{\ess}{ess}
\DeclareMathOperator{\loc}{loc}
\DeclareMathOperator{\intt}{int}

\begin{document}
\renewcommand{\thefootnote}{\fnsymbol{footnote}}

\title{Wasserstein Gradient Flows for Moreau Envelopes of $f$-Divergences in Reproducing Kernel Hilbert Spaces}
\author{
Viktor Stein\footnotemark[3]
\and
Sebastian Neumayer\footnotemark[2]
\and
Nicolaj Rux\footnotemark[3]
\and
Gabriele Steidl\footnotemark[3] 
}
 \date{\today}
\maketitle
\footnotetext[3]{Institute of Mathematics,
TU Berlin,
Stra{\ss}e des 17. Juni 136, 
10587 Berlin, Germany,
\{stein, steidl,rux\}@math.tu-berlin.de}
\footnotetext[2]{Institute of Mathematics,
TU Chemnitz,
Reichenhainer Straße 39,
09126 Chemnitz, Germany
sebastian.neumayer@mathematik.tu-chemnitz.de
}

\begin{abstract}
Commonly used $f$-divergences of measures, e.g., the Kullback-Leibler divergence, are subject to limitations regarding the support of the involved measures. 
A remedy is regularizing the $f$-divergence by a squared maximum mean discrepancy (MMD) associated with a characteristic kernel $K$.
We use the kernel mean embedding to show that this regularization can be rewritten as the Moreau envelope of some function on the associated reproducing kernel Hilbert space.
Then, we exploit well-known results on Moreau envelopes in Hilbert spaces to analyze the MMD-regularized $f$-divergences, particularly their gradients.
Subsequently, we use our findings to analyze Wasserstein gradient flows of MMD-regularized $f$-divergences.
We provide proof-of-the-concept numerical examples for flows starting from empirical measures.
Here, we cover $f$-divergences with infinite and finite recession constants.
Lastly, we extend our results to the tight variational formulation of $f$-divergences and numerically compare the resulting flows.
\end{abstract}

%--------------------------------------------------------------------------
\section{Introduction}
%----------------------------------------------------------------------
In variational inference \cite{BKM2017,JGJS1999} and generative modeling \cite{AFHHSS2021,GPMXWOCB14}, a common task is to minimize the $f$-divergence for a fixed target measure over some hypothesis space.
For this, many different $f$-divergences were deployed in the literature, such as the Kullback-Leibler (KL) divergence \cite{KL1951}, Tsallis-$\alpha$ divergences \cite{NN11}, power divergences \cite{LMS2017}, Jeffreys and Jensen-Shannon divergences \cite{L1991}, and Hellinger distances \cite{H1909}.
If the recession constant is infinite, then the hypothesis space in the above tasks reduces to reweighted target samples.
To eliminate this disadvantage, regularized $f$-divergences can be applied.
The probably most well-known case is the regularization of $F = \KL(\cdot \mid \nu)$ with target measure $\nu$ by the squared Wasserstein-2 distance.
This regularization appears in the backward scheme
for computing the gradient flow of $F$ in the Wasserstein geometry \cite{JKO1998}.
It can be considered as a Moreau envelope 
of $F$ in the Wasserstein-2 space.

In this paper, we discuss the regularization of $f$-divergences by a squared maximum mean discrepancy (MMD) induced by a characteristic kernel $K$. 
Since the space of signed Borel measures embeds into the associated reproducing kernel Hilbert space (RKHS), the MMD can be rewritten as a distance in this RKHS.
As our first contribution, we establish a link between our MMD-regularized $f$-divergences and Moreau envelopes of a specific functional on this RKHS.
Here, two main challenges arise.
First, covering $f$-divergences both with a finite and an infinite recession constant makes the analysis more involved.
Second, as the kernel mean embedding is not surjective, we 
must suitably extend the $f$-divergence to the RKHS and have to prove that the corresponding functional is lower semicontinuous.
Due to the established link, many of the well-known properties of Moreau envelopes also hold for our regularized $f$-divergences.
As our second contribution, we analyze Wasserstein gradient flows of the regularized $f$-divergences, particularly the associated particle gradient flows.
Using our theoretical insights, we prove that the regularized $f$-divergences are $\lambda$-convex along generalized geodesics in the Wasserstein space for sufficiently smooth kernels.
Then, we show that their subdifferential consists of just one element and use this to determine the Wasserstein gradient flow equation.
Finally, we deal with particle flows, which are proven to be Wasserstein gradient flows that start in an empirical measure with an empirical target measure.
For the numerical simulations, we focus on the Tsallis-$\alpha$ divergence, which is smoothed based on the inverse multiquadric kernel.
In the case of an infinite recession constant, we can apply the representer theorem to obtain a finite-dimensional problem.
If the recession constant is finite, we must impose a lower bound on the regularization parameter. 
We simulate the gradient flows for three different target measures from the literature.
Since $\alpha = 1$ corresponds to the KL divergence, we are especially interested in the behavior for different values of $\alpha$.
Choosing $\alpha$ moderately larger than one seems to improve the convergence of the gradient flow.
A similar behavior was observed in \cite{GBPRK2022} for another regularization.

\subsection{Related work}
Our work is inspired by the paper of Glaser et al.\ \cite{GAG2021} (which builds on \cite{AZG2021,NWJ2007}) on Wasserstein gradient flows with respect to the MMD-regularized KL divergence.
In contrast to their paper, we deal with arbitrary $f$-divergences and relate the functionals to Moreau envelopes in Hilbert spaces.
This helps to streamline the proofs.
Moreover, our simulations cover the more general Tsallis-$\alpha$ divergences.
 
The opposite perspective is to regularize MMDs by $f$-divergences \cite{KNSZ2023}.
Their analysis covers entropy functions $f$ with an infinite recession constant and probability measures with additional conditions for the moments.
The paper focuses on kernel methods of moments as an alternative to the generalized method of moments.
In contrast, we deal with gradient flows.

Wasserstein-1 regularized $f$-divergences for measures on compact metric spaces are considered in \cite{T2021}.
When choosing their exponent $\alpha$ equal to two, one obtains a Moreau envelope.
However, their definition of $f$-divergence is slightly more general than ours since it also accounts for negative measures. 

The authors of \cite{BDKPR23} investigated the regularization of $f$-divergences using the infimal convolution with general integral probability metrics.
Again, only entropy functions $f$ with an infinite recession constant were considered.
Choosing an MMD as the integral probability metric leads, in contrast to our paper, to a regularization of $f$-divergences with the \emph{non-squared} MMD.
For this setting, no interpretation as Moreau envelope is possible.
Instead, their regularized objective can be interpreted as a Pasch-Hausdorff envelope \cite[Sec.~9B]{RW2009} or Moreau-Yosida envelope \cite[Chp.~9]{M1993}.
Wasserstein gradient flows of these regularized $f$-divergences and their discretization are discussed in \cite{GBPRK2022}, although neither the existence nor uniqueness of the Wasserstein gradient flow is proven.

After the submission of this work, Chen et al. published \cite{CMGKGS24}, which deals with the properties of $D_{f, \nu}^{\lambda}$
and its Wasserstein gradient flow for the Tsallis-2 entropy $f \coloneqq (\cdot - 1)^2$.
Due to the simple structure of $f$, they can derive a closed-form expression for the minimization problem defining $D_{f, \nu}^{\lambda}$ and show that this expression is equal to $d_{\tilde{K}}^2$ for a modified kernel $\tilde{K}$.

\subsection{Outline of the paper}
In Section~\ref{sec:prelim}, we collect basic facts from convex analysis, especially about Moreau envelopes.
Further, we recall the notation of RHKSs and MMDs.
Then, in Section~\ref{sec:f-div}, we discuss $f$-divergences both with finite and an infinite recession constant.
We introduce their MMD-regularized counterparts and establish the relation to Moreau envelopes of a specific proper, convex, and lower semicontinuous function 
on the RKHS associated with the kernel of the MMD. 
In Section~\ref{sec:wgf}, we recall the concept of Wasserstein gradient flow and prove the existence and uniqueness of the Wasserstein gradient flow with respect to MMD-regularized $f$-divergences.
We discuss the simulation of these Wasserstein gradient flows when both the target and the starting point of the flow are empirical measures in Section~\ref{sec:discr}.
Here, we show that the representer theorem can also be applied for the case of $f$-divergences with finite recession constant.
Then, we illustrate the behavior of such flows with three numerical examples in Section~\ref{sec:numerics}.
Finally, we conclude in Section~\ref{sec:conclusions}.
We collect examples of entropy functions, their conjugates, and their associated $f$-divergences in Appendix~\ref{app:ent}. Further, we discuss regularizing the tight variational formulations of $f$-divergences in Appendix~\ref{sec:tight_formulation}.
Details for the duality gaps are added in Appendix~\ref{app:C}.
Further ablation plots and implementation details are provided in the supplementary material included in the arXiv version (\url{https://arxiv.org/abs/2402.04613}).

%--------------------------------------------------------------------------
\section{Convex Analysis in Reproducing Kernel Hilbert spaces} \label{sec:prelim}
%--------------------------------------------------------------------------
This section contains the necessary preliminaries and notions.
In Subsection~\ref{sec:moreau}, we start with basic facts from convex analysis in Hilbert spaces, especially Moreau envelopes and proximal mappings, which can be found, e.g., in the textbook \cite{BC2011}.
Our Hilbert spaces of choice will be RKHSs, which are properly introduced in Subsection~\ref{sec:MMD}.
In particular, we require their relation to the space of signed Borel measures in terms of the so-called kernel mean embedding.
This embedding also relates the distance in RKHSs with the MMD of measures.

%--------------------------------------------------------------------------
\subsection{Moreau Envelopes in Hilbert Spaces} \label{sec:moreau}
%--------------------------------------------------------------------------
Let $\H$ be a \emph{Hilbert} space with inner product 
$\langle \cdot,\cdot\rangle_{\H}$ and corresponding norm $\| \cdot \|_{\H}$.
The domain of an extended function
$F \colon \H \to (-\infty,\infty]$
is defined by $\dom(F) \coloneqq \{ h \in \H: F(h) < \infty \}$, and $F$ is proper if $\dom(F) \not = \emptyset$.
By $\Gamma_0(\H)$, we denote the set of proper, convex, lower semicontinuous extended real-valued functions on $\H$.
The \emph{Fenchel conjugate function} of a proper function $F \colon \H \to (-\infty,\infty]$
is given by 
\begin{equation}
    F^*(h) = \sup_{g \in \H} \{ \langle h,g \rangle_{\H} - F(g) \}.
\end{equation}
The conjugate $F^*$ is convex and lower semicontinuous.
The \emph{subdifferential} of a function $F \in \Gamma_0(\H)$
at $h \in \dom(F)$ is defined as the set
\begin{equation}
\partial F (h) \coloneqq \bigl\{p \in \H: F(g) \ge F(h) + \langle p, g-h \rangle_{\H} \; \forall g \in \H\bigr\}.
\end{equation}
If $F$ is differentiable at $h$, then
$\partial F (h) = \{ \nabla F(h) \}$.
Further, $p \in \partial F(h)$ 
implies $h \in \partial F^*(p)$.

Next, recall that the \emph{Moreau envelope} of a function $G \in \Gamma_0(\H)$ 
is defined by
\begin{equation}\label{eq:MoreauEnvPrimal}
    G^\lambda (h) \coloneqq \min_{g \in \H} \Bigl\{G(g) + 
    \frac{1}{2\lambda} \|h -g\|_{\H}^2 \Bigr\}, 
    \quad \lambda > 0,
\end{equation}
where the minimizer is unique.
Hence, the \emph{proximal map}
$\prox_{\lambda G}\colon \H \to \H$ with
\begin{equation}\label{eq:ptox}
    \prox_{\lambda G} (h) \coloneqq \argmin_{g \in \H} \Bigl\{G(g) + 
    \frac{1}{2\lambda} \|h -g\|_{\H}^2 \Bigr\}, 
    \quad \lambda > 0,
\end{equation}
is well-defined.
The Moreau envelope has the following advantageous properties.

\begin{theorem}\label{prop:moreau}
The Moreau envelope of $G\in \Gamma_0(\H)$ has the following properties:
\begin{enumerate}
    \item \label{item:MoreauProp1}
    The dual formulation of $G^\lambda\colon \H \to \R$ reads as
    \begin{equation} \label{eq:MoreauEnvDual}
        G^\lambda(h)
        = \max_{p \in \H} \Bigl\{ \langle h, p \rangle_{\H} - G^*(p) - \frac{\lambda}{2} \| p \|_{\H}^2 \Bigr\}.
    \end{equation}
    If $\hat p$ maximizes \eqref{eq:MoreauEnvDual}, then $\hat g = h - \lambda \hat p$ minimizes \eqref{eq:MoreauEnvPrimal} and vice versa.
    
    \item \label{item:MoreauProp2}
    The function $G^\lambda$ is Fréchet differentiable with derivative given by
    \begin{equation} \label{eq:MoreauDerivative}
        \nabla G^\lambda (h)
        = \lambda^{-1} \big( h - \prox_{\lambda G}(h)\big).
    \end{equation}
    In particular, it holds that $\nabla G^\lambda$ is $\frac{1}{\lambda}$-Lipschitz and that $G^\lambda$ is continuous.
    
    \item \label{item:MoreauProp3}
    For $\lambda \searrow 0$, we have $G^{\lambda} \nearrow G$,
    and for $\lambda \to \infty$
    that $G^{\lambda} \searrow \inf_{x \in \R^d}\{G(x)\}$ pointwise.
    \end{enumerate}
\end{theorem}

%--------------------------------------------------------------------------
\subsection{Reproducing Kernel Hilbert Spaces and MMDs} \label{sec:MMD}
%--------------------------------------------------------------------------
A Hilbert space $\H$ of real-valued functions on $\mathbb R^d$ is called a
\emph{reproducing kernel Hilbert space} (RKHS), 
if the point evaluations $h \mapsto h(x)$, $h \in \H$, are continuous for all $x \in \R^d$.
There exist various textbooks on RKHS from different points of view, see, e.g., \cite{SC08,zhou2007,SC09,YSCH24}.
By \cite[Thm.~4.20]{SC08}, every RKHS admits a unique symmetric, positive definite function $K\colon\R^d \times \R^d \to \R$ 
which is determined by the 
reproducing property
\begin{equation} \label{repr}
    h(x)
    = \langle h,K(x,\cdot) \rangle_{\H} \quad
    \text{for all } h \in \H.
\end{equation}
In particular, we have that $K(x,\cdot) \in \H$ for all $x \in \R^d$.
Conversely, for any symmetric, positive definite function $K\colon\R^d \times \R^d \to \R$, 
there exists a unique RKHS with reproducing kernel $K$, denoted by $\H_K$ \cite[Thm.~4.21]{SC08}.
Recall that $K$ is called positive definite if for all $N \in \N$, all scalars $\alpha_1, \ldots, \alpha_N \in \R$ and all points $x_1, \ldots, x_N \in \R^d$ we have
\begin{equation*}
    \sum_{i, j = 1}^{N} \alpha_i \alpha_j K(x_i, x_j)
    \ge 0.
\end{equation*}

\begin{assumption} \label{assumption:bdC0}
In the following, we use the term ,,kernel'' for symmetric, positive definite functions $K\colon\R^d \times \R^d \to \R$ that are 
\begin{enumerate}
    \item \label{item:bdKernel}
    bounded, i.e., $\sup_{x \in \R^d} K(x,x) < \infty$, and fulfil
    
    \item \label{item:denseInC0}
    $K(x,\cdot) \in \C_0(\R^d)$ for all $x \in \R^d$.
\end{enumerate}
\end{assumption}
The properties (\ref{item:bdKernel}) and (\ref{item:denseInC0}) are equivalent to the fact that $\H_K \subset \C_0(\R^d)$. 
Further, the embedding is continuous: $\H_K \hookrightarrow \C_0(\R^d)$ \cite[Cor.~3]{SGS2018}.

RKHSs are closely related to the Banach space $\M(\R^d)$ of finite signed Borel measures equipped with total variation norm $\| \cdot\|_{\TV}$.
Later, we also need its subset $\mathcal M_+(\R^d)$ of non-negative measures.
For any $\mu \in \mathcal M(\R^d)$, there exists a unique $m_\mu \in \H_K$ 
with
\begin{align} 
    \E_{\mu}[h]
    & \coloneqq \langle h, \mu \rangle_{\C_0 \times \M}
    = \int_{\R^d} h(x) \d {\mu}(x)
    = \int_{\R^d} \langle h, K(x, \cdot) \rangle_{\H_{ K}}  \d {\mu}(x) \notag\\
    &= \Bigl\langle h, \int_{\R^d} K(x, \cdot) \d {\mu}(x) \Bigr\rangle_{\H_{K}} 
    = \langle h, m_\mu \rangle_{\H_{K}} \label{eq:repr}
\end{align}
for all $h \in \H_K$.
The linear, bounded mapping $m\colon \mathcal M(\R^d) \to \H_K$ with $\mu \mapsto m_{\mu}$ given by 
\begin{equation}\label{kme}
    m_\mu(x)
    = \langle K(x,\cdot), \mu \rangle
    = \int_{\R^d} K(x,y) \d \mu(y),
\end{equation}
is called \emph{kernel mean embedding} (KME) \cite[Sec.~3.1]{MFSS17}.

\begin{assumption} \label{assumption:characteristic}
    In this paper, we restrict our attention to so-called \emph{characteristic kernels} $K$, for which the KME is injective.
\end{assumption}

A kernel $K$ is characteristic if and only if $\H_{K}$ is dense in $( \C_0(\R^d), \| \cdot \|_{\infty})$, see \cite{SGS2018}.
The KME is not surjective and we only have $\overline{(\text{ran}(m), \| \cdot \|_{\H_K})} = \H_K$, see \cite{SF2021}.

The \emph{maximum mean discrepancy} (MMD) \cite{BGRKSS06, GBRSS13} $d_K\colon \mathcal M(\R^d) \times \mathcal \mathcal \M(\R^d) \to \R$ 
is defined as 
\begin{align} 
d_K(\mu,\nu)^2 &\coloneqq \int_{\R^d \times \R^d} K(x,y) \d (\mu(x) - \nu(x)) \d (\mu(y) - \nu(y))\notag\\
&=
\int_{\R^d \times \R^d} K(x,y) \d \mu(x) \d \mu(y) - 2 \int_{\R^d \times \R^d} K(x,y) \d \mu(x) \d\nu(y)\notag\\
& \quad + 
\int_{\R^d \times \R^d} K(x,y) \d \nu(x) \d\nu(y). \label{eq:DK}
\end{align}
By the reproducing property \eqref{repr}, see \cite[Lemma~4]{GBRSS13}, \eqref{eq:DK} can be rewritten as 
\begin{equation} \label{mmd-kme}
d_K (\mu,\nu) = \|m_\mu - m_\nu\|_{\H_K}.
\end{equation}
Since the KME is injective, $d_K$ is a metric on $\mathcal M(\R^d)$ and,
in particular, $d_K (\mu,\nu) = 0$ if and only if $\mu = \nu$.
The widely used \textit{radial kernels} are of the form $K(x, y) = \phi(\| x - y \|_2^2)$ for some continuous function $\phi \colon [0, \infty) \to \R$.

\begin{remark} \label{rem1}
By Schoenberg's theorem \cite[Thm.~7.13]{W2004}, a radial kernel $K$ is positive definite if and only if $\phi$ is completely monotone on $[0, \infty)$, that is $\phi \in \C^{\infty}((0, \infty)) \cap \C([0, \infty))$ and $(-1)^k \phi^{(k)}(r) \ge 0$ for all $k \in \N$ and all $r > 0$. 
Hence, $\phi$ and $\phi''$ are decreasing on $(0,\infty)$.
Moreover, the Hausdorff–Bernstein–Widder theorem \cite[Thm.~7.11]{W2004} asserts that $\phi$ is completely monotone on $[0, \infty)$ if and only if there exists a non-negative finite Borel measure $\nu$ on $[0, \infty)$ such that
\begin{equation}
    \phi(r) = \int_{0}^{\infty} e^{- r t} \d{\nu}(t), \qquad \forall r \ge 0.
\end{equation}
By \cite[Prop.~11]{SFL2010}, $K(x, y) = \phi(\| x - y \|_2^2)$ is characteristic if and only if $\supp(\nu) \ne \{ 0 \}$.
Examples of radial characteristic kernels are the Gaussians with $\phi(r) = \smash{\exp(- \frac{1}{2 \sigma^2} r})$, $\sigma \in \R$
and the inverse multiquadric with $\phi(r) = \smash{(\sigma + r)^{-\frac{1}{2}}}$,  $\sigma > 0$.
\end{remark}

We have the following regularity result, where both $\phi'(0)$ and $\phi''(0)$ denote the one-sided derivatives.

\begin{lemma}\label{lem:regularity}
    For a radial kernel $K(x,y) = \phi(\|x-y\|_2^2)$ with 
    $\phi \in \C^2([0,\infty))$ it holds that
    $\H_K \hookrightarrow C^2(\R^d)$.
   Then, we have for any $y,\tilde y \in \R^d$ and all $i \in \{ 1, \ldots, d \}$ that
    \begin{equation}
        \|\partial_{y_i}K(\cdot,y)\|_{\H_K}^2 = -2\phi^{\prime}(0)
    \end{equation}
    and
    \begin{equation}
        \|\partial_{y_i}K(\cdot,y) - \partial_{\tilde{y}_i}K(\cdot,\tilde y)\|_{\H_K}^2 \le 
        4 \phi^{\prime\prime}(0) \left(2 |y_i - \tilde y_i|^2 + \|y- \tilde y\|_2^2\right).
    \end{equation}
    Furthermore, we have for any $h \in \H_K$ that
    \begin{equation} \label{2a}
        \| \nabla h(y) - \nabla h(\tilde y)\|_2 \le 2 \|h\|_{\H_K} \sqrt{\phi^{\prime\prime}(0) (d+2)} \, \|y - \tilde y \|_2.
    \end{equation}
 \end{lemma}

\begin{proof}
    The continuity of the embedding follows from \cite[Cor.~4.36]{SC08}.
    Next, for $x, y \in \R^d$ we directly compute
    $\partial_{x_i} \partial_{y_i} K(x,y) = -4 \phi^{\prime\prime}(\|x-y\|_2^2)(x_i-y_i)^2 - 2\phi^\prime(\|x-y\|_2^2)$.
    By applying \cite[Lem.~4.34]{SC08} for the feature map $y \mapsto K(\cdot,y)$,  we get
    \begin{equation}\label{eq:KernelProduct}
        \bigl\langle \partial_{y_i}K(\cdot,y), \partial_{\tilde{y}_i}K(\cdot,\tilde y) \bigr \rangle_{\H_k} = \partial_{y_i} \partial_{\tilde{y}_i} K(y,\tilde y)
        = -4 \phi^{\prime\prime}(\|y-\tilde y\|_2^2)(y_i-\tilde y_i)^2 - 2\phi^{\prime}(\|y-\tilde y\|_2^2)
    \end{equation}
    and in particular
    $\|\partial_{y_i}K(\cdot,y)\|_{\H_K}^2 = -2\phi^\prime(0)$.
    By \eqref{eq:KernelProduct} and since $\phi'$ is Lipschitz continuous with Lipschitz constant $\|\phi''\|_\infty =\phi''(0)$, see Remark~\ref{rem1}, it holds
    \begin{align}
        \Vert \partial_{y_i} K(\cdot,y) - \partial_{\tilde{y}_i} K(\cdot,\tilde y)\Vert_{\H_K}^2 
        &= -4\phi^{\prime}(0)  + 8 \phi^{\prime\prime}(\|y-\tilde y\|_2^2)(y_{i}- \tilde y_{i})^2 +  4\phi^{\prime}(\|y-\tilde y\|_2^2)
        \notag\\
        &\le
        4 \phi^{\prime\prime}(0) \left(2 (y_{i}-\tilde y_{i})^2 + \|y-\tilde y\|_2^2 \right).
        \label{lipa}
    \end{align}
    The final assertion follows by
    \begin{align}
        |\partial_{y_i} h(y) - \partial_{\tilde{y}_i} h(\tilde y) | 
        &=
         |\partial_i \langle h,K(\cdot,y) \rangle_{\H_K} - \partial_i \langle h,K(\cdot,\tilde y) \rangle_{\H_K} |  \\
        & =
        | \left\langle h, \partial_{y_i}K(\cdot,y) - \partial_{y_i}K(\cdot,\tilde y)
        \right \rangle_{\H_K} |\notag\\
        &\le 
        \|h\|_{\H_K} \|\partial_{y_i}K(\cdot,y) - \partial_{y_i}K(\cdot,\tilde y)\|_{\H_K}.
        \end{align}
    This finishes the proof.
\end{proof}
In the rest of this paper, kernels always have to fulfill Assumptions \ref{assumption:bdC0} and \ref{assumption:characteristic}.

%--------------------------------------------------------------
\section{MMD-regularized \texorpdfstring{$f$}{f}-Divergences and Moreau Envelopes} \label{sec:f-div}
%--------------------------------------------------------------------------
Let us briefly describe the path of this section.
First, we define $f$-divergences $D_f$ of non-negative measures 
for entropy functions $f$ with infinite or finite recession constant. Such $f$-divergences were first introduced by \cite{C1964,AS1966}, and we refer to \cite{LV1987} for a detailed overview.
We prove some of their properties in Subsection~\ref{sec:ff}, where allowing a finite recession constant makes the proofs more expansive.
Then, in Subsection~\ref{sec:regf}, we deal with the associated functionals $D_{f,\nu} \coloneqq D_f(\cdot|\nu)$ with a fixed target measure $\nu$.
If the recession constant is infinite, then \smash{$D_{f,\nu}$} is only finite for measures that are absolutely continuous with respect to $\nu$.
This disadvantage can be circumvented by using the MMD-regularized functional
\begin{equation}\label{eq:Regfunc}
    D_{f,\nu}^{\lambda} (\mu) \coloneqq \inf_{\sigma \in \M_+(\R^d)} \Big\{ D_{f,\nu}(\sigma) + \frac{1}{2 \lambda} \underbrace{d_K(\mu, \sigma)^2}_{\|m_\mu - m_\sigma\|_{\H_K}^2} \Big\}, \quad \lambda>0.
\end{equation}
A different form of MMD-regularized $f$-divergences, where $\mu, \nu \in \P(\R^d)$ and the infimum in \eqref{eq:Regfunc} is only taken over $\mathcal P(\R^d)$, is discussed in Appendix \ref{sec:tight_formulation}.

Now, our main goal is to show that \eqref{eq:Regfunc} can be rewritten
as the concatenation of the Moreau envelope of some $G_{f,\nu} \colon \H_K \to (-\infty,\infty]$, and the KME.
To this end, we
introduce $G_{f,\nu} = D_{f,\nu} \circ m^{-1}$ on $\text{ran}(m)$ and set $G_{f,\nu} = \infty$ otherwise.
This function is proper and convex.
We prove that $G_{f, \nu}$ is lower semicontinuous, where we have to face the difficulty that $\text{ran}(m) \not = \H_K$.
Then, $G_{f,\nu} \in \Gamma_0(\H_K)$  yields the desired Moreau envelope identification
\begin{equation}
    D_{f,\nu}^{\lambda} (\mu) = G_{f,\nu}^{\lambda} (m_{\mu}) \coloneqq \min_{g \in \H_K} \Big\{ G_{f,\nu}(g) + \frac{1}{2 \lambda} \|g - m_\mu\|_{\H_K}^2 \Big\}.
\end{equation}
which allows to exploit Theorem~\ref{prop:moreau} to show various of its properties in Subsection~\ref{sec:prop}.

%-----------------------------------------
\subsection{\texorpdfstring{$f$}{f}-Divergences} \label{sec:ff}
%-----------------------------------------

A function $f\colon\mathbb R \to [0,\infty]$ is called an \emph{entropy function} if $f \in \Gamma_0(\R)$ with $f(t) = \infty$ for $t < 0$ and $f(1) = 0$.
Its \emph{recession constant} is given by $f_{\infty}' = \lim_{t \to \infty} \frac{f(t)}{t}$. 
Then, $f^* \in \Gamma_0(\R)$ is non-decreasing and $\text{int}(\dom(f^*)) = ( - \infty, f_{\infty}')$, see, e.g., \cite{LMS2017}.
Further, $f^*$ is continuous on $\dom(f^*)$ and
$f^*(0)=0$, in particular $0 \in \dom(f^*)$.
By definition of the subdifferential and since $f(1) = 0$, it follows that
$0 \in \partial f(1)$
and thus $1 \in \partial f^*(0)$.
Several entropy functions with their recession constants and conjugate functions are collected in Table~\ref{tab:entropyFunctions} in Appendix~\ref{app:ent}.

Let $f$ be an entropy function and $\nu \in \M_+(\R^d)$.
Recall that every measure $\mu \in \M_+(\R^d)$ 
admits a \emph{unique Lebesgue decomposition}
$\mu = \rho \nu + \mu^s$, where
$\rho \in L^1(\R^d,\nu)$, $\rho \ge 0$ and
$\mu^s \perp \nu$, i.e., there exists a Borel set $\mathcal A \subseteq \R^d$ such that $\nu(\mathbb R^d\setminus \mathcal A) = 0$ and $\mu^s(\mathcal A) = 0$ \cite[Lemma~2.3]{LMS2017}. 
The $f$-\emph{divergence} $D_f \colon\M_+(\R^d) \times \M_+(\R^d) \to [0,\infty]$ between a measure $\mu = \rho \nu + \mu^s\in \M_+(\R^d)$ and $\nu \in \M_+(\R^d)$ is defined by
\begin{align}
    D_f(\rho \nu + \mu^s \mid \nu)
    & = \int_{\R^d} f \circ\rho \d{\nu} + f_{\infty}' \mu^s(\R^d) = \sup_{\substack{g \in \C_b(\R^d), \\ f^* \circ g \in \C_b(\R^d)}} 
	 \biggl \{\int_{\R^d} g \d{\mu} - \int_{\R^d} f^* \circ g \d{\nu} \biggr\} \label{eq:fDiv}
\end{align}
with the usual convention $0 \cdot (\pm \infty) = 0$, 
see \cite[Eq.~(2.35), Thm.~2.7, Rem.~2.8]{LMS2017}.
The function $D_f$ is jointly convex and non-negative, see \cite[Cor.~2.9]{LMS2017}.
Moreover, we will need the following lemma.
%----------------------------------------
\begin{lemma}\label{lem:definit}
    Let $f \in \Gamma_0(\R)$ be an entropy function
    with the unique minimizer $1$.
    \begin{enumerate}
        \item 
        Divergence: For $\mu,\nu \in \M_+(\R^d)$, the relation $D_f(\mu,\nu) = 0$ implies that $\mu = \nu$.

        \item
         We have $f_{\infty}'>0$ and thus $\lim_{t\to \infty } f(t)=\infty$.
    \end{enumerate}
\end{lemma}
%-----------------------------------------
\begin{proof}
    \begin{enumerate}
        \item
        Suppose we have $\nu \in \M_+(\R^d)$ and $\mu = \rho \nu + \mu^s \in \M_+(\R^d)$ such that $D_f(\mu \mid \nu) = 0$.
        Since both summands in the primal definition \eqref{eq:fDiv} of $D_f$ are non-negative, they must be equal to zero.
        In particular, $\mu^s(\R^d) = 0$ and since $\mu^s \in \M_+(\R^d)$, this implies $\mu^s = 0$.
        Hence, $\mu = \rho \nu$ and
        \begin{equation}
            D_f(\mu \mid \nu)
            = D_f(\rho \nu \mid \nu)
            = \int_{\R^d} (f\circ \rho)(x) \d{\nu}(x)
            = 0.
        \end{equation}
        Since $f$ is non-negative and $\nu \in \M_+(\R^d)$, we must have $f \circ \rho = 0$ $\nu$-a.e.
        Since  $f(1) = 0$ is the unique minimum, this implies $\rho = 1$ $\nu$-a.e., namely that $\mu = \nu$.

        \item 
        Fix $\tau > 1$.
        For any $t>\tau$ we write $\tau$ as $\tau= 1 - \lambda + \lambda t$, with $\lambda = \frac{\tau - 1}{t - 1} \in (0, 1)$. By the convexity of $f$ we have
        \begin{equation}\label{eq:Tau}
            0
            < f(\tau)
            \le (1 - \lambda) f(1) + \lambda f(t)
            = \lambda f(t)
            = \frac{\tau - 1}{t - 1} f(t).
        \end{equation}
        Since $1$ is the unique minimizer and $\tau >1$, \eqref{eq:Tau} implies 
        \begin{equation*}
            f_{\infty}' = \lim_{t\to \infty} \frac{f(t)}{t}=\lim_{t \to \infty} \frac{f(t)}{t-1}
            \ge \frac{f(\tau)}{\tau - 1}
            >0.
        \end{equation*}
        Immediately, $f_{\infty}'>0$ yields $\lim_{t\to \infty} f(t)=\infty$.
    \end{enumerate}
\end{proof}

\begin{assumption}\label{eq:Assum3}
    From now on, we assume that all entropy functions fulfill the assumption in Lemma~\ref{lem:definit}, that is, the minimizer in $1$ is unique. 
\end{assumption}

Examples of $f$-divergences are contained in Table~\ref{tab:fDivergences} in Appendix~\ref{app:ent}.
Assumption~\ref{eq:Assum3} is fulfilled
for all $f$-divergences in Table~\ref{tab:entropyFunctions}
except for the Marton divergence, the hockey stick divergence, and the
trivial zero divergence. However, it is not hard to check that the Marton divergence and the hockey stick divergence are positive
definite on the space of \textit{probability} measures, $\P(\R^d)$, too.
Below is an example of a non-trivial $f$-divergence that does not have this property.
\begin{example}[Rescaled Marton divergence]
    Let $f(t) = \max(\frac{1}{2} - x, 0)^2$ on $[0, \infty)$.
    Then $f_{\infty}' = 0$ and we have for any absolutely continuous $\nu \in \P(\R^d)$ that
    $D_f(\frac{1}{2} \nu + \frac{1}{2} \delta_0 \mid \nu) = f(\frac{1}{2}) = 0$.
\end{example}

Recall that $\M(\R^d)$ is the dual space of $\C_0(\R^d)$.
A sequence of measures $(\mu_n)_n \subset \M(\R^d)$ converges weak* to a measure $\mu \in \M(\R^d)$ if we have for all $g \in \C_0(\R^d)$ that $\langle g, \mu_n \rangle \to \langle g, \mu \rangle$ as $n \to \infty$.
Moreover, if $(\mu_n)_n \subset \M_+(\R^d)$ is bounded, then $\mu \in \M_+(\R^d)$, see \cite[Lemma~4.71,~Cor.~4.74]{PPST2023}. 

%-------------------------------

The following lemma also follows from the more general \cite[Thm.~2.34]{AFP00}, but we prefer to give a simpler proof for our setting to make the paper self-contained.
\begin{lemma}\label{lem:ApproxC0}
For any fixed $\nu \in \M_+(\R^d)$, the
$f$-divergence \eqref{eq:fDiv} can be rewritten as
\begin{align}
    D_f(\rho \nu + \mu^s \mid \nu) = \sup_{\substack{g \in \C_0(\R^d; \dom(f^*)
    )}} 
    \biggl \{\int_{\R^d} g \d{\mu} - \int_{\R^d} f^* \circ g \d{\nu} \biggr\}.\label{eq:fDiv_1}
\end{align}
Therefore, $D_f$ is jointly weak* lower semicontinuous.
\end{lemma}

\begin{proof}
In the proof, we must be careful concerning the support of $\nu$.

\begin{enumerate}
    \item 
    $\eqref{eq:fDiv} \ge \eqref{eq:fDiv_1}$: This direction is obvious since $g\in \C_0(\R^d; \dom(f^*))$, $0 \in \dom(f^*)$ and the continuity of $f^*$ on its domain imply $g\in \C_b(\R^d)$ and $f^* \circ g \in \C_b(\R^d)$.

    \item 
    $\eqref{eq:fDiv} \le \eqref{eq:fDiv_1}$: Let $g \in \C_b(\R^d)$ with $f^* \circ g \in \C_b(\R^d)$.
    Using continuous cutoff functions, we can approximate $g$ by a family of functions $(g_k)_{k \in \N}\subset \C_0(\R^d; \dom(f^*))$ with $\sup g \ge \max g_k \ge \inf g_k \ge \min(\inf g,0)$, which implies $\vert f^* \circ g_k \vert \le \sup \vert f^* \circ g \vert$, for any $k \in \N$, and $\lim_{k \to \infty} g_k(x) = g(x)$ for all $x \in \R^d$.
    Further, the continuity of $f^*$ on its domain implies that $(f^*\circ g_k)_{k \in \N}$ satisfies $\lim_{k \to \infty} (f^* \circ g_k)(x)= (f^* \circ g)(x)$.
    Now, the claim follows as in the first part using the dominated convergence theorem.

    \item 
    Let $(\mu_n)_{n}$ and $(\nu_n)_n$ converge weak* to $\mu$ and $\nu$,
    respectively.
    Since $f^*$ is continuous on its domain, where $f^*(0) = 0$ and 
    $g \in \C_0(\R^d; \dom(f^*))$, it follows that $f^* \circ g \in \C_0(\R^d)$.
    Then, we have for $g \in \C_0(\R^d; \dom(f^*))$ that
    \begin{equation}\label{eq:conv}
        \lim_{n \to \infty} \left\{\langle g,\mu_n\rangle - \langle f^*\circ g,\nu_n\rangle \right\} = \langle g,\mu\rangle - \langle f^*\circ g,\nu\rangle\end{equation}
    and by inserting \eqref{eq:conv} into \eqref{eq:fDiv_1} we get 
    \begin{align}
        D_f(\mu \mid \nu)
        &= \sup_{\substack{g \in \C_0(\R^d; \dom(f^*)
        )}}\Bigl\{\lim_{n \to \infty} \langle g,\mu_n\rangle - \langle f^*\circ g,\nu_n\rangle\Bigr\}\notag\\
        &\le 
        \liminf_{n \to \infty}\sup_{\substack{g \in \C_0(\R^d; \dom(f^*)
        )}} \bigl\{\langle g,\mu_n\rangle - \langle f^*\circ g,\nu_n\rangle \bigr\}
        = \liminf_{n \to \infty} D_f(\mu_n \mid \nu_n).
    \end{align}
    \end{enumerate}
    Thus, $D_f$ is jointly weak* lower semicontinuous.
\end{proof}

%-----------------------------------------
\subsection{Regularized \texorpdfstring{$f$}{f}-Divergences} \label{sec:regf}
%-----------------------------------------
To overcome the drawback that $D_f(\mu,\nu)$ requires $\mu$ to be absolutely continuous with respect to $\nu$ if $f_{\infty}' = \infty$, 
we introduce
the MMD-\emph{regularized $f$-divergence} 
$D_{f}^{\lambda} \colon \M_+(\R^d) \times \M_+(\R^d) \to [0,\infty)$ as
\begin{equation} \label{eq:RegFDiv}
    D_{f}^{\lambda} (\mu \mid \nu) \coloneqq  \inf_{\sigma \in \M_+(\R^d)} \Big\{ D_{f}(\sigma \mid \nu) + \frac{1}{2 \lambda} d_K(\mu, \sigma)^2 \Big\}, \quad \lambda > 0.
\end{equation}
For fixed $\nu \in \M_+(\R^d)$, we investigate the functional 
$D_{f,\nu} \coloneqq D_f(\cdot \mid \nu)\colon \M(\R^d) \to [0,\infty]$, set to $\infty$ on $\M(\R^d) \setminus \M_+(\R^d)$.
Similarly, we consider its regularized version 
$\smash{D_{f,\nu}^{\lambda}} \colon \M(\R^d) \to [0,\infty)$ that was announced in \eqref{eq:Regfunc}.
Note that \smash{$D_{f,\nu}^{\lambda}$} is well-defined also for nonpositive measures. 
Moreover, we always have \smash{$D_{f,\nu}^{\lambda}(\mu) < \infty$}.

Now, we aim to reformulate \eqref{eq:Regfunc} as the Moreau envelope of a certain function on $\H_K$.
Using the KME in \eqref{kme}, we see that
\begin{equation} \label{eq:tildegnufExtendsDf}
    D_{f,\nu}(\mu)
    = G_{f,\nu} (m_\mu),
\end{equation}
where $G_{f,\nu} \colon \H_{K} \to [0, \infty]$ is given by
\begin{equation} \label{eq:g_nu}
   G_{f,\nu}(h)
   \coloneqq \left\{
   \begin{array}{ll}
   D_{f,\nu}(\mu), & \text{if } \exists \, \mu \in \M_+(\R^d) \text{ s.t. } h = m_{\mu}, \\
        \infty, & \text{else.}
    \end{array}
    \right.
\end{equation}
Since $m^{-1}$ is linear on $\ran(m)$ and $D_f$ is jointly convex, the concatenation $G_{f,\nu}$ is convex.
Further, $G_{f,\nu}(m_{\nu}) = 0$, so that $G_{f,\nu}$ is also proper.

We now prove that although $\ran(m)$ is not closed in $\H_K$, the function $G_{f,\nu}$ is lower semicontinuous.
Hence, the theory for Moreau envelopes on $\Gamma_0(\H_K)$ applies.
\begin{lemma} \label{lemma:Gfnu_lsc}
    The function $G_{f, \nu} \colon \H_K \to [0, \infty]$ is lower semicontinuous.
\end{lemma}
\begin{proof}
    Fix $h\in \H_{K}$ and let $(h_n)_{n} \subset \H_{K}$ with $h_n \to h$. 
    We need to show that ${G_{f,\nu}}(h)\le \liminf_{n\to \infty} {G_{f,\nu}}(h_n)$.
    If $\liminf_{n\to \infty} {G_{f,\nu}}(h_n) = +\infty$, then we are done.
    For $\liminf_{n\to \infty} {G_{f,\nu}}(h_n) < +\infty$, we pass to a subsequence which realizes the limes inferior.
    We choose this subsequence such that ${G_{f,\nu}}(h_n)<\infty$ for all $n\in \mathbb{N}$.
    Then, there exist $\mu_n\in \M_+(\R^d)$ with $m_{\mu_n}=h_n$.
    In principle, two cases can occur.

    \textbf{Case 1:} \textit{The sequence $(\|\mu_n\|_{\TV})_{n}$ diverges to $\infty$.} \\
    Since $D_{f, 0} = f_{\infty}' \cdot \| \cdot \|_{\TV}$, this is impossible for $\nu = 0$.
    If $\nu \in \M_+(\R^d) \setminus \{ 0 \}$, then Jensen's inequality implies
    \begin{align}
        {G_{f,\nu}}(h_n)
        & = {D_{f,\nu}}(\mu_n)
        = \int_{\R^d} f\circ \rho_n \d\nu + f_{\infty}' \mu_n^s(\R^d) \notag\\
        & \ge \| \nu \|_{\TV} f\left(\int_{\R^d} \frac{\rho_n}{\| \nu \|_{\TV}} \d\nu \right) + f_{\infty}' \mu_n^s(\R^d)\notag\\
        & = \| \nu \|_{\TV} f\left(\frac{1}{\| \nu \|_{\TV}} \left(\mu_n(\R^d)-\mu_n^s(\R^d)\right)\right) + f_{\infty}' \mu_n^s(\R^d).\label{eq:LowerBoundFunc}
    \end{align}
    Since $\lim_{n\to \infty} {G_{f,\nu}} (h_n)<\infty$ and $f_{\infty}' >0$, \eqref{eq:LowerBoundFunc} implies that the sequence $(\mu_{n}^s(\R^d))_n$ must be bounded.
    By Lemma~\ref{lem:definit} the same holds for the sequence $(\mu_{n}(\R^d)-\mu_{n}^s(\R^d))_n$.
    Hence, also $(\mu_{n}(\R^d))_n=(\|\mu_{n}\|_{\TV})_n$ is bounded, which contradicts our initial assumption and thus this case cannot occur.

    \textbf{Case 2:} \textit{The sequence $(\mu_n)_n$ satisfies $\liminf_{n\to \infty} \|\mu_n\|_{\TV}<\infty $.}\\
    Since $\C_0(\R^d)$ is separable and $\C_0(\R^d)^* \cong \M(\R^d)$, the Banach-Alaoglu theorem \cite[Cor.~5.4.2]{K2023} implies that there exists a subsequence $(\mu_{n_k})_k$ which converges weak* to some $\mu\in \M(\R^d)$.
    Since $K(x, \cdot) \in \C_0(\R^d)$ for any $x\in \R^d$, we thus have
    \begin{equation}\label{eq:weak_conv}
        \lim_{k\to \infty} m_{\mu_{n_k}}(x)
        = \lim_{k\to \infty }\int_{\R^d} K(x, y) \d\mu_{n_k}(y)
        = \int_{\R^d} K(x, y) d\mu(y)
        = m_{\mu}(x).
    \end{equation}
    Since $(m_{\mu_{n_k}})_k$ converges to $h$ in $\H_{K}$, \eqref{eq:weak_conv} entails
    \begin{equation}
        h(x)
        = \lim_{k\to \infty} m_{\mu_{n_k}}(x)
        = m_{\mu}(x) \quad \forall x\in \R^d.
    \end{equation}
    Thus, $h$ is represented by the measure $\mu$.
    Furthermore, $\mu \in \M_+(\R^d)$ as $(\mu_n)_{n}$ is bounded.
    The weak* lower semicontinuity of ${D_{f,\nu}}$ from Lemma \ref{lem:ApproxC0} yields
    \begin{align}
        {G_{f,\nu}}(h)
        ={D_{f,\nu}}(\mu)\le \liminf_{k\to \infty } {D_{f,\nu}}(\mu_{n_k})
        = \liminf_{k\to \infty } {G_{f,\nu}}(h_{n_k})
        =\lim_{n\to \infty} {G_{f,\nu}}(h_n).
     \end{align}
\end{proof}
%------------------------------------------------------------------------------
The following lemma and its proof build upon the theory of convex integral functionals initiated by \cite{R1968,R1971} and developed in, e.g., \cite{BL1993}.
\begin{lemma}\label{prop:LowerHull}
The conjugate function of $G_{ f,\nu}$ in \eqref{eq:g_nu} is given by
    \begin{equation} \label{eq:gnuf*}
        G_{f,\nu}^*(h ) 
        = \mathbb E_{\nu}[f^* \circ h] 
        + \iota_{\C_0(\R^d, \overline{\dom(f^*)})}(h)
    \end{equation}
\end{lemma}
%-----------------------------------------------------------------

\begin{proof}
    Using \eqref{eq:repr} and \eqref{eq:tildegnufExtendsDf}, we obtain for $h \in \H_K$ that
    \begin{align*}
        G_{f,\nu}^*(h)
        & = \sup_{g \in \H_K} \bigl\{\langle g, h \rangle_{\H_K} - G_{f,\nu}(g)\bigr\}
        = \sup_{\mu \in \M_+(\R^d)} \bigl\{\langle m_{\mu}, h \rangle_{\H_K} - D_{f,\nu}(\mu)\bigr\} \\
        & = \sup_{\mu \in \M_+(\R^d)} \bigl\{\langle h, \mu \rangle_{\C_0 \times \M} - D_{f,\nu}(\mu)
        \bigr\}.
    \end{align*}
    Applying \cite[Prop.~24]{AH2021} (a similar argument can be found in the earlier paper \cite{P18}) with (in their notation) $(\Omega, \F)$ being $\R^d$ equipped with its Borel $\sigma$-algebra, $X = \M(\R^d)$ and $Y$ equal to the space of bounded measurable functions on $\R^d$, we get that $(X, Y)$ is $\nu$-decomposable with $\Xi = \{ \emptyset \}$, so that $\ess \ \text{im}_{\Xi}(h) = \ran(h)$, and thus
    \begin{equation*}
        \sup_{\mu \in \M_+(\R^d)} \bigl\{\langle h, \mu \rangle_{\C_0 \times \M} - D_{f,\nu}(\mu)
        \bigr\}
        = \begin{cases}
            \E_{\nu}[f^* \circ h], & \text{if } \ran(h) \subset \overline{\dom(f^*)}, \\
            \infty, & \text{else.}
        \end{cases}
    \end{equation*}
    for any $h \in Y$, so in particular for any $h \in \H_K$.
\end{proof}

Next, we establish the link between \smash{$D_{f,\nu}^{\lambda}$} and the Moreau envelope of $G_{f,\nu}$.

%------------------------------------
\begin{corollary} \label{cor:coincidence}
For any fixed $\nu \in \M_+(\R^d)$ let $G_{f,\nu}$ be defined by 
\eqref{eq:g_nu}.
Then, it holds for $\mu \in \M(\R^d)$ 
that
\begin{equation} \label{eq:MoreauPrimalgnuf}
    D_{f,\nu}^{\lambda}(\mu)
    = G_{f,\nu}^{\lambda} (m_\mu).
    \end{equation} 
\end{corollary}

%--------------------------------------
\begin{proof}
    Since $G_{f, \nu} \in \Gamma_0(\H_K)$, we have that 
    \begin{equation} \label{eq:MoreauPrimalgnuf2}
        G_{f,\nu}^{\lambda} (m_\mu)
        = \min_{g \in \H_K} \Bigl\{ G_{f,\nu}(g) + \frac{1}{2\lambda} \|g - m_{\mu} \|_{\H_K}^2 \Bigr\}
    \end{equation}
    is well-defined and clearly the unique minimizer fulfills $g \in \text{ran}(m)$. Hence we can substitute $g = m_{\sigma}$ for $\sigma \in \M_+(\R^d)$ to obtain
    \begin{equation*}
        G_{f, \nu}^{\lambda} (m_\mu)
        = \min_{\sigma \in \M_+(\R^d)} \big\{D_{f, \nu}(\sigma) + \frac{1}{2 \lambda} \| m_{\mu} - m_{\sigma} \|_{\H_K}^2 \big\},
    \end{equation*}
    which yields the assertion.
\end{proof}

%-----------------------------------------
\subsection{Properties of MMD-regularized \texorpdfstring{$f$}{f}-Divergences} \label{sec:prop}
%-----------------------------------------
Now, we combine Corollary~\ref{cor:coincidence} with the properties of Moreau envelopes in Theorem~\ref{prop:moreau} to prove various properties of the MMD-regularized functional $D_{f,\nu}^\lambda$ in a sequence of corollaries.
For the KL divergence and the $\chi^2$-divergence, these properties have been shown differently in \cite{GAG2021,CMGKGS24} without using Moreau envelopes.

\begin{corollary}[Dual formulation]\label{lemma:RegFDivDual}
    For $\nu \in \M_+(\R^d)$ and $\mu \in \M(\R^d)$, we have that
    \begin{equation} \label{eq:RegFDivDual}
        D_{f,\nu}^\lambda(\mu)
        = \max_{\substack{p \in \H_{K} \\ p(\R^d) \subseteq \overline{\dom(f^*)}}} \Bigl\{ \E_{\mu}[p] - \E_{\nu}[f^* \circ p] - \frac{\lambda}{2} \| p \|_{\H_{K}}^2 \Bigr\}.
    \end{equation}
    If $\hat{p} \in \H_K$ maximizes \eqref{eq:RegFDivDual}, then $\hat{g} = m_{\mu} - \lambda \hat{p}$ minimizes \eqref{eq:MoreauPrimalgnuf2} and vice versa.
    Further, it holds 
    \begin{equation} \label{some_est}
        \frac{\lambda}{2} \Vert \hat p\Vert_{\H_{K}}^2 
        \le 
        D_{f,\nu}^{\lambda}(\mu) 
        \le 
        \Vert \hat p\Vert_{\H_{K}} d_K(\mu, \nu),
    \end{equation}
    and, in particular,
    \begin{equation*}
        D_{f,\nu}^{\lambda}(\mu) 
        \le 
        \Vert \hat p\Vert_{\H_{K}} (\Vert m_{\mu} \Vert _{\H_{K}}+ \Vert m_{\nu} \Vert_{\H_{K}})
        \quad \text{and} \quad \Vert \hat p\Vert_{\H_{K}} \le \tfrac{2}{\lambda} d_K(\mu,\nu).
    \end{equation*}
\end{corollary}

\begin{proof}
\begin{enumerate}
    \item 
    From Corollary~\ref{cor:coincidence} and Theorem~\ref{prop:moreau}(\ref{item:MoreauProp1}), we obtain
    \begin{equation}
        D_{f,\nu}^{\lambda}(\mu)= G_{f,\nu}^{\lambda}(m_\mu)
        = \max_{p \in \H_K} \Bigl\{ \langle p, m_{\mu} \rangle_{\H_K} - G_{f,\nu}^*(p) - \frac{\lambda}{2} \| p \|_{\H_K}^2 \Bigr\}.
    \end{equation}
    By \eqref{eq:repr}, we have $\langle p, m_{\mu} \rangle_{\H_K} = \E_{\mu}[p]$ and plugging in \eqref{eq:gnuf*} yields the first assertion.
    By Theorem~\ref{prop:moreau} we also get the primal-dual relation.

    \item 
    Using \eqref{eq:MoreauPrimalgnuf}, there exists a $\hat{g} \in \dom(G_{f, \nu})$ such that
    \begin{equation} \label{inter}
        D_{f,\nu}^{\lambda}(\mu) = G_{f,\nu}(\hat g) + \frac{1}{2\lambda} \|\hat g - m_{\mu}\|_{\H_K}^2 \ge \frac{1}{2\lambda} \|\hat g - m_{\mu}\|_{\H_K}^2 = \frac{\lambda}{2} \|\hat p\|_{\H_K}^2,
    \end{equation}
    which is the first lower estimate.
    For the upper one, we use that $f(1)=0$, so that $f^*(p) \ge p$ for all $p \in \R$.
    Then, together with \eqref{eq:repr}, the upper estimate follows by
    \begin{equation}
        D_{f,\nu}^\lambda(\mu)
        = \E_{\mu}[\hat p] - \E_{\nu}[f^* \circ \hat p] - \frac{\lambda}{2} \| \hat p \|_{\H_{K}}^2
        \le \E_{\mu - \nu}[\hat p] 
        \le \Vert \hat p\Vert_{\H_{K}} \Vert m_{\mu} - m_{\nu} \Vert_{\H_{K}},
    \end{equation}
    where we employed the Cauchy-Schwarz inequality in $\H_K$ and \eqref{inter}.
\end{enumerate}
\end{proof}

\begin{remark}
In \cite[Eq.~(2)]{GAG2021}, Glaser et al. introduced a so-called KALE
functional. This is exactly the
dual formulation \eqref{eq:RegFDivDual} multiplied by $1+\lambda$
for the Kullback-Leibler entropy function $f_{\KL}$ from Table~\ref{tab:entropyFunctions}.
As expected, the dual function of the KALE functional \cite[Eq.~(6)]{GAG2021} coincides with our primal formulation \eqref{eq:RegFDiv}.
Indeed, we have $D_f(\sigma \mid \nu) < \infty$ if and only if there exists a density $\rho \in L^1(\R^d, \nu ; \R_{\ge 0})$ such that $\sigma = \rho \nu$. Thus,
\begin{align}
    \min_{\sigma \in \M_+(\R^d)} D_f(\sigma \mid \nu) + \frac{1}{2 \lambda} d_K(\mu, \sigma)^2
    = \min_{\substack{\rho \in L^1(\R^d) \\ \rho \ge 0}} \int_{\R^d} f_{\KL}\circ \rho \d{\nu} + \frac{1}{2 \lambda} \| m_{\rho \nu} - m_{\mu} \|_{\H_K}^2.
\end{align} 
\end{remark}

We now discuss topological properties of $D_{f, \nu}^{\lambda}$, generalizing \cite[Thm.~1]{GAG2021}, where only metrization of $\P_1(\R^d) \subset B_{\varphi(0)}(0)$ is shown.
\begin{corollary}[Topological properties]\label{lemma:topology}\hfill
    \begin{enumerate}
        \item
        For every $\nu \in \M_+(\R^d)$, the function $D_{f,\nu}^{\lambda}\colon \M(\R^d) \to \R$ is weakly continuous.
        
        \item
        If $D_f$ is a divergence, then $D_f^{\lambda}$ in \eqref{eq:RegFDiv} is a divergence as well.
        
        \item \label{item:Top3}
        If $D_f$ is a divergence, then $D_f^{\lambda}$ \enquote{metrizes} the topology on the closed balls $B_r(\mu_0) \coloneqq \{\mu \in \mathcal M_+(\R^d): d_K(\mu,\mu_0) \le r\} \subset (\mathcal M_+(\R^d), d_K)$ for any $r > 0$ and any $\mu_0 \in \M_+(\R^d)$ in the sense that for $\mu_n, \mu \in B_r(\mu_0)$ it is equivalent that $D_{f}^{\lambda}(\mu_n \mid \mu) \to 0$ and that $d_K(\mu_n, \mu) \to 0$.
    \end{enumerate}
\end{corollary}

\begin{proof}
    \begin{enumerate}
        \item 
        Let $(\mu_n)_{n \in \N}$ converge weakly to $\mu$, i.e.,
        $\lim_{n \to \infty} \langle f,\mu_n\rangle = \langle f,\mu\rangle$
        for all $f \in \C_b(\R^d)$.
        Then, we have by \cite[Lemma~10]{SGS2018} that $m_{\mu_n} \to m_{\mu}$ in $\H_K$.
        Since $G^\lambda_{f,\nu}$ is continuous by Theorem~\ref{prop:moreau}(\ref{item:MoreauProp3}),
        this implies together with Corollary~\ref{cor:coincidence} that
        \begin{equation}
            D^\lambda_{f,\nu} (\mu) = G^\lambda_{f,\nu} (m_\mu) = \lim_{n \to \infty} G^\lambda_{f,\nu} (m_{\mu_n}) = 
            \lim_{n \to \infty}D^\lambda_{f,\nu} (\mu_n).
        \end{equation}

        \item 
        The definiteness follows from \eqref{eq:RegFDiv} since both summands are non-negative.
        We refer to \cite[Thm.~75.4]{BDKPR23} for a more detailed proof in a slightly different setting.

        \item 
        Since $f$ has its unique minimizer at $x = 1$, we have $0 \in \partial f(x)$ only for $x = 1$.
        Hence, $\partial f^*(0) = \{ 1 \}$, and thus $f^*$ is differentiable at zero with $(f^*)'(0) = 1$.   
        Let $\mu_n,\mu \in B_r(\mu_0)$ for all $n \in \N$.
        As $G^\lambda_{f,\nu}$ is continuous, the relation $d_K(\mu_n, \mu) = \Vert m_{\mu_n} - m_{\mu} \Vert_{\H_K}\to 0$ implies $D^\lambda_{f}(\mu_n \mid \mu) = G^\lambda_{f,\mu} (m_{\mu_n}) \to D^\lambda_{f}(\mu \mid \mu) = 0$.
    
        For the reverse direction, assume that $\smash{D^\lambda_{f}}(\mu_n \mid \mu)\to 0$.
        For any $n\in\N$, we define $g_n \coloneqq  m_{\mu_n- \mu} = m_{\mu_n} - m_\mu$, for which it holds that 
        \begin{equation}\label{eq:BoundPotential}
            \| g_n \|_{\C_0}
            \le C \| g_n \|_{\H_K}
            = C d_K(\mu_n, \mu)
            \le C \big( d_K(\mu_n, \mu_0) + d_K(\mu_0, \mu) \big)
            \le 2 C r,
        \end{equation}
        where $C > 0$ is the embedding constant from $\H_K \hookrightarrow \C_0(\R^d)$.
        Hence, it holds for any \smash{$\frac{f_{\infty}'}{2 C r} > \epsilon > 0$} that $\epsilon g_n(\R^d) \subset \dom(f^*)$ since $\epsilon | g_n(x) | < \smash{\frac{f_{\infty}'}{2 C r}} | g_n(x) | \le f_{\infty}'$ for all $x \in \R^d$.
        By \eqref{eq:RegFDivDual}, we further get that
        \begin{align}
            D^\lambda_{f}(\mu_n \mid \mu)
            &\ge \E_{\mu_n}\big[\epsilon g_n\big]
            - \E_\mu \left[f^*\circ(\epsilon g_n) \right]
            - \frac{\lambda \epsilon^2}{2} \| g_n \|_{\H_{K}}^2
            \notag\\
            &= 
            \epsilon \| g_n \|^2_{\H_{K}} - \E_{\mu}[f^* \circ (\epsilon g_n) - \epsilon g_n] - \frac{\lambda \epsilon^2}{2} \| g_n \|_{\H_{K}}^2 \notag\\
            & = \Bigl(\epsilon - \frac{\lambda \epsilon^2}{2}\Bigr) \| g_n \|^2_{\H_{K}} - \E_{\mu}[f^* \circ (\epsilon g_n) - \epsilon g_n].\label{eq:LowerBoundTop}
         \end{align}
        Using Taylor's theorem with the Peano form of the remainder, there exists $h \colon (-\infty, f_{\infty}') \to \R$ with $\lim_{t \to 0} \frac{h(t)}{t} = 0$ such that $f^*(t) = t + h(t)$ for all $t \in (- \infty, f_{\infty}')$.
        Thus, we get
        \begin{equation}\label{eq:TaylorExp}
            \E_{\mu}[f^* \circ (\epsilon g_n) - \epsilon g_n]
            = \int_{\R^d}h \circ (\epsilon g_n) \d{\mu}
            \le \Bigl\|\frac{ h \circ (\epsilon g_n)}{\epsilon g_n} \Bigr\|_{\infty} \Vert {\epsilon g_n} \|_{\infty} \mu(\R^d)
        \end{equation}
        with the convention $\frac{h(\epsilon g_n(x))}{\epsilon g_n(x)} = 0$ if $g_n(x) = 0$.
        Plugging \eqref{eq:TaylorExp} into \eqref{eq:LowerBoundTop} yields
        \begin{equation}
            D^\lambda_{f}(\mu_n \mid \mu)
            \ge \epsilon \| g_n \|_{\H_K} \left( \Bigl(1 - \frac{\lambda}{2} \epsilon\Bigr) \| g_n \|_{\H_K} - C \Bigl\|\frac{ h \circ (\epsilon g_n)}{\epsilon g_n} \Bigr\|_{\infty} \mu(\R^d) \right).\label{eq:LowerBoundTop2}
        \end{equation}
        
        Since $\lim_{t \to 0} \frac{h(t)}{t} = 0$ and $\epsilon g_n$ is uniformly bounded, equation \eqref{eq:BoundPotential} implies that the second factor in \eqref{eq:LowerBoundTop2} converges to $\| g_n \|_{\H_K}$ for $\epsilon \searrow 0$.
        Thus, we can pick $\epsilon>0$ independent of $n$ such that $D^\lambda_{f}(\mu_n \mid \mu) \ge \frac{\epsilon}{2} \| g_n \|^2_{\H_{K}} = \frac{\epsilon}{2} d_K(\mu_n,\mu)^2$.
        From this, we infer $d_K(\mu_n,\mu) \to 0$.
    \end{enumerate}
\end{proof}

Now, we investigate the two asymptotic regimes of $D_f^\lambda$.

\begin{corollary}[Limits for $\lambda \to 0$ and $\lambda \to \infty$] \label{lemma:lambdaToZero}
\hfill
\begin{enumerate}
    \item \label{item:lambdatoZero1}
    It holds $\lim_{\lambda \to 0} D_{f, \nu}^\lambda(\mu) = D_{f, \nu}(\mu)$ for $\mu, \nu \in \M_+(\R^d)$.
    
    \item
    We have that \smash{$D_{f,\nu}^\lambda$} converges to $D_{f,\nu}$ in the sense of Mosco: 
    For $\mu \in \M_+(\R^d)$, a monotonically decreasing $(\lambda_n)_{n \in \N} \subset (0, \infty)$ with $\lambda_n \to 0$, and $(\mu_n)_{n \in \N} \subset \M_+(\R^d)$ with $\mu_n \weakly \mu$ it holds that
    \begin{equation}\label{eq:liminf}
        D_{f,\nu}(\mu) \le \liminf_{n \to \infty} D_{f,\nu}^{\lambda_n}(\mu_n).
    \end{equation}
    Further, there exists a sequence $(\tilde{\mu}_n)_{n \in \N} \subset \M_+(\R^d)$ with $\tilde{\mu}_n \to \mu$ such that
    \begin{equation}\label{eq:limsup}
        D_{f,\nu}(\mu) = \lim_{n \to \infty} D_{f,\nu}^{\lambda_n}(\tilde{\mu}_n).
    \end{equation}
    
    \item \label{item:lambdatoZero3}
    For any $r > 0$ and any $\mu_0 \in \M_+(\R^d)$, it holds
    \begin{equation}
        \lim_{\lambda \to \infty} \sup_{\mu \in B_r(\mu_0)}\Bigl \vert (1+\lambda) D_{f}^{\lambda}(\mu \mid \nu) - \frac{1}{2} d_K(\mu, \nu)^2 \Bigr \vert = 0.
    \end{equation}
    \end{enumerate} 
\end{corollary}

\begin{proof}
\hfill
\begin{enumerate}
    \item
    Since $f_{\infty}' > 0$, we have $0 \in \intt(\dom(f^*))$.
    By Theorem~\ref{prop:moreau}(\ref{item:MoreauProp3}), we get pointwise convergence to $G_{f, \nu}$.
    Then, the statement follows by the final part of Lemma~\ref{prop:LowerHull} and Corollary~\ref{cor:coincidence}.

    \item 
    By Theorem~\ref{prop:moreau}(\ref{item:MoreauProp3}), the sequences \smash{$(G_{f, \nu}^{\lambda_n}(h))_{n \in \N}$} are monotonically increasing with $\sup_{n \in \N} G_{f, \nu}^{\lambda_n}(h) = G_{f, \nu}(h)$ for any $h \in \H_K$.
    As $G_{f, \nu}$ is lower semicontinuous, \cite[Rem.~2.12]{Braides2006} implies that \smash{$(G_{f, \nu}^{\lambda_n})_{n \in \N}$} $\Gamma$-converges to $G_{f, \nu}$.
    More precisely, it holds that $G_{f, \nu}(h) \le \liminf_{n \to \infty} G_{f, \nu}^{\lambda_n}(h_n)$ for every $h \in \H_K$ and any sequence $(h_n)_{n \in \N} \subset \H_K$ with $h_n \to h$.
    Further, for every $h \in \H_K$ it holds that $G_{f, \nu}(h) = \lim_{n \to \infty} \smash{G_{f, \nu}^{\lambda_n}}(h)$.
    The statement now follows from the fact that \smash{$D_{f, \nu}^{\lambda} = G_{f, \nu}^{\lambda} \circ m$} by Corollary~\ref{cor:coincidence} and since $\mu_n \weakly \mu$ in $\M_+(\R^d)$ implies $m_{\mu_n} \to m_{\mu}$ by \cite[Lemma~10]{SGS2018}.

    \item
    From \eqref{eq:RegFDiv}, we infer \smash{$(1+\lambda) D_{f}^{\lambda}(\mu \mid \nu) \le \frac{1 + \lambda}{2 \lambda} d_K(\mu, \nu)^2$}.
    To get a lower bound, we proceed as for Corollary~\ref{lemma:topology}(\ref{item:Top3}).
    For $g_{\mu,\lambda} \coloneqq \smash{\frac{1}{\lambda}} m_{\mu - \nu}$ with $\lambda > \smash{\frac{C}{f_{\infty}'}} | \mu - \nu |(\R^d) > 0$, we have that $g_{\mu,\lambda}(\R^d) \subset \dom(f^*)$.
    Analogously to \eqref{eq:LowerBoundTop} and \eqref{eq:LowerBoundTop2} (using the same $h$), we get
    \begin{align}
       (1+\lambda) D^{\lambda}_{f}(\mu \mid \nu)
        & \ge \frac{(1+\lambda)}{\lambda} d_K(\mu, \nu) \left( \frac{1}{2} d_K(\mu, \nu) - C \nu(\R^d) \Bigl\| \frac{h \circ g_{\mu,\lambda}}{g_{\mu,\lambda}} \Bigr\|_{\infty} \right)\notag\\
        & \ge \frac{1}{2} d_K(\mu, \nu)^2 - C \nu(\R^d) d_K(\mu, \nu) \frac{1+\lambda}{\lambda} \Bigl\| \frac{h \circ g_{\mu,\lambda}}{g_{\mu,\lambda}} \Bigr\|_{\infty}.\label{eq:LowerEst}
    \end{align}
    Combining \eqref{eq:LowerEst} with $(1+\lambda) D_{f}^{\lambda}(\mu \mid \nu) \le \frac{1 + \lambda}{2 \lambda} d_K(\mu, \nu)^2$, we get for $\mu \in B_r(\mu_0)$ that
    \begin{align}
        &\Bigl| (1 + \lambda) D_{f, \nu}^{\lambda}(\mu) - \frac{1}{2} d_K(\mu, \nu)^2 \Bigr|\notag\\
        \le &d_K(\mu, \nu) \max\bigg(\frac{d_K(\mu, \nu)}{2 \lambda}, C \nu(\R^d) \frac{1+\lambda}{\lambda} \Bigl\| \frac{h \circ g_{\mu,\lambda}}{g_{\mu,\lambda}} \Bigr\|_{\infty} \bigg)\notag\\
        \le &\bigl(d_K(\mu_0, \nu)+ r\bigr) \max\bigg(\frac{d_K(\mu_0, \nu) + r}{2 \lambda}, C \nu(\R^d) \frac{1+\lambda}{\lambda} \Bigl\| \frac{h \circ g_{\mu,\lambda}}{g_{\mu,\lambda}} \Bigr\|_{\infty} \bigg).
    \end{align}
    Here, the first term in the maximum converges to zero as $\lambda \to \infty$.
    As for Corollary~\ref{lemma:topology}(\ref{item:Top3}), $\Vert g_{\mu,\lambda}\Vert_\infty \le  \frac{C}{\lambda}(d_K(\mu_0, \nu) + r)$ together with $\lim_{t \to 0} \frac{1}{t} h(t) = 0$ yields that also the second term converges to zero.
    Thus, the claim follows.
\end{enumerate}
\end{proof}

The  Moreau envelope interpretation of $D_{f,\nu}^{\lambda}$ allows the calculation of its gradient without the implicit function theorem, which is used to justify the calculations for the KALE functional in \cite[Lem.~2]{GAG2021}.
Furthermore, we strengthen their result by proving Fréchet differentiability instead of G$\hat{\text{a}}$teaux differentiability.

\begin{corollary}[Gradient] \label{lem:Grad} 
    The function $D_{f,\nu}^{\lambda}\colon \M(\R^d) \to [0,\infty)$
    is Fr\'echet differentiable and
        $\nabla D_{f,\nu}^{\lambda}(\mu) = \hat p$, 
    where $\hat p \in \H_{K}$ 
    is the maximizer in \eqref{eq:RegFDivDual}.
    Further, the mapping $\nabla D_{f,\nu}^{\lambda} \colon \mathcal M (\R^d) \to \H_K$ is $\frac{1}{\lambda}$-Lipschitz with respect to $d_K$.
\end{corollary}

\begin{proof}
    By Theorem~\ref{prop:moreau}(\ref{item:MoreauProp2}) and Corollary~\ref{lemma:RegFDivDual}, we obtain 
    \begin{equation}
        \nabla G^{\lambda}_{f,\nu}(m_\mu) = \lambda^{-1}\bigl(m_\mu - \prox_{\lambda G_{f,\nu}}(m_\mu)\bigr) = \hat p.
    \end{equation}
    As a concatenation of a Fr\'echet differentiable and a continuous linear mapping, \smash{$D_{f,\nu}^{\lambda} = G^{\lambda}_{f,\nu} \circ m$} is Fr\'echet differentiable.
    The Fr\'echet derivative of the KME $m\colon \M(\R^d) \to \H_{K}$ at $\mu \in \M(\R^d)$ is $\d m_{\mu} = m \in L(\M(\R^d); \H_{K})$.
    Using these computations together with the chain rule, we get for any  $\sigma \in \mathcal M(\R^d)$ that
    \begin{equation}\label{deriv}
        \d D_{f,\nu}^{\lambda}(\mu) [\sigma]
        =
        \d (G^{\lambda}_{f,\nu} \circ m)(\mu)[\sigma]
        = \d G^{\lambda}_{f,\nu} (m_{\mu}) [m_\sigma]
        = \langle \hat p, \sigma \rangle_{\C_0 \times \M}.
    \end{equation}
    Hence, we identify 
    $\nabla D_{f,\nu}^{\lambda}(\mu) = \hat p \in \C_0(\R^d)$.
    Finally, we have by Theorem~\ref{prop:moreau}(\ref{item:MoreauProp2}) 
    that
    \begin{equation} \label{fini}
    \| \nabla G^{\lambda}_{f,\nu}(m_\mu) - \nabla G^{\lambda}_{f,\nu}(m_\sigma) \|_{\H_K}
    \le \lambda^{-1} \|m_\mu - m_\sigma \|_{\H_K}
    = 
    \lambda^{-1} d_K(\mu,\sigma),
    \end{equation}
    see \cite[Prop.~12.28]{BC2011}, which completes the proof.
\end{proof}

%--------------------------------------------------------------------------
\section{Wasserstein Gradient Flows of Regularized \texorpdfstring{$f$}{f}-Divergences} \label{sec:wgf} 
%--------------------------------------------------------------------------
Now, we discuss Wasserstein gradient flows of \smash{$D^\lambda_{f,\nu}$}.
This requires some preliminaries from \cite[Secs.~8.3, 9.2, 11.2]{AGS2008}, which we adapt to our setting in Section \ref{subsec:wgfPrelims}.

\subsection{Wasserstein Gradient Flows} \label{subsec:wgfPrelims}
%--------------------------------------------------------------
We consider the \emph{Wasserstein space} $\P_2(\R^d)$ of Borel probability measures with finite second moment, equipped
with the \emph{Wasserstein distance} 
\begin{equation}\label{eq:W2}
W_2(\mu,\nu)^2 \coloneqq\min_{\pi\in\Gamma(\mu,\nu)}\int_{\R^d\times\R^d}\|x-y\|_2^2\d \pi(x,y),
\end{equation}
where $\Gamma(\mu,\nu)=\{\pi\in\P_2(\R^d\times\R^d):(P_1)_\#\pi=\mu,\,(P_2)_\#\pi=\nu\}$.
Here, $T_{\#}\mu \coloneqq \mu \circ T^{-1}$ denotes the \emph{push-forward} of $\mu$ via 
the measurable map $T$, and
$P_i(x) \coloneqq x_i$, $i = 1,2$, for $x = (x_1,x_2) \in \R^{d} \times \R^{d}$.

A curve $\gamma \colon [0,1] \to \P_2(\R^d)$, $t \mapsto \gamma_t$ is a \emph{geodesic} if 
$W_2(\gamma_{s}, \gamma_{t}) = W_2(\gamma_0,\gamma_1)\vert t - s\vert$ for all $s, t \in [0,1]$.
The Wasserstein space is geodesic, i.e., any measures $\mu_0, \mu_1 \in \P_2(\R^d)$ can be connected by a geodesic.
These are all of the form 
\begin{equation}
    \gamma_t
    = \big( (1-t)P_1 +tP_2 \big)_{\#} \hat \pi, \qquad t \in [0, 1],
\end{equation} 
where $\hat \pi \in \Gamma (\mu_0,\mu_1)$ realizes $W_2(\mu_0, \mu_1)$ in \eqref{eq:W2}.

For $\F\colon \P_2(\R^d) \to (-\infty,\infty]$, we set $\dom(\F) \coloneqq \{ \mu \in \P_2(\R^d): \F(\mu) < \infty \}$.
The function $\F$ is \emph{$M$-convex along geodesics} with $M \in \R$ if, for every 
$\mu_0, \mu_1 \in \dom(\F)$,
there exists a geodesic $\gamma \colon [0, 1] \to \P_2(\R^d)$ 
between $\mu_0$ and $\mu_1$ such that
\begin{equation*}
    \F(\gamma_t) 
    \le
    (1-t) \, \F(\mu_0) + t \, \F(\mu_1) 
    - \frac{M}{2} \, t (1-t) \, W_2(\mu_0, \mu_1)^2, 
    \qquad t \in [0,1].
\end{equation*}
Frequently, we need a more general notion of convexity.
Based on the set of three-plans with base $\sigma \in \P_2(\R^d)$ given by
\begin{equation}
    \Gamma_\sigma(\mu,\nu)
    \coloneqq
    \bigl\{ 
    \zb{\alpha} \in \P_2(\R^d \times \R^d \times \R^d)
    :
    (P_1)_\# \zb \alpha = \sigma,
    (P_2)_\# \zb \alpha = \mu,
    (P_3)_\# \zb \alpha = \nu
    \bigr\},
\end{equation}
a \emph{generalized geodesic} $\gamma_{\boldsymbol{\alpha}} \colon [0, 1] \to \P_2(\R^d)$
joining $\mu$ and $\nu$ with base $\sigma$ is defined as
\begin{equation} \label{ggd}
  \gamma_{\boldsymbol{\alpha},t} 
  \coloneqq 
  \bigl(
  (1-t ) P_2 
  + t P_3
  \bigr)_\# \zb \alpha,
  \qquad
  t \in [0, 1],
\end{equation}
where $\zb \alpha \in \Gamma_\sigma (\mu,\nu)$ with $(P_{1,2})_\# \zb{\alpha} \in \Gamma^{\opt}(\sigma,\mu)$ and $(P_{1,3})_\# \zb{\alpha} \in \Gamma^{\opt}(\sigma,\nu)$.
Here, $\Gamma^{\opt}(\mu, \nu)$ denotes the \emph{set of optimal transport plans} that minimize \eqref{eq:W2}.
The plan $\zb{\alpha}$ can be interpreted as transport from $\mu$ to $\nu$ via $\sigma$.
A function $\F\colon \P_2(\R^d) \to (-\infty,\infty]$ is
\emph{$M$-convex along generalized geodesics}
if, for every $\sigma,\mu,\nu \in \dom(\F)$,
there exists a generalized geodesic $\gamma_{\boldsymbol{\alpha}} \colon [0,1] \to \P_2(\R^d)$ such that
\begin{equation} \label{eq:gg}
  \F(\gamma_{\boldsymbol{\alpha},t})
  \le 
  (1-t) \, \F(\mu) 
  + t \, \F(\nu) 
  - \frac{M}{2} \, t(1-t) \, W_{\zb \alpha}^2(\mu,\nu), 
  \qquad t \in [0,1],
\end{equation}
where 
\begin{equation}
W_{\zb{\alpha}}^2 (\mu, \nu)
\coloneqq \int_{\R^d \times \R^d \times \R^d} \|y - z\|_2^2 \d \zb{\alpha}(x,y,z).
\end{equation}
Each function that is $M$-convex along generalized geodesics 
is also $M$-convex along geodesics
since generalized geodesics with base $\sigma = \mu$ are actual geodesics.

The \emph{strong (reduced) Fr\'echet subdifferential} $\partial\F(\mu)$
of a proper, lower semicontinuous, geodesically convex function 
$\F\colon \P_2(\R^d) \to (-\infty,\infty]$ at $\mu$ consists of $v \in L_2(\R^d, \mu; \R^d)$ such that for every $\eta \in \P_2(\R^d)$ and every $\pi \in \Gamma(\mu,\eta)$, it holds 
\begin{equation}\label{eq:subdiff_ineq}
    \F (\eta) - \F(\mu) \ge \int_{\R^d \times \R^d} \langle v(x_1), x_2 - x_1\rangle \d \pi(x_1,x_2) + o(C_2(\pi)),
\end{equation}
where $C_2(\pi)^{2} \coloneqq \int_{\R^d \times \R^d} \| x_1 - x_2\|_2^2 \d \pi(x_1,x_2)$ and the asymptotic $o(\cdot)$ has to be understood with respect to the $W_2$ metric, see \cite[Eq.~(10.3.7)]{AGS2008}.

A curve $\gamma\colon (0,\infty)\to\P_2(\R^d)$ 
is called locally 2-\emph{absolutely continuous} 
if there exists a Borel velocity field $v\colon\R^d \times (0, \infty) \to\R^d$, $(x, t) \mapsto v_t(x)$ with 
$(t \mapsto \| v_t \|_{L^2(\R^d; \gamma_t)}) \in L_{\loc}^2(0, \infty)$
such that the continuity equation 
\begin{equation}
    \partial_t\gamma_t +\nabla\cdot(v_t\gamma_t)
    = 0    
\end{equation}
is fulfilled on $(0,\infty) \times\R^d$ in a weak sense, i.e., 
\begin{equation}
    \int_0^\infty \int_{\R^d} \partial_t \varphi(t,x) + \langle \nabla_x\varphi(t,x), v_t(x) \rangle \d \gamma_t(x) \d t
    = 0 \quad\quad \text{for all } \varphi \in \C_{c}^{\infty}\bigl((0,\infty)\times \R^d\bigr).
\end{equation}
A locally 2-absolutely continuous curve
$\gamma \colon (0,\infty) \to \P_2(\R^d)$
with velocity field $v_t$ in the regular tangent space
  \begin{equation*}
      \textrm{T}_{\gamma_t} \P_2(\R^d) \coloneqq \overline{\{ \nabla \phi: \phi \in \C_{\textrm{c}}^{\infty}(\R^d) \}}^{L_2(\R^d, \gamma_t; \R^d)}
  \end{equation*}
of
  $\P_2(\R^d)$ at $\gamma_t$
  is called a \emph{Wasserstein gradient flow of}
  $\F\colon \P_2(\R^d) \to (-\infty, \infty]$
  if 
  \begin{equation}\label{wgf}
      v_t  \in - \partial \F(\gamma_t), \quad \text{for a.e. } t > 0.
  \end{equation}
In principle, \eqref{wgf} involves the so-called reduced Fr\'echet subdifferential $\partial \F$. 
We will show the $D^\lambda_{f,\nu}$ are Fr\'echet differentiable such that we can use the strong Fr\'echet subdifferential in \eqref{eq:subdiff_ineq} instead.
Then, we have the following theorem \cite[Thm.~11.2.1]{AGS2008}.

\begin{theorem}\label{ambrosio}
Assume that $\F\colon\P_2(\R^d)\to(-\infty,\infty]$ is proper, lower semicontinuous, coercive 
and $M$-convex along generalized geodesics.
Given $\mu_0 \in \overline{\dom (\F)}$, there exists a unique Wasserstein gradient flow of $\F$ with $\gamma(0+) = \mu_0$.
\end{theorem}

\begin{remark}
    By \cite[Eq.~(11.2.1b)]{AGS2008} any proper $M$-convex functional that is bounded from below by some constant is coercive in the sense of \cite[Eq.~(2.1.2b)]{AGS2008}.
\end{remark}

%--------------------------------------------------------------

\subsection{Wasserstein Gradient Flow of \texorpdfstring{$D_{f,\nu}^{\lambda}$}{the regularized f-Divergence functional}} \label{sec:4.2}
%--------------------------------------------------------------
Now, we show that \smash{$D_{f,\nu}^{\lambda}$} is locally Lipschitz, $M$-convex along generalized geodesics, and that $\partial D_{f,\nu}^{\lambda}(\mu)$ is a singleton for every $\mu \in \P_2(\R^d)$.
To this end, we rely on \cite[Prop.~2, Cor.~3]{VG23}, which we adapt to our setting in the following lemma.
\begin{lemma}\label{lem:dK_embed}
Let the kernel $K$ fulfill
$K(x, x) + K(y, y) - 2K(x, y)\le  C_{\emb}^2 \| x - y \|_2^2$
with some constant $C_{\emb} >0$.
Then, it holds
\begin{equation}\label{eq:grib}
d_K(\mu, \nu) \le C_{\emb} W_2(\mu, \nu).
\end{equation}
If $K(x,y) = \Phi(x-y)$ is translation invariant, then we get $C_{\emb} = \sqrt{\lambda_{\max} (-\nabla^2 \Phi (0))}$.
For radial kernels $K(x,y) = \phi(\|x-y\|_2^2)$, we have
\begin{equation}
\nabla^2 \Phi(x) = 4 \phi^{\prime\prime}(\|x\|_2^2) x x^\tT + 2 \phi'(\|x\|_2^2) \id,
\end{equation}
so that
$
 C_{\emb} = \sqrt{-2 \phi'(0)}
$.
\end{lemma}
Note that the authors in \cite{VG23} found $C_{\emb} = \sqrt{ \lambda_{\max} (-\nabla^2 \Phi (0)) } \Phi (0)$ instead, which we could not verify.
Now, we prove the local Lipschitz continuity of $D_{f,\nu}^{\lambda}$ with Theorem~\ref{prop:moreau}(\ref{item:MoreauProp2}) and Lemma~\ref{lem:dK_embed}.
\begin{lemma}
    The function $D_{f,\nu}^{\lambda}\colon (\P_2(\R^d),W_2) \to [0,\infty)$
    is locally Lipschitz continuous.
\end{lemma}

\begin{proof}
   By Theorem~\ref{prop:moreau}(\ref{item:MoreauProp2}), 
   we know that $\smash{G_{f,\nu}^{\lambda}}\colon \H_K \to [0,\infty)$ is continuously Fr\'echet differentiable.
   Hence, it is locally Lipschitz continuous.
   Since $m \colon (\P_2(\R^d),d_K) \to \H_K$ is an isometry and $G_{f,\nu}^{\lambda} \circ m = D_{f, \nu}^{\lambda}$,
   the claim follows using \eqref{eq:grib}.
\end{proof}

Next, we show the $M$-convexity of \smash{$D^\lambda_{f,\nu}$} along generalized geodesics.
Our Moreau envelope interpretation simplifies the proof in \cite[Lem.~4]{GAG2021} since we do not need to invoke the implicit function theorem.

\begin{theorem} \label{thm:M}
Let $K(x,y) = \phi( \|x-y\|_2^2)$ with $\phi \in \C^2([0, \infty))$.
Then, $D^\lambda_{f,\nu}\colon \P_2(\R^d) \to [0,\infty)$ is $(-M)$-convex along generalized geodesics with $M \coloneqq \frac{8}{\lambda}\sqrt{(d+2) \phi^{\prime\prime}(0) \phi(0)}$.
\end{theorem}

\begin{proof}
    Let $\mu_1, \mu_2, \mu_3 \in \P_2(\R^d)$ and $\gamma \colon [0, 1] \to \P_2(\R^d)$, $t \mapsto ( (1 - t) P_2 + t P_3)_{\#} \bm \alpha$ be a generalized geodesic associated to a three-plan $\bm \alpha \in \Gamma_{\mu_1}(\mu_2, \mu_3)$.
    Furthermore, $\tilde{\gamma} \colon [0, 1] \to \P_2(\R^d)$, $t \mapsto (1 - t) \mu_2 + t \mu_3$ denotes the linear interpolation between $\mu_2$ and $\mu_3$.
    Since \smash{$G_{f,\nu}^{\lambda}$} and \smash{$D_{f,\nu}^{\lambda}$} are linearly convex, it holds for $t \in [0, 1]$ that
    \begin{align}
        D^\lambda_{f,\nu}\left(\gamma_t \right) 
        &\le (1-t) D^\lambda_{f,\nu}(\mu_2) + t D^\lambda_{f,\nu}(\mu_3) + D^\lambda_{f,\nu}\left(\gamma_t \right) - D^\lambda_{f,\nu}\left(\tilde \gamma_t\right)\notag\\
        & \le (1-t) D^\lambda_{f,\nu}(\mu_2) + t D^\lambda_{f,\nu}(\mu_3) + \bigl\langle \nabla G_{f,\nu}^\lambda (m_{\gamma_t}), m_{\gamma_t}- m_{\tilde \gamma_t} \bigr\rangle_{\H_K} \label{eq:est_m_conv}.
    \end{align}
    We consider the third summand.
    Let $\hat p_t$ maximize the dual formulation \eqref{eq:MoreauEnvDual} of \smash{$G_{f,\nu}^\lambda (m_{\gamma_t})$}.
    Then, we know by Theorem~\ref{prop:moreau}(\ref{item:MoreauProp2}) that
    $\nabla G_{f,\nu}^\lambda(m_{\gamma_t}) = \hat{p}_t$, so that
    \begin{align}
        &\big\langle \nabla G_{f,\nu}^\lambda(m_{\gamma_t}), m_{\gamma_t}- m_{\tilde \gamma_t}\big\rangle_{\H_K}\notag\\
        = &\int_{\R^d \times \R^d \times \R^d} \hat{p}_t\big( (1 - t) y_2 +t y_3\big) -\bigl( (1 - t) \hat{p}_t(y_2) + t \hat{p}_t(y_3) \bigr)\d \bm \alpha(y_1, y_2, y_3).
    \end{align}
    Due to \eqref{2a}, $\nabla \hat{p}_t$ is Lipschitz continuous with $\Lip(\nabla \hat{p}_t) \le 2 \Vert \hat{p}_t \Vert_{\H_K} \sqrt{\phi^{\prime\prime}(0) (d+2)}$.
    Hence, the descent lemma \cite[Lemma~1.2.3]{N2003} implies that $\hat{p}_t$ is $(-\Lip(\nabla \hat{p}_t))$-convex (in the sense of \cite[Def.~2.4.1]{AGS2008} as a function on the metric space $(\R^d, \| \cdot \|_2)$) and
    \begin{align}
        \bigl\langle \nabla G_{f,\nu}^\lambda(m_{\gamma_t}), m_{\gamma_t-\tilde \gamma_t}\bigr\rangle_{\H_K}
        &\le \frac{t (1 - t)}{2} \Lip(\nabla \hat{p}_t) \int_{\R^d \times \R^d \times \R^d} \| y_2 - y_3 \|_2^2 \d{\bm \alpha}(y_1, y_2, y_3) \notag\\
        &= \frac{1}{2} \Lip(\nabla \hat{p}_t) t (1 - t) W_{\bm \alpha}^2(\mu_2,\mu_3).
    \end{align}
    By \eqref{some_est}, we have $\|\hat p_{t}\|_{\H_K} \le \frac{2}{\lambda}(\|m_{\gamma_t}\|_{\H_K} + \|m_{\nu}\|_{\H_K})$.
    For any $\mu \in \P_2(\R^d)$, it holds
    \begin{equation} 
        \|m_\mu\|_{\H_K}^2
        = \int_{\R^d \times \R^d} K(x,y) \d{\mu}(y) \d{\mu}(x) 
        = \int_{\R^d \times \R^d} \phi(\|x-y\|_2^2) \d{\mu}(y) \d{\mu}(x) 
        \le \phi(0),
    \end{equation} 
    which implies
    \begin{equation}\label{brutal}
        \|\hat p_{t}\|_{\H_K} 
        \le \frac{4}{\lambda} \sqrt{\phi(0)}.
     \end{equation}
     Thus, for all $t \in [0, 1]$, we have 
     $
        \Lip(\nabla \hat{p}_t)
        \le M \coloneqq \frac{8}{\lambda} \sqrt{\phi(0) \phi''(0) (d+2)}$,
    and $D^\lambda_{f,\nu}$ is $-M$-convex along generalized geodesics.
\end{proof}

The next proposition determines the strong subdifferential of $D_{f,\nu}^{\lambda}$.
The result generalizes \cite[Lemma~3]{GAG2021} to arbitrary $f$-divergences satisfying Assumption \ref{eq:Assum3}.
\begin{lemma} \label{lemma:RegFDivDerivative}
    Let $K(x,y) = \phi( \|x-y\|_2^2)$ and $\phi \in \C^2([0, \infty))$.
    Then, it holds for any $\mu \in \P_2(\R^d)$ that $\partial D_{f,\nu}^{\lambda}(\mu) = \{\nabla \hat p\}$, where $\hat p \in \H_{K}$ maximizes \eqref{eq:RegFDivDual}.
\end{lemma}

\begin{proof}
    First, we show that $\nabla \hat p \in \partial \smash{D_{f,\nu}^{\lambda}}(\mu)$.
    By convexity of \smash{$G_{f,\nu}^{\lambda}$} and Lemma~\ref{lem:Grad}, we obtain for any $\eta \in \P_2(\R^d)$ and any $\pi \in \Gamma(\mu,\eta)$ that
    \begin{align} \label{eq:ineq_nablag}
    \begin{aligned}
        D_{f,\nu}^{\lambda}(\eta) - D_{f,\nu}^{\lambda}(\mu) 
        &= G_{f,\nu}^{\lambda}(m_\eta) - G_{f,\nu}^{\lambda}(m_\mu)
        \ge \langle \nabla G_{f,\nu}^{\lambda}(m_\mu), m_\eta - m_\mu\rangle_{\H_K}  \\
        &= \langle \hat p, m_\eta - m_\mu\rangle_{\H_K} 
        = \int_{\R^d} \hat p(x_2) \d \eta(x_2) - \int_{\R^d} \hat p(x_1) \d \mu(x_1) \\
        &= \int_{\R^d \times \R^d} \hat p(x_2) - \hat p(x_1) \d \pi(x_1,x_2).
    \end{aligned}
    \end{align}
    Combining \eqref{2a} and \eqref{brutal}, we get that $\hat p \in \H_K$ has a Lipschitz continuous gradient with Lipschitz constant $L \coloneqq \frac{8}{\lambda} \sqrt{\phi(0) \phi^{\prime\prime} (0) (d+2)}$.
    Hence, we obtain by the descent lemma \cite[Lemma~1.2.3]{N2003} that
    \begin{equation} \label{eq:descentlemma}
        \begin{aligned}
        D_{f,\nu}^{\lambda}(\eta) - D_{f,\nu}^{\lambda}(\mu) 
        &\ge 
        \int_{\R^d \times \R^d} \langle \nabla \hat p (x_1), x_2 - x_1\rangle - \frac{L}{2} \Vert x_2 - x_1\Vert_2^2 \d \pi(x_1,x_2) \\
        & = 
        \int_{\R^d \times \R^d} \langle \nabla \hat p (x_1), x_2 - x_1\rangle \d \pi(x_1,x_2) - \frac{L}{2} C_2^2(\pi).
        \end{aligned}
    \end{equation}
    Next, we show that $\nabla \hat p$ is the only subgradient.
 Suppose that $v\colon \R^d\to {\R^d}$ with ${\|\cdot\|_2\circ v} \in L^2({\R^d},\mu)$ is another subgradient. 
We define the perturbation map $A_t(x)\coloneqq x-t(\nabla \hat p(x)-v(x))$ with $t>0$, the perturbed measures $ \mu_t\coloneqq (A_t)_\# \mu$, and the induced plan $\pi_t\coloneqq (\id,A_t)_\# \mu\in \Gamma(\mu,\mu_t)$.
For $\pi_t$ it holds
\begin{equation} \label{eq:CpitvsCpi1}
    C_2(\pi_t)^2
    = \!\int_{\R^d} \!\|x-A_t(x)\|_2^2 \d\mu(x) 
    = t^2 \!\!\int_{\R^d} \!\|\nabla \hat p(x)-v(x)\|_2^2\d\mu(x)
    = t^2C_2(\pi_1)^2.
\end{equation}
Hence, $t\searrow 0$ implies $C_2(\pi_t)\to 0$. 
Since $G_{f,\nu}^\lambda$ is Fr\'echet differentiable, $T(\eta,\mu) \coloneqq D_{f,\nu}^{\lambda}(\eta)-D_{f,\nu}^{\lambda}(\mu) - \langle \nabla G_{f,\nu}^{\lambda}(m_{\mu}), m_{\eta}-m_{\mu}\rangle_{\H_K}$ satisfies
\begin{align*}
    \lim_{\|m_{\eta-\mu}\|_{\H_K}\to 0} \frac{T(\eta,\mu)}{\|m_{\eta-\mu}\|_{\H_K}}=0.
\end{align*}
Using a similar reasoning as for \eqref{eq:descentlemma}, we get
\begin{align}
    &\quad D_{f,\nu}^\lambda(\mu_t)-D_{f,\nu}^\lambda (\mu)-\int_{\R^d \times \R^d}\langle v(x_1), x_2-x_1 \rangle \d\pi_t(x_1,x_2)\notag\\
    & = D_{f,\nu}^\lambda(\mu_t)-D_{f,\nu}^\lambda(\mu)-\int_{\R^d \times \R^d}\bigl\langle \nabla \hat p(x_1)-(\nabla \hat p(x_1)-v(x_1)), x_2-x_1 \bigr\rangle \d\pi_t(x_1,x_2)\notag\\
    & \le T(\mu_t,\mu) + \frac{1}{2} LC_2(\pi_t)^2+\int_{\R^d \times \R^d}\langle \nabla \hat p(x_1)-v(x_1), x_2-x_1 \rangle \d\pi_t(x_1,x_2)\notag\\
    &= T(\mu_t,\mu) + \frac{1}{2} LC_2(\pi_t)^2 - t \int_{\R^d}\| \nabla \hat p(x)-v(x)\|_2^2\d\mu(x) \notag\\
    &=T(\mu_t,\mu) + \frac{1}{2} L C_2(\pi_t)^2- t C_2(\pi_1)^2.\label{eq:EstSub}
\end{align}
By \eqref{eq:grib} we have
\begin{equation}\label{eq:UpBoundDK}
    d_K(\mu_t,\mu)
    \le C_{\emb} W_2(\mu_t,\mu)
    \le C_{\emb}C_2(\pi_t).
\end{equation}
Using \eqref{eq:CpitvsCpi1} and \eqref{eq:UpBoundDK}, we further estimate \eqref{eq:EstSub} as
\begin{align*}
    &\quad \frac{1}{C_2(\pi_t)}\left( D_{f,\nu}^\lambda(\mu_t)-D_{f,\nu}^\lambda (\mu)-\int_{\R^d \times \R^d}\langle v(x_1), x_2-x_1 \rangle \d\pi_t(x_1,x_2)\right)\\&
    \le \frac{T(\mu_t,\mu) + \frac{1}{2} L C_2(\pi_t)^2 - tC_2(\pi_1)^2}{C_2(\pi_t)}\\
    & \le \frac{T(\mu_t,\mu)}{d_K(\mu_t,\mu)}\cdot C_{\emb} + \frac{1}{2} LC_2(\pi_t)-C_2(\pi_1)
    \xrightarrow{t\searrow 0}
    -C_2(\pi_1).
    % -\frac{M}{2}C_2(\pi_t)-\frac{1}{C_2(\pi_t)} t\int_{\R^d \times \R^d}\| \nabla \hat p(x)-v(x)\|_2^2\d\mu(x)
    %= -\frac{M}{2}C_2(\pi_t) - \frac{1}{C_2(\pi_t)} \frac{C_2(\pi_t)}{C_2(\alpha_1)}C_2(\alpha_1)^2\\&
    %=-\frac{M}{2}C_2(\pi_t)-C_2(\alpha_1)
\end{align*}
Since $v\in \partial D_{f,\nu}^\lambda(\mu)$, we must have $C_2(\pi_1)=0$.
Hence, $\pi_1 \in \Gamma(\mu, \mu_1)$ implies that $\mu_1=\mu$.
Due to $\mu_1=(A_1)_\#\mu$, we must have $\nabla \hat p=v$ $\mu$-a.e.
\end{proof}

Based on Theorem~\ref{ambrosio} and Lemma \ref{lemma:RegFDivDerivative}, we have the following result.

\begin{corollary}
    Let $K(x,y) = \phi( \|x-y\|_2^2)$ 
    be a radial kernel
    and 
    $\phi \in \C^2([0, \infty))$.
    For any $\gamma_t \in \P_2(\R^d)$, let $\hat p_{\gamma_t}$ 
    denote the maximizer in the dual formulation \eqref{eq:RegFDivDual} of 
    $D^\lambda_{f,\nu}(\gamma_t)$.
    Then, for any $\mu_0 \in \P_2(\R^d)$, the equation 
    \begin{equation}\label{eq:WassersteinParticle}
        \partial_t \gamma_t - \div(\gamma_t \nabla \hat p_{\gamma_t}) = 0.
    \end{equation}
    has a unique weak solution $\gamma_t$ fulfilling $\gamma(0+) = \mu_0$, where $\gamma(0+) \coloneqq \lim_{t \searrow 0} \gamma_t$.
\end{corollary}

\paragraph{Wasserstein Gradient Flows for Empirical Measures.}
Finally, we investigate the Wasserstein gradient flow \eqref{eq:WassersteinParticle} starting in an empirical measure 
\begin{equation} \label{startemp}
\mu_0 \coloneqq \frac{1}{N}\sum_{i=1}^N  \delta_{x_{i}^{(0)}}.
\end{equation}
To solve \eqref{eq:WassersteinParticle} numerically, we consider a \textit{particle functional} $D_N \colon \R^{d N} \to [0, \infty)$ of $D_{f,\nu}^{\lambda}$ defined for 
$x \coloneqq (x_1, \ldots,x_N)$ as 
\begin{equation}
  D_{N}(x) \coloneqq D_{f, \nu}^{\lambda}\biggl(\frac{1}{N} \sum_{k = 1}^{N} \delta_{x_k}\biggr),
\end{equation}
and consider the particle flow $x\colon [0,\infty) \to \R^{dN}$ with
\begin{equation}\label{eq:ParticleFlow}
    \dot x(t) = - N \nabla D_{N}\big(x(t)\big), \quad x(0) \coloneqq \left(x_1(0),\ldots,x_N(0)\right).
\end{equation}
In general, solutions of \eqref{eq:ParticleFlow} differ from those of \eqref{eq:WassersteinParticle}.
However, for our setting, we can show that the flow \eqref{eq:ParticleFlow} induces a Wasserstein gradient flow.

\begin{proposition} \label{flow}
    Assume that $K(x,y) = \phi( \|x-y\|_2^2)$ and $\phi \in \C^2([0, \infty))$.
    Let $x(t) = (x_{i}(t))_{i = 1}^{N} \subset \R^{d N}$, $t\in [0,\infty)$, be a solution of \eqref{eq:ParticleFlow}.
    Then, the corresponding curve of empirical measures $\gamma_N \colon [0, \infty) \to \P_2(\R^d)$ given by
    \begin{equation}
        \gamma_N(t) = \frac{1}{N} \sum_{i = 1}^{N} \delta_{x_{i}(t)}
    \end{equation}
    is a Wasserstein gradient flow of $D_{f, \nu}^\lambda$ starting in $\mu_0$.
\end{proposition}

\begin{proof}
    For the vector-valued measure $M_N(t) \coloneqq \frac{1}{N} \sum_{i = 1}^{N} \dot x_{i}(t) \delta_{x_{i}(t)}
    = -\sum_{i = 1}^{N} \nabla_{x_i} D_N(x)\delta_{x_{i}(t)}
    $ and $\varphi \in C_c^\infty((0,\infty)\times \R^d)$ it holds that
    \begin{align}
    &\int_0^\infty \biggl(\int_{\R^d} \partial_t \varphi(t,x) \d \gamma_N(t) + \int_{\R^d}  \nabla_x\varphi(t,x) \d M_N(t) \biggr)\d t\notag\\
    = & \frac{1}{N}\sum_{i=1}^N \int_0^\infty \bigl(\partial_t \varphi(t,x_i(t))  + \nabla_x\varphi(t,x_i(t)) \cdot \dot x_{i}(t) \bigr)\d t\notag\\
    =& \frac{1}{N} \sum_{i=1}^N \int_0^\infty \frac{\d}{\d t} \varphi(t,x_i(t)) \d t = 0.
    \end{align}
    Hence, $\partial_t \gamma_N + \div(M_N) = 0$ holds in weak sense.
    To show $M_N(t) = -\nabla \hat p_{\gamma_N(t)}\gamma_N(t)$, we first note that 
    $D_N(x) = \smash{G_{f,\nu}^{\lambda}}(m_{\gamma_N})$ and $m_{\gamma_N} = \frac{1}{N}\sum_{i=1}^N K(\cdot,x_i)$.
    For the second term, we obtain by Lemma~\ref{lem:regularity} that $\nabla_{x_i} m_{\gamma_N} = \frac{1}{N} \nabla_{x_i} K(\cdot,x_i)$.
    Next, we apply the chain rule, Corollary~\ref{lem:Grad} and the derivative reproducing property of $\partial_j K(\cdot,x_i)$ to get
    \begin{equation}
        N\nabla_{x_i} D_{N}(x) = \Bigl(\bigl \langle \hat p_{\gamma_N}, \partial_j K(\cdot,x_{i}) \bigr \rangle_{\H_K}\Bigr)_{j=1}^d
        = 
        \nabla \hat p_{\gamma_N}(x_i).
    \end{equation}
Consequently, we obtain $\partial_t \gamma_N - \div(\gamma_N \nabla \hat p_{\gamma_N}) = 0$ as required.
\end{proof}

\begin{remark}[Consistent discretization]
    An advantage of the regularized $f$-divergences \smash{$D_{f, \nu}^{\lambda}$}, $\nu\in \P_2(\R^d)$, is that $D_{f, \nu}^{\lambda}(\mu) < \infty$ for any $\mu \in \P_2(\R^d)$ (even if $D_{f, \nu}(\mu) = \infty$).
    This is important when approximating the Wasserstein gradient flow of the functional $D_{f, \nu}$ starting in  $\mu_0 \in \dom(D_{f, \nu})$ numerically with gradient flows starting in empirical measures $\mu_{0,N} = \frac{1}{N}\sum_{i=1}^N \delta_{x_{i}(0)}$.
    Since $\mu_{0,N}$ might not be in $\dom(D_{f, \nu})$, we choose $(\lambda_N)_{N \in \N}$ with $\lim_{N \to \infty} \lambda_N=0$ such that $\sup_{N \in \N} \smash{D_{f, \nu}^{\lambda_N}}(\mu_{0,N}) < \infty$.
    Then, as shown in Corollary~\ref{lemma:lambdaToZero}(\ref{item:lambdatoZero1}), the functionals \smash{$D_{f, \nu}^{\lambda_N}$} Mosco converge to $D_{f, \nu}$.
    If we have a uniform lower bound on their convexity-modulus along generalized geodesics, then \cite[Thm.~11.2.1]{AGS2008} implies that the Wasserstein gradient flows of $D_{f, \nu}^{\lambda_N}$ starting in $\mu_{0,N}$ converge locally uniformly in $[0,\infty)$ to the flow of $D_{f, \nu}$ starting in $\mu_{0}$.
    Whether such a lower bound exists is so far an open question.
    In Theorem~\ref{thm:M}, we only established a lower bound on the weak convexity modulus of \smash{$D_{f, \nu}^{\lambda_N}$} which scales as $1/\lambda_N$.
    A similar convergence result was recently established in \cite{LMSS2020} for regularization with the Wasserstein distance $W_2$ instead of the MMD $d_K$.
    However,  their proof cannot be directly extended to our setting.
\end{remark}

%--------------------------------------------------------------------------
\section{Computation of Flows for Empirical Measures} \label{sec:discr}
%--------------------------------------------------------------------------
In this section, we are interested in the Wasserstein gradient flows $\gamma$ of \smash{$D_{f,\nu}^\lambda$} with initial measure \eqref{startemp} and target measure $\nu = \frac{1}{M} \sum_{j = 1}^{M} \delta_{y_j}$.
For the KL divergence, such flows were considered in 
\cite{GAG2021} under the name KALE flows.
The Euler forward discretization of the Wasserstein gradient flow \eqref{eq:ParticleFlow}
 with step size $\tau > 0$ 
 is given by the sequence $(\gamma_n)_{n \in \N} \subset \P_2(\R^d)$ defined by 
\begin{equation}
    \gamma_{n + 1}
    \coloneqq (\id - \tau \nabla \hat p_n)_{\#} \gamma_n, 
\end{equation}
where $\hat p_n$ maximizes the dual formulation \eqref{eq:RegFDivDual} of $D_{f, \nu}^{\lambda}(\gamma_n)$.
Since the push-forward of an empirical measure by a measurable map is again an empirical measure with the same number of particles,  we have that
\begin{equation}
    \gamma_{n}
    = \frac{1}{N} \sum_{i = 1}^{N} \delta_{x_i^{(n)}}, 
\end{equation}
where
\begin{equation} \label{compute}
    x_i^{(n + 1)}
    = x_i^{(n)} - \tau \nabla \hat p_n(x_i^{(n)}), \qquad  i = 1, \ldots, N.
\end{equation}
Since both $\gamma_n$ and $\nu$ are empirical measures, the dual problem \eqref{eq:RegFDivDual} for $\hat p_n$ becomes
\begin{equation}\label{dual_d}
   \hat p_n
    = \argmax_{\substack{p \in \H_K \\ p(\R^d) \subseteq \overline{\dom(f^*)}}} 
    \biggl\{
    \frac1N \sum_{i=1}^N p( x_i^{(n)} )
    - 
    \frac{1}{M} \sum_{j=1}^{M} 
    f^*\left( p(y_j) \right)
    - 
    \frac{\lambda}{2} \| p \|_{\H_K}^2
    \biggr \}.
\end{equation}

\subsection{Discretization Based on the Representer Theorem}
If $f_{\infty}' = \infty$ or $\text{supp}(\nu) = \R^d$, then the constraint $p(\R^d) \subset \dom(f^*)$ in \eqref{dual_d} becomes superfluous. By the representer theorem \cite{SHS2001}, the solution of \eqref{dual_d} is of the form
\begin{equation}\label{eq:RepSol}
    \hat p_n = \sum_{k = 1}^{M+N} b_k^{(n)} K(\cdot, z_k^{(n)}),
    \qquad b^{(n)} \coloneqq (b_k^{(n)})_{k = 1}^{M + N} \in \R^{M + N},
\end{equation} 
where
\begin{equation}
    (z_1^{(n)}, \ldots, z_{N+M}^{(n)})
    = \big(y_1, \ldots, y_{M}, x_1^{(n)}, \ldots, x_{N}^{(n)}\big) \in \R^{d \times (M+N)}.  
\end{equation}
We denote the associated kernel matrix by $K^{(n)} \coloneqq ( K(z_i^{(n)}, z_j^{(n)}))_{i, j = 1}^{M + N}$.

If $f'_{\infty} < \infty$, then we cannot drop the constraint $p(\R^d) \subset \overline{\dom(f^*)}$ in \eqref{eq:RegFDivDual} to apply the representer theorem.
The situation is similar for the primal problem \eqref{eq:RegFDiv}, where the condition $f'_{\infty} < \infty$ entails that we have to consider measures $\mu$ that are not absolutely continuous with respect to the (empirical) target $\nu$. 
In the following, we provide a condition under which computing $\hat{p}_n$ for \eqref{dual_d} reduces to a tractable finite-dimensional problem.
To simplify the notation, we omit the index $n$ and use the shorthand $L(p) = \E_{\mu}[p] - \E_{\nu}[f^* \circ p] - \frac{\lambda}{2} \| p \|_{\H_K}^2$ for the objective in \eqref{eq:RegFDivDual}.
\begin{lemma} \label{lemma:reprThmForFiniteRecConst}
    Let $f$ be an entropy function with $f'_\infty \in (0, \infty)$.
    If $\mu,\nu\in \M_+(\R^d)$ and $\lambda > 2d_K(\mu,\nu)\sqrt{\phi(0)}/{f'_\infty}$, then 
    \begin{equation}
        \argmax_{\substack{p\in \H_k \\ p(\R^d)\subseteq \overline{\supp(f^*)}}} L(p) = \argmax_{ \|p\|_{\H_K}<\frac{f'_\infty}{\sqrt{\phi(0)}} } L(p). 
    \end{equation}
    If $\mu,\nu$ are empirical measures with $Z=\supp(\mu)\cup \supp(\nu)$, then the solution $\hat p$ of the discrete problem \eqref{dual_d} has a finite representation of the form \eqref{eq:RepSol}.
\end{lemma}

\begin{proof}
    First, we introduce shorthand notations for the feasible sets 
    \begin{equation}
        C \coloneqq \bigl\{ p \in \H_K: p(\R^d) \subset \overline{\dom(f^*)} \bigr\}
        \quad \text{and} \quad 
        X \coloneqq \biggl\{ p \in \H_K: \|p\|_{\H_K}<\frac{f'_\infty}{\sqrt{\phi(0)}}\biggr\}.
    \end{equation}
    It holds for any $p\in X$ that $\|p\|_\infty \le \sqrt{\phi(0)}\|p\|_{\H_K}< f'_\infty$, which implies $X\subseteq C$.
    By \eqref{some_est}, the maximizer $\hat p$ of \eqref{eq:RegFDivDual} satisfies $\|\hat p\|_{\H_K}\le\frac{2}{\lambda}d_K(\mu,\nu)$.
   For $\lambda > 2d_K(\mu,\nu)\sqrt{\phi(0)}/{f'_\infty}$, we get $\hat p \in X$ due to $\|\hat p\|_{\H_K} \le \frac{2}{\lambda } d_K(\mu,\nu)< f'_\infty/\sqrt{\phi(0)}$.
   Since $X\subseteq C$, we get that $\hat p$ maximizes $L$ over $X$.
   This proves the first claim.

    For the second part, consider the parameterization $\Psi\colon \H_K\to X$ with $\Psi(0)\coloneqq 0$ and $\Psi(p)\coloneqq \frac{p}{\|p\|}\psi(\|p\|)$ otherwise, where $\psi\colon [0,\infty)\to [0,\frac{f'_\infty}{\phi(0)}]$ is given by
    \begin{equation}
        \psi(t)=\frac{2f'_\infty}{\pi \sqrt{\phi(0)}}\arctan(t).
    \end{equation}
    The bijectivity of $\arctan$ directly implies the bijectivity of $\Psi$. 
    Hence, the reparameterized problem $\argmax_{q\in \H_K} L(\Psi(q))$ has the unique maximizer $\hat q = \Psi^{-1}(\hat p)$.
    Since $\hat q\in \H_k$, we can decompose $\hat q$ as $\hat q = \hat s+ \hat r$, where $\hat s\in \text{span}\{ K(\cdot, z) : z\in Z\}$ and $\hat r \in \text{span}\{ K(\cdot, z) : z\in Z\}^\perp$.
    Then, for all $z\in Z$ we get that
    \begin{equation}
        \hat q(z)=\langle \hat q,K(\cdot, z)\rangle =\langle \hat s+\hat r,K(\cdot, z)\rangle =\langle \hat s,K(\cdot, z)\rangle =\hat s(z).
    \end{equation}
    Hence, it holds that
\begin{align}
    &L(\Psi(\hat q))
    = \E _\mu[\Psi(\hat q)]-\E_\nu [ f^*\circ \Psi(\hat q)]-\frac{\lambda}{2}\|\Psi(\hat q)\|_{\H_K}^2\notag\\
    =&\frac{1}{|\supp (\mu)|}\sum_{z\in \supp(\mu)} \Psi(\hat s(z))
    - \frac{1}{|\supp(\nu)|}\sum_{z\in \supp(\nu)} f^*\circ \Psi(\hat s(z)) 
    - \frac{\lambda}{2}\|\Psi(\hat q)\|_{\H_K}^2\notag\\
    =&\E_\mu[\Psi(\hat s)]-\E_\nu[f^*\circ \Psi(\hat s)]
    -\frac{\lambda}{2}\|\Psi(\hat q)\|_{\H_K}^2
\end{align}
Now, $\|\Psi(\hat q)\|_{\H_K}=\psi(\|\hat q\|_{\H_K})$, $\|\hat q \|_{\H_K}^2=\|\hat s\|_{\H_K}^2+\|\hat r\|_{\H_K}^2$ and the strict monotonicity of $\Psi^2$ imply
\begin{align}
    &L(\Psi(\hat q))
    =\E_\mu[\Psi(s)]-\E_\nu[f^*\circ \Psi(\hat s)] -\frac{2 \lambda (f'_\infty)^2}{(\pi \sqrt{\phi(0)})^2}\arctan\Bigl(\sqrt{\|\hat s\|_{\H_K}^2+\|\hat r\|_{\H_K}^2}\Bigr)^2\notag\\
    \le& \E_\mu[\Psi(\hat s)]-\E_\nu[f^*\circ \Psi(\hat s)] -\frac{2 \lambda (f'_\infty)^2}{(\pi \sqrt{\phi(0)})^2}\arctan(\|\hat s\|_{\H_K})^2
    =L(\Psi(\hat s)).
\end{align}
Hence, we must have $\hat q = \hat s$.
This directly implies $\hat p =\Psi(\hat q) \in \text{span}\{ K(\cdot, z) : z\in Z\}$, which means that we can write $\hat p = \sum_{z\in Z} b_zK(\cdot, z)$ and $b\in \R^Z$.
\end{proof}

\subsection{Numerical Implementation Details}
\label{Numerical Implementation Details}
Recall that both for infinite and finite recession constant (see Lemma \ref{lemma:reprThmForFiniteRecConst}), the solution $\hat p_n$ to the dual problem has the representation
\begin{equation*}
    \hat p_n
    = \sum_{k=1}^{N+M} b_k^{(n)} K\bigl(\cdot, z_k^{(n)}\bigr).
\end{equation*}
To determine the coefficients $b^{(n)} = (b_k^{(n)})_{k=1}^{M+N}$, we look instead at the primal problem \eqref{eq:MoreauPrimalgnuf2}, which by Corollary \ref{lemma:RegFDivDual} has a solution of the form
\begin{equation} \label{eq:gn}
    \hat g_n = m_{\gamma_n} - \lambda \hat p_n
    = \sum_{k = 1}^{M+N} \beta_k^{(n)} K\bigl(\cdot,z_k^{(n)}\bigr) = m_{\hat \sigma_n}, \qquad \hat \sigma_n \coloneqq \sum_{k = 1}^{M+N } \beta_k^{(n)}\delta_{z_k^{(n)}},
\end{equation}
with the non-negative coefficients
\begin{equation} \label{eq:BetavsB}
    \beta_k^{(n)} 
    \coloneqq 
    \begin{cases}
        - \lambda b_k^{(n)}, & \text{if } k \in \{1, \ldots,M\}
        \\
         \frac{1}{N} - \lambda b_k^{(n)}, & \text{if } k 
         \in \{ M+1, \ldots, M+N \}.
    \end{cases}
\end{equation}
Hence, to compute $b^{(n)}$, we can minimize the primal objective only with respect to the coefficients $\beta^{(n)}$, which results in the objective
\begin{equation} \label{eq:primalObjective}
    J\bigl(\beta^{(n)}\bigr)= {D_{f,\nu}}(\hat \sigma_n)+\frac{1}{2\lambda} \| m_{\hat \sigma_n} -m_{\gamma_n} \|_{\H_K}^2.
\end{equation}
With the change of variables $q^{(n)} \coloneqq -\lambda M b^{(n)} \in \R^{N + M}$, we get $\smash{\tfrac{\d\hat \sigma_n}{\d\nu}}(z_k^{(n)}) =  q_k^{(n)}$ for $k \in \{1,\ldots, M \}$ and $\hat \sigma_n^s (\R^d) = 1 + \smash{\frac{1}{M}\sum_{k=M+1}^{M+N} q_k^{(n)}}$. 
Plugging this into \eqref{eq:primalObjective} yields
\begin{align}
    J\bigl(q^{(n)}\bigr) & =\int_{\R^d} f\circ \frac{\d\hat \sigma_n}{\d \nu} d\nu +f'_\infty \hat \sigma_n^s(\R^d) +\frac{1}{2\lambda M^2}(q^{(n)})^{\tT} K^{(n)} q^{(n)} \label{eq:primalObjective2}\\
    & = \frac{1}{M} \sum_{k=1}^M f\bigl(q_k^{(n)}\bigr) + \frac{f'_\infty}{M}\left(M + \sum_{k=M+1}^{M+N} q_k^{(n)}\right) + \frac{1}{2\lambda M^2} (q^{(n)})^{\tT} K^{(n)} q^{(n)}. \notag
\end{align}
For many kernels like the inverse multiquadric or the Gaussian, $K^{(n)}$ is positive definite, so that $J$ is strongly convex.

The constraint $\beta^{(n)}\ge 0$ translates into the constraints
\begin{equation}\label{eq:boxconst}
    q_k^{(n)}
    = -\lambda b_k^{(n)} M
    \begin{cases}
        \ge 0, & k \in \{ 1,\ldots, M \}, \\
        = \beta_k^{(n)} - \frac{M}{N}\ge -\frac{M}{N}, & k \in \{ M+1,\ldots, M+N \}.
    \end{cases}
\end{equation}
If $f_{\infty}' = \infty$, then $q_{k + M}^{(n)} = - \frac{M}{N}$ for $k \in \{ 1, \ldots, N \}$.
Hence, we only need to optimize over \smash{$(q_k^{(n)})_{k = 1}^{M}$} and consider a reduced objective $J$ instead.
Given the optimal coefficients $\hat q^{(n)}$, the update step \eqref{compute} becomes 
\begin{equation} \label{eq:compute2}
    x_j^{(n + 1)}
    = x_j^{(n)} + \frac{\tau}{\lambda M} \sum_{k=1}^{M+N} \hat q_k^{(n)} \nabla K\bigl(\cdot, z_k^{(n)}\bigr)(x_j^{(n)}),
    \qquad j \in \{ 1,\ldots, N \}.
\end{equation}

Now, we detail the computation of $\hat q^{(n)}$ based on a forward-backward splitting algorithm.
%Since $J$ might not be differentiable, we can not use an L-BFGS-B algorithm like in \cite{GAG2021}.
This requires the proximal mapping $\prox_f$ of $f$.\footnote{We provide $\prox_f$ in closed form for several entropies $f$ in Table~\ref{table:proxies}.
For the others, we employ a naive derivative-free method to solve the one-dimensional strongly convex problem defining $\prox_f$.}
For $f_{\infty}' < \infty$, we decompose $J$ from \eqref{eq:primalObjective2} into the proper, convex and lower semicontinuous function
\begin{equation*}
    h_1 \colon \R^{M + N} \to [0, \infty], \qquad
    q \mapsto \frac{1}{M} \sum_{k = 1}^{M} f(q_k) 
    + \sum_{k = M + 1}^{M + N} \iota_{[-\frac{M}{N}, \infty)}(q_k)
\end{equation*}
and the smooth function
\begin{equation*}
    h_2 \colon \R^{M + N} \to \R, \qquad
    q \mapsto f_{\infty}'\left( 1 + \frac{1}{M} \sum_{k = 1}^{N} q_{M + k}\right) 
    + \frac{1}{2 \lambda M^2} q^{\tT} K q
\end{equation*}
with Lipschitz continuous gradient
\begin{equation*}
    \nabla h_2(q)
    = \frac{f_{\infty}'}{M} \begin{bmatrix} 0_M \\ \1_N \end{bmatrix} + \frac{1}{\lambda M^2} K q.
\end{equation*}
If $f_{\infty}' = \infty$, we decompose $J$ into
\begin{gather*}
    \tilde{h}_1 \colon \R^M \to \R, \qquad q \mapsto \frac{1}{M} \sum_{k = 1}^{M} f(q_k), \\
    \tilde{h}_2 \colon \R^{M} \to \R, \qquad q \mapsto \frac{1}{2 \lambda M^2} q^{\tT} K_{yy} q - \frac{1}{\lambda M N} \1_N^{\tT} K_{x y}^{(n)} q + \frac{1}{2 \lambda N^2} \1_{N}^{\tT} K_{x x}^{(n)} \1_N.
\end{gather*}
The gradient of $\tilde{h}_2$ reads
\begin{equation*}
    \nabla \tilde{h}_2(q) = \frac{1}{\lambda M^2} K_{y y} q - \frac{1}{\lambda M N} (K_{x y}^{(n)})^{\tT} \1_N.
\end{equation*}
Since $h_1$ and $\tilde{h}_1$ are separable, the same holds for their proximal operators.
The FISTA \cite{BT2009} iterations in Algorithm \ref{alg:Fista1} and \ref{alg:Fista2}, respectively, converge to the unique minimizer $q^*$ of $J$.
FISTA has the optimal $O(k^{-2})$ convergence rate, meaning that $J(q^{(k)}) - J(q^*) \le \frac{2 L}{(k + 1)^2} \| x^{(0)} - x^* \|_2^2$.
Since the proximal operator in Algorithm \ref{alg:Fista1} is independent of the particles $x_j$, we can use the same approximation throughout.

\begin{algorithm}[t]
\caption{FISTA for solving \eqref{eq:RegFDivDual} related to $D_{f, \nu}^{\lambda}(\mu)$ if $f_{\infty}' = \infty$}\label{alg:Fista1}
\KwData{$q_0 \in \R^M$, $t_0 = 1$}
\KwResult{Approximate minimizer $q^*$ of $J$}
\For{$k=0,\ldots$ until convergence}{
$t_{k + 1} \gets \frac{1}{2}\bigl(1 + \sqrt{1 + 4 t_k^2}\bigr)$\;
$z^{(k)} \gets q^{(k)} + \frac{1}{t_{k + 1}} (t_k - 1) \left(q^{(k)} - q^{(k - 1)}\right)$\;
$q_j^{(k + 1)} \gets \prox_{\frac{\lambda M}{\| K_{yy} \|_2} f}\Bigl( \bigl(\id - \frac{1}{\| K_{yy} \|_2} K_{y y}\bigr)z_j^{(k)} + \frac{1}{\lambda M N} (K_{x y}^{(n)})^{\tT}\1_N \Bigr)$\; % $, \quad j \in \{ 1, \ldots, M \}$
} 
\end{algorithm}

\begin{algorithm}[t]
\caption{FISTA for solving \eqref{eq:RegFDivDual} related to $D_{f, \nu}^{\lambda}(\mu)$ if $f_{\infty}' < \infty$}\label{alg:Fista2}
\KwData{$q_0 \in \R^{M + N}$, $t_0 = 1$}
\KwResult{Approximate minimizer $q^*$ of $J$}
\For{$k=0,\ldots,$ until convergence}{
$t_{k + 1} \gets \frac{1}{2}\bigl(1 + \sqrt{1 + 4 t_k^2}\bigr)$\;
$z^{(k)} \gets q^{(k)} + \frac{1}{t_{k + 1}} (t_k - 1) \bigl(q^{(k)} - q^{(k - 1)}\bigr)$\;
$q_j^{(k + 1)} \gets \prox_{\frac{\lambda M}{\| K \|_2} f}\Bigl(\bigl( \bigl(\id - \frac{1}{\| K \|_2} K\bigr)z^{(k)}\bigr)_j\Bigr), \quad j = 1, \ldots, M$\;
$q_{M + j}^{(k + 1)} \gets \max\left(-\frac{M}{N}, \bigl(\bigl(\id - \frac{1}{\| K \|_2} K\bigr)z^{(k)}\bigr)_{M + j} - f_{\infty}' \frac{\lambda M}{\| K \|_2}\right), \,\,\, j=1,\ldots,N$\;
}
\end{algorithm}

%--------------------------------------------------------------------------
\section{Numerical Results} \label{sec:numerics}
%--------------------------------------------------------------------------
In this section, we use three target measures from the literature to compare how fast the discrete Wasserstein gradient flow (WGF) \eqref{eq:compute2} for different $D_{f, \nu}^{\lambda}$ converge\footnote{The code is available at \url{https://github.com/ViktorAJStein/Regularized_f_Divergence_Particle_Flows}.
We use pytorch \cite{pytorch} to benefit from GPU parallelization.}.
We always choose $N = M = 900$ particles.
Further, we use the inverse multiquadric kernel from Remark~\ref{rem1}.
We focus on the Tsallis-$\alpha$ divergence for $\alpha \ge 1$ because the corresponding entropy function is differentiable on the interior of its domain and $f'_\infty = \infty$.
Other choices are covered in the code repository.
For $\alpha \ge 1$, the Tsallis-$\alpha$ divergence $D_{f_{\alpha}}(\mu \mid \nu)$ between $\mu, \nu \in \mathcal P(\R^d)$ with $\mu \ll \nu$ reads
\begin{equation}
    D_{f_{\alpha}}(\mu \mid \nu)
    = \frac{1}{\alpha - 1} \left(\int_{\R^d} \Bigl(\frac{\d \mu}{\d \nu}(x)\Bigr)^{\alpha} \d{\nu}(x) - 1\right).
\end{equation}
In our experiments, the commonly used KL divergence (corresponding to the limit for $\alpha \searrow 1$) is outperformed in terms of convergence speed if $\alpha$ is moderately larger than one.
We always observed sublinear convergence to the target $\nu$ in terms of functional values. 
This matches the rate proven in \cite[Prop.~3]{GAG2021} (the same proof is possible for any $f$-divergence) under an a priori boundedness assumption, which, however, is never fulfilled if $\mu$ and $\nu$ have disjoint supports.
We also observe fast empirical convergence in terms of the MMD and the Wasserstein distance.
\begin{figure}[tp]
    \centering
    \rotatebox{90}{
        \begin{minipage}{.07\textwidth}
        \centering 
        \small $1$
    \end{minipage}} \hfill
    \begin{subfigure}[t]{.16\textwidth}
        \includegraphics[width=\linewidth]{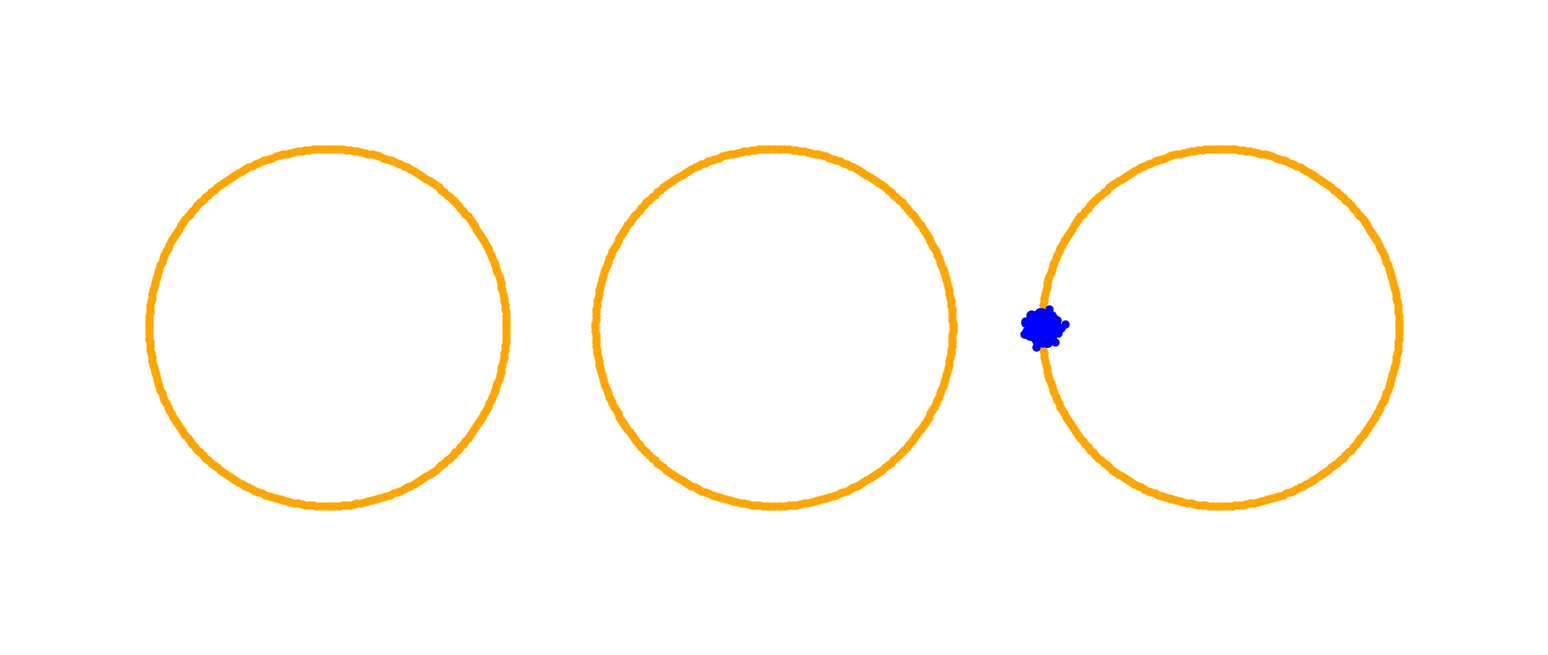}
    \end{subfigure}%
    \begin{subfigure}[t]{.16\textwidth}
        \includegraphics[width=\linewidth]{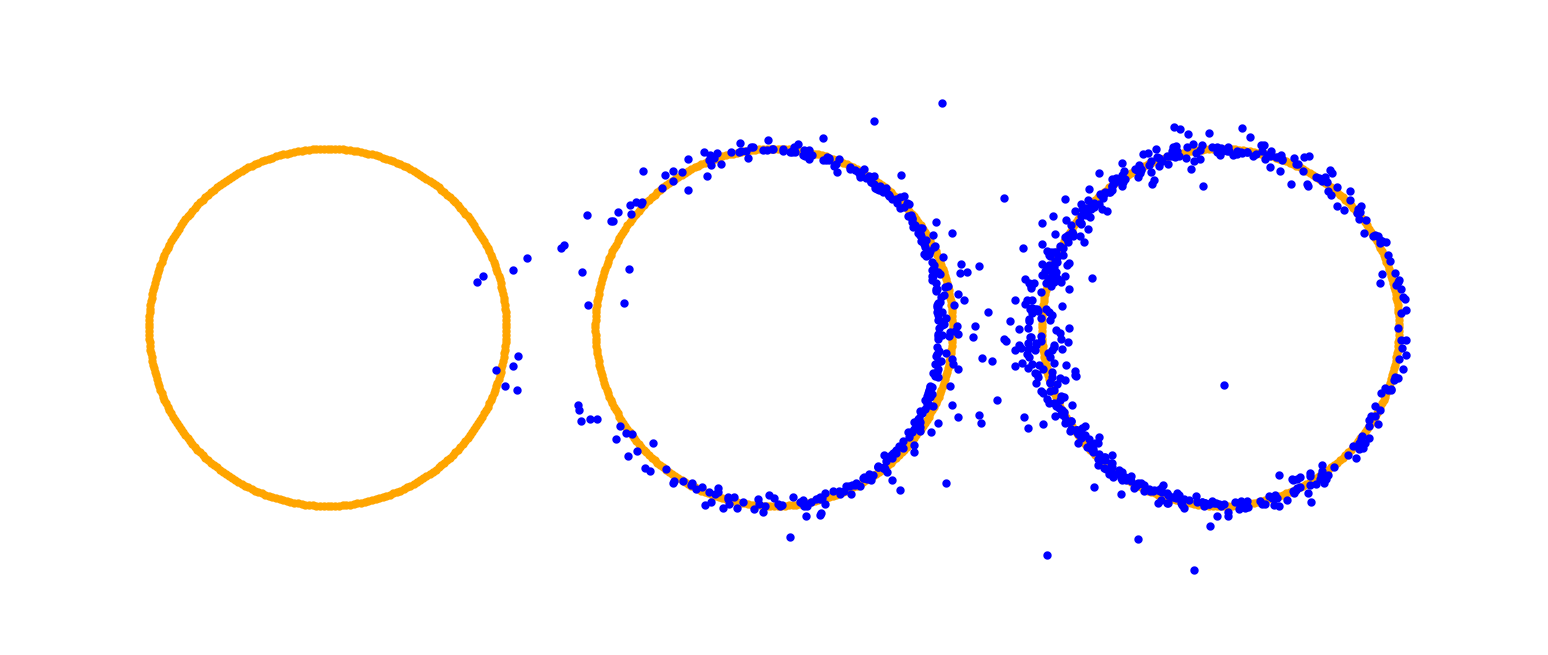}
    \end{subfigure}%
    \begin{subfigure}[t]{.16\textwidth}
        \includegraphics[width=\linewidth]{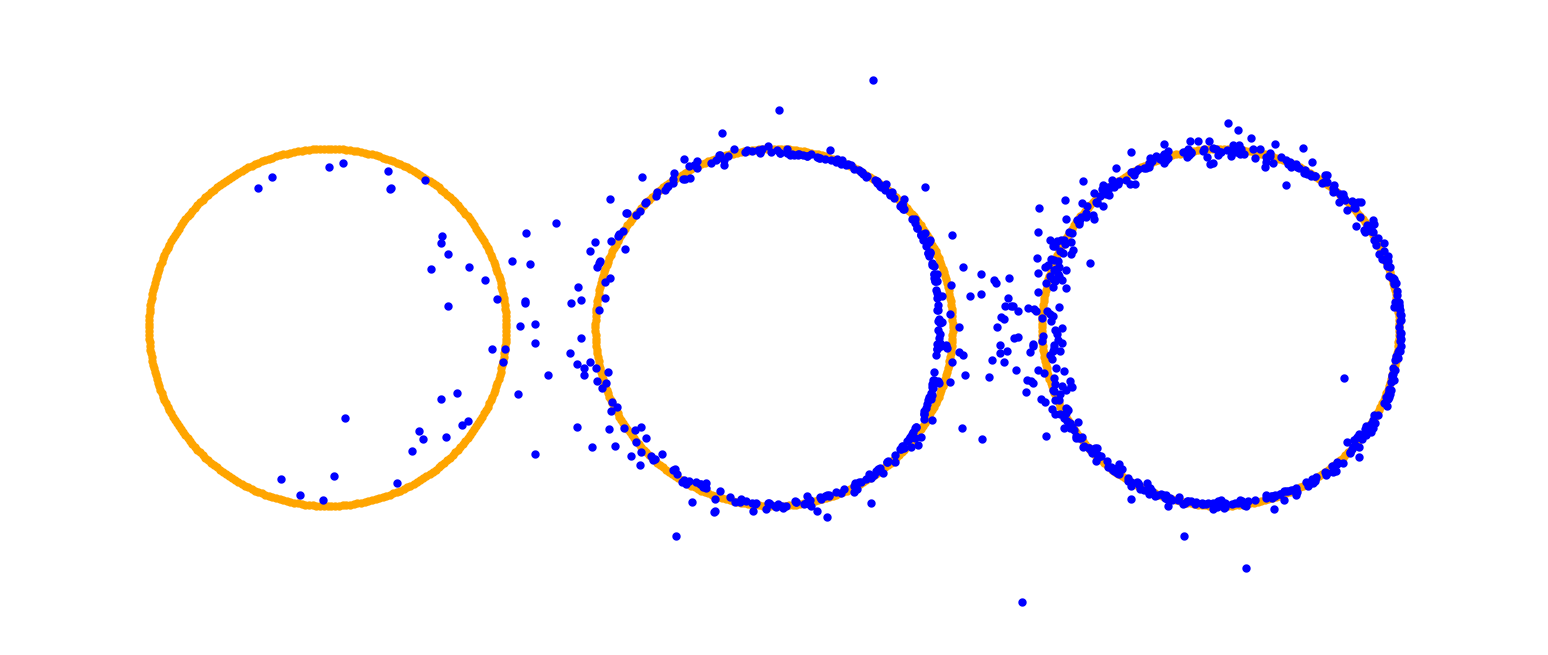}
    \end{subfigure}%
    \begin{subfigure}[t]{.16\textwidth}
        \includegraphics[width=\linewidth]{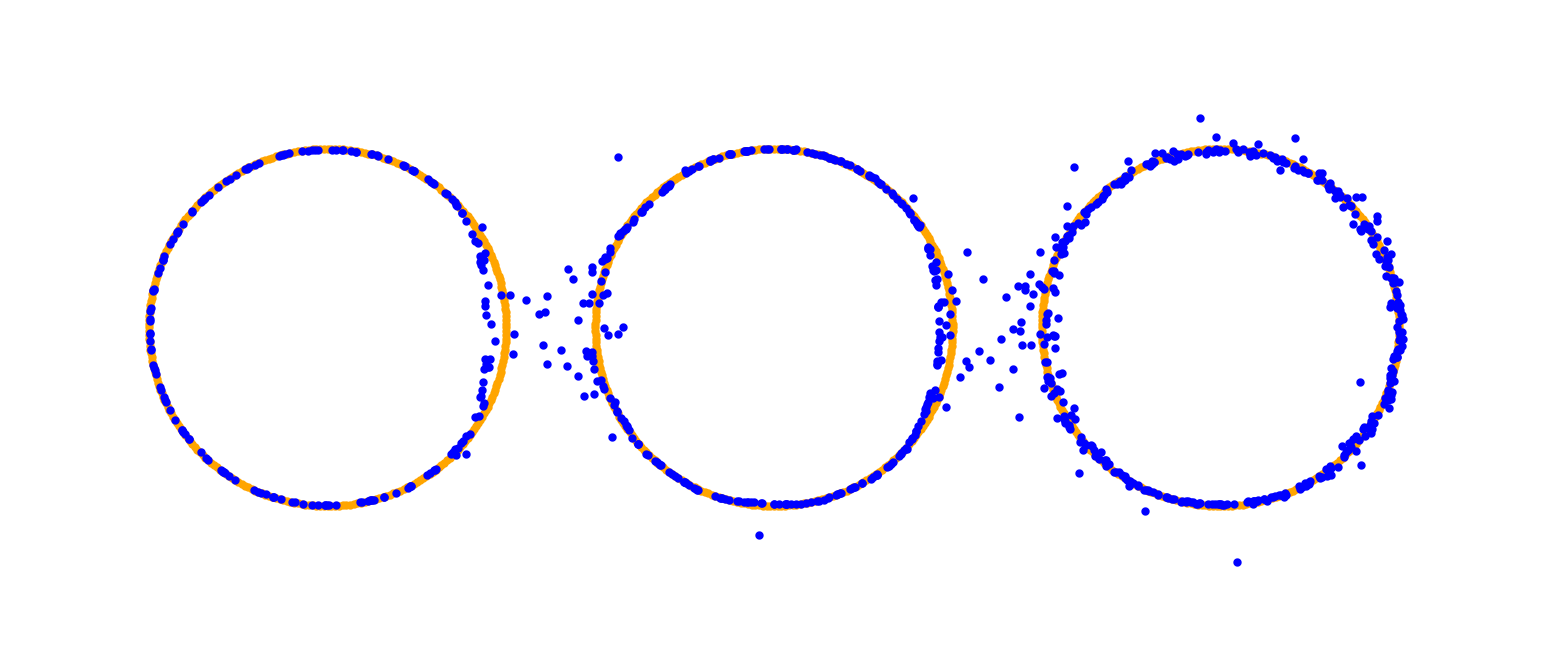}
    \end{subfigure}%
    \begin{subfigure}[t]{.16\textwidth}
        \includegraphics[width=\linewidth]{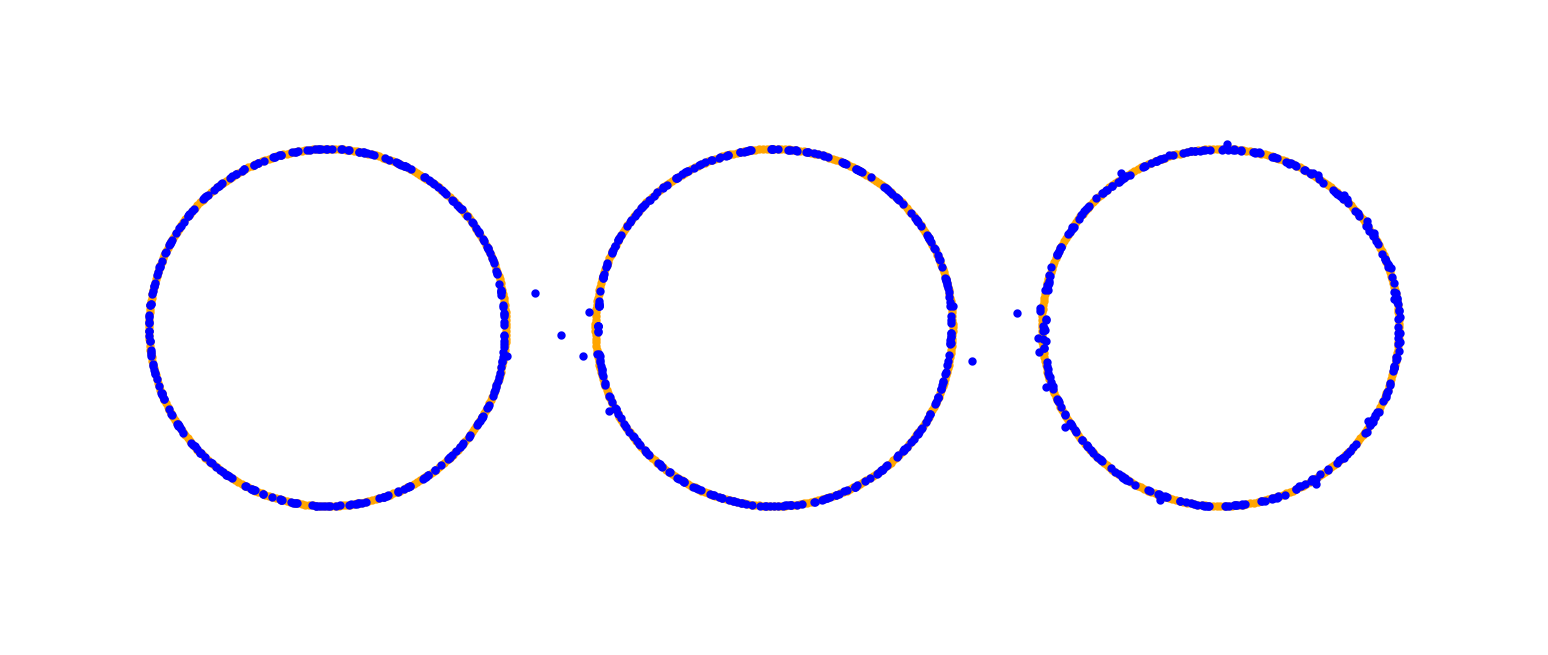}
    \end{subfigure}%
    \begin{subfigure}[t]{.16\textwidth}
        \includegraphics[width=\linewidth]{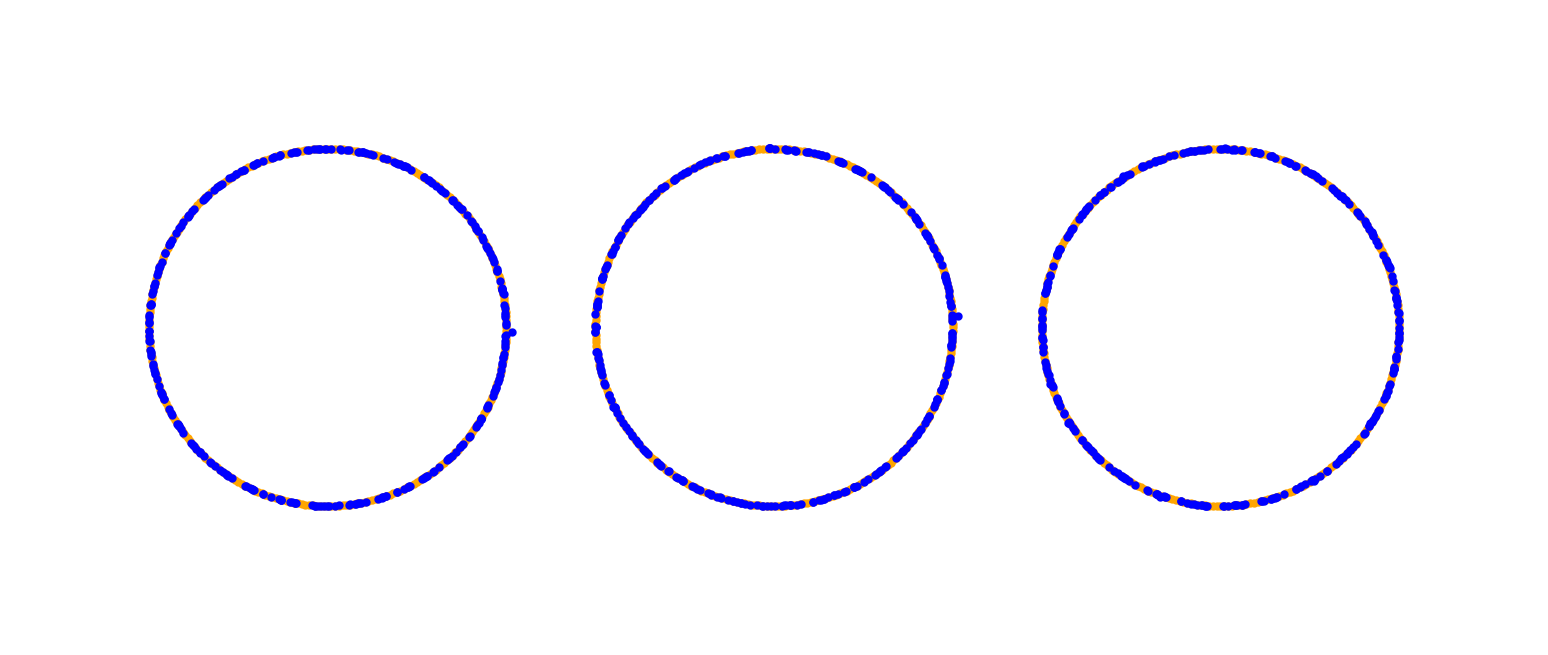}
    \end{subfigure}\\ %
    \rotatebox{90}{
        \begin{minipage}{.07\textwidth}
        \centering 
        \small $3$
    \end{minipage}} \hfill
    \begin{subfigure}[t]{.16\textwidth}
        \includegraphics[width=\linewidth]{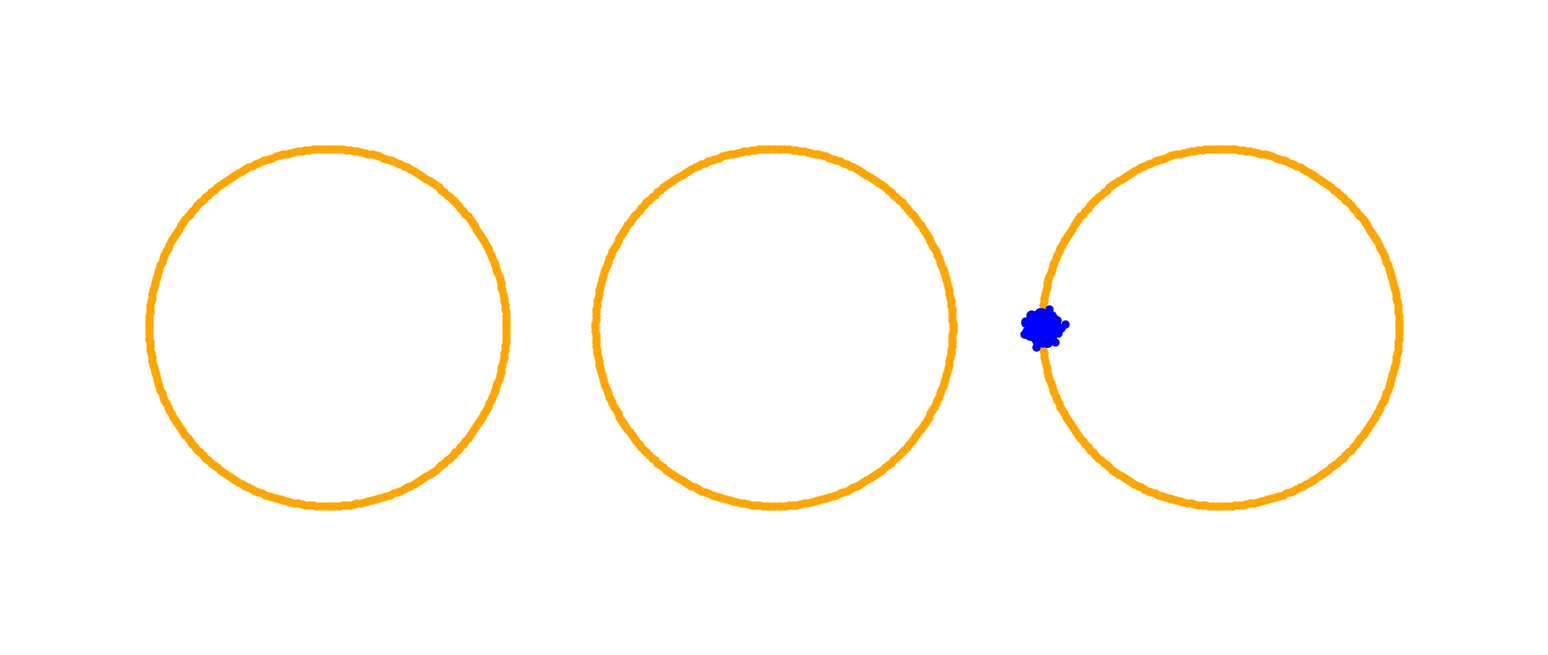}
    \end{subfigure}%
    \begin{subfigure}[t]{.16\textwidth}
        \includegraphics[width=\linewidth]{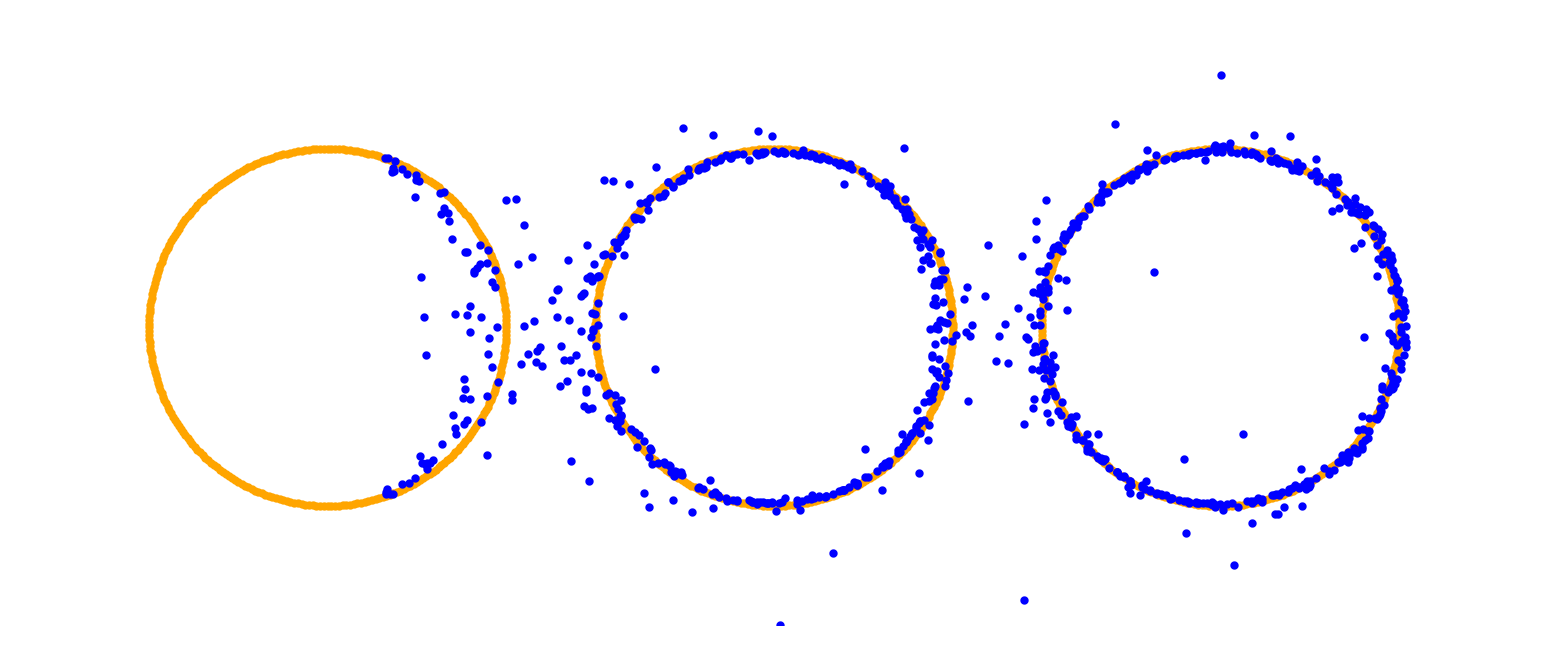}
    \end{subfigure}%
    \begin{subfigure}[t]{.16\textwidth}
        \includegraphics[width=\linewidth]{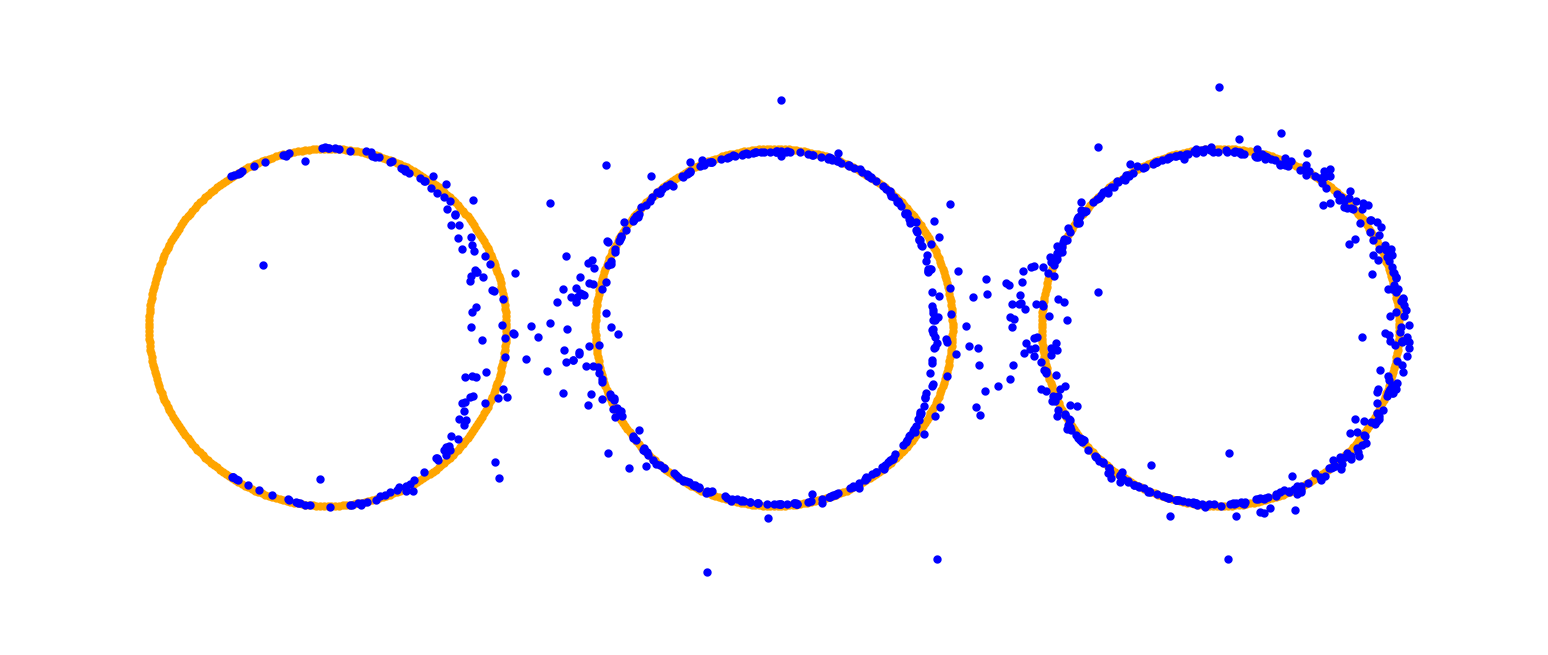}
    \end{subfigure}%
    \begin{subfigure}[t]{.16\textwidth}
        \includegraphics[width=\linewidth]{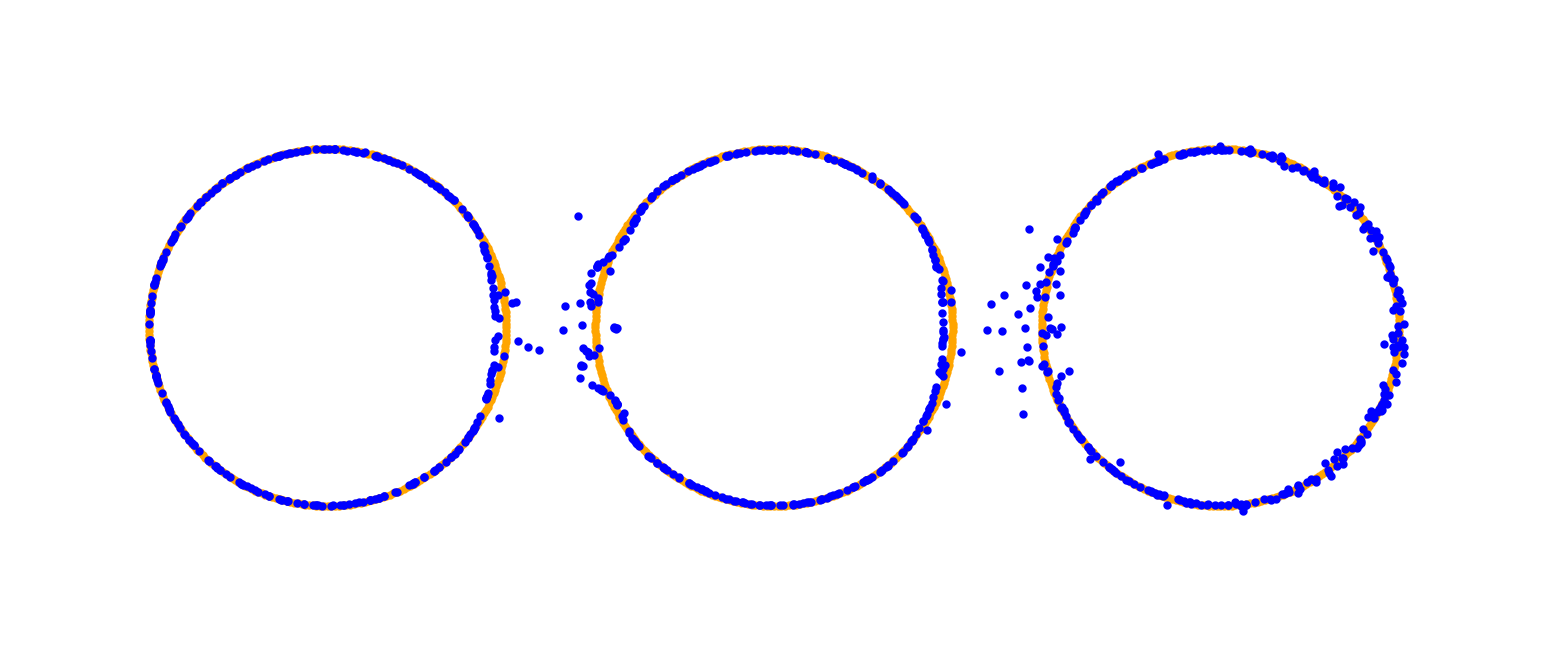}
    \end{subfigure}%
    \begin{subfigure}[t]{.16\textwidth}
        \includegraphics[width=\linewidth]{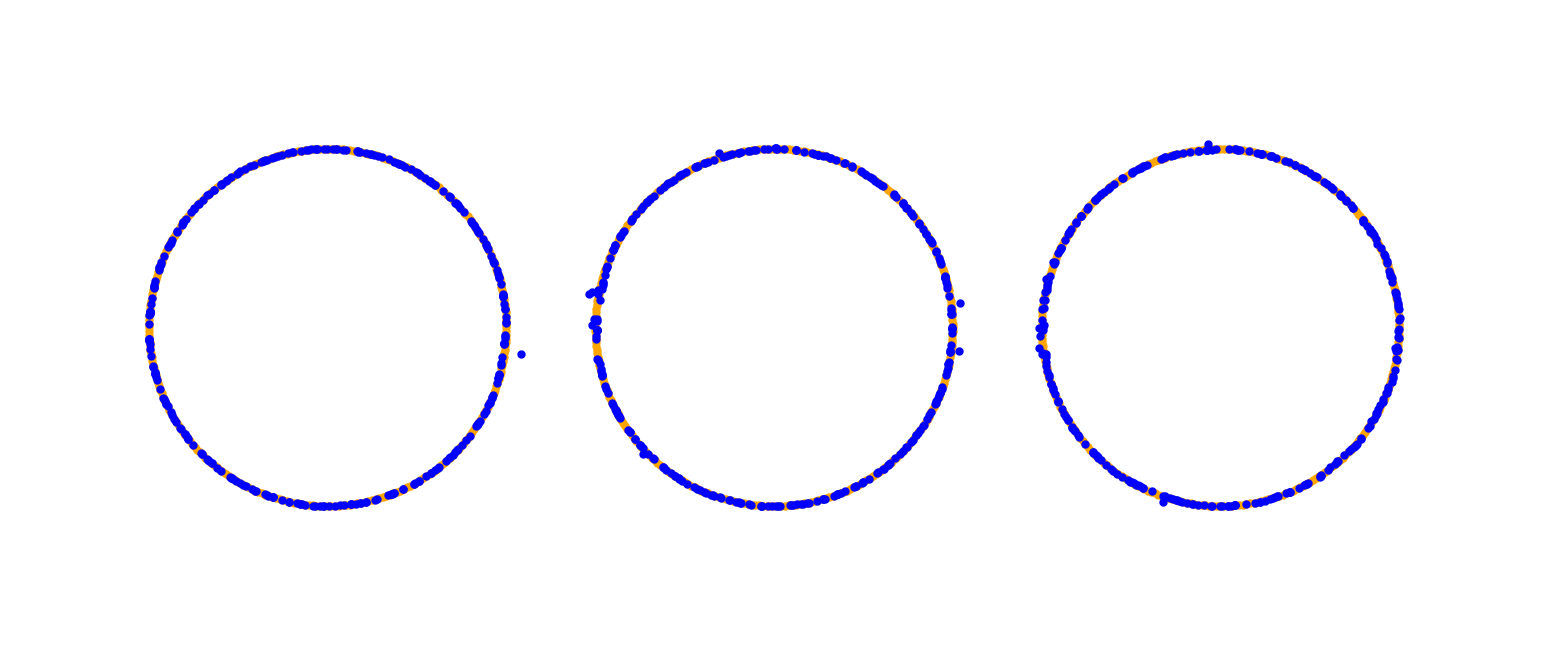}
    \end{subfigure}%
    \begin{subfigure}[t]{.16\textwidth}
        \includegraphics[width=\linewidth]{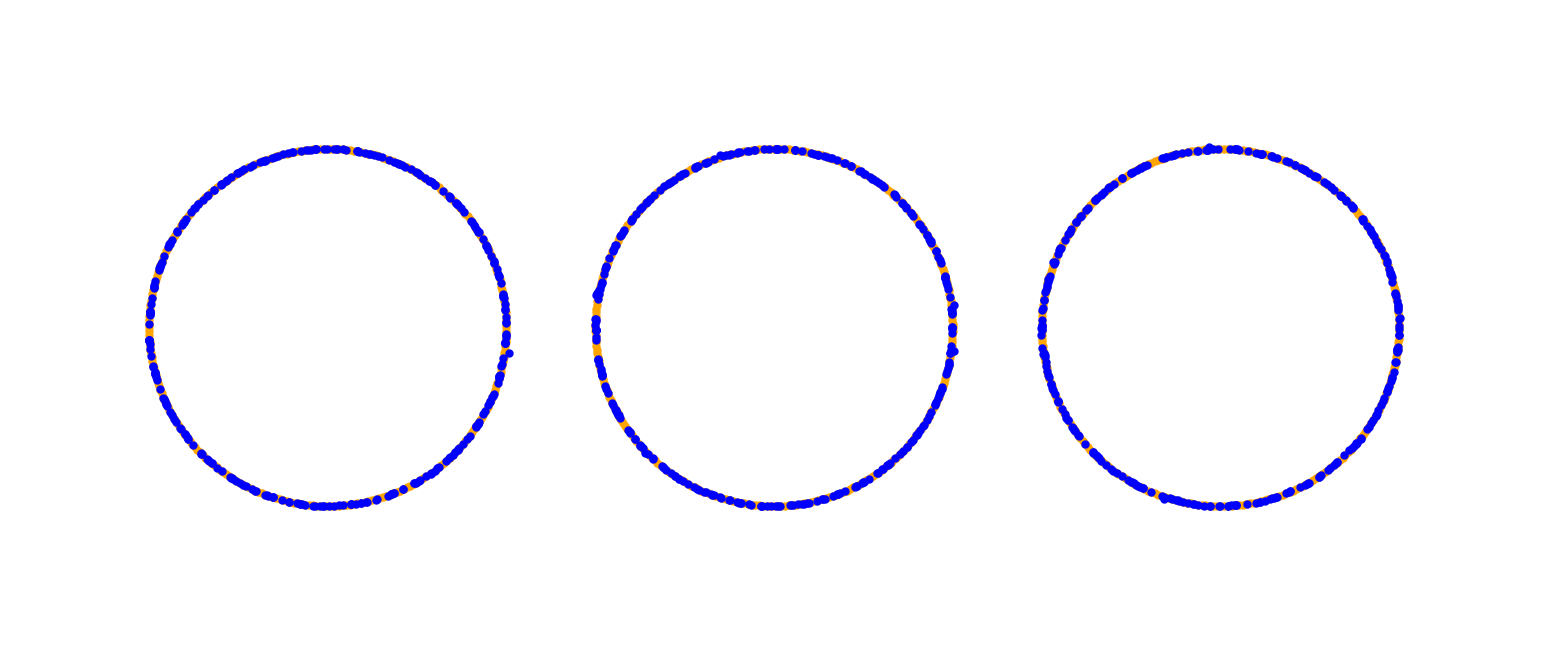}
    \end{subfigure}\\ %
    \rotatebox{90}{
        \begin{minipage}{.08\textwidth}
        \centering 
        \small $7.5$
    \end{minipage}} \hfill
    \begin{subfigure}[t]{.16\textwidth}
        \includegraphics[width=\linewidth]{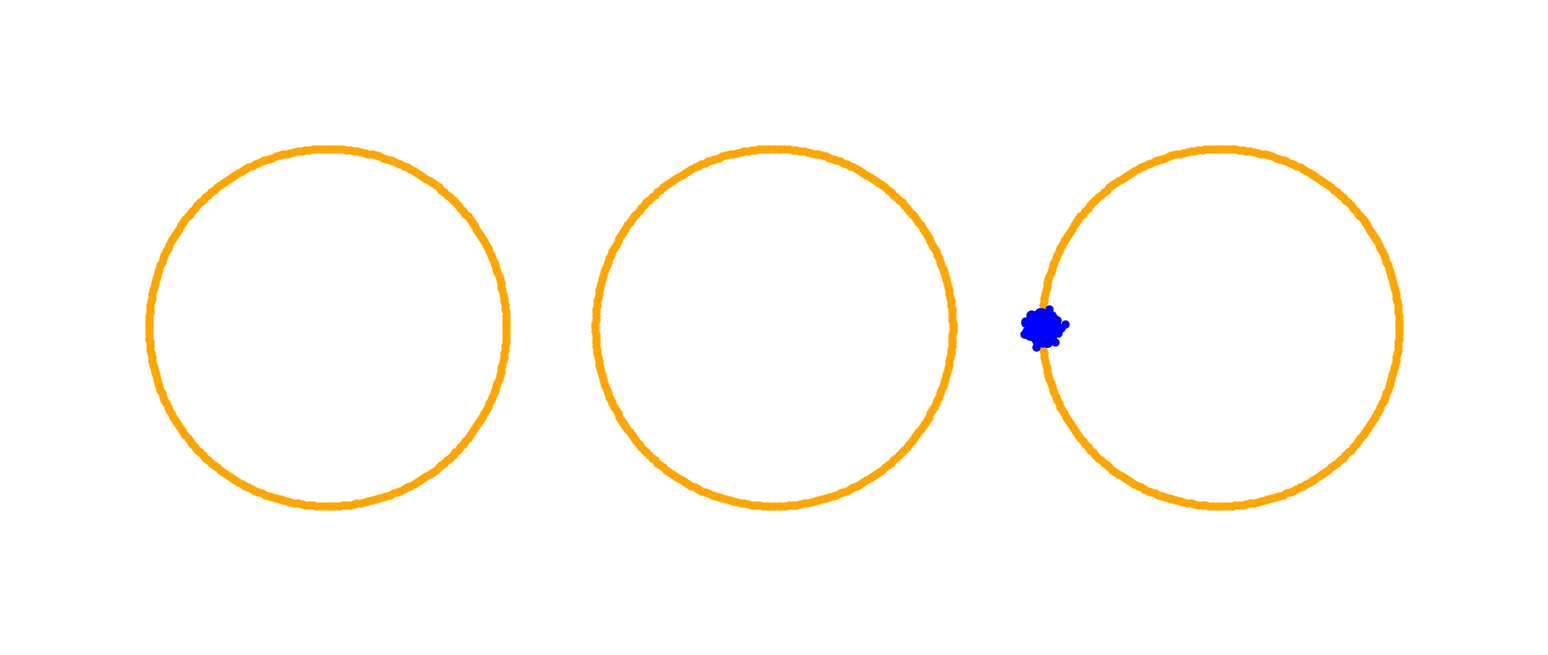}
    \caption*{t=0}
    \end{subfigure}%
    \begin{subfigure}[t]{.16\textwidth}
        \includegraphics[width=\linewidth]{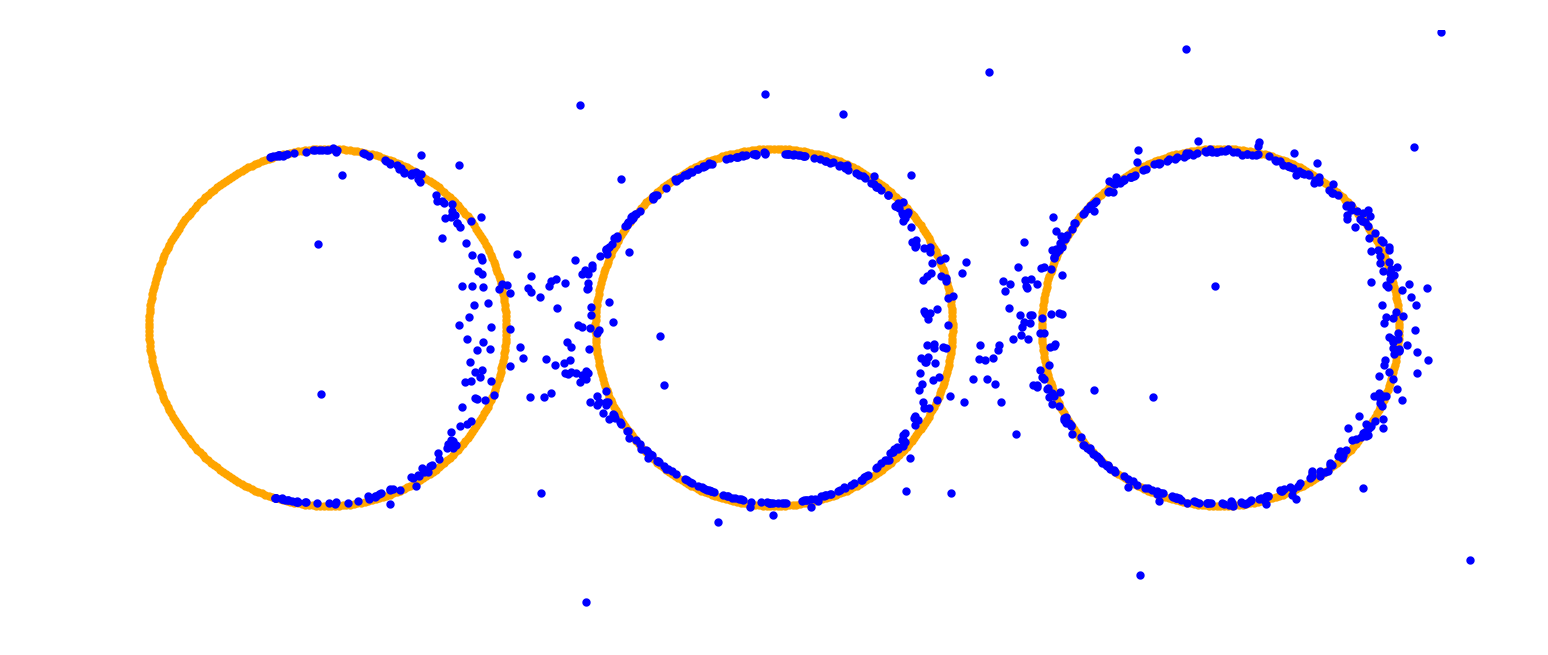}
    \caption*{t=0.1}
    \end{subfigure}%
    \begin{subfigure}[t]{.16\textwidth}
        \includegraphics[width=\linewidth]{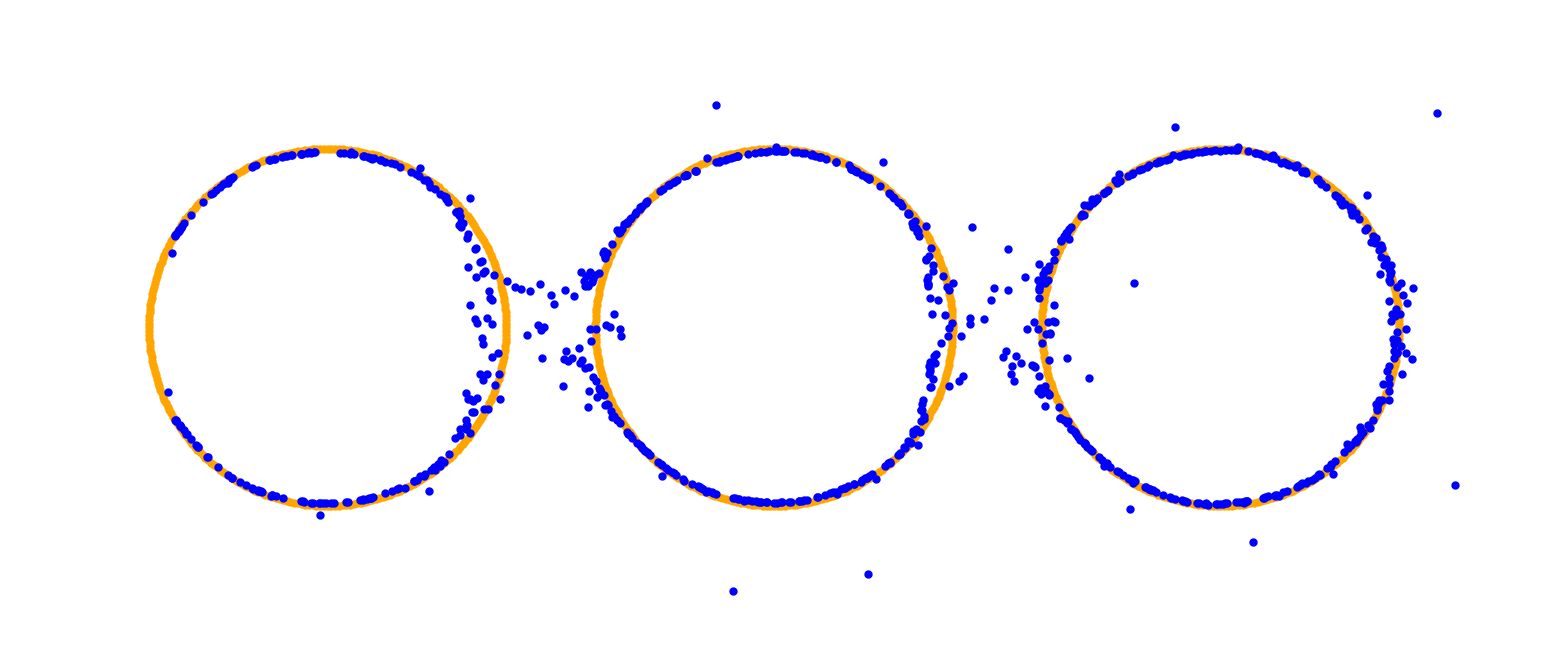}
    \caption*{t=0.2}
    \end{subfigure}%
    \begin{subfigure}[t]{.16\textwidth}
        \includegraphics[width=\linewidth]{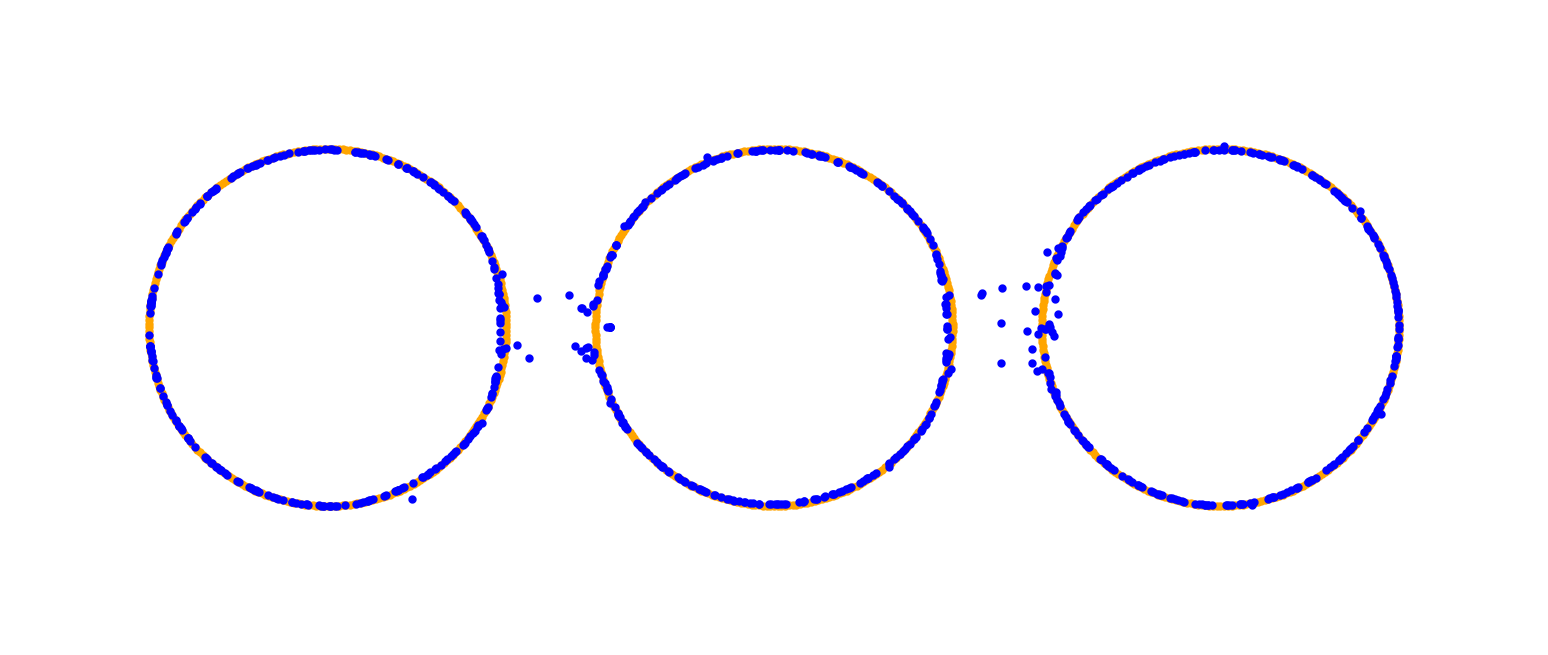}
    \caption*{t=1}
    \end{subfigure}%
    \begin{subfigure}[t]{.16\textwidth}
        \includegraphics[width=\linewidth]{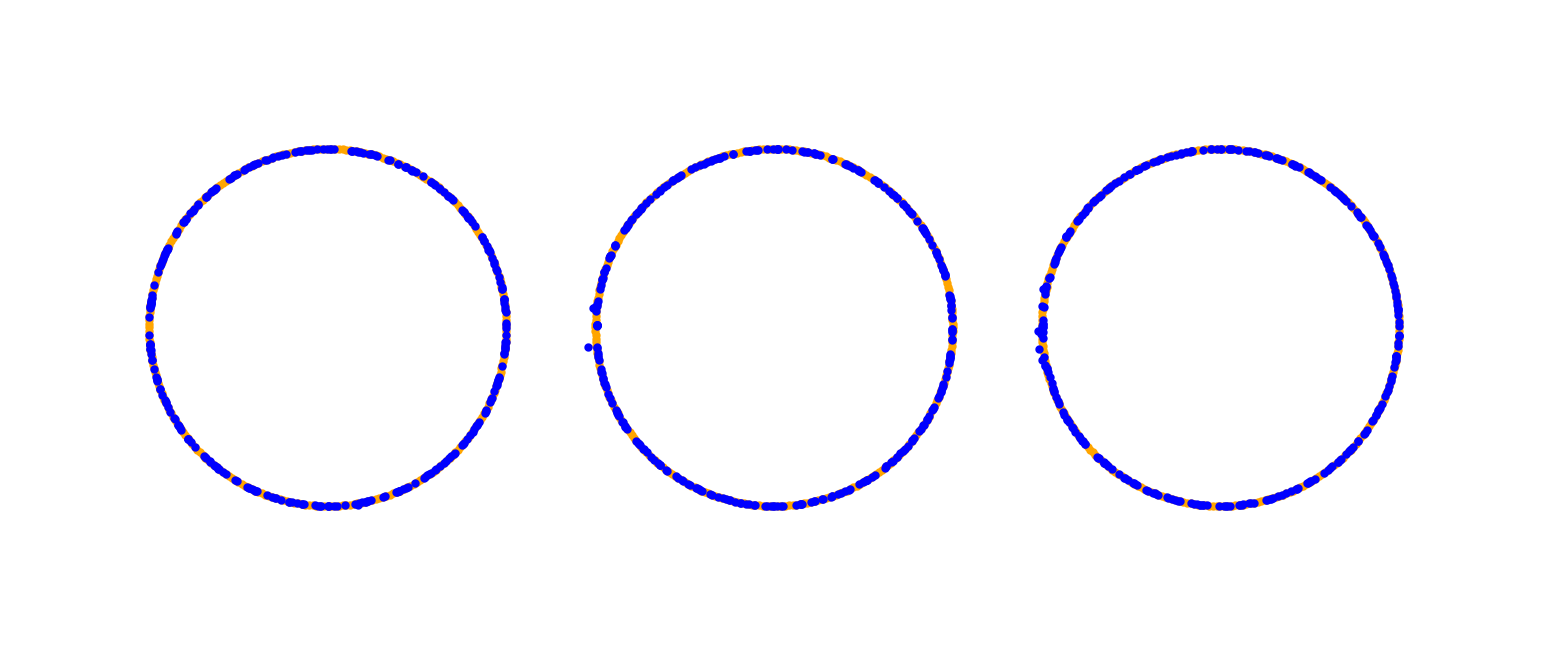}
    \caption*{t=10}
    \end{subfigure}%
    \begin{subfigure}[t]{.16\textwidth}
        \includegraphics[width=\linewidth]{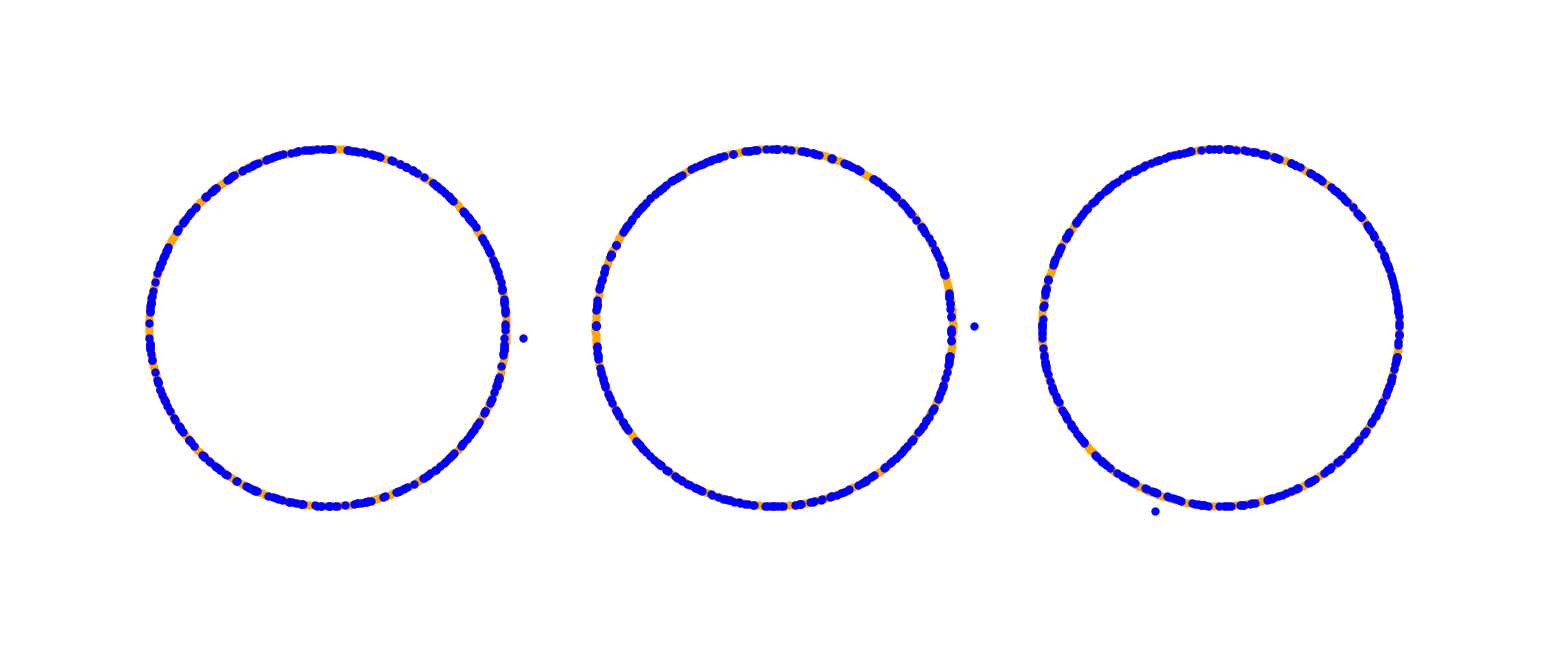}
    \caption*{t=50}
    \end{subfigure}%
    \caption{WGF of the regularized Tsallis-$\alpha$ divergence
    $D_{f_{\alpha}, \nu}^{\lambda}$ for $\alpha \in \{ 1, 3, 7.5 \}$.}
    \label{fig:circlesAlpha}
    \vspace{.2cm}

    \nextfloat
    \centering
    \includegraphics[width=.49\textwidth]{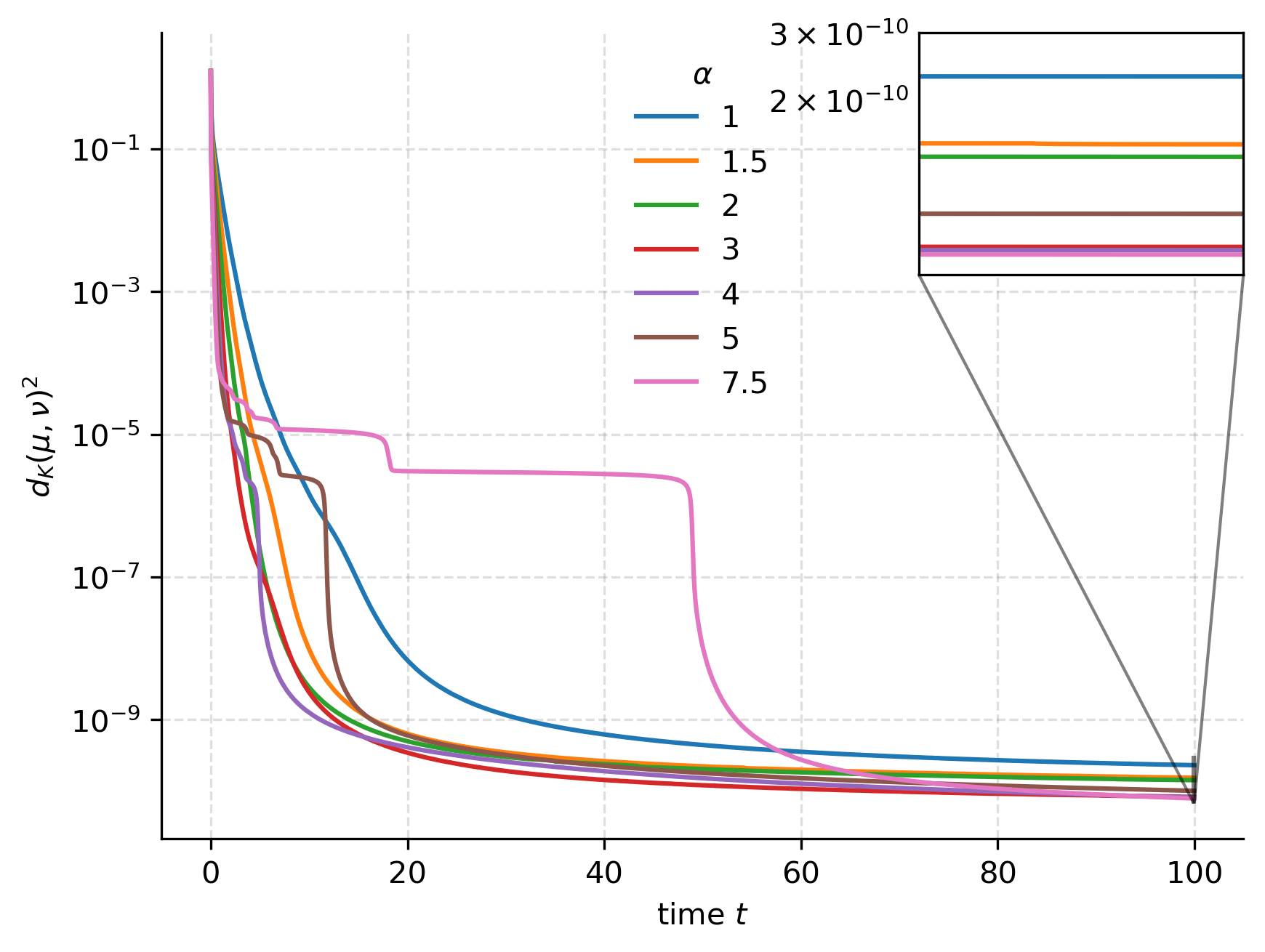}
    \includegraphics[width=.49\textwidth]{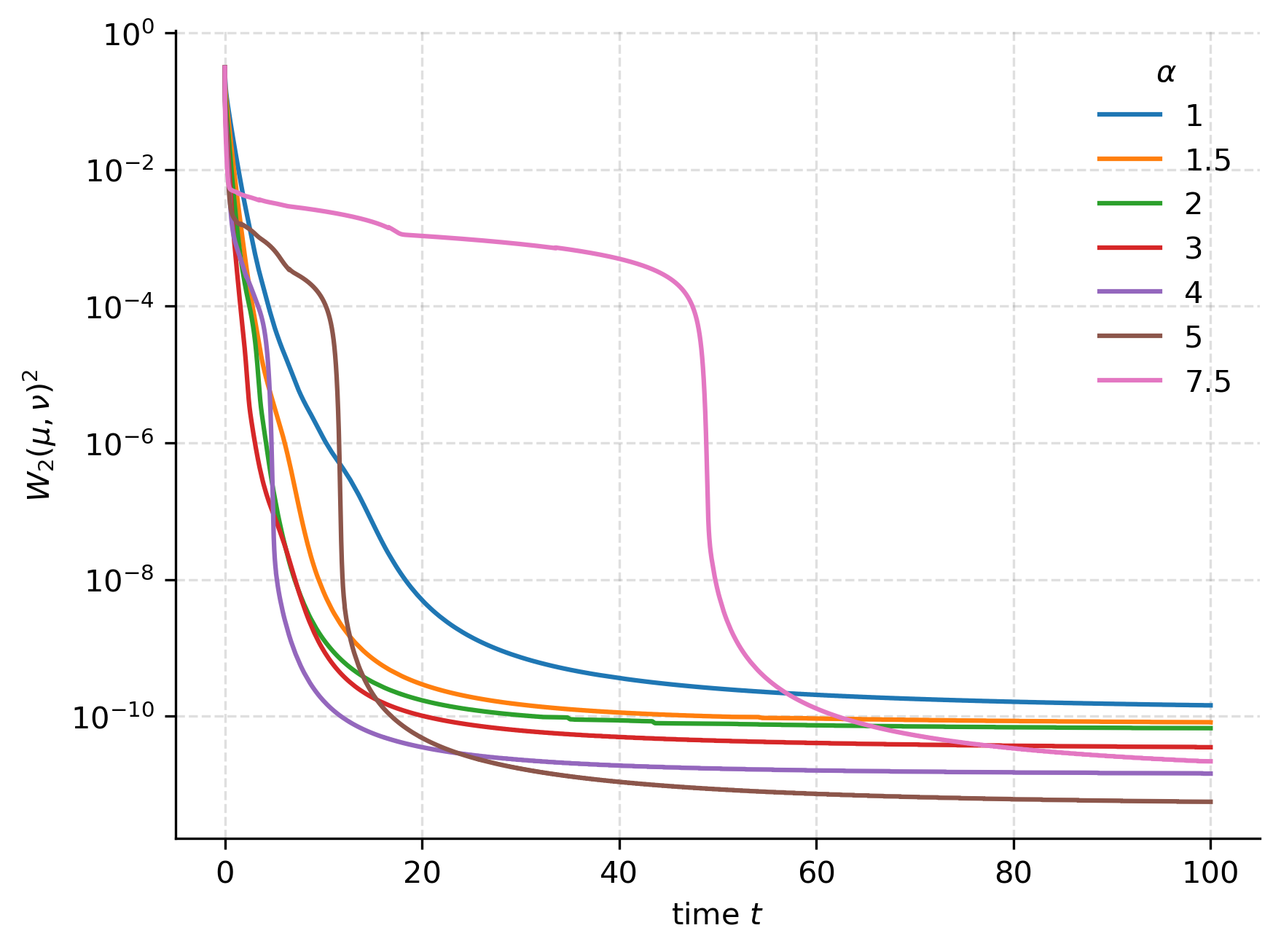}
    \caption{Comparison of squared MMD  (\textbf{left}) and $W_2$ distance (\textbf{right}) along the flow for different $\alpha$, where $\nu$ is the three circles target.
    Note the logarithmic axis.}
    \label{fig:circles_comparison}
\end{figure}

\begin{figure}
    \centering
    \begin{subfigure}[t]{.2\textwidth}
        \includegraphics[width=\linewidth, valign=c]{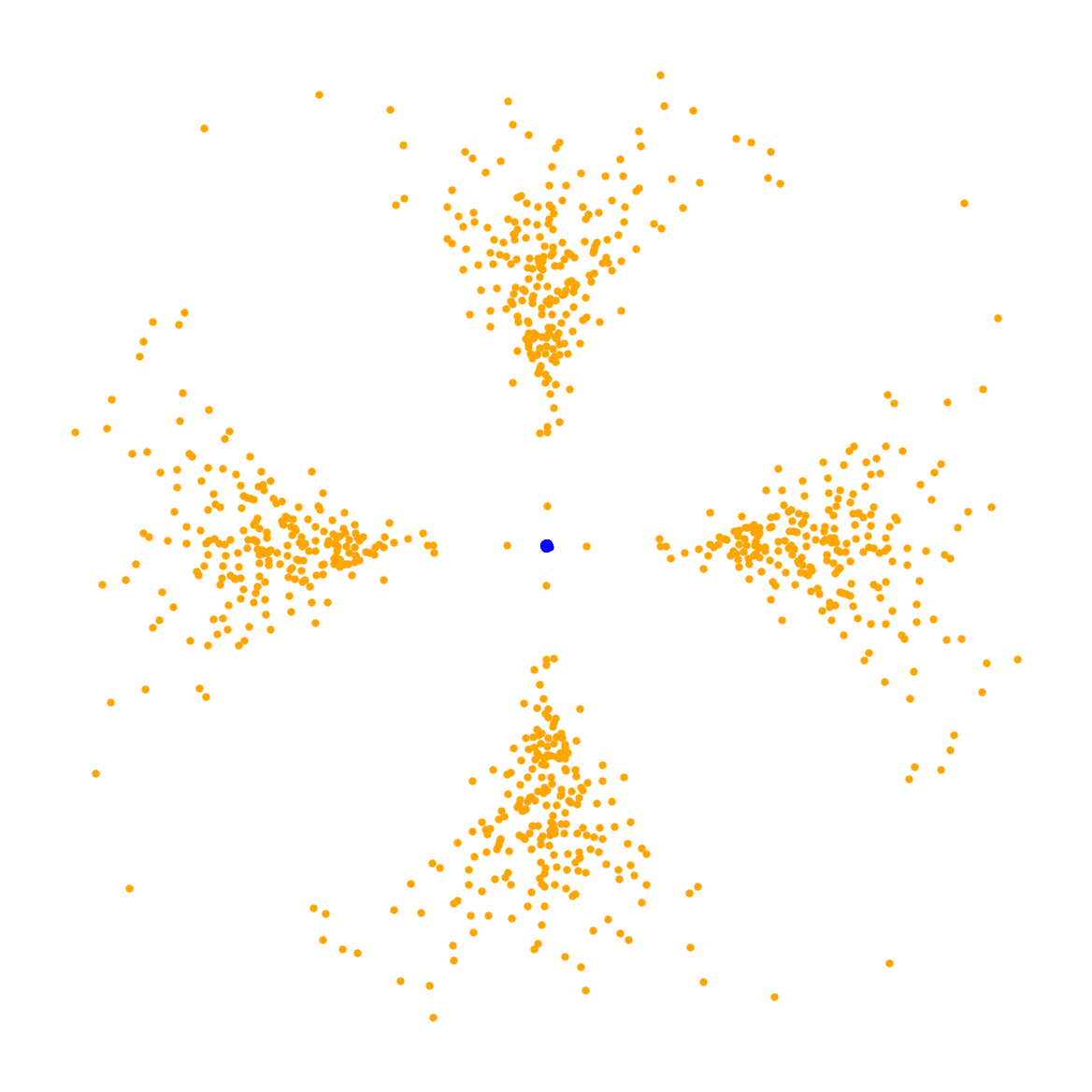}
    \caption*{t=0}
    \end{subfigure}%
    \begin{subfigure}[t]{.2\textwidth}
        \includegraphics[width=\linewidth, valign=c]{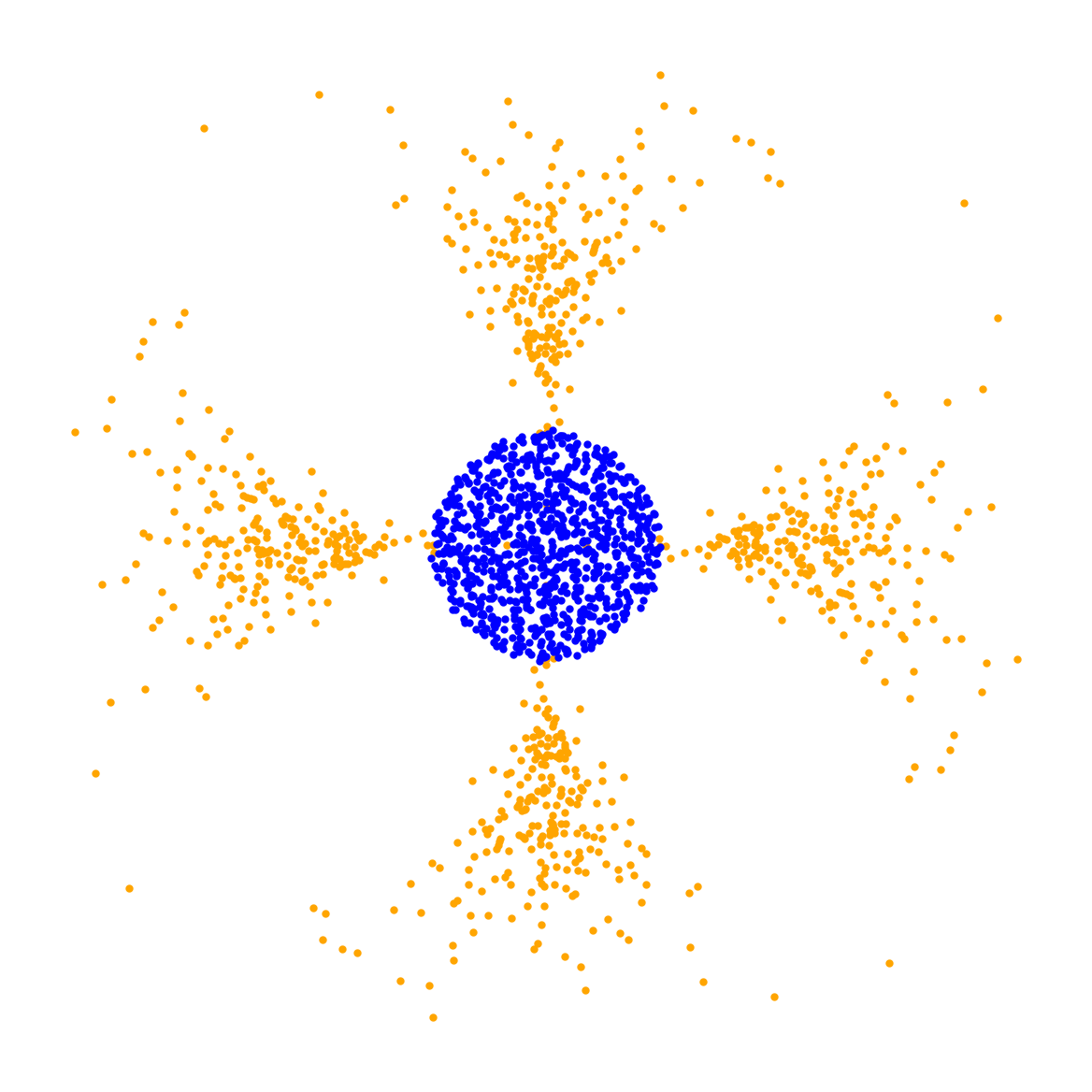}
    \caption*{t=0.1}
    \end{subfigure}%
    \begin{subfigure}[t]{.2\textwidth}
        \includegraphics[width=\linewidth, valign=c]{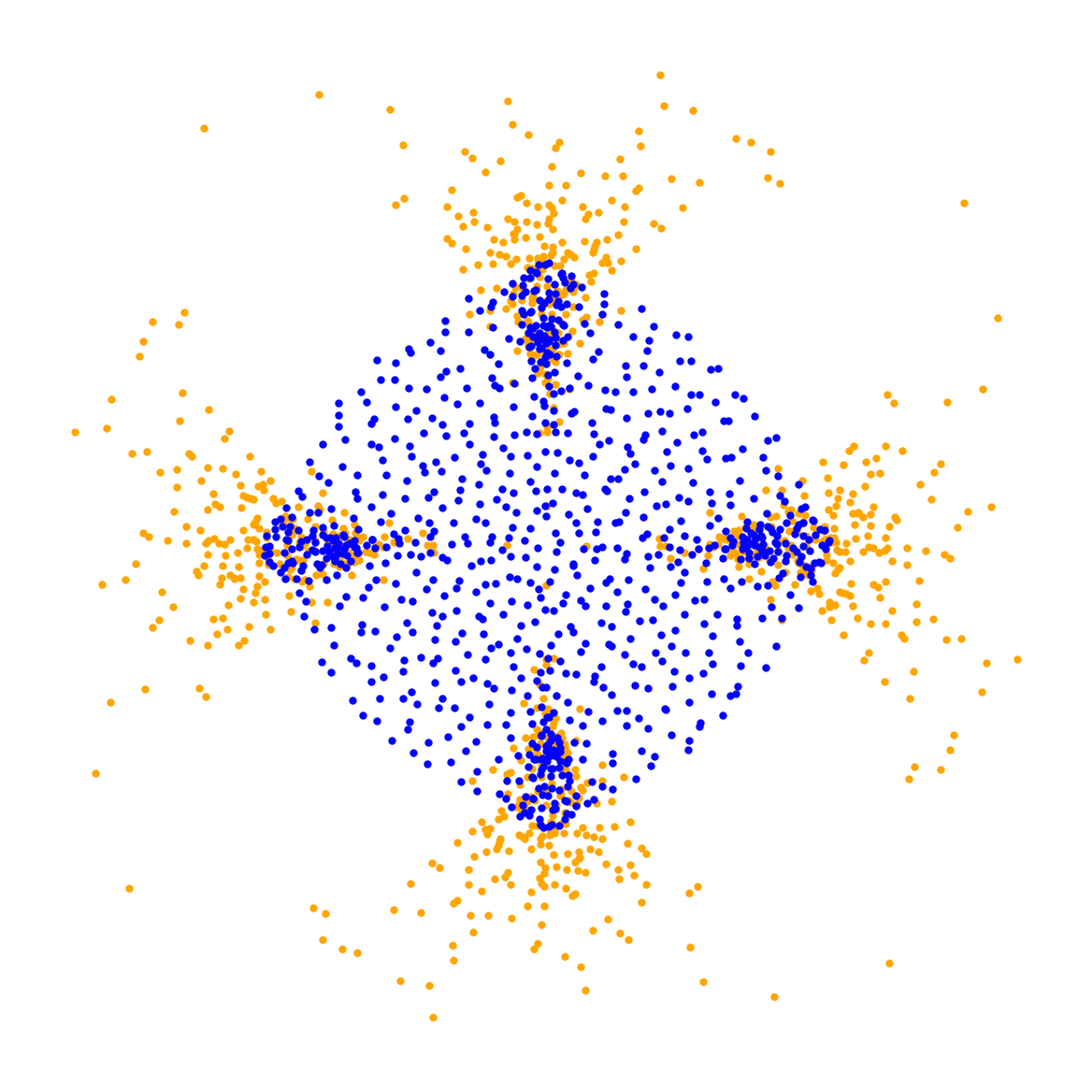}
    \caption*{t=1}
    \end{subfigure}%
    \begin{subfigure}[t]{.2\textwidth}
        \includegraphics[width=\linewidth, valign=c]{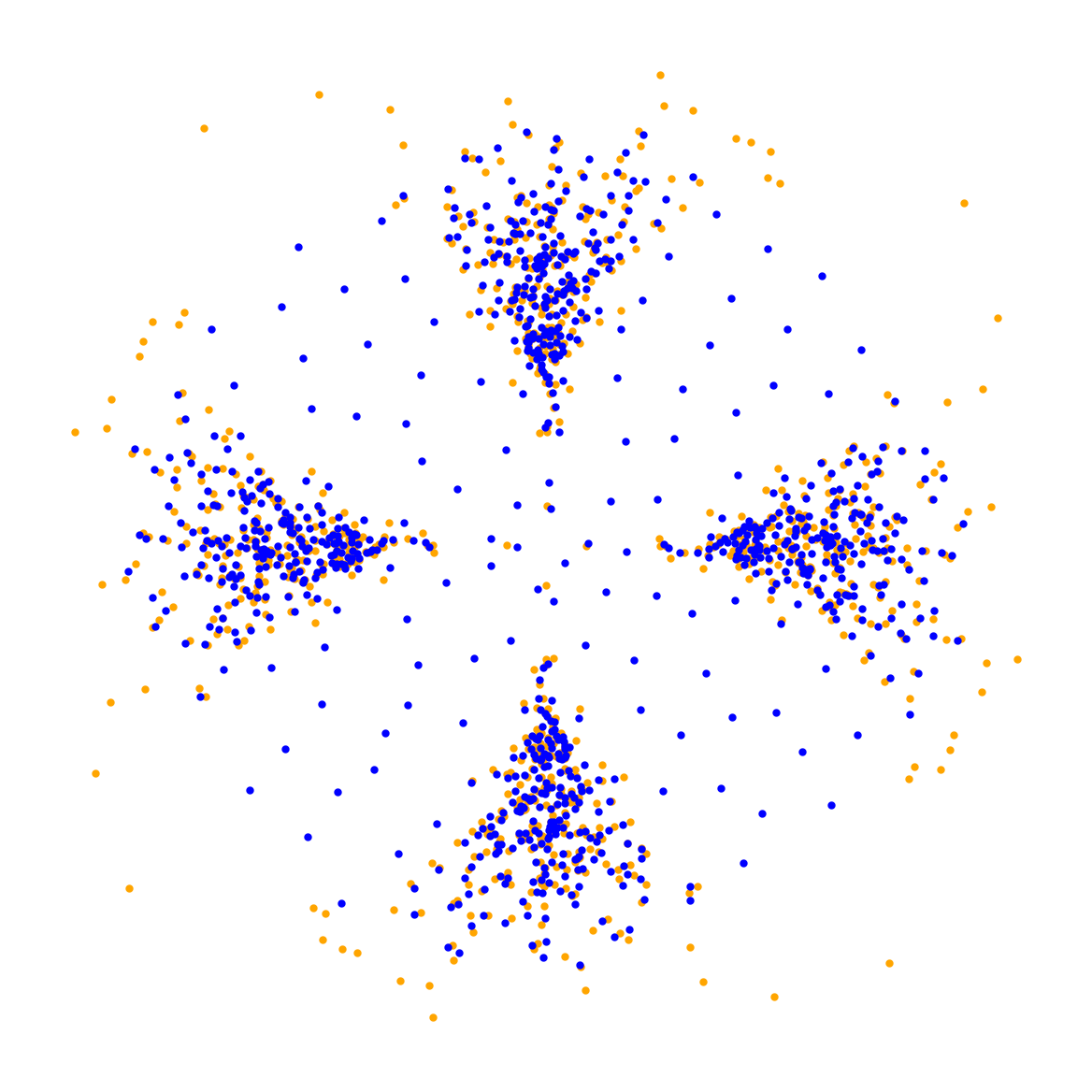}
    \caption*{t=10}
    \end{subfigure}%
    \begin{subfigure}[t]{.2\textwidth}
        \includegraphics[width=\linewidth, valign=c]{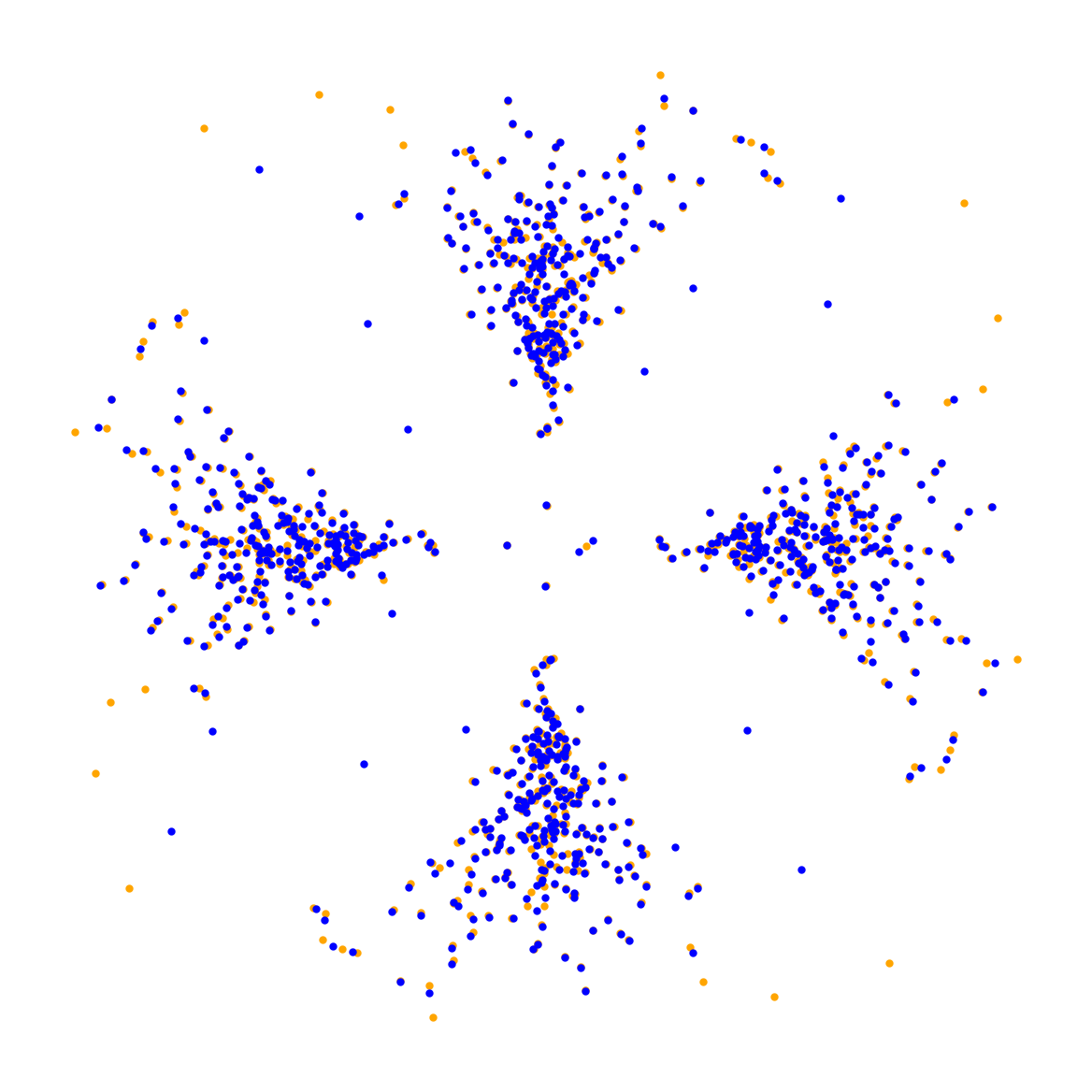}
    \caption*{t=50}
    \end{subfigure}%
    \caption{WGF of the regularized Tsallis-$\alpha$ divergence  $D_{f_{7.5}, \nu}^{\lambda}$ for Neals cross.}
    \label{fig:Cross}
    \vspace{.2cm}
    
    \nextfloat
    \centering
    \includegraphics[width=.49\textwidth]{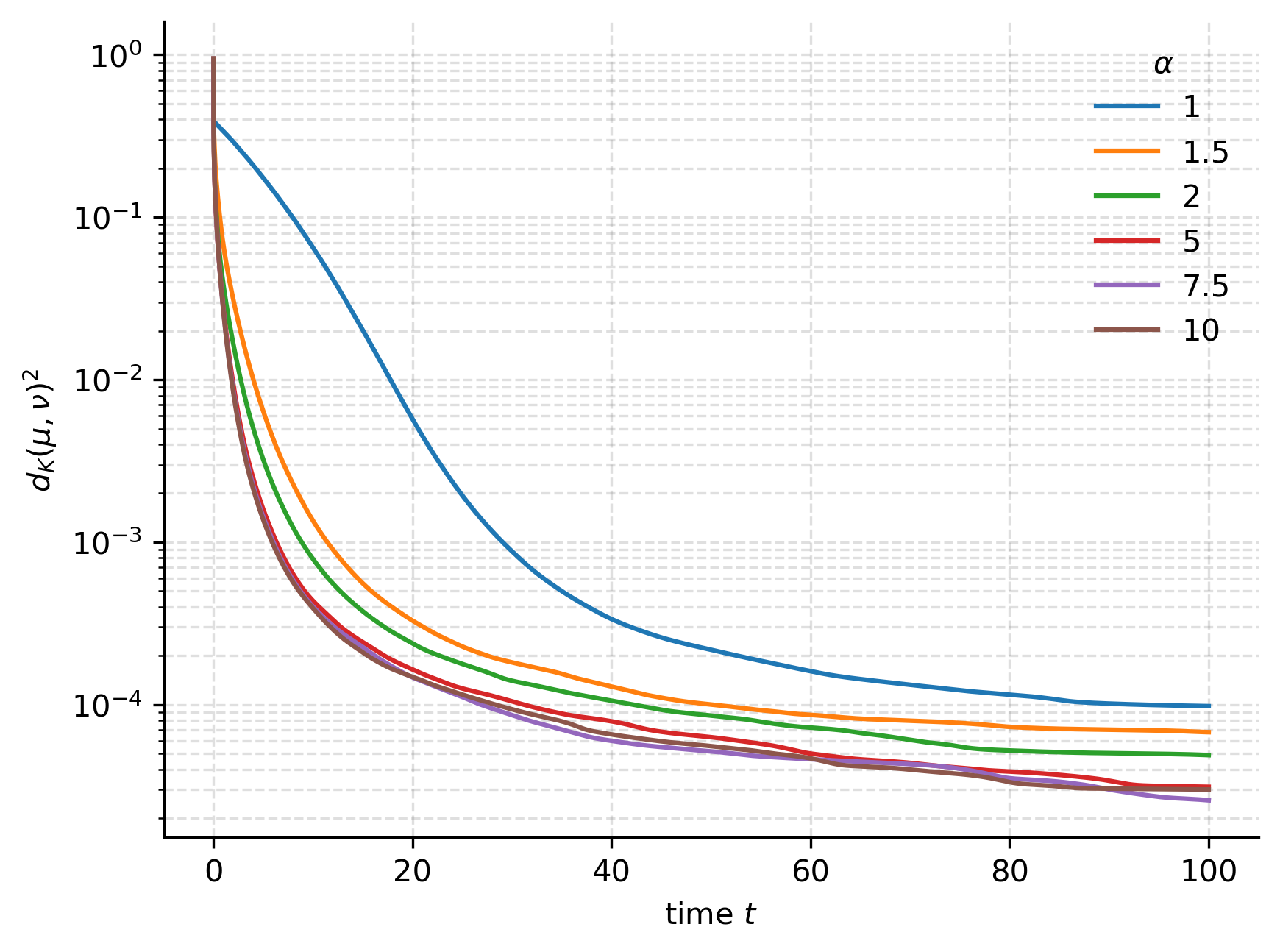}
    \includegraphics[width=.49\textwidth]{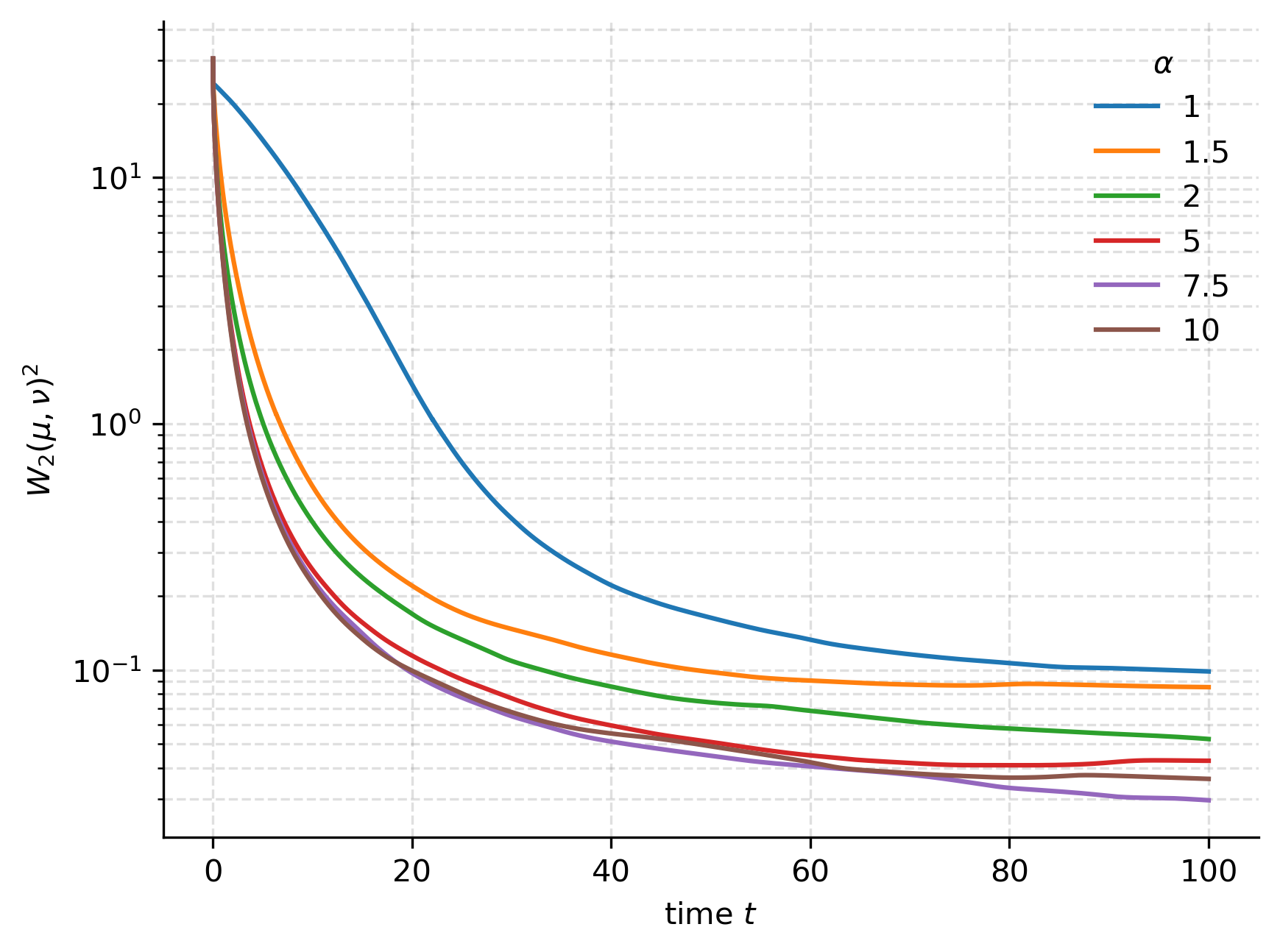}
    \caption{Comparison of the flow for $D_{f_{\alpha}, \nu}^{\lambda}$, where $\nu$ is Neal's cross.
    We depict the squared MMD \textbf{(left)} and Wasserstein distance \textbf{(right)} to $\nu$ for different $\alpha$.}
    \label{fig:CrossAlphaCompare}
\end{figure}

\textbf{Three rings target}
First, we consider the three rings target from \cite[Fig.~1]{GAG2021}.
Our simulations are provided in Figure \ref{fig:circlesAlpha}.
The starting point $\mu_0$ of the flow are samples from a normal distribution with variance $s^2 = \num{2d-3}$ around the leftmost point on the rightmost circle.
Further, we choose the kernel width $\sigma^2 = \num{5e-2}$, regularization parameter $\lambda = \num{e-2}$, and step size $\tau = \num{e-3}$.

For larger $\alpha$, the particles advance faster towards the leftmost circle in the beginning, and accordingly, the MMD decreases the fastest; see also Figure~\ref{fig:circles_comparison}.
On the other hand, for too large $\alpha$, there are more outliers (points that are far away from the rings) at the beginning of the flow.
Even at $t = 50$, some of them remain between the rings, which contrasts our observations for smaller $\alpha$.
This is also reflected by the fact that the MMD and $W_2$ values plateau for these values of $\alpha$ before they finally converge to zero.
The \enquote{sweet spot} for $\alpha$ seems to be $\alpha \in [3,4]$ since then the MMD and $W_2$ loss drop below $\num{e-9}$ the fastest.
For the first few iterations, the MMD values are monotone with respect to $\alpha$: the lower curve is the one belonging to $\alpha = 7.5$, the one above belongs to $\alpha = 5$, and the top one belongs to $\alpha = 1$.
For the last time steps, this order is nearly identical.
If $\alpha \in \{ 10, 50, 100, 500 \}$, then the flow behaves even worse in the sense that the plateau phases becomes longer, i.e., both the $W_2$ and the squared MMD loss converge to 0 even slower.

\textbf{Neal's cross target}
Inspired by \cite[Fig.~1f]{Xu2023}, the target $\nu$ comprises four identical versions of Neal's funnel, each rotated by 90 degrees about the origin, see Figure~\ref{fig:Cross}.
We generate the samples $\{ (x_{1, k}, x_{2, k}) \}_{k = 1}^{N}$ by drawing normally distributed samples $x_{2, k} \sim \NN(7.5, 2)$ and $x_{1, k} \sim \NN(0, e^{\frac{1}{3} x_{2, k}})$, where $\NN(m, s^2)$ is a univariate normal distribution with mean $m$ and variance $s^2$.
For our simulations, we choose $\alpha = 7.5$, $\lambda = \num{e-2}$, $\tau = \num{e-3}$ and $\sigma^2 = 0.25$.

We observe that the particles (blue) are pushed towards the regions with a high density of target particles (orange) and that the low-density regions at the ends of the funnel are not matched exactly.
Since the target $\nu$ is usually obtained by drawing samples from a distribution, this behavior is acceptable.
Figure~\ref{fig:CrossAlphaCompare} shows that for $\alpha = 1$ the gradient flow with respect to $D_{f_{\alpha}, \nu}^{\lambda}$ recovers the target slower, both in terms of the MMD or the $W_2$ metric and that $\alpha = 7.5$ performs best.

\textbf{Bananas target}
This experiment is inspired by a talk from Aude Genevay\footnote{MIFODS Workshop on Learning with Complex Structure 2020, see \url{https://youtu.be/TFdIJib_zEA?si=B3fsQkfmjea2HCA5}.}, see Figure \ref{fig:disconnectedTarget}.
The target $\nu$ is multimodal and the \enquote{connected components} of its support are far apart.
Additionally, the chosen initial measure does not take the properties of $\supp(\nu)$ into account.
Still, we observe the convergence of the particles to the target $\nu$.
We also observe a mode-seeking behavior: the particles concentrate near the mean of the right \enquote{banana} first, which is even more pronounced for the left \enquote{banana} at later time points.

Since the parameter $\lambda$ penalizes the disjoint support condition, we choose $\lambda$ to be higher than for the other targets, namely $\lambda = 1$, so that the particles are encouraged to \enquote{jump} from one mode to another.
Further, we choose $\tau = \num{e-1}$, $\sigma = \num{2e-2}$ and $\alpha = 3$ for the simulations.
Empirically, we observed that the value of $\alpha$ makes no difference, though. 

\begin{figure}[tbp]
    \centering
    \begin{subfigure}[t]{.245\textwidth}
      \includegraphics[width=\linewidth, valign=c]{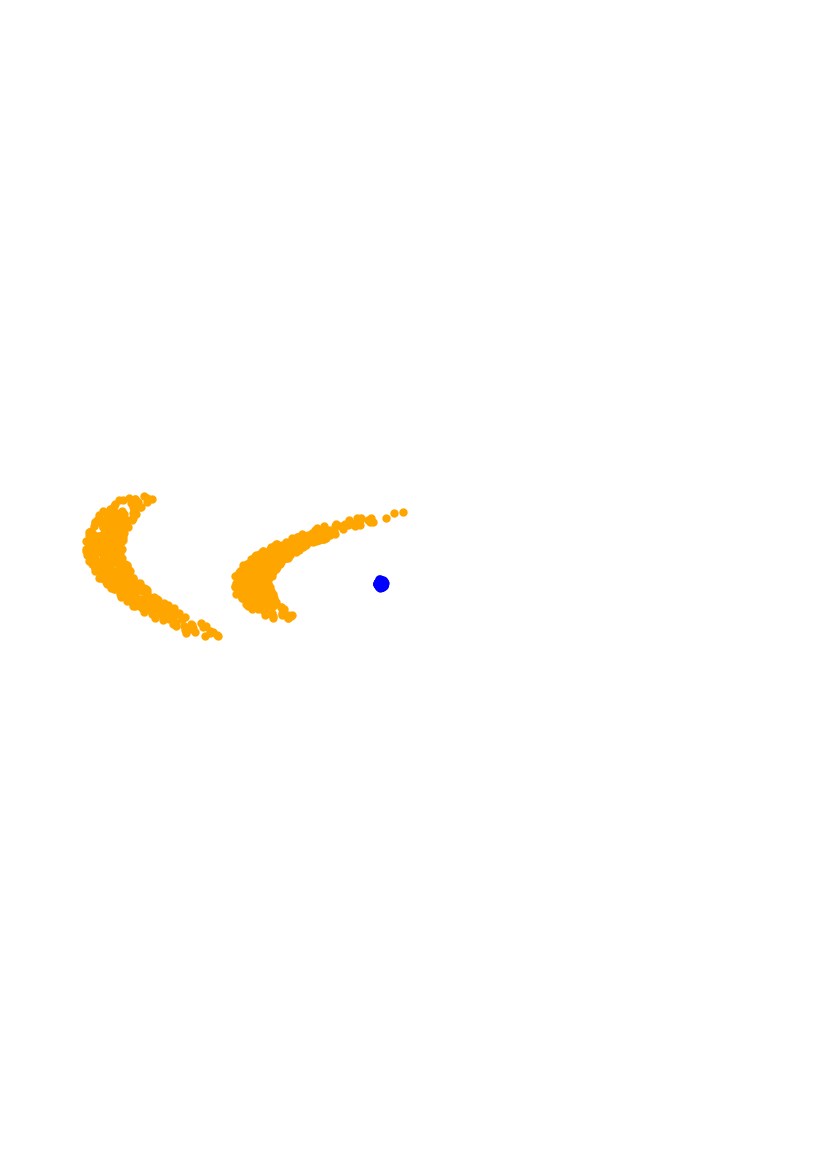}
    \caption*{t=0}
    \end{subfigure}%
    \begin{subfigure}[t]{.245\textwidth}
      \includegraphics[width=\linewidth, valign=c]{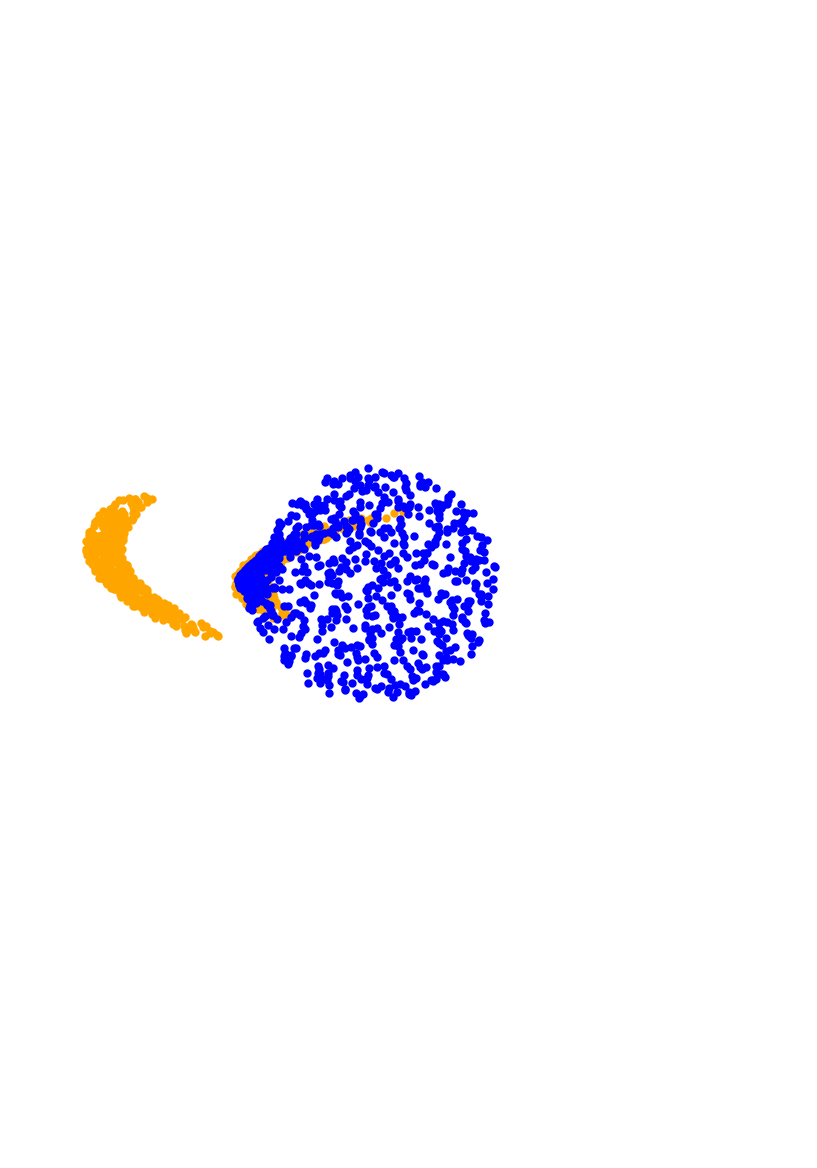}
    \caption*{t=10}
    \end{subfigure}%
    \begin{subfigure}[t]{.245\textwidth}
      \includegraphics[width=\linewidth, valign=c]{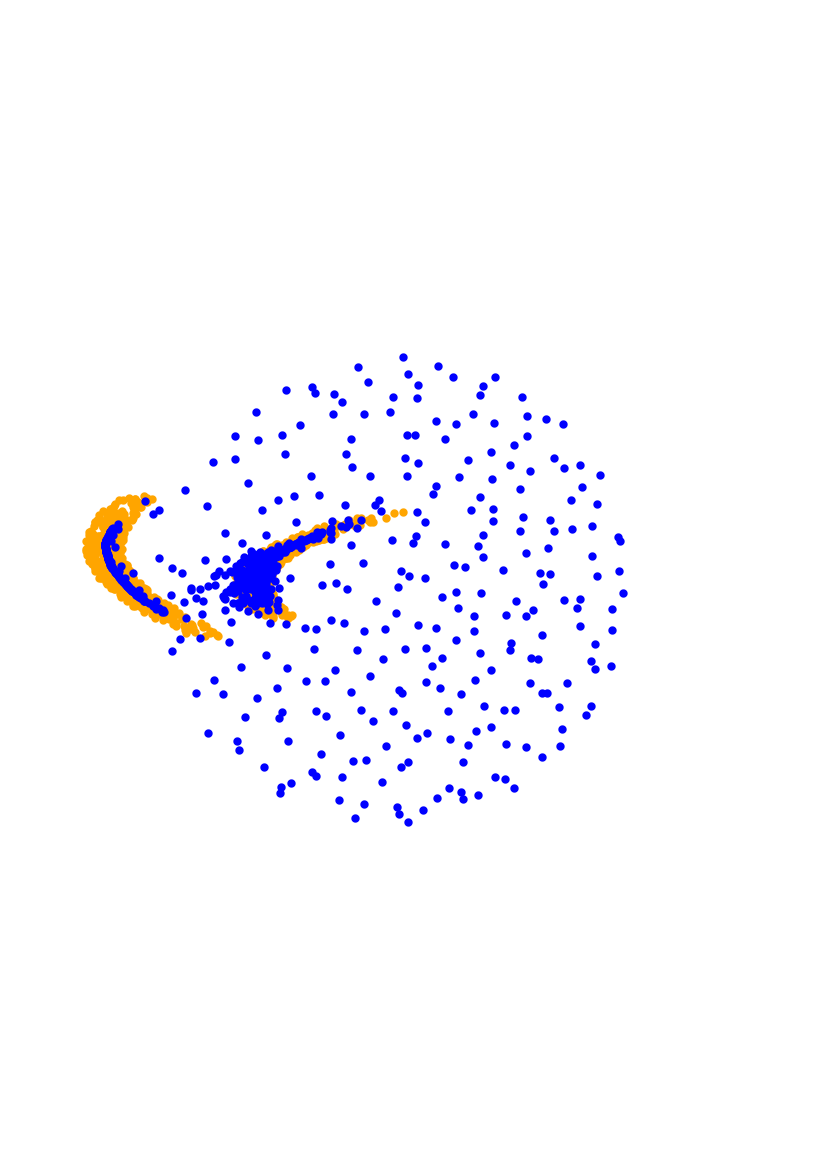}
    \caption*{t=100}
    \end{subfigure}
    
    \begin{subfigure}[t]{.245\textwidth}
        \includegraphics[width=\linewidth, valign=c]{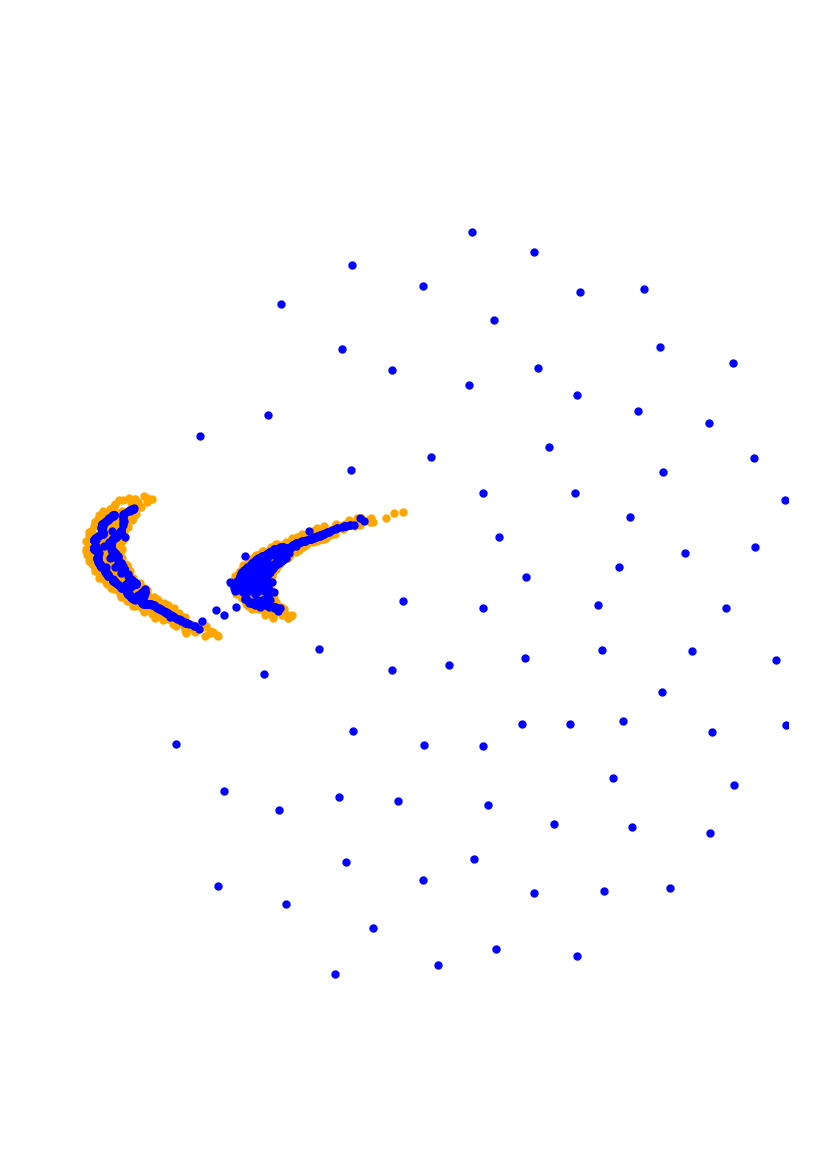}
        \caption*{t=1000}
    \end{subfigure}%
    \begin{subfigure}[t]{.245\textwidth}
      \includegraphics[width=\linewidth, valign=c]{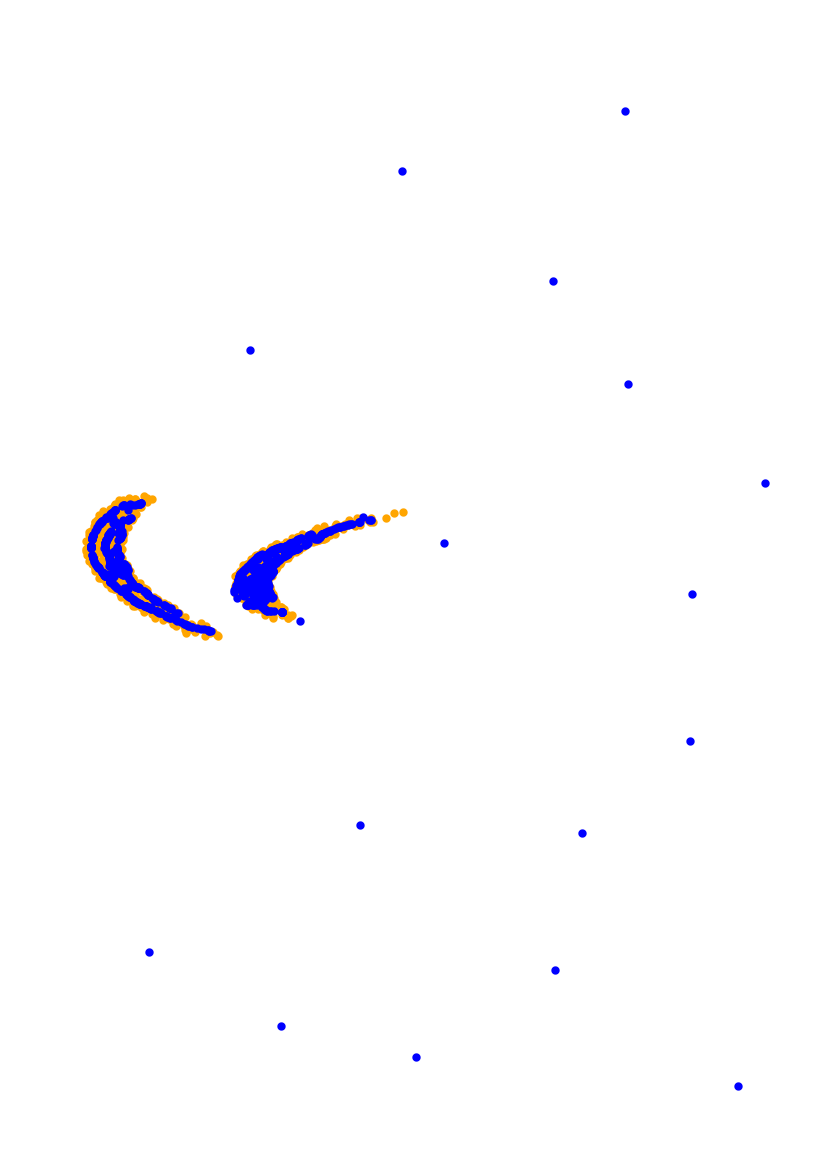}
    \caption*{t=10000}
    \end{subfigure}%
     \begin{subfigure}[t]{.245\textwidth}
      \includegraphics[width=\linewidth, valign=c]{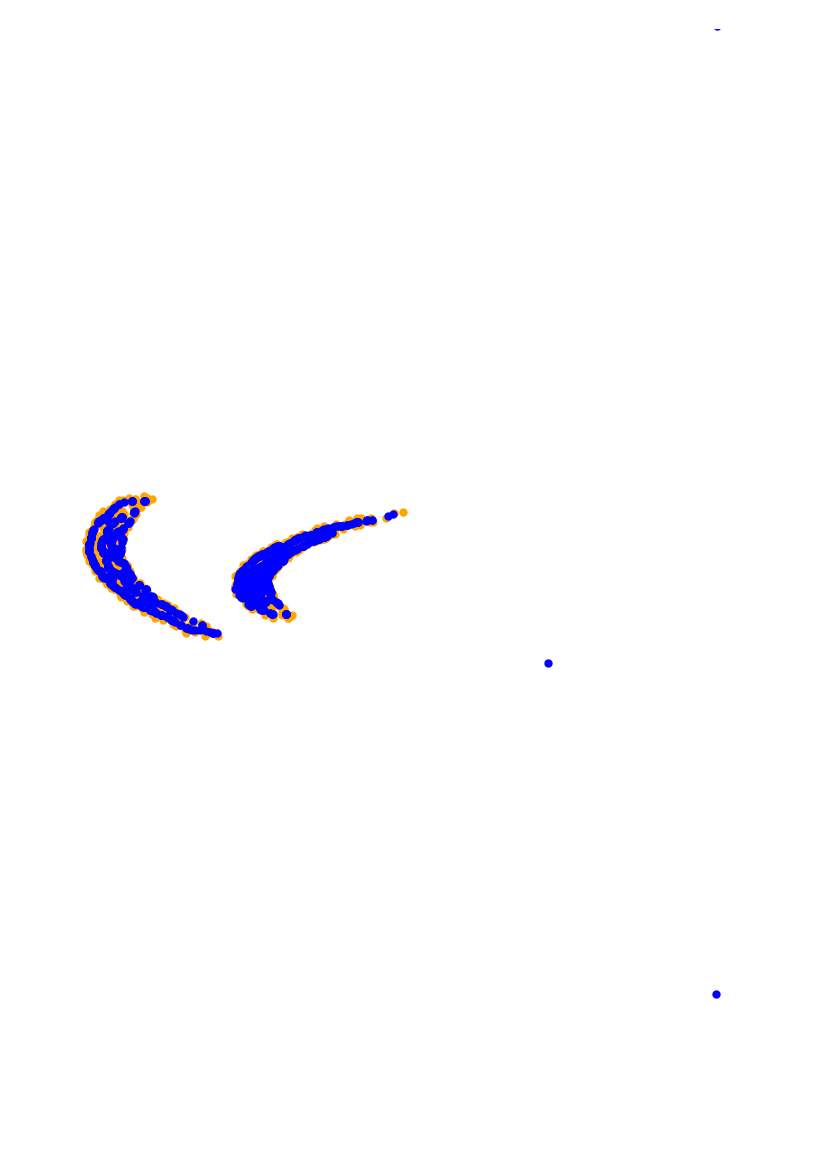}
    \caption*{t=1000000}
    \end{subfigure}%
    \caption{WGF of the regularized Tsallis-3 divergence
    $D_{f_{3}, \nu}^{\lambda}$ with the bananas target.
   }
    \label{fig:disconnectedTarget}
\end{figure}

\textbf{Finite recession constant and annealing}
Lastly, in Figures~\ref{fig:tvflow} and \ref{fig:tsallis12flow}, we show that the three rings target is also well recovered for entropy functions with finite recession constant.
To illustrate the convergence of our procedure, we first choose the total variation entropy function, $\tau = \num{e-3}$ and the kernel width $\sigma = \num{5e-2}$.
According to Lemma \ref{lemma:reprThmForFiniteRecConst}, we choose $\lambda = \num{2.25}$.
As a second example, we choose the $\frac{1}{2}$-Tsallis entropy function, $\tau = \num{e-3}$ and the kernel width $\sigma = \num{5e-2}$ as well as $\lambda = \num{2.25}$ (so that Lemma \ref{lemma:reprThmForFiniteRecConst} applies).

Since smaller $\lambda$ leads to better target recovery, we use an annealing heuristic (similar to \cite{C20}): Ultimately, we want to sample for the target distribution.
Hence, we reduce the value of $\lambda$ along the flow.
This is justified as follows.
For the Wasserstein gradient flow $(\gamma_t)_t$ with respect to $\smash{D_{f, \nu}^{\lambda}}$ for a fixed $\lambda$, we expect that $\smash{D_{f}^{\lambda}}(\gamma_t, \nu) \to 0$ as $t \to \infty$.
Then, by Corollary \ref{lemma:topology}(3), the same holds for $d_K(\gamma_t, \nu)$, see also Figures \ref{fig:circles_comparison} and \ref{fig:CrossAlphaCompare} for empirical evidence.
Based on Lemma \ref{lemma:reprThmForFiniteRecConst}, we can thus reduce the value of $\lambda$ along the flow while maintaining the finite-dimensional minimization problem.
In this last example, we divide the value of $\lambda$ by five three times, namely at $t \in \{5, 10, 20\}$.
This indeed improved the convergence.
\begin{figure}[tbp]
    \centering
    \begin{subfigure}[t]{.2\textwidth}
        \includegraphics[width=\linewidth, valign=c]{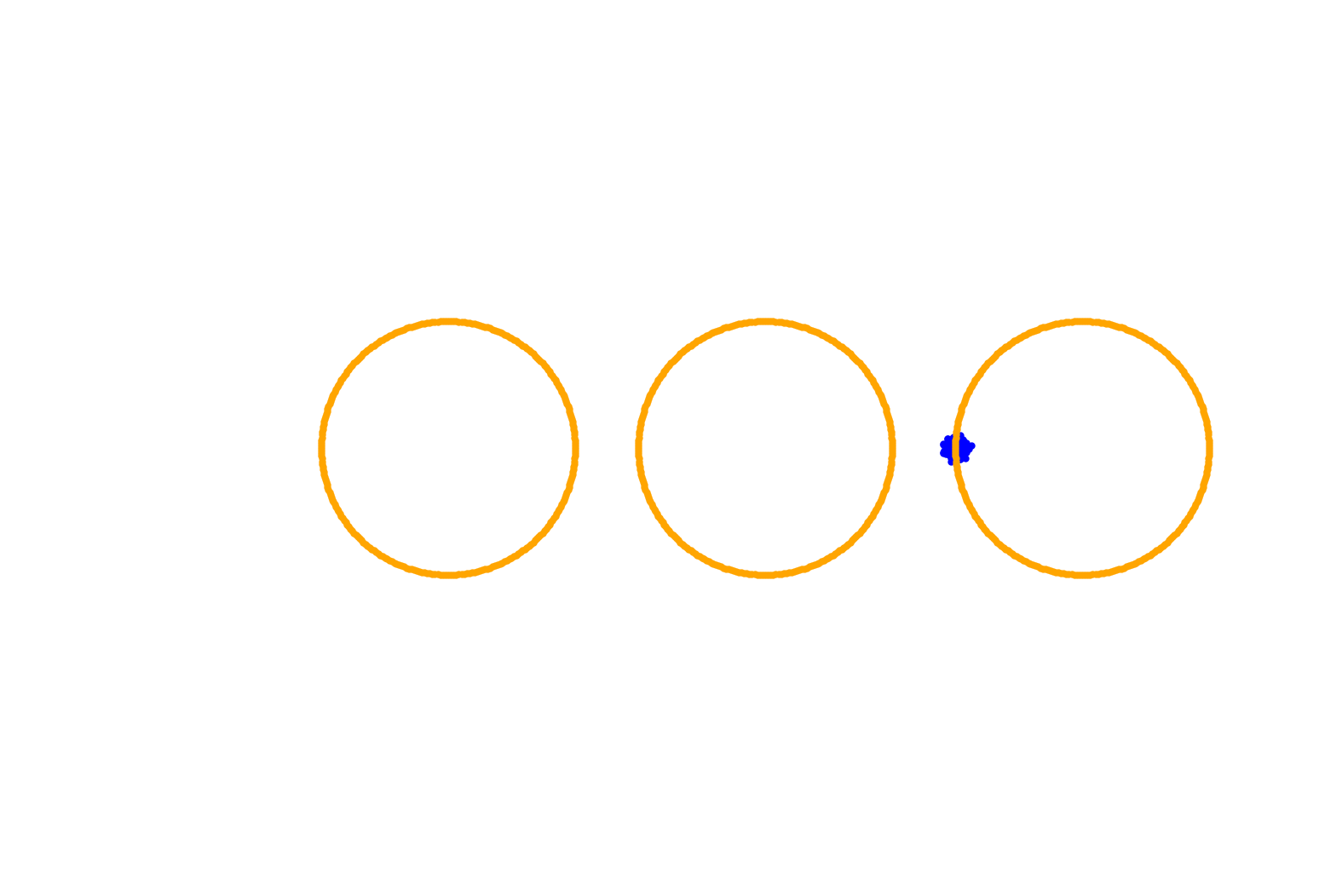}
    \caption*{t=0}
    \end{subfigure}%
    \begin{subfigure}[t]{.2\textwidth}
        \includegraphics[width=\linewidth, valign=c]{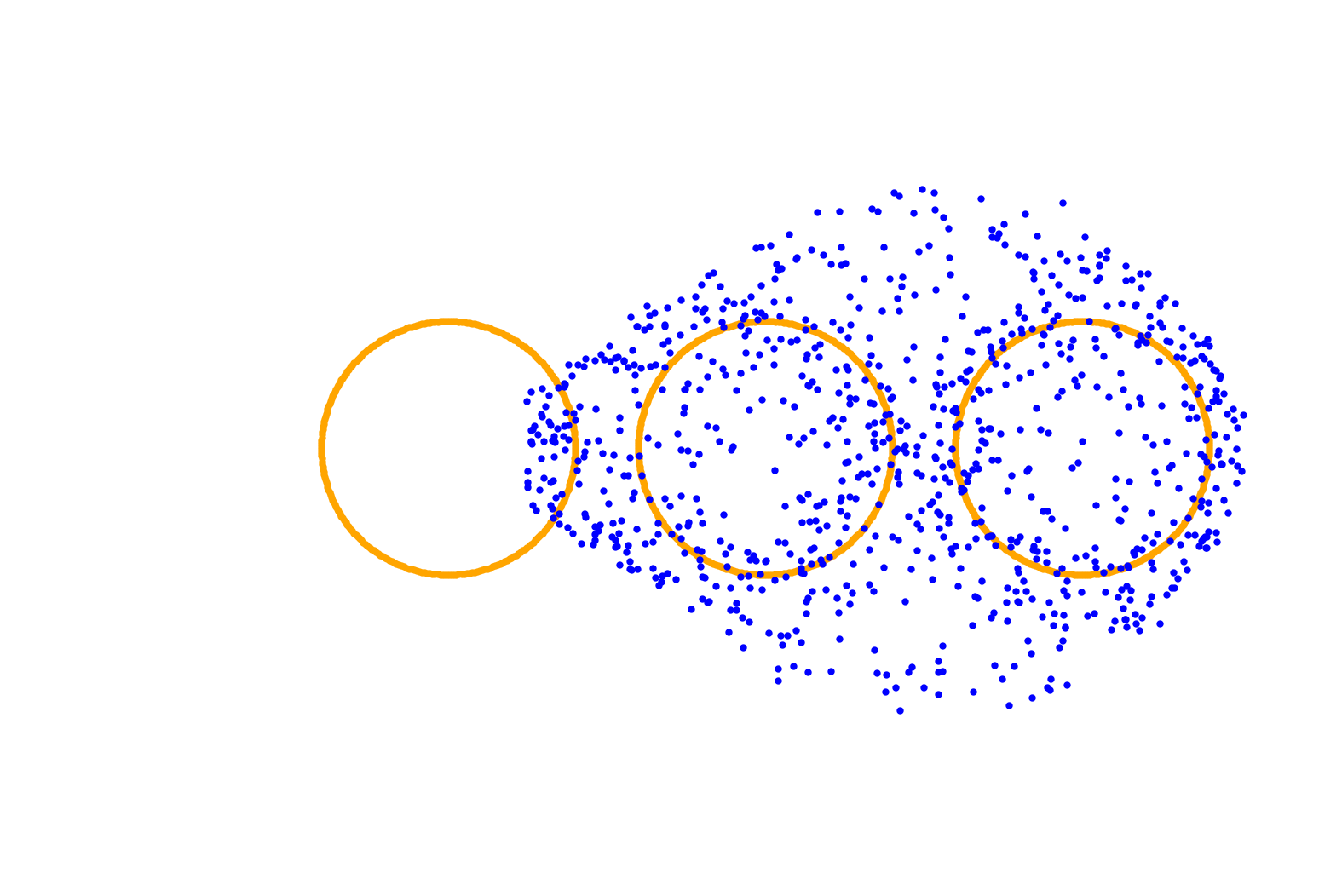}
    \caption*{t=1}
    \end{subfigure}%
    \begin{subfigure}[t]{.2\textwidth}
        \includegraphics[width=\linewidth, valign=c]{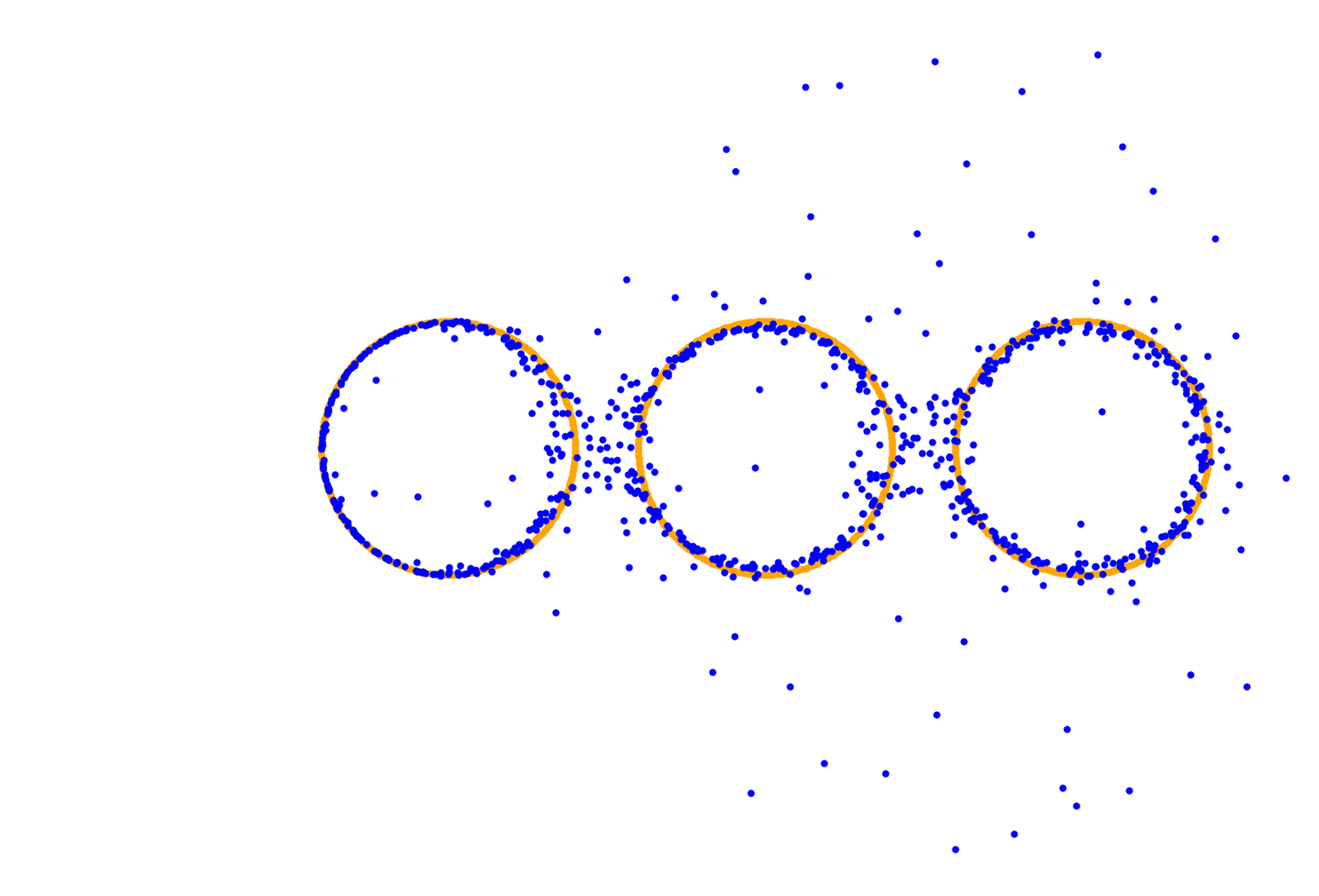}
    \caption*{t=10}
    \end{subfigure}%
    \begin{subfigure}[t]{.2\textwidth}
        \includegraphics[width=\linewidth, valign=c]{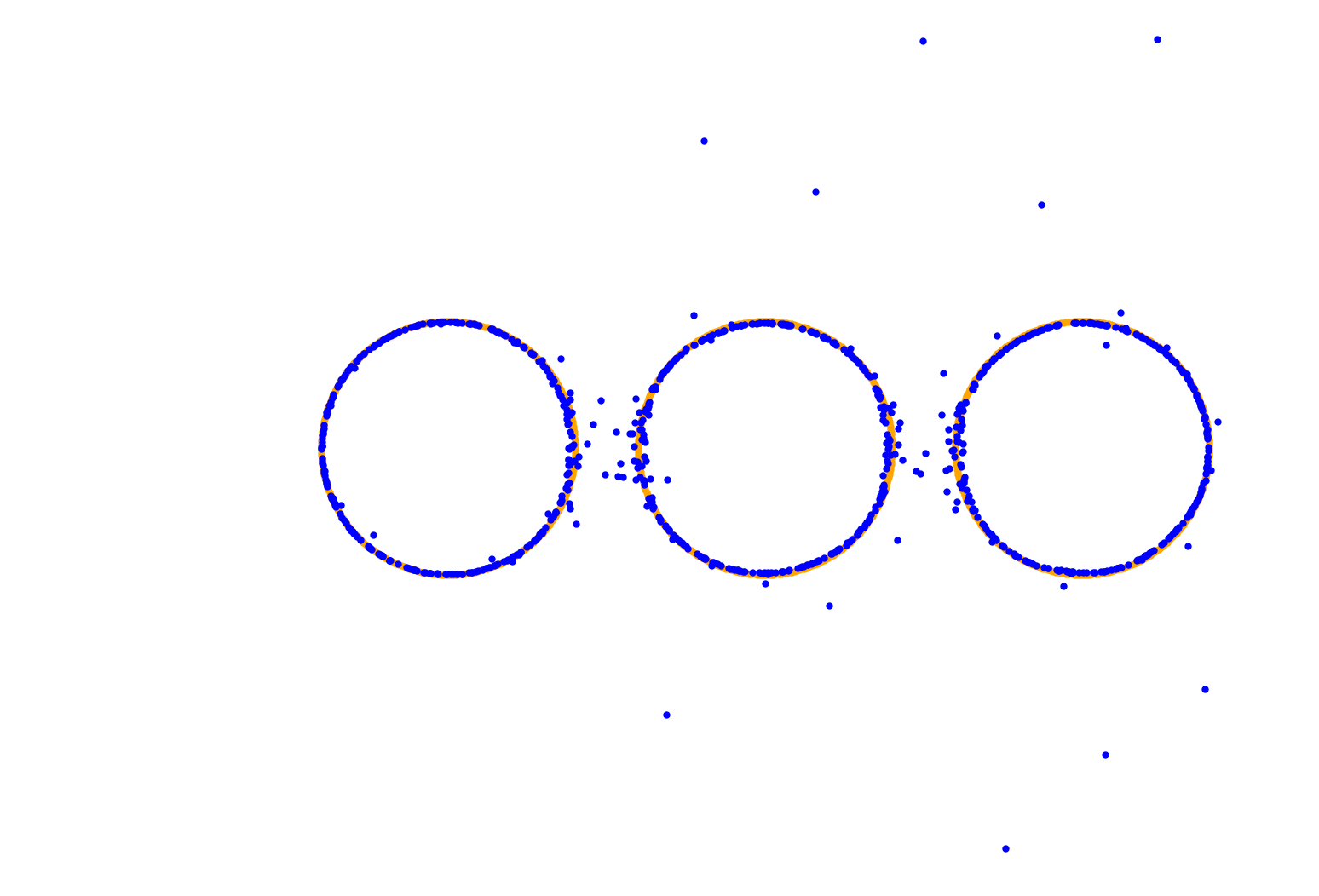}
    \caption*{t=50}
    \end{subfigure}%
    \begin{subfigure}[t]{.2\textwidth}
        \includegraphics[width=\linewidth, valign=c]{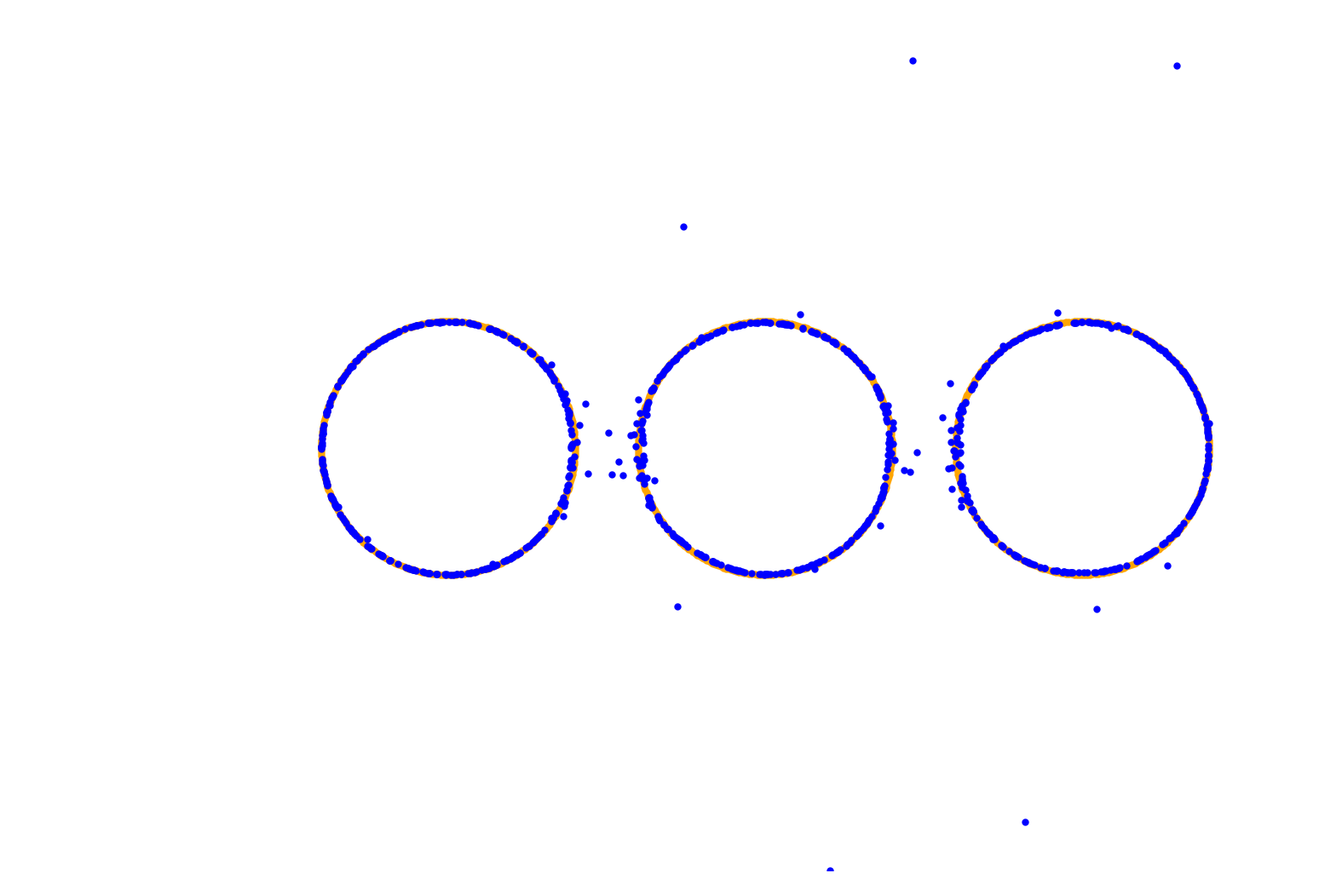}
    \caption*{t=100}
    \end{subfigure}%
    \caption{WGF of the regularized TV divergence $D_{f_{\TV}, \nu}^{\lambda}$ with the three rings target.}
    \label{fig:tvflow}
\end{figure}
\begin{figure}[tbp]
    \centering
    \begin{subfigure}[t]{.2\textwidth}
        \includegraphics[width=\linewidth, valign=c]{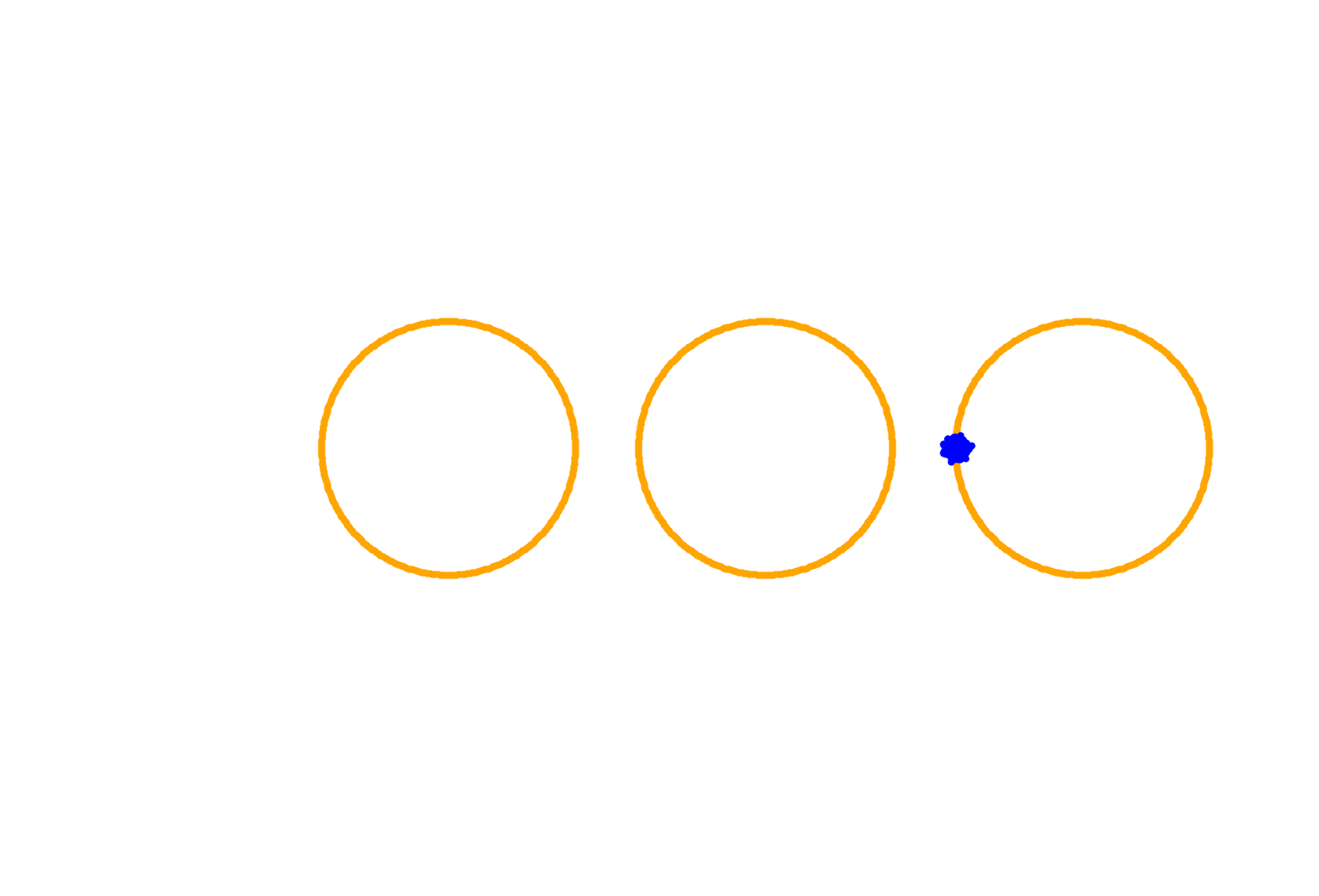}
        \includegraphics[width=\linewidth, valign=c]{img/finite_rec_annealing/annealing0.png}
    \caption*{t=0}
    \end{subfigure}%
    \begin{subfigure}[t]{.2\textwidth}
        \includegraphics[width=\linewidth, valign=c]{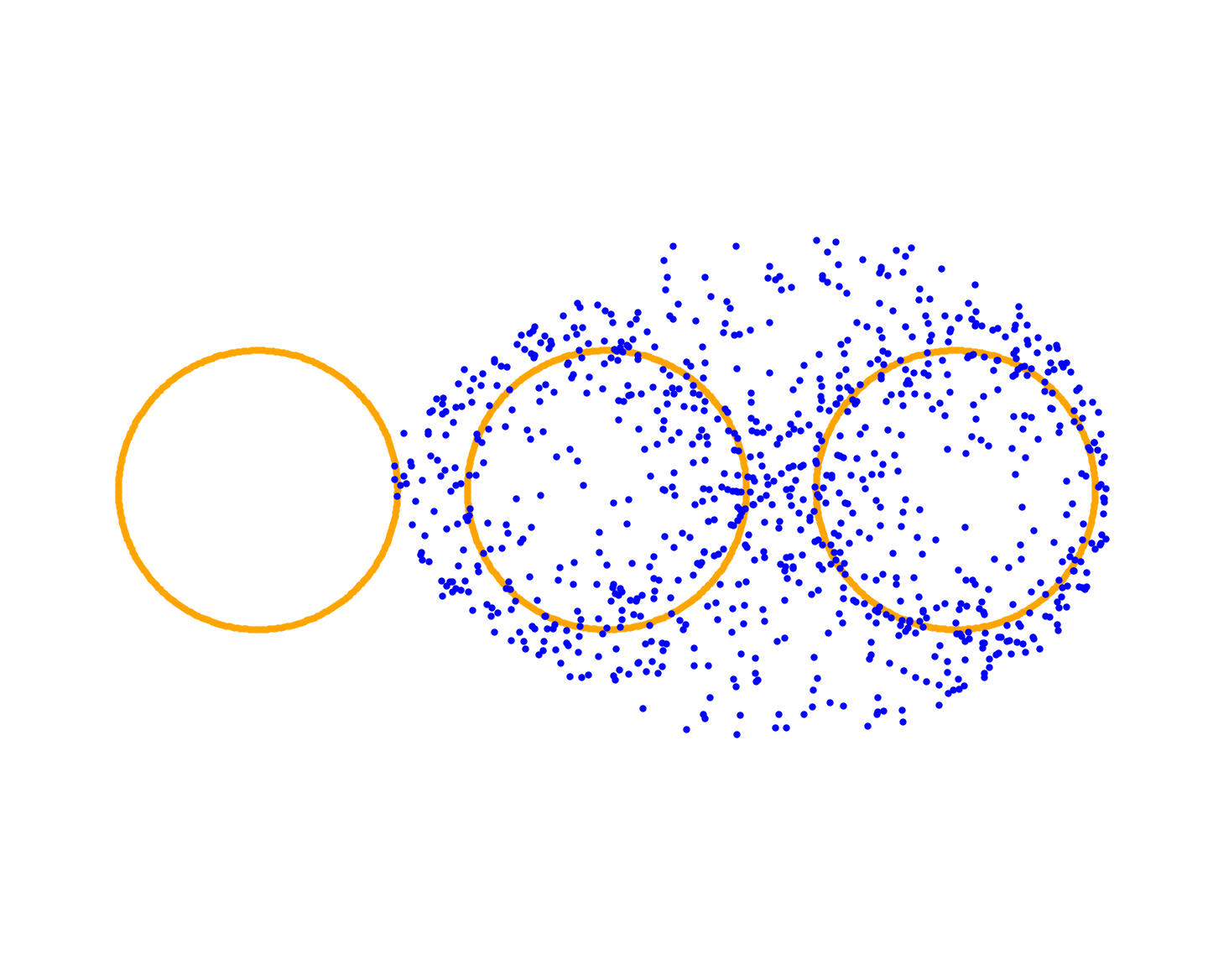}
        \includegraphics[width=\linewidth, valign=c]{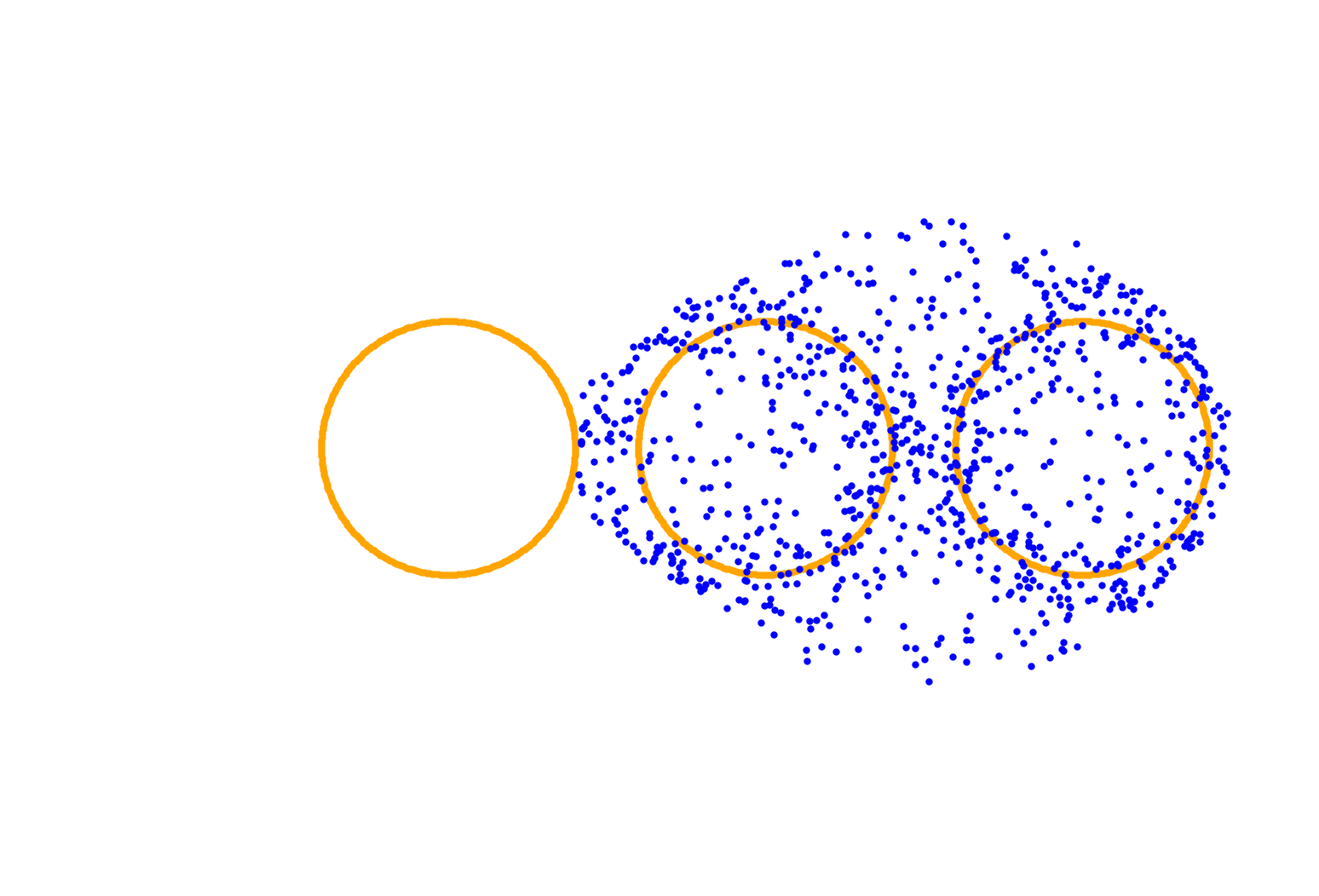}
    \caption*{t=1}
    \end{subfigure}%
    \begin{subfigure}[t]{.2\textwidth}
        \includegraphics[width=\linewidth, valign=c]{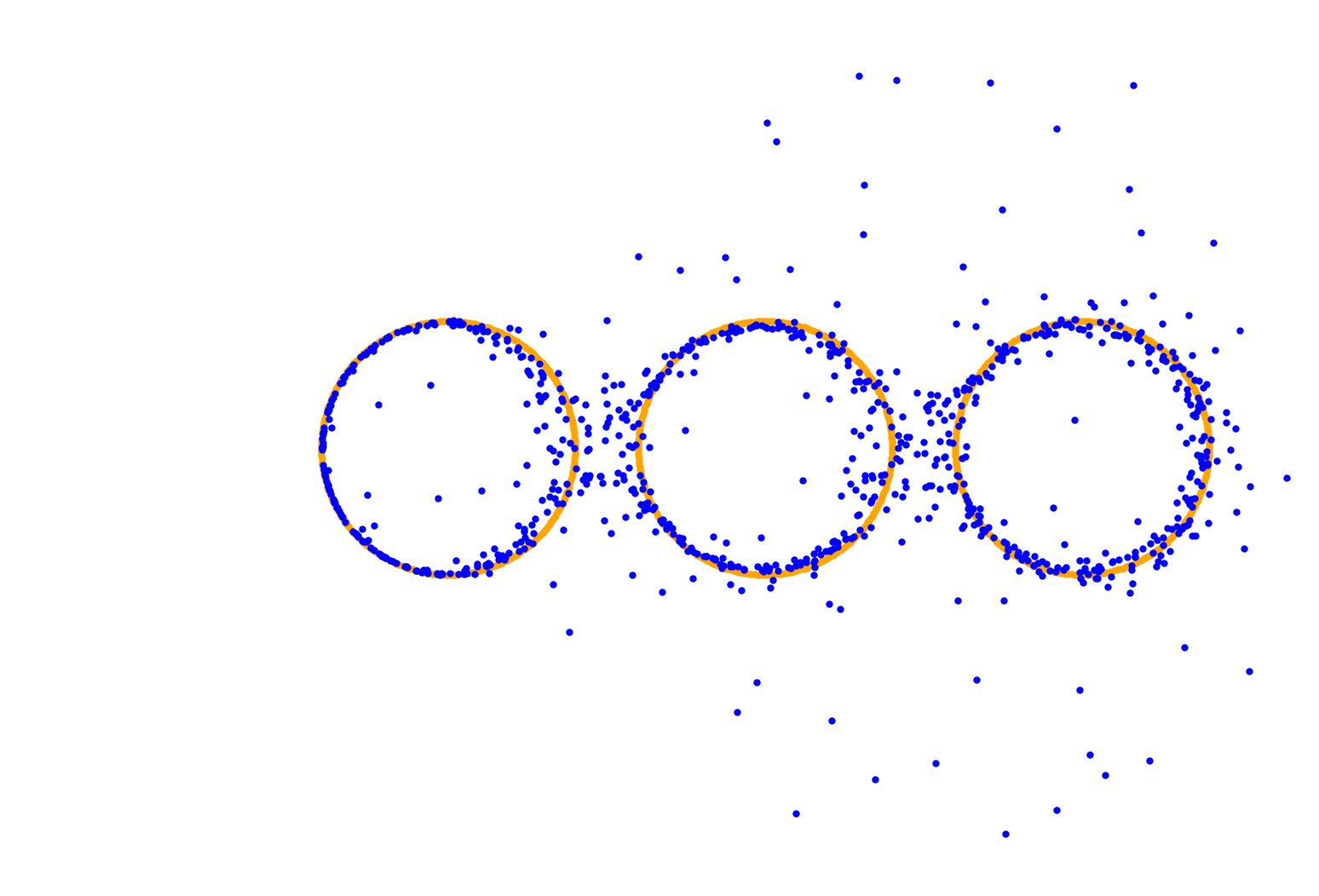}
        \includegraphics[width=\linewidth, valign=c]{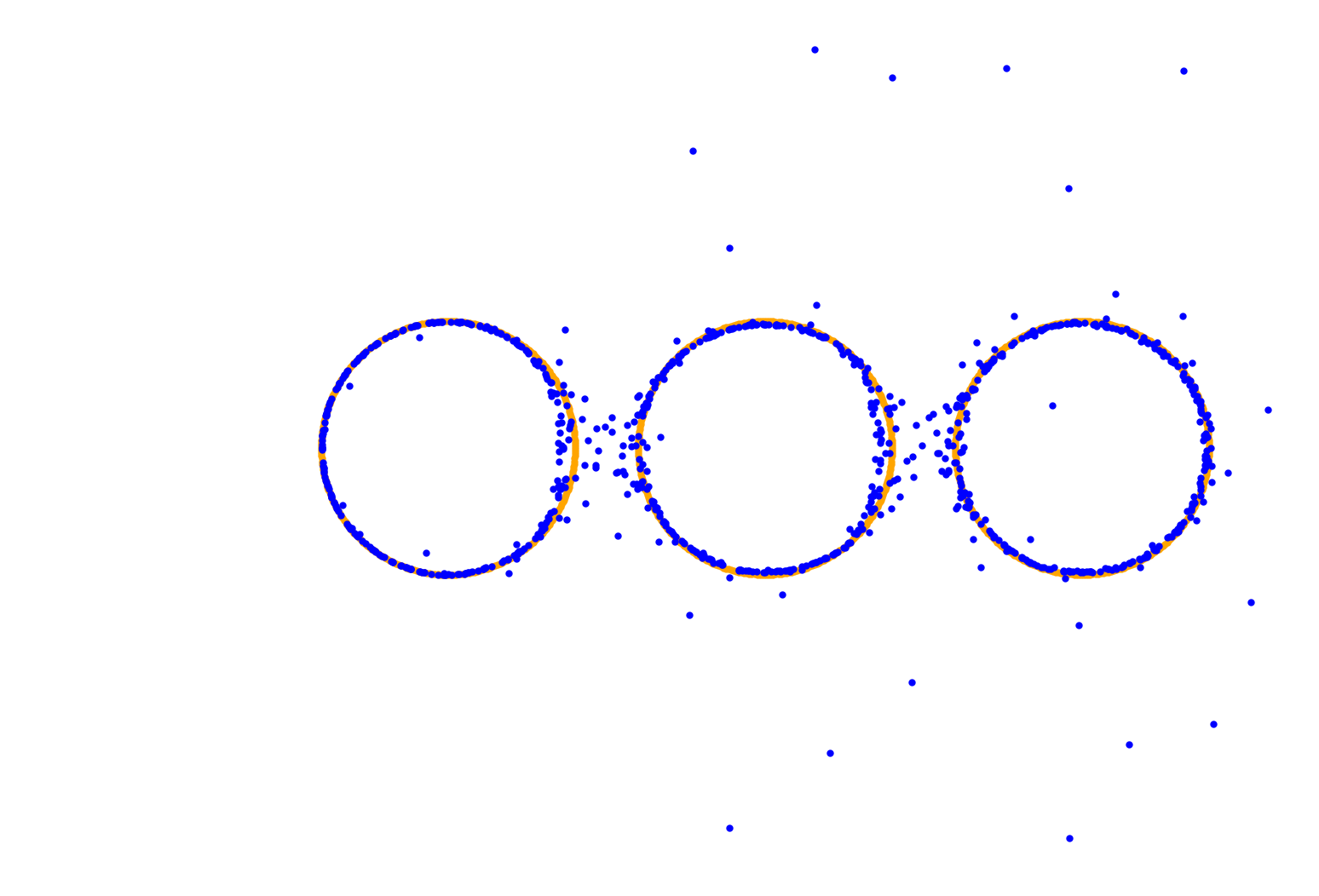}
    \caption*{t=10}
    \end{subfigure}%
    \begin{subfigure}[t]{.2\textwidth}
        \includegraphics[width=\linewidth, valign=c]{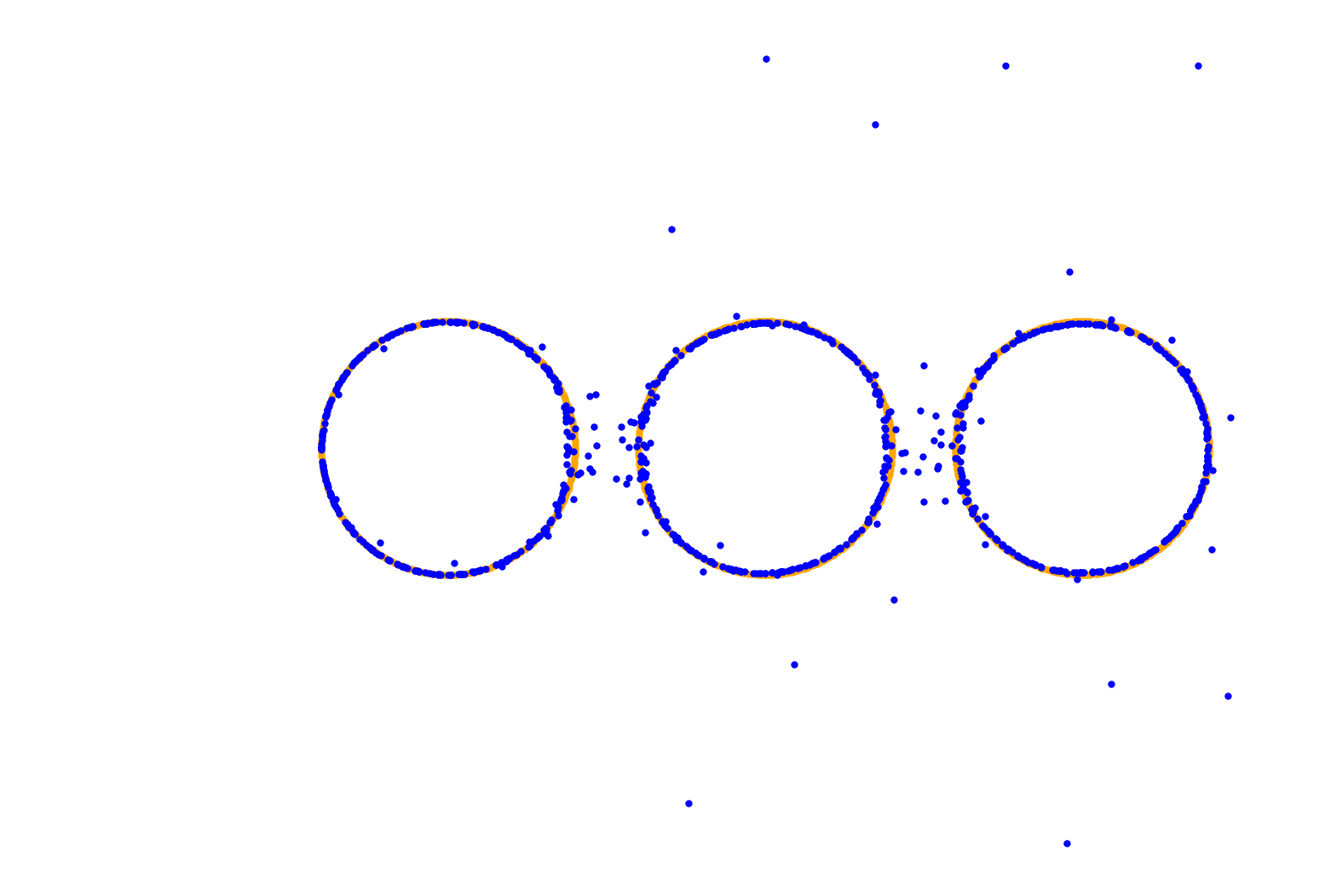}
        \includegraphics[width=\linewidth, valign=c]{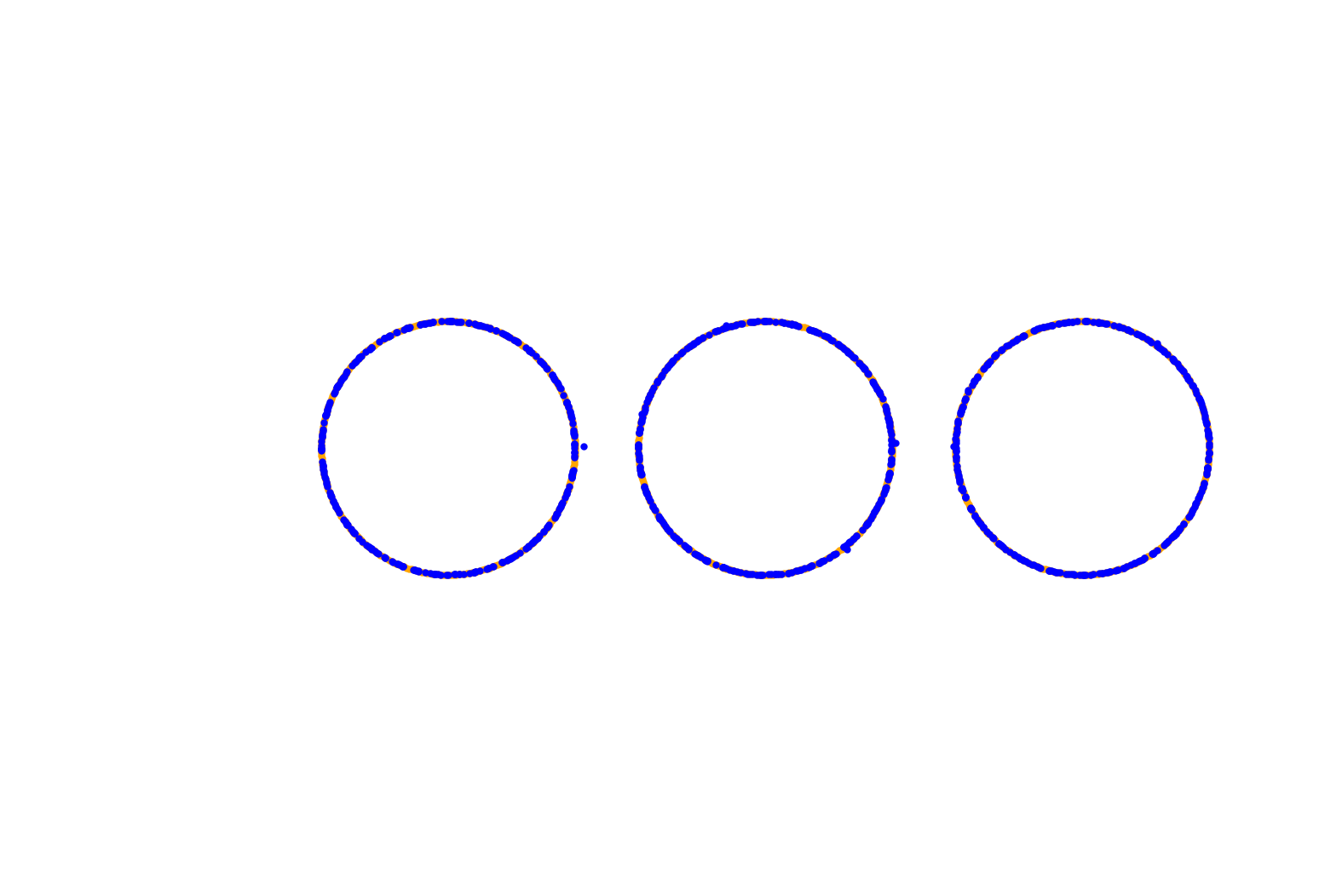}
    \caption*{t=50}
    \end{subfigure}%
    \begin{subfigure}[t]{.2\textwidth}
        \includegraphics[width=\linewidth, valign=c]{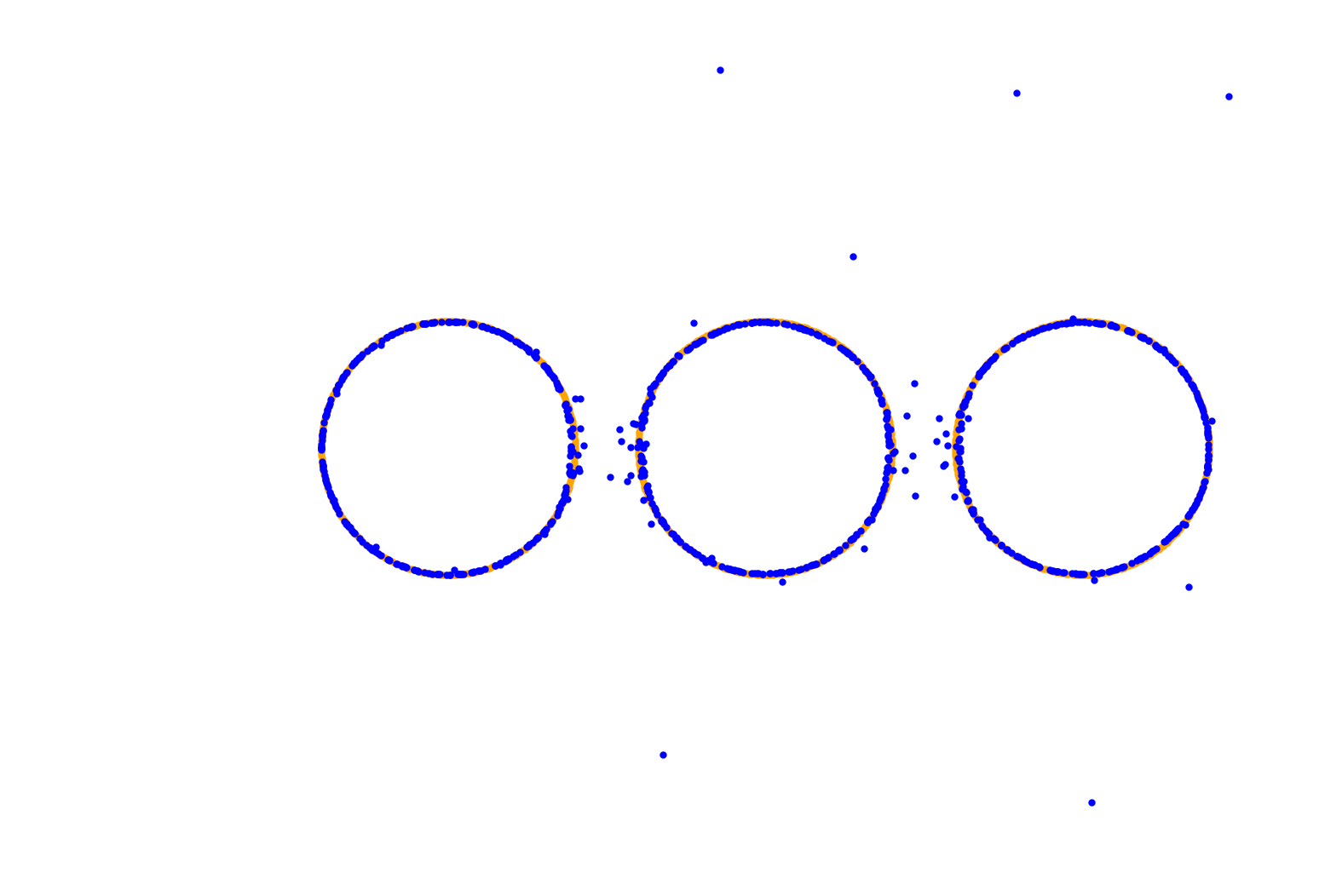}
        \includegraphics[width=\linewidth, valign=c]{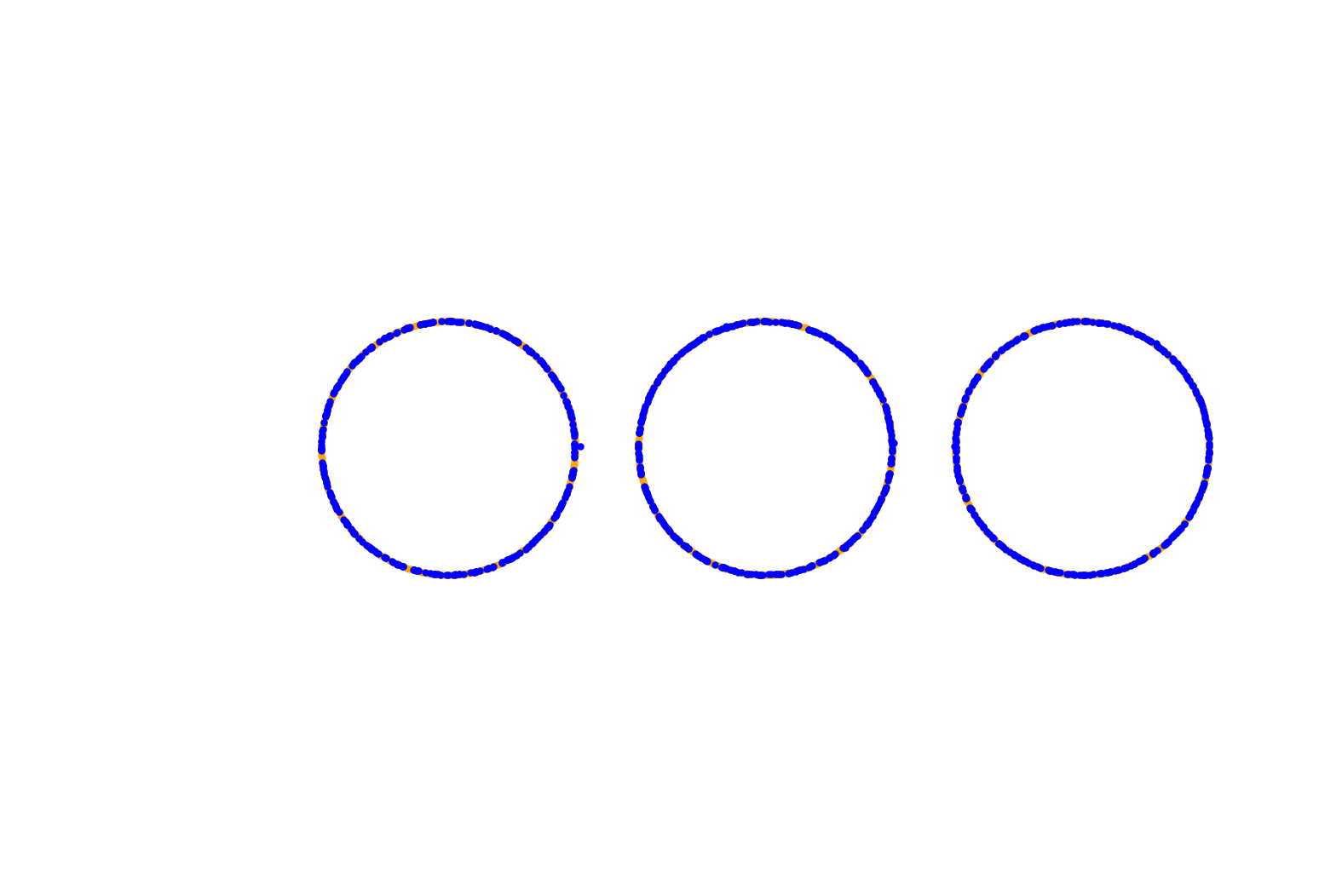}
    \caption*{t=100}
    \end{subfigure}%
    \caption{WGF of the regularized $\frac{1}{2}$-Tsallis divergence \smash{$D_{f_{\frac{1}{2}}, \nu}^{\lambda}$} without (top) and with annealing (bottom), where $\nu$ is the three rings target.}
    \label{fig:tsallis12flow}
\end{figure}
%--------------------------------------------------------------------------
\section{Conclusions and Limitations} \label{sec:conclusions}
%--------------------------------------------------------------------------
We have proven that the proposed MMD regularized $f$-divergences are $M$-convex along generalized geodesics, and calculated their gradients.
This ensures the existence and uniqueness of the associated Wasserstein gradient flows.
When considering particle flows for $f$-divergences with $f'_\infty = \infty$ or if the regularization parameter $\lambda$ is chosen in a certain way, each forward Euler step reduces to solving a finite-dimensional, strongly convex optimization problem.
Proving the observed sublinear convergence rate of the flows is subject to future work.

Further, we would like to extend this paper's theory to non-differen\-tiable kernels such as the Laplace kernel, i.e., the Matérn-$\frac{1}{2}$ kernel, and to non-bounded kernels like Coulomb kernels or Riesz kernels, which are of interest, e.g., in generative modeling \cite{HGBS2024,HWAH2024}
and image dithering \cite{EhlGraNeuSte21}.

Recently, we became aware of \cite{CMF20}, where they 
consider infimal convolution functionals as smoother loss functions in generative adversarial networks.
In particular, the infimal convolution of the MMD arising from the Gaussian kernel and convex functions appears promising.
Potentially, our results can contribute in this direction, which is also in the spirit of \cite{BDKPR23,FT2018,SE2020}.

\section*{Acknowledgments}
We thank the anonymous reviewer for pointing us to the tight variational formulation.
V.S.\ and G.S.\ gratefully acknowledge funding from the BMBF project \enquote{VI-Screen} with number 13N15754.
V.S.\ extends his gratitude to J.-F. Bercher for making the paper \cite{V1973} available.

%--------------------------------------------------------------------------
%--------------------------------------------------------------------------
\bibliographystyle{abbrv}
\bibliography{Bibliography}

\clearpage
\appendix

%-------------------------------------------------------------------------
\section{Entropy Functions and Their \texorpdfstring{$f$}{f}-Divergences} \label{app:ent}

In Table \ref{tab:entropyFunctions} and \ref{tab:fDivergences}, we give an extensive overview on entropy functions $f$ together with their recession constants, convex conjugates, and associated $f$-divergences.
The $(\star)$ marks that we were not able to find a closed form for the convex conjugate.

Let us briefly discuss some cases of the $f$-divergences below:
\begin{itemize}
    \item
    The Tsallis-2-divergence is also called $\chi^2$ or Pearson divergence \cite{P1900}.
    The Tsallis-$\frac{1}{2}$-divergence is equal to the Matusita-$\frac{1}{2}$-divergence and equal to half of the squared Hellinger distance \cite{H1909}.
    In the limit $\alpha \nearrow 1$, the Matusita entropy function converges to the TV entropy function.

    \item 
    For $\alpha \to 1$, both the power and the Tsallis divergence recover the Kullback-Leibler divergence.
    However, only the power divergence recovers the reverse Kullback-Leibler divergence for $\alpha \to 0$, while the Tsallis entropy function converges to 0.
    
    \item
    The $\frac{1}{2}$-Lindsay divergence is called triangular discrimination or Vincze-Le Cam \cite[Eq.~(2)]{V1981} divergence.
    The Lindsay divergence interpolates between the $\chi^{2}$ divergence and the reverse $\chi^{2}$ divergence, which are recovered in the limits $\alpha \searrow 1$ and ${\alpha \nearrow 0}$, respectively.
    In the limit ${\alpha \to 1}$, the perimeter-type divergence recovers the Jensen-Shannon divergence, and in the limit $\alpha \to 0$, it recovers $\frac{1}{2}$ times the TV-divergence.
    The special case $\alpha = \frac{1}{2}$ already appears in \cite[p.~342]{Oe1982}.
    
    \item
    For $\alpha \searrow 1$, Vadja's $\chi^{\alpha}$ divergence becomes the TV-divergence - the only (up to multiplicative factors) $f$-divergence that is a metric.
    
    \item
    The Marton divergence plays an essential role in the theory of concentration of measures; see \cite[Rem.~7.15]{PW2022} and the references therein.

    \item 
    If $f = \iota_{1}$, then $D_{f, \nu}^{\lambda} = \frac{1}{2 \lambda} d_K(\cdot, \nu)^2$ recovers the squared MMD.

    \item 
    Not all functions $f$ listed below fulfill the assumptions of this paper. 
    For example, we require that $f$ has a unique minimizer at $1$, which is not true for the last three entropy functions.
\end{itemize}\newpage

\begin{table}[H]
    \caption{Entropy functions with recession constants and conjugates.
    \label{tab:entropyFunctions}}
    {\scriptsize \rotatebox{90}{
       \begin{tabular}{@{}|c || l | c| c | c@{}} \toprule
       name & entropy $f(x)$ & $f_{\infty}'$ & $f^*(y)$ 
        \\ \midrule
        \begin{tabular}[x]{@{}c@{}} Tsallis \cite{T1998}, \\ $\alpha \in (0, 1) \cup (1, \infty)$ \end{tabular} & $\frac{1}{\alpha - 1} \left( x^{\alpha} - \alpha x + \alpha - 1 \right)$ & $\iota_{(0, 1)}(\alpha) - \frac{\alpha}{\alpha - 1}$ & $\left( \frac{\alpha - 1}{\alpha} y + 1\right)_{+}^{\frac{\alpha}{\alpha - 1}} - 1 + \1_{(0, 1)}(\alpha) \cdot \iota_{\left(- \infty, \frac{\alpha}{1 - \alpha}\right)}(y)$
        \\ \hline
        \begin{tabular}[x]{@{}c@{}} power divergence, \\ $\alpha \in \R \setminus \{ 0, 1 \}$ \cite{LMS2017} \end{tabular} & $\frac{1}{\alpha (\alpha - 1)} \left( x^{\alpha} - \alpha x + \alpha - 1 \right)$ & $
        \begin{cases}
            \infty, & \alpha \ge 1, \\
            \frac{1}{1 - \alpha}, & \alpha < 1
        \end{cases}$ & \begin{tabular}[x]{@{}c@{}} $\begin{cases} \frac{1}{\alpha} \left( (\alpha - 1) y + 1\right)_+^{\frac{\alpha}{\alpha - 1}} - \frac{1}{\alpha}, & \text{if } \alpha > 1, \\
            g_{\alpha}(y) + \iota_{\left(- \infty, \frac{1}{1 - \alpha}\right)}(y), & \text{if } 0 < \alpha < 1, \\ g_{\alpha}(y) + \iota_{\left(- \infty, \frac{1}{1 - \alpha}\right]}(y), & \text{if } \alpha < 0, \end{cases}$, \\ where $g_{\alpha}(y) \coloneqq \frac{1}{\alpha} \left( (\alpha - 1) y + 1\right)^{\frac{\alpha}{\alpha - 1}}
            - \frac{1}{\alpha}$ \end{tabular}
        \\
        \hline 
        \begin{tabular}[x]{@{}c@{}} Kullback-Leibler \\ \cite{KL1951} \end{tabular} & $x \ln(x) - x + 1$ & $\infty$ & $e^{y} - 1$
        \\
        \hline 
        Jeffreys \cite[Eq.~1]{J1946} & $(x - 1) \ln(x)$ & $\infty$ & $y + 2 + W_0(e^{1 - y}) + \frac{1}{W_0(e^{1 - y})} $\\ \hline
        \begin{tabular}[x]{@{}c@{}} Vajda's $\chi^{\alpha}$, \\ $\alpha > 1$ \cite{V1972,V1973} \end{tabular} & $| x - 1 |^{\alpha}$ & $\infty$ & $\begin{cases}
            y + (\alpha - 1) \left( \frac{| y |}{\alpha} \right)^{\frac{\alpha}{\alpha - 1}}, & \text{if } y \ge - \alpha, \\
            - 1, & \text{else.}
        \end{cases}$ \\ \hline
        \begin{tabular}[x]{@{}c@{}} equality indicator \\ \cite{BLNS2023,SFVTP2019} \end{tabular}
        & $\iota_{\{1\}}$ & $\infty$ & $y$
        \\ 
        \hline
        \begin{tabular}[x]{@{}c@{}} Jensen-Shannon \\ \cite{L1991} \end{tabular} & $x \ln(x) - (x + 1) \ln\left(\frac{x + 1}{2}\right)$  & $\ln(2)$ & $\iota_{(- \infty, \ln(2))}(y) - \ln(2 - e^{y})$.
        \\ \hline
        \begin{tabular}[x]{@{}c@{}} $\alpha$-Lindsay, \\ $\alpha \in [0, 1)$ \cite{L94} \end{tabular} & $\frac{(x - 1)^2}{\alpha + (1 - \alpha) x}$ & $\frac{1}{1 - \alpha}$ & $\begin{cases}
           \infty & \text{if } y > \frac{1}{1 - \alpha}, \\
           \frac{\alpha (\alpha - 1) y - 2 \sqrt{(\alpha - 1) y + 1} + 2}{(\alpha - 1)^2} & \text{else.}
        \end{cases}$
        \\ \hline
        \begin{tabular}[x]{@{}c@{}} Perimeter-type, \\ $\alpha \in \R \setminus \{ 0, 1 \}$ \\ \cite{Oe1996,OeV2003} \end{tabular}
        &  $\frac{\sgn(\alpha)}{1 - \alpha}\left( (x^{\frac{1}{\alpha}} + 1)^{\alpha} - 2^{\alpha - 1}(x + 1)\right)$ &
        $ \begin{cases}
            \ln(2),                                    & \alpha = 1, \\
            \frac{1}{1 - \alpha} (1 - 2^{\alpha - 1}), & \alpha > 0, \\
            \frac{1}{2},                               & \alpha = 0,  \\
            \frac{1}{1 - \alpha} 2^{\alpha - 1},       & \alpha < 0.
        \end{cases}$
        &
        \begin{tabular}[x]{@{}c@{}} $\begin{cases}\frac{y - \frac{\sgn(\alpha)}{1 - \alpha}\left(  h_{\alpha}(y)^{\frac{\alpha}{\alpha - 1}} - 2^{\alpha - 1} \right)}{\left( h_{\alpha}(y)^{\frac{1}{\alpha-1}} -1\right)^{\alpha}}
        + \frac{\sgn(\alpha)}{1 - \alpha} 2^{\alpha - 1}, & y < f_{\infty}', \\
        \frac{1}{1 - \alpha} (1 - 2^{\alpha - 1}) + \iota_{(- \infty, 0)}(\alpha), & y = f_{\infty}', \\
        \infty, & \text{else} \end{cases}$, \\ where $h_{\alpha}(y) \coloneqq \frac{1 - \alpha}{\sgn(\alpha)} y + 2^{\alpha - 1}$ \end{tabular}
        \\ \hline
        Burg \cite{KL1951}
        & $- \ln(x) + x - 1$ & $1$ & $-\ln(1 - y) + \iota_{(- \infty, 1)}(y)$ 
        \\ \hline
        \begin{tabular}[x]{@{}c@{}} Symmetrized Tsallis, \\ $s \in (0, 1)$ \cite{CF1962} \end{tabular} & $1 + x - (x^s + x^{1 - s})$ & 1 & $(\star)$ \\ \hline
        \begin{tabular}[x]{@{}c@{}} Matusita $\alpha \in (0, 1)$ 
        \\ \cite[Eq.~(4.1.30)]{B1977} \\ \cite[p.~158]{J1948} \end{tabular} & $|1 - x^{\alpha} |^{\frac{1}{\alpha}}$ & 1 & $\left( y - | y |\right)^{\frac{1}{1 - \alpha}}\left(1 - \sgn(y) | y |^{\frac{\alpha}{1 - \alpha}} \right)^{- \frac{1}{\alpha}} + \iota_{(- \infty, 1)}(y)$ \\ \hline 
        \begin{tabular}[x]{@{}c@{}} Kafka $\alpha \in (0, 1]$ \\\cite{PV1990}, \cite{K1991} \end{tabular} & $|1 - x |^{\frac{1}{\alpha}} (1 + x)^{\frac{\alpha - 1}{\alpha}}$ & 1 & $(\star)$ \\ \hline 
        \begin{tabular}[x]{@{}c@{}} total variation \\ \cite{BLNS2023,SFVTP2019} \end{tabular} & $|x-1|$ &1&  
        $\begin{cases} \max(-1,y) & \text{if } y \le 1, \\ 
        \infty & \text{otherwise } 
        \end{cases}$\\ \hline
        \begin{tabular}[x]{@{}c@{}} Marton \\ \cite[Rem.~7.15]{PW2022} \end{tabular} & $\max(0, 1 - x)^2$ & 0 & $\begin{cases} \infty, & \text{if } y > 0, \\ \frac{1}{4} y^2 + y, & \text{if } -2 \le y \le 0, \\ - 1 &\text{else.} \end{cases}$ 
        \\ \hline
        \begin{tabular}[x]{@{}c@{}} Hockey stick \\ \cite[Eq.~(67)]{SV2016} \end{tabular}
        & $\max(0, 1 - x)$ & 0 & $\begin{cases} \max(-1,y) & \text{if } y \le 0, \\ 
        \infty & \text{otherwise } 
        \end{cases}$
        \\ \hline
        zero \cite{BLNS2023,SFVTP2019}& $\iota_{[0,\infty)}$ & 0 & $\iota_{(-\infty,0)}(y)$\\\bottomrule
        \end{tabular}
        }}
\end{table} 

\begin{table}[th]
    \caption{The $f$-divergences belonging to the entropy functions from Table~\ref{tab:entropyFunctions}.
        \label{tab:fDivergences}}
    {\rotatebox{90}{\small
       \begin{tabular}{l  l} \toprule
        divergence $D_f$ & $D_f(\mu \mid \nu)$, where $\mu = \rho \nu + \mu^s$ 
        \\ \midrule
        Tsallis-$\alpha$, $\alpha \in (0, 1)$ & $\frac{1}{\alpha - 1}\left[ \int_{\R^d} p(x)^{\alpha} - 1 \d{\nu}(x) - \alpha \big(\mu(\R^d) - \nu(\R^d)\big) \right]$ 
        \\ 
        Tsallis-$\alpha$, $\alpha > 1$ & $\begin{cases}
            \frac{1}{\alpha - 1}\left[ \int_{\R^d} \rho(x)^{\alpha} - 1 \d{\nu}(x) - \alpha \big(\mu(\R^d) - \nu(\R^d)\big) \right], & \text{if } \mu^s = 0, \\
            \infty, & \text{else.}
        \end{cases}$ 
        \\ \hline
        power divergence, $\alpha < 1, \alpha \ne 0$ & 
        $\int_{\R^d} \frac{\rho(x)^{\alpha} - 1}{\alpha (\alpha - 1)} \d{\nu}(x) + \frac{1}{\alpha - 1} \big( \nu(\R^d) - \mu(\R^d) \big)$
        
        \\
        power divergence, $\alpha > 1 $ &
        $\begin{cases}
            \int_{\R^d} \frac{\rho(x)^{\alpha} - 1}{\alpha (\alpha - 1)} \d{\nu}(x) + \frac{1}{\alpha - 1} \big( \nu(\R^d) - \mu(\R^d) \big), & \text{if } \mu^s = 0, \\
            \infty, & \text{else.}
        \end{cases}$ \\ 
        \hline 
        Kullback-Leibler & $\begin{cases}
        \int_{\R^d} \rho(x) \ln\big(\rho(x)\big) \d{\nu}(x) - \mu(\R^d) + \nu(\R^d) & \text{if } \mu^s = 0, \\
        \infty & \text{else.}
        \end{cases}$ \\
        \hline 
        Jeffreys & $\begin{cases}
            \int_{\R^d} \big( \rho(x) - 1 \big) \ln(\rho(x)) \d{\nu}(x), & \text{if } \mu^s = 0 \text{ and } \nu(\{ \rho = 0 \}) = 0, \\
            \infty, & \text{else}
        \end{cases}$ \\ \hline
        Vajda's $\chi^{\alpha}$, $\alpha > 1$ & $\begin{cases} \int_{\R^d} | \rho(x) - 1 |^{\alpha} \d{\nu}(x), & \text{if } \mu^s = 0, \\
        \infty, & \text{else.}
        \end{cases}$ \\ \hline
        equality indicator & $\begin{cases} 0 & \text{if } \mu = \nu, \\ 
        \infty & \text{otherwise } 
        \end{cases}$
        \\ \hline
        Jensen-Shannon & $\int_{\R^d} \rho(x) \ln\big(\rho(x)\big) - \big( \rho(x) + 1 \big) \ln\left( \frac{1}{2} \big( \rho(x) + 1 \big) \right) \d{\nu}(x) + \ln(2) \mu^s(\R^d)$ 
        \\ \hline
        $\alpha$-Lindsay, $\alpha \in [0, 1)$ & $\int_{\R^d} \frac{\big( \rho(x) - 1 \big)^2}{\alpha + (1 - \alpha) \rho(x)} \d{\nu}(x) + \frac{1}{1 - \alpha} \mu^s(\R^d)$ \\ \hline
        Perimeter-type, $\alpha \in \R \setminus \{ 0, 1 \}$ & $\frac{\sgn(\alpha)}{1 - \alpha} \left[ \int_{\R^d} \left( \rho(x)^{\frac{1}{\alpha}} + 1 \right)^{\alpha} \d{\nu}(x) - 2^{\alpha - 1}\big( \mu(\R^d) + \nu(\R^d) \big) + \1_{(0, \infty)}(\alpha) \mu^s(\R^d) \right]$
        \\ \hline
        Burg & $\begin{cases} \int_{\R^d} \ln\big(\rho(x)\big) \d{\nu}(x) + \mu(\R^d) - \nu(\R^d), & \text{if } \nu(\{ \rho = 0 \}) = 0, \\ \infty, & \text{else.} \end{cases}$.
         \\ \hline
        Symmetrized Tsallis & $\mu(\R^d) + \nu(\R^d) - \int_{\R^d} \rho(x)^s + \rho(x)^{1 - s} \d{\nu}(x)$ \\ \hline
        Matusita & $\int_{\R^d} \left| 1 - \rho(x)^{\alpha} \right|^{\frac{1}{\alpha}} \d{\nu}(x) + \mu^s(\R^d)$ \\ \hline
        Kafka & $\int_{\R^d} \big| 1 - \rho(x) \big|^{\alpha} \big(1 + \rho(x)\big)^{\frac{\alpha - 1}{\alpha}} \d{\nu}(x) + \mu^s(\R^d)$ \\ \hline
        total variation & $\int_{\R^d} | \rho(x) - 1 | \d{\nu}(x) + \mu^s(\R^d)$ \\ \hline
        Marton & $\frac{1}{2} \int_{\R^d} | \rho(x) - 1 | \d{\nu}(x) + \frac{1}{2}\big( \nu(\R^d) - \mu(\R^d) \big)$
        \\ \hline
        Hockey stick & $\int_{\R^d} (1 - \rho(x))_+ \d{\nu}(x) $\\ \hline
        zero & $0$ \\ \bottomrule
        \end{tabular}
        }
        }
\end{table}

\begin{table}[th]
    \centering
    \begin{tabular}{@{}cl} \toprule
        $f$                 & $\prox_{\lambda f}(x)$ \\ \midrule
        $f_{\KL}$           & $\lambda W\left(\frac{1}{\lambda} \exp\left(\frac{x}{\lambda}\right)\right)$ \\ \
        Tsallis-2           & $\frac{1}{2 \lambda + 1}(2 \lambda + x)_+$ \\ \
        Burg                & $\frac{1}{2}\left(x - \lambda + \sqrt{(x - \lambda)^2 + 4 \lambda}\right)$ \\ \
        TV                  & $\begin{cases} S_{\lambda}(x - \lambda), & \text{if } x \ge - \lambda, \\ 0, & \text{else.} \end{cases}$ \\ \
        Marton              & $\begin{cases} \frac{1}{1 + 2 \lambda} (x + 2 \lambda)_+, & \text{if } x \le 1, \\ x, & \text{else.}  \end{cases}$ \\ \
        equality indicator  & 1 \\ \
        zero                & $(x)_+$ \\ \bottomrule
    \end{tabular}
    \caption{Closed forms for proximal operators of some entropy functions.
    Here, $\lambda > 0$, $S_{\lambda}$ denotes the soft-shrinkage operator, $W$ is the (principal branch of the) Lambert W function \cite{CGHJK1996}, and $(\cdot)_+$ the non-negative part of a real number.}
    \label{table:proxies}
\end{table}

\newpage 
\section{Regularization of the Tight Variational Formulation} \label{sec:tight_formulation}
The tight variational formulation of $f$-divergences \cite{RRGDP12,NWJ2007,T2021} is derived as the biconjugate of the restricted objective
\begin{equation}\label{eq:TightFDiv}
    \tilde{D}_{f, \nu} \colon \M(\R^d) \to [0, \infty], \qquad
    \mu \mapsto
    D_{f, \nu}(\mu) + \iota_{\P(\R^d)}(\mu).
\end{equation}
This indeed leads to another dual formulation of $D_{f, \nu}$ for $\mu,\nu \in \mathcal P(\R^d)$, namely
\begin{equation}\label{eq:TightDual}
    D_{f, \nu}(\mu)
    = \sup_{g \in Y} \{ \E_{\mu}[g] - (\tilde{D}_{f, \nu})^*(g) \},
\end{equation}
where $Y$ denotes the space of bounded measurable functions.
By \cite[Prop.~28]{AH2021}, the conjugate $(\tilde{D}_{f, \nu})^*$ is determined by the one-dimensional convex problem
\begin{equation}\label{eq:TightConj}
    (\tilde{D}_{f, \nu})^*(g)
    = \inf\left\{ \int_{\R^d} f^*(g(x) + \lambda) \d{\nu}(x) - \lambda: \lambda + \sup(g) \le f_{\infty}' \right\}.
\end{equation}
This is different than the closed form \eqref{eq:gnuf*} in the non-tight case.
An algorithm to solve \eqref{eq:TightConj} is given in \cite[Algo.~1]{T2021}.
Sometimes, as for $f = f_{\KL}$, the functional $(\tilde{D}_{f, \nu})^*(g)$ is available in closed form: $(\tilde{D}_{f_{\KL}, \nu})^*(g) = \log\left(\int_{\R^d} e^{g} \d{\nu}\right)$.
The dual \eqref{eq:TightDual} has proven advantageous towards estimating $D_{f, \nu}$ for $\mu,\nu \in \mathcal P(\R^d)$ from samples since the objective in \eqref{eq:TightDual} is always larger or equal than the one in the original dual \eqref{eq:fDiv_1}.

Given the dual form \eqref{eq:RegFDivDual} of regularized $f$-divergences, it appears interesting to restrict the dual \eqref{eq:TightDual} to $\H_K$ and adding the regularizer ${- \frac{\lambda}{2} \| \cdot \|_{\H_K}^2}$, leading to
\begin{equation}\label{eq:TightDualReg}
    \tilde{D}_{f, \nu}^{\lambda}(\mu)
    \coloneqq \max_{h \in \H_k}\Bigl\{ \E_{\mu}[h] - (\tilde{D}_{f, \nu})^*(h) - \frac{\lambda}{2} \| h \|_{\H_K}^2\Bigr\}, \qquad \mu \in \P(\R^d).
\end{equation}
In contrast to the dual \eqref{eq:RegFDivDual}, there is no constraint on $h$.
However, for every $h \in \H_K$, we have to solve \eqref{eq:TightConj} to obtain $(\tilde{D}_{f, \nu}^{\lambda})^*(h)$.
For discrete  $\mu, \nu \in \mathcal P(\R^d)$, the objective \eqref{eq:TightDualReg} already appears in \cite[Eq.~(7)]{RRGDP12}.
In the following subsection we further analyze this \emph{tight} formulation of the MMD-regularized $f$-divergence.

\subsection{Theory}
First, to derive a primal formulation, we prove $\tilde G_{f,\nu}\in \Gamma_0(\mathcal H_K)$ for the extended functional $\tilde G_{f,\nu}$.

\begin{lemma}\label{lem:tildeG}
    Let $f'_\infty=\infty$.
    Then, the function
     \begin{equation*}
        \tilde{G}_{f, \nu} \colon \H_K \to [0, \infty], \qquad
        h \mapsto \begin{cases}
            {D}_{f, \nu}(\mu), & \text{if } \exists \mu \in \P(\R^d) \text{ such that } h = m_{\mu}, \\ 
            \infty, & \text{else,}
        \end{cases}
    \end{equation*}
    is lower semicontinuous.
    In particular, we have $\tilde G_{f,\nu}\in \Gamma_0(\mathcal H_K)$.
\end{lemma}

\begin{proof}
We follow the proof of Lemma~\ref{lemma:Gfnu_lsc}.
For $(h_n)_{n} \subset \H_{K}$ with $h_n \to h \in \H_{K}$, we need to show ${\tilde G_{f,\nu}}(h)\le \liminf_{n\to \infty} {\tilde G_{f,\nu}}(h_n)$.
If $\liminf_{n\to \infty} {\tilde G_{f,\nu}}(h_n) = +\infty$, we are done.
For $\liminf_{n\to \infty} {\tilde G_{f,\nu}}(h_n) < +\infty$, we pass to a subsequence without renaming such that ${\tilde G_{f,\nu}}(h_n)<\infty$ for all $n\in \mathbb{N}$ that realizes the $\liminf$.
Then, there exist $\mu_n\in \P(\R^d)$ with $m_{\mu_n}=h_n$.

Since $f'_\infty = \infty$ and $D_f(\mu_n \,| \,\nu)=\tilde G_{f,\nu}(h_n)<\infty$, we have $\mu_n\ll \nu$ for all $n\in \N$.
Since $f$ is an entropy function with $f_{\infty}' = \infty$, its conjugate $f^*$ is continuous and strictly monotone on $(0,\infty) \subset \dom(f^*)$ with $f^*(0)=0$ and $f^*(t) \to \infty$ for $t\to \infty$ by \cite[Sec.~2.3]{LMS2017}.
For the closed Euclidean ball $B_r \subset \R^d$ around $0$ with radius $r$, we define
\begin{align}
    g \colon (0, \infty) \to [0, \infty], \qquad
    r \mapsto \begin{cases} (f^*)^{-1}(\nu(\R^d\setminus B_r)^{-1})& \text{ if } \nu(B_r)\neq 1,\\
    +\infty & \text{ else.}\end{cases}
\end{align}
Due to \cite[Eq.~(2.20)]{LMS2017} we have $\lim_{r\to \infty} g(r)=\infty$.
The estimate becomes trivial when $\nu(B_r)=1$, hence we assume $\nu(B_r)<1$.
For any $r>0$, we estimate
\begin{align*}
\tilde G_{f,\nu}(h_n)&
= \int_{\R^d} f\circ \frac{\d\mu_n}{\d\nu}\d\nu 
\ge \int_{\R^d\setminus B_r} f\circ \frac{\d\mu_n}{\d\nu}-g(r)\frac{\d\mu_n}{\d\nu}\d\nu +g(r)\mu_n(\R^d \setminus B_r) \\&
\ge \int_{\R^d\setminus B_r} -f^*(g(r))\d\nu +g(r)\mu_n(\R^d \setminus B_r) 
\ge -1 + g(r)\mu_n(\R^d \setminus B_r).
\end{align*}
%where we used the conventions $f^*(\infty) \coloneqq \infty$ and $0 \cdot \infty \coloneqq 0$.
Since $\liminf_{n\to \infty} {\tilde G_{f,\nu}}(h_n) < +\infty$, the sequence $(\tilde G_{f,\nu}(h_n))_n$ is bounded by some $B>0$. 
To show that $(\mu_n)_n$ is tight, let $\varepsilon>0$ and choose $r$ large enough to satisfy $\frac{B+1}{\varepsilon}<g(r)$. 
Then, we can estimate with the convention $\frac{B+1}{\infty}=0$ that
\begin{align*}
    \varepsilon
    > \frac{B+1}{g(r)}
    \ge \frac{\tilde G_{f,\nu}(h_n)+1}{g(r)}
    \ge %\frac{g(r)\mu_n(\R^d\setminus B_r)}{g(r)}=
    \mu_n(\R^d\setminus B_r).
\end{align*}
Hence, $(\mu_n)_n$ is tight and by \cite[Thm.~5.1]{B2013} it has a weakly convergent subsequence $\mu_{n_k} \weakly \mu\in \mathcal P(\R^d)$.
Similarly to \eqref{eq:weak_conv}, we obtain $h=m_\mu$.
Using Lemma \ref{lem:ApproxC0}, we get that $\tilde G_{f,\nu}$ is lower semicontinuous.
Since $\P(\R^d)$ and $D_{f,\nu}$ are convex, we finally get $\tilde G_{f,\nu}\in \Gamma_0(\mathcal H_K)$.
\end{proof}
Now, we can show that the primal formulation of $\tilde{D}_{f, \nu}^{\lambda}$ is like \eqref{eq:RegFDiv}, but with feasible set $\P(\R^d)$ instead of $\M_+(\R^d)$.

\begin{lemma}[Primal formulation of $\tilde{D}_{f, \nu}^{\lambda}$]
Assume that $f'_\infty = \infty$.
Let $\lambda > 0$ and $\mu,\nu \in \P(\R^d)$.
Then we have
\begin{equation}
    \tilde{D}_{f, \nu}^{\lambda}(\mu)
    = \min_{\sigma \in \P(\R^d)}\Bigl\{ D_{f, \nu}(\sigma) + \frac{1}{2 \lambda} d_K(\sigma, \nu)^2 \Bigr\}
    \label{eq:RegfuncTight}
\end{equation}
and the relationship with the maximizer in \eqref{eq:TightDualReg} is $m_{\hat{\sigma}} = m_{\mu} - \lambda \hat{h}$.
Moreover, we have
\begin{equation}\label{eq:dualtight}
    \frac{\lambda}{2} \| \hat{h} \|_{\H_K}^2
    \le \tilde{D}_{f, \nu}(\mu)
    \le \| \hat{h} \|_{\H_K} d_K(\mu, \nu).
\end{equation}
\end{lemma}
\begin{proof}
We follow the proof of Corollary~\ref{lemma:RegFDivDual}.
The Moreau envelope $\tilde{G}_{f, \nu}^{\lambda}$ of $\tilde{G}_{f, \nu}$ is
    \begin{align*}
        \tilde{G}_{f, \nu}^{\lambda}(m_{\mu})
        & = \min_{\sigma \in \P(\R^d)} \Bigl\{ \tilde{D}_{f, \nu}(\sigma) + \frac{1}{2 \lambda} d_K(\sigma,\mu)^2 \Bigr\}.
    \end{align*}
    Moreover, for $p \in \H_K$, we have
    \begin{equation*}
        (\tilde{G}_{f, \nu})^*(p)
        = \sup_{h \in \H_K} \langle p, h \rangle_{\H_K} - \tilde{G}_{f, \nu}(h)
        = \sup_{\mu \in \P(\R^d)} \langle p, \mu \rangle_{\C_0 \times \M} - \tilde{D}_{f, \nu}(\mu)
        = (\tilde{D}_{f, \nu})^*(p).
    \end{equation*}
    Hence, incorporating Theorem~\ref{prop:moreau}, $\tilde{G}_{f, \nu}$ can be written as
    \begin{align*}
        \tilde{G}_{f, \nu}^{\lambda}(m_{\mu})
        & = \max_{p \in \H_K} \Bigl\{ \E_{\mu}[p] - (\tilde{G}_{f, \nu})^*(p) - \frac{\lambda}{2} \| p \|_{\H_K}^2 \Bigr\} \\
        & = \max_{p \in \H_K} \Bigl\{ \E_{\mu}[p] - (\tilde{D}_{f, \nu})^*(p) - \frac{\lambda}{2} \| p \|_{\H_K}^2 \Bigr\} = \tilde{D}_{f, \nu}^{\lambda}(\mu).
    \end{align*}

    Both the relationship between the optimizers and the lower bound in \eqref{eq:dualtight} follow as in Corollary~\ref{lemma:RegFDivDual}.
    For the upper bound, note that for $g \in \H_K$, we have
    \begin{equation*}
        \tilde{G}_{f, \nu}^*(g)
        = \sup_{h \in \H_K} \langle g, h \rangle - \tilde{G}_{f, \nu}(h)
        \ge \langle g, m_{\nu} \rangle -  \tilde{G}_{f, \nu}(m_{\nu})
        = \E_{\nu}[g]
    \end{equation*}
    and thus
    \begin{equation*}
        \tilde{D}_{f, \nu}^{\lambda}(\mu)
        = \E_{\mu}[\hat{h}] - (\tilde{G}_{f, \nu})^*(\hat{h}) - \frac{\lambda}{2} \| \hat{h} \|_{\H_K}^2
        \le \E_{\mu}[\hat{h}] - \E_{\nu}[\hat{h}] \leq \Vert \hat h\Vert_{\H_{K}} d_K(\mu, \nu).
    \end{equation*}
\end{proof}

The primal form \eqref{eq:RegfuncTight} appears in \cite[Thm.~6]{BDKPR23} for integral probability metrics instead of just MMDs.
However, their restrictions on $f$ are stronger than ours.
\begin{remark}[Difference between tight and non-tight formulation]
    The difference between $D_{f, \nu}^{\lambda}$ and $\tilde{D}_{f, \nu}^{\lambda}$ becomes obvious for $\nu = \delta_y$ with $y \in \R^d$ and $\mu = \delta_x$ with $x \in \R^d \setminus \{ y \}$.
    If $f_{\infty}' = \infty$, then we have
    \begin{equation*}
        \tilde{D}_{f, \delta_y}^{\lambda}(\delta_x)
        = D_{f, \delta_y}(\delta_y) + \frac{1}{2 \lambda} d_K(\delta_x, \delta_y)
        = \frac{1}{2 \lambda} \bigl( K(x, x) + K(y, y) - 2 K(x, y)\bigr),
    \end{equation*}
    whereas
    \begin{equation*}
       {D}_{f, \delta_y}^{\lambda}(\delta_x) = \min_{a \ge 0} f(a) + \frac{1}{2 \lambda} \big( a^2 K(y, y) + K(x, x) - 2 a K(x, y) \big).
    \end{equation*}
    For the Tsallis-$2$ entropy function $f = f_2 = (\cdot - 1)^2$ we get
    \begin{equation*}
        \argmin_{a \ge 0} f(a) + \frac{1}{2 \lambda} d_K(a \delta_x, \delta_y)^2
        = \frac{2 \lambda + K(x, y)}{2 \lambda + K(y, y)}.
    \end{equation*}
    Since the radial kernel $K$ is a decreasing function of the squared norm, the optimal $a^* = K(x, y)$ is in general smaller than $1$.
    If, additionally we require that $K(x, x) = 1$ for all $x \in \R$, then we have the optimal value
    \begin{equation*}
        D_{f, \delta_y}^{\lambda}(\delta_x)
        = \frac{1 - K(x, y)}{\lambda} \frac{K(x, y) + 4 \lambda + 1}{2 (2 \lambda + 1)}.
    \end{equation*}
\end{remark}

The theoretical justification for the existence of Wasserstein gradient flows based on Theorem~\ref{ambrosio} uses Theorem~\ref{thm:M} and Lemma~\ref{lemma:RegFDivDerivative}, which only require the dual estimates \eqref{some_est} and not the dual \eqref{eq:RegFDivDual} itself.
Hence, the argumentation remains valid for the tight formulation \eqref{eq:TightDualReg}.
In Lemma~\ref{lemma:reprThmForFiniteRecConst}, we showed that the dual representation \eqref{eq:TightDualReg} has a finite representation even for $f'_\infty<\infty$.
Unfortunately, this result does not generalize straightforwardly.

\subsection{Numerical Implementation Details}
Following Section \ref{sec:discr}, we propose an Euler forward scheme for the Wasserstein-2 gradient flow of $\tilde{D}_{f, \nu}^{\lambda}$ for $\nu = \frac{1}{M} \sum_{j = 1}^{M} \delta_{y_j}$ and $f_{\infty}' = \infty$.
The primal problem \eqref{eq:RegfuncTight} is finite-dimensional since $\sigma \in \P(\R^d)$ must be absolutely continuous with respect to $\nu$.
In particular, only $M$ weights have to be optimized in \eqref{eq:gn} and $\hat{\sigma}_n(\R^d) = \sum_{k = 1}^{M + N} \beta_k^{(n)} = 1$.
Thus, we have $\beta_k^{(n)} = 0$, $k \in \{ M + 1, \ldots, M + N \}$, i.e.,
\begin{equation}\label{eq:ConstrDual}
    b_k^{(n)} = \frac{1}{\lambda N}, \quad k = M + 1, \ldots, M + N,
    \quad \text{and} \quad
    \sum_{k = 1}^{M} b_k^{(n)} = - \frac{1}{\lambda}.
\end{equation}
This can be reformulated as
\begin{equation}\label{eq:SimpConstr}
    q_{M + j}^{(n)} = - \frac{M}{N},\, j = 1, \ldots, N, \quad q_k^{(n)} \ge 0,\, k = 1, \ldots, M, \quad
    \text{and }\sum_{k = 1}^{M} q_k^{(n)} = M.
\end{equation}
This means that $\frac{1}{M} (q_k^{(n)})_{k = 1}^{M}$ is an element of the unit simplex in $\R^M$.
We minimize the objective $J$ from \eqref{eq:primalObjective2} subject to \eqref{eq:SimpConstr} with exponentiated gradient descent, that is, mirror descent where the mirror map is the negative Shannon entropy.
We use an Armijo line search or Polyak rule \cite[Algorithm~3]{YL22} for determining the step size in each iteration, which leads to convergence if $J$ is differentiable \cite{LC2019}.
Both approaches only require one gradient per iteration. 
The advantage of the Polyak step size is that it also only requires one function evaluation.
Mirror descent works relatively well in high dimension \cite{BT2003} and is parallelizable on the GPU.
In our simulations, the gradient flow for $\tilde{D}_{f, \nu}^{\lambda}$ tends to stay closer to the target and has a lower number of outliers; see Figure~\ref{fig:tight_vs_nontight}.
Although the differences are most pronounced for the bananas target, a similar effect occurs for the others.

\begin{figure}[tbp]
    \centering
    \begin{subfigure}[t]{.245\textwidth}
      \includegraphics[width=\linewidth, valign=c]{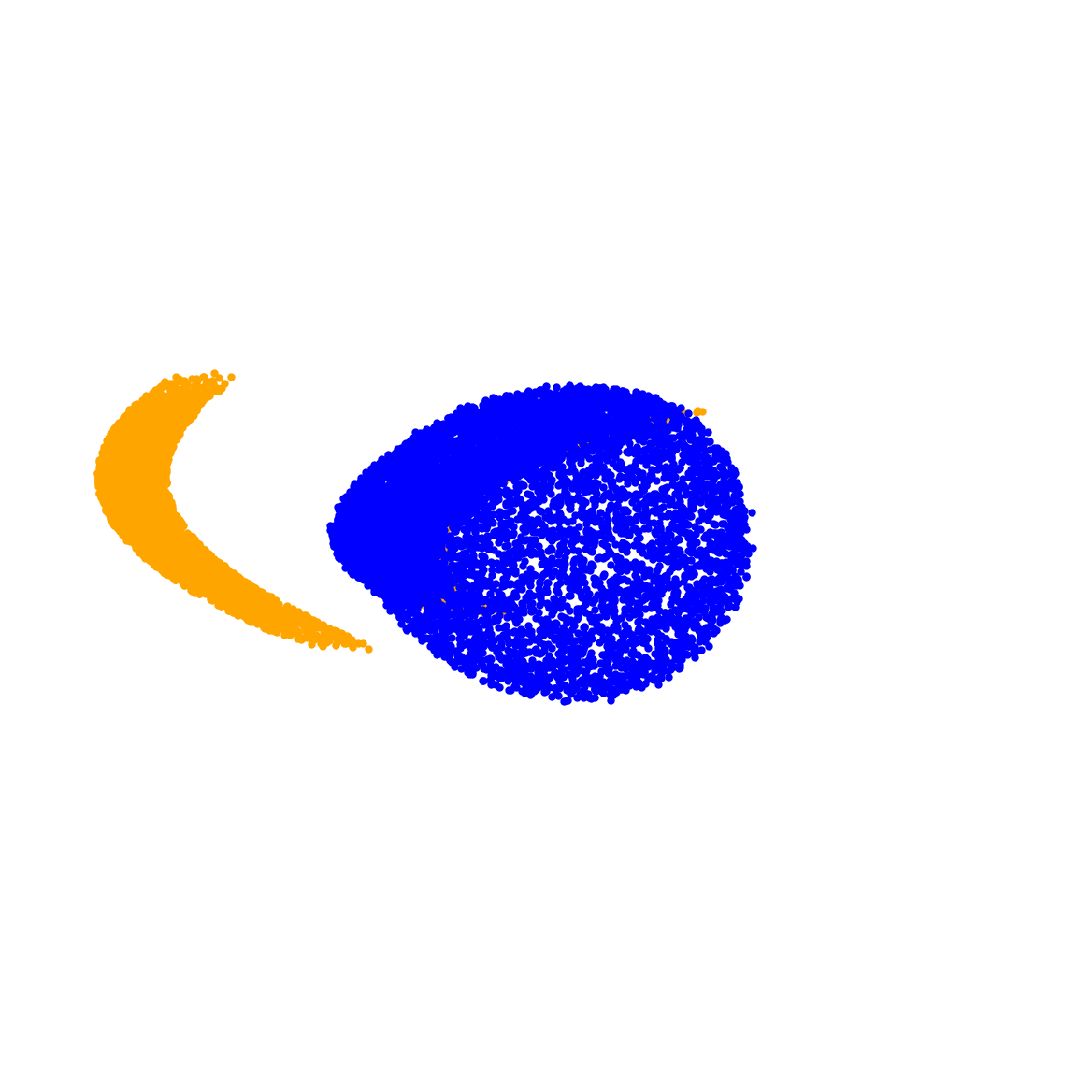}
    \caption*{t=0.1}
    \end{subfigure}%
    \begin{subfigure}[t]{.245\textwidth}
      \includegraphics[width=\linewidth, valign=c]{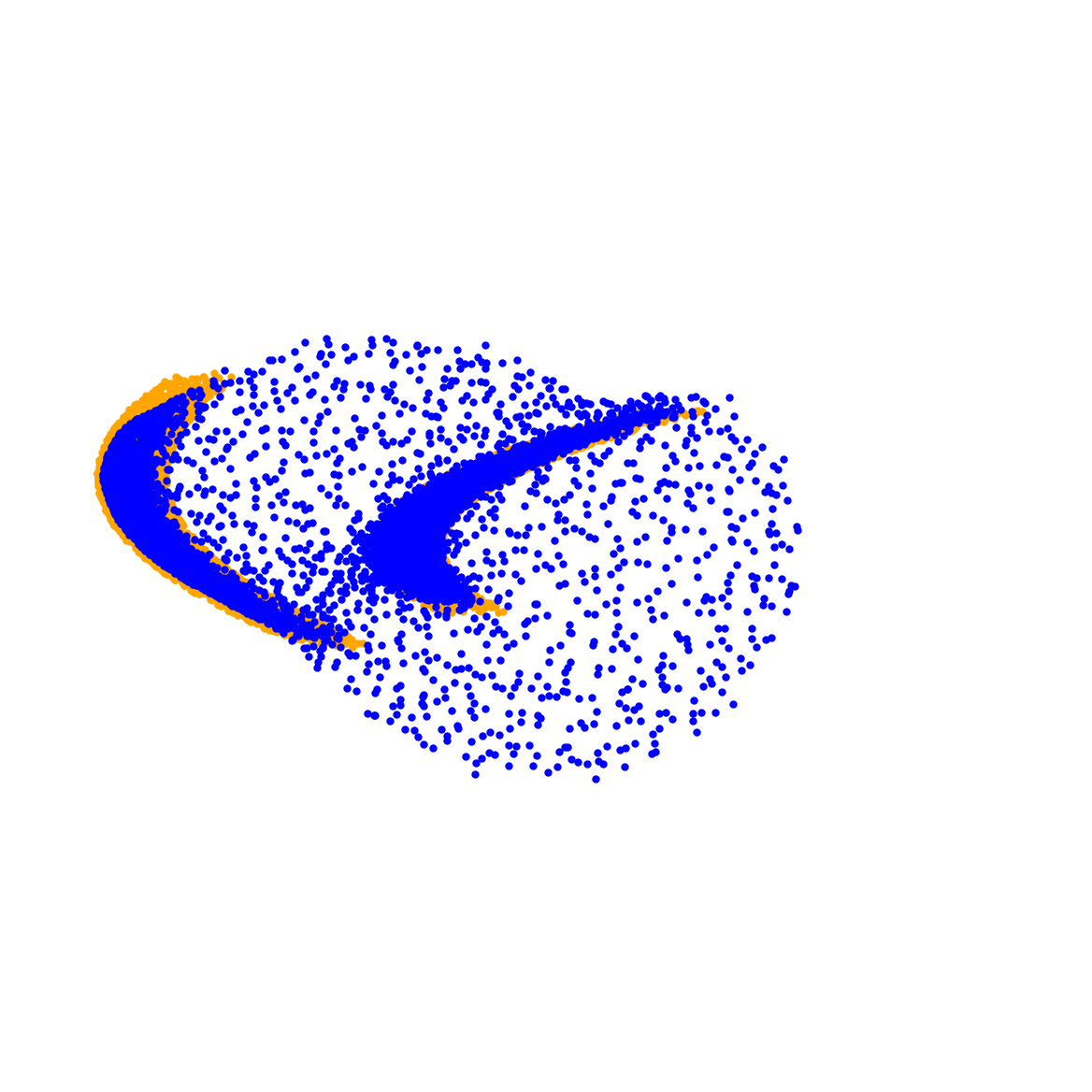}
    \caption*{t=1}
    \end{subfigure}%
    \begin{subfigure}[t]{.245\textwidth}
      \includegraphics[width=\linewidth, valign=c]{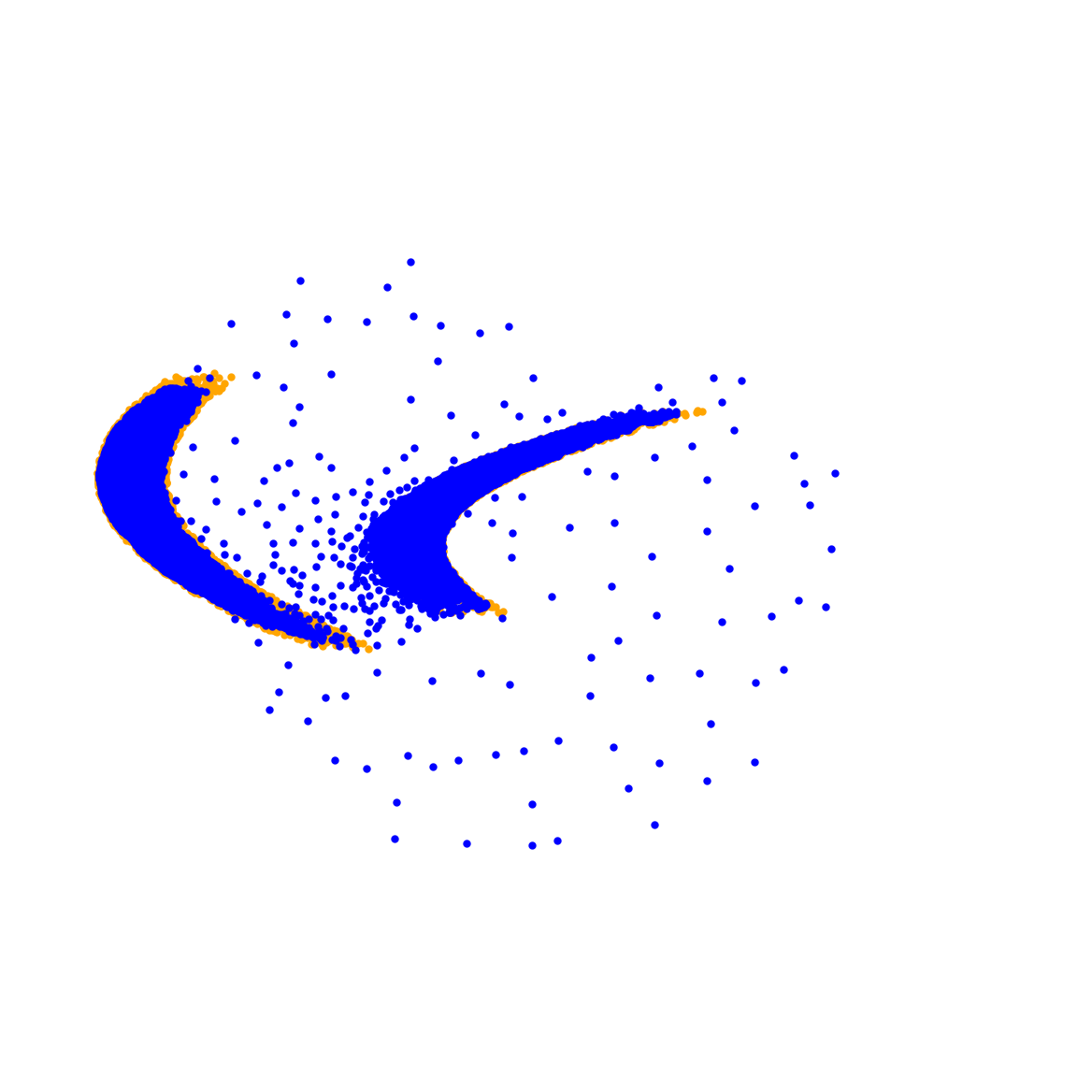}
    \caption*{t=10}
    \end{subfigure}
    \begin{subfigure}[t]{.245\textwidth}
        \includegraphics[width=\linewidth, valign=c]{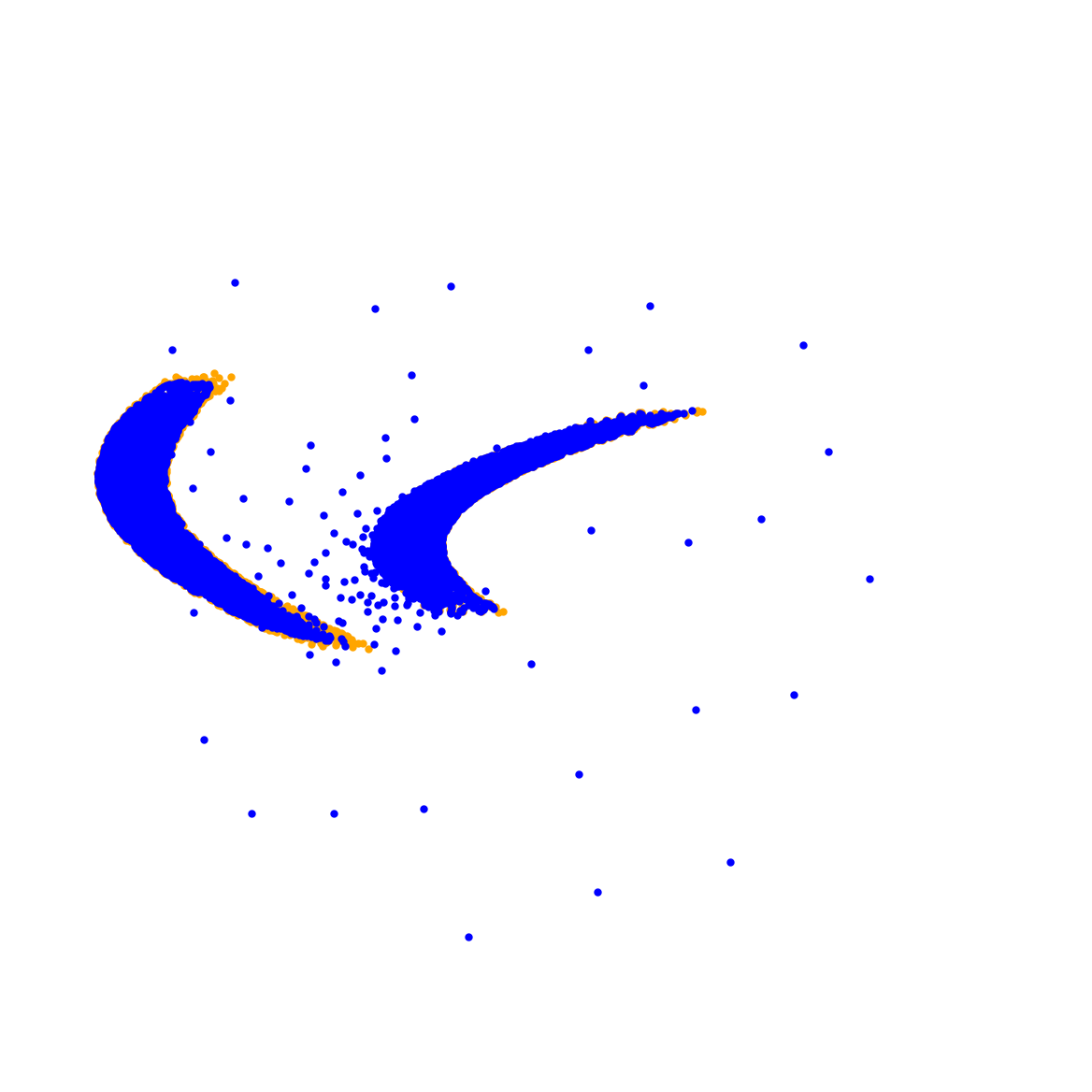}
        \caption*{t=50}
    \end{subfigure}%

    \begin{subfigure}[t]{.245\textwidth}
      \includegraphics[width=\linewidth, valign=c]{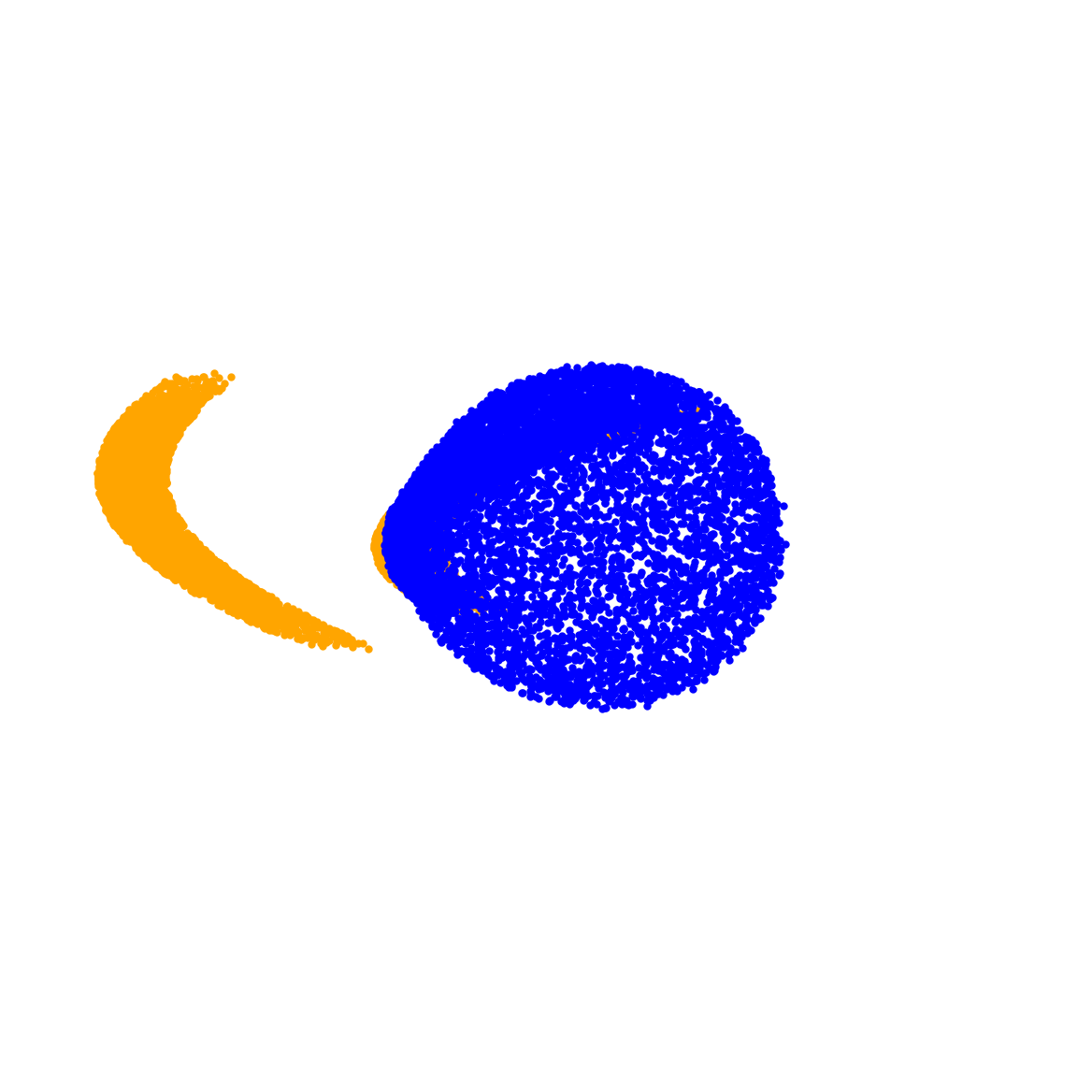}
    \caption*{t=0.1}
    \end{subfigure}%
    \begin{subfigure}[t]{.245\textwidth}
      \includegraphics[width=\linewidth, valign=c]{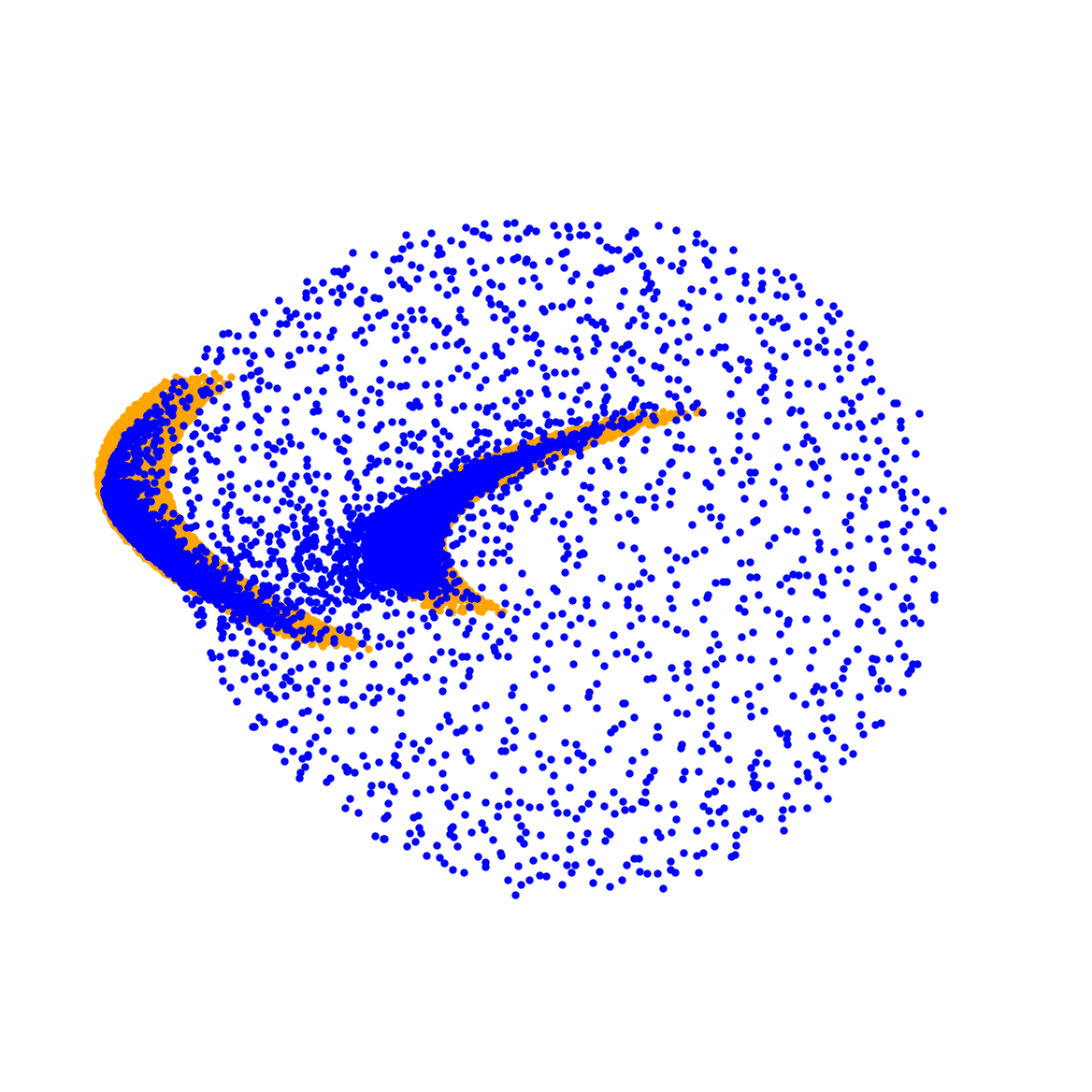}
    \caption*{t=1}
    \end{subfigure}%
    \begin{subfigure}[t]{.245\textwidth}
      \includegraphics[width=\linewidth, valign=c]{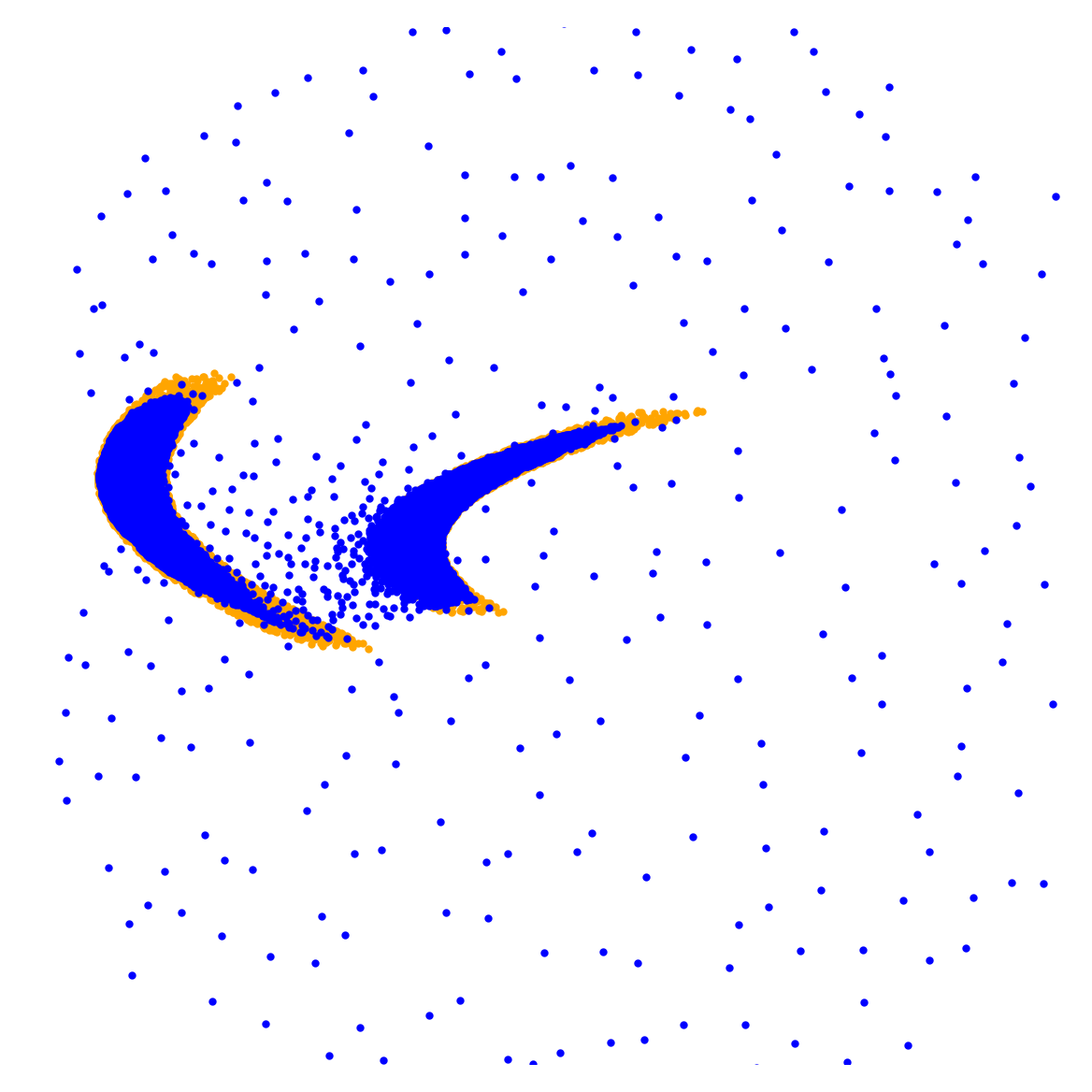}
    \caption*{t=10}
    \end{subfigure}
    \begin{subfigure}[t]{.245\textwidth}
        \includegraphics[width=\linewidth, valign=c]{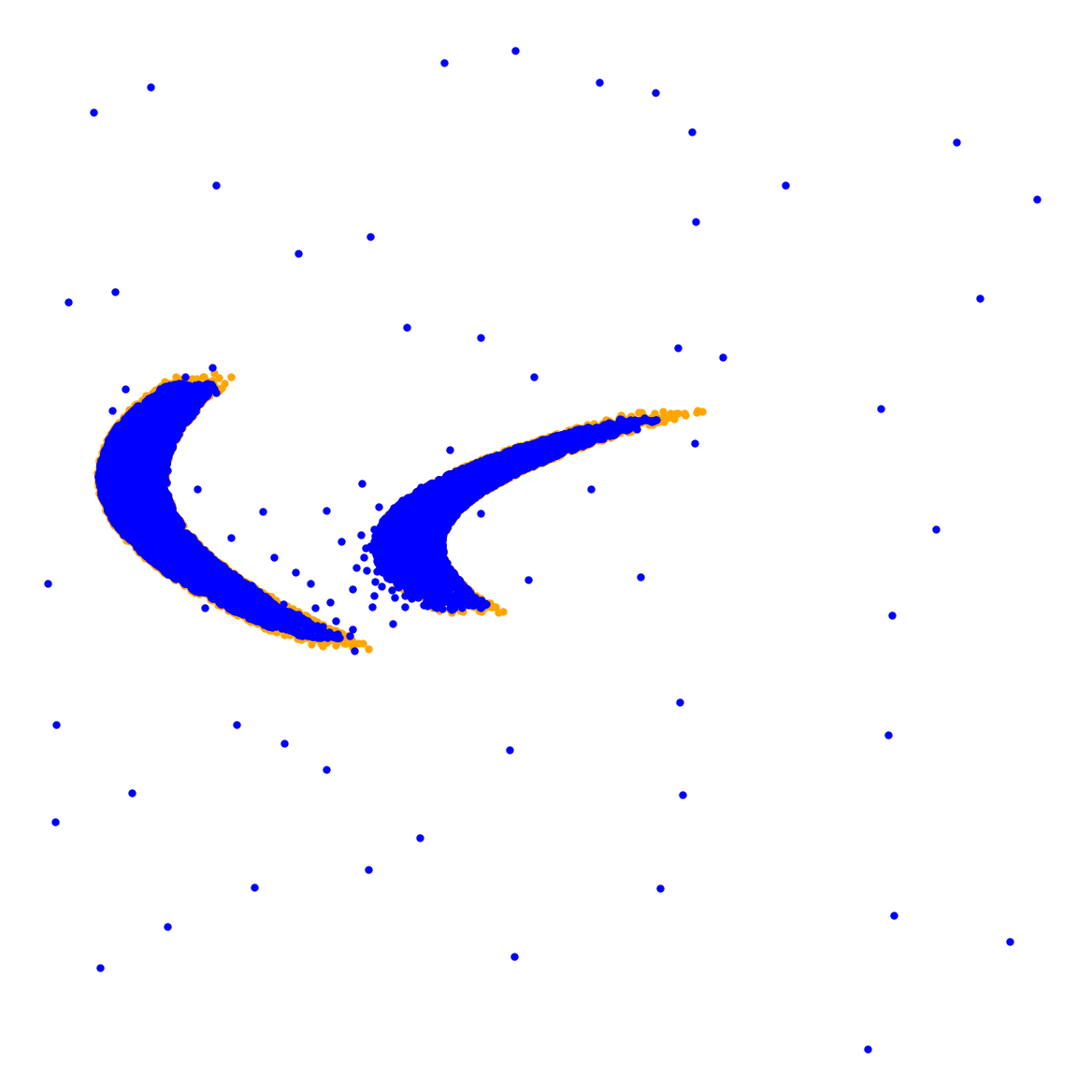}
        \caption*{t=50}
    \end{subfigure} \\
    \centering
    \caption{WGF of the tight version $\tilde{D}_{f, \nu}^{\lambda}$ (top) and non-tight one $D_{f, \nu}^{\lambda}$ (bottom) of the regularized Tsallis-2 divergence for the bananas target and $N = M = 9000$.
   }
    \label{fig:tight_vs_nontight}
\end{figure}

\begin{remark}[Tractability of FISTA for the primal problem]
    The separability of $\prox_{\tilde h_1}$ is lost when we restrict the minimization of $J$ to the simplex.
    Hence, it is unclear if FISTA would be effective when applied to \eqref{eq:TightFDiv}.
    Due to \eqref{eq:ConstrDual}, the same problem would also occur for the dual objective.
\end{remark}

\section{Duality gaps and optimality criteria} \label{app:C}

Here, we collect some details about tracking the accuracy of the numerical solvers for the dual problem and when to terminate the iterations. 
The duality gap is the difference between the dual optimal value and the primal optimal value, namely
\begin{align*}
    \text{dg}^{(n)}
    & \coloneqq \bigl|- D(b^{n}) - J(q^{(n)})\bigr|.
\end{align*}
Since we only minimize $J$, we instead calculate the primal pseudo-duality gap
\begin{equation*}
    \text{dg}_p^{(n)}
    \coloneqq \Bigl| - D\Bigl(- \frac{1}{\lambda M} q^{(n)}\Bigr) - J(q^{(n)}) \Bigr|.
\end{equation*}
%and analogously, if we only solve the dual problem, we can calculate the dual pseudo-duality gap 
%\begin{equation*}
    %\text{dg}_d^{(n)}
    %\coloneqq \left| - D(b^{(n)}) - J\left(- \lambda M b^{(n)}\right) \right|.
%\end{equation*}
To account for small functional values towards the end of the iteration, we can replace this with its relative counterpart, where we divide by the minimum of the two summands.
When computing the particle flows, we terminate the time stepping if the \enquote{kinetic energy} $- \partial_t D_{f, \nu}^{\lambda}(\gamma_n) = \frac{1}{N} \sum_{i = 1}^{N} | \nabla p_n\big(x_i^{(n)}\big) |^2$ is small.

%%%%%%%%%%%%%%%%%%%%%%%%%%%%%%%%%%%%%%%%%%%%%%%%%%%%%%%%%%%%%%%%%%%%%%%%%%%%%%%%%%%%%%%%%%%%%%%%%%
\section{Supplementary Material} \label{sec:supplementary}

Here, we provide more numerical experiments and give some implementation details.

\subsection{Implementation Details}
To leverage parallel computing on the GPU, we implemented our model in \texttt{PyTorch} \cite{pytorch}.
Furthermore, we use the \texttt{Matplotlib} package \cite{matplotlib} for plotting the figures and the \texttt{POT} package \cite{POT21} for calculating $W_2(\mu, \nu)^2$ along the flow.
Solving the dual \eqref{dual_d} turned out to be much more time-consuming than solving the primal problem \eqref{eq:MoreauPrimalgnuf2}, so we exclusively outlined the implementation for the latter.
As a sanity check, we calculate the \enquote{pseudo-duality gap}, which is the difference between the value of the primal objective at the solution $q$ and the value of the dual objective at the corresponding dual certificate $\frac{1}{\lambda N} [- q, \1_N]$.
Then, the relative pseudo-duality gap is computed as the quotient of the pseudo-duality gap and the minimum of the absolute value of the involved objective values.
Since the particles get close to the target towards the end of the flow, we use double precision throughout (although this deteriorates the benefit of GPUs).
With this, all quantities can still be accurately computed and evaluated.

\subsection{The Kernel Width}
For the first ablation study, we use the parameters from Section~\ref{sec:numerics}, namely $\alpha = 5$, $\lambda = \num{e-2}$, $\tau = \num{e-3}$ and $N = 900$.
In Figure \ref{fig:ablationCirclesIMQtau},
we see that the kernel width has to be calibrated carefully to get a sensible result.
We observe that the width $\sigma^2 = \num{e-3}$ is too small.
In particular, the repulsion of the particles is too powerful, and they immediately spread out before moving towards the rings, but only very slowly.
If $\sigma^2 = \num{e-2}$, everything works out reasonably, and there are no outliers.
For $\sigma^2 = 1$, the particles only recover the support of the target very loosely and for $\sigma^2 = 10$, there is no recovery.
\begin{figure}[pt]
    \centering
    \rotatebox{90}{
        \begin{minipage}{.05\textwidth}
        \centering 
        \small $\num{e-3}$
    \end{minipage}} \hfill
    \begin{subfigure}[t]{.14\textwidth}
        \includegraphics[width=\linewidth, valign=c]{img/Width/sigma2=0.01/Reg_tsa5flow,lambd=.01,tau=.001,IMQ,.01,900,25,circles-0}
    \end{subfigure}%
    \begin{subfigure}[t]{.14\textwidth}
        \includegraphics[width=\linewidth, valign=c]{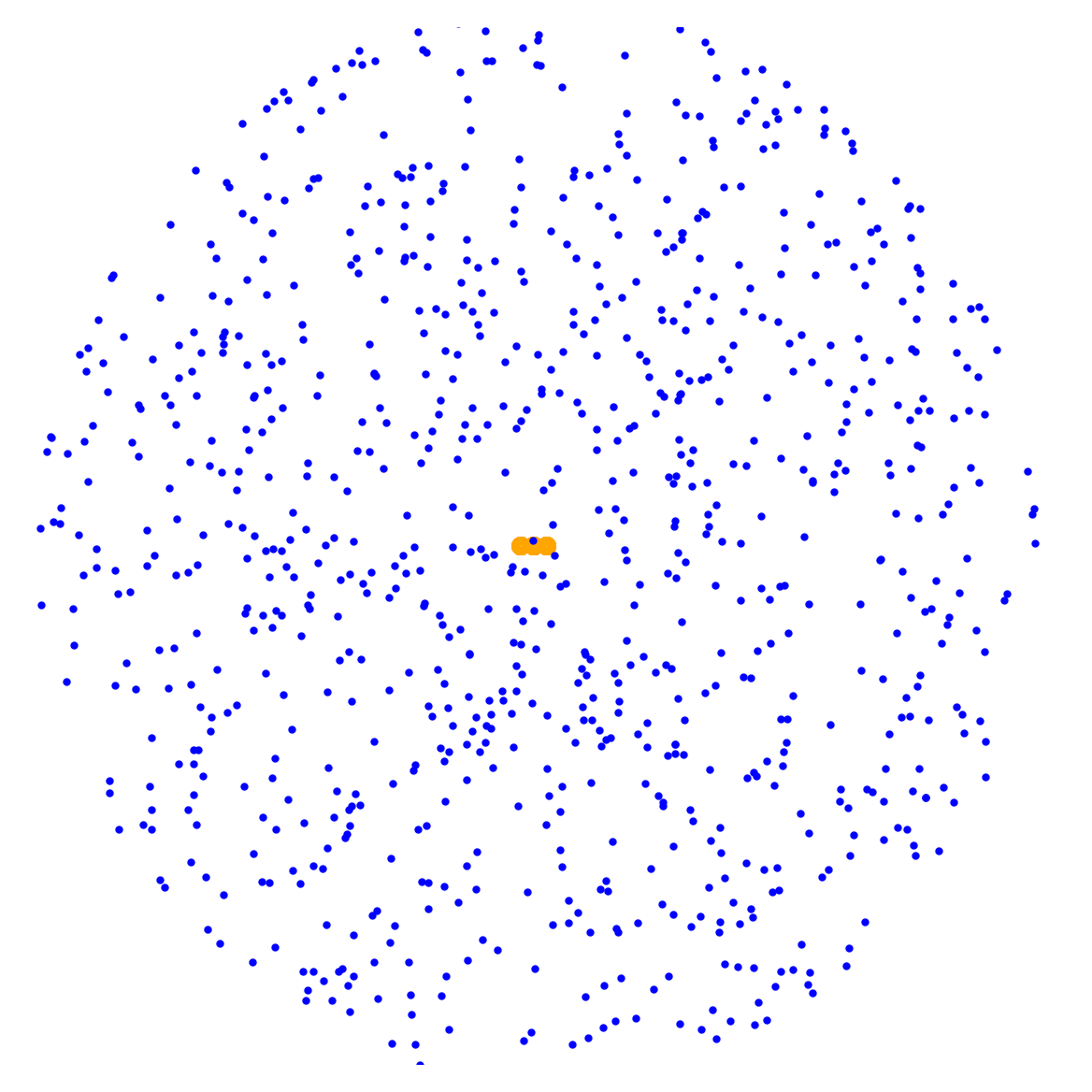}
    \end{subfigure}%
    \begin{subfigure}[t]{.14\textwidth}
        \includegraphics[width=\linewidth, valign=c]{{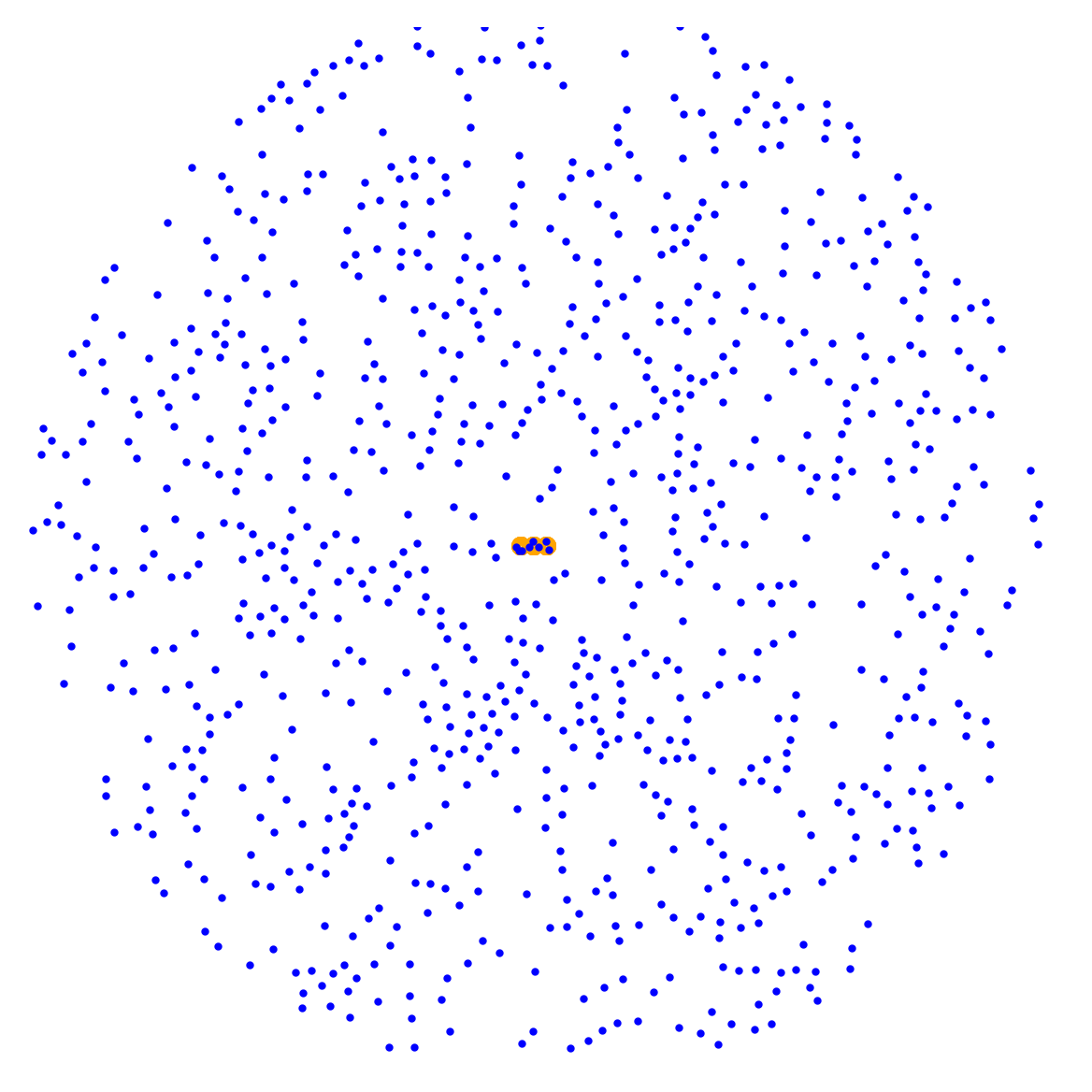}}
    \end{subfigure}%
    \begin{subfigure}[t]{.14\textwidth}
        \includegraphics[width=\linewidth, valign=c]{{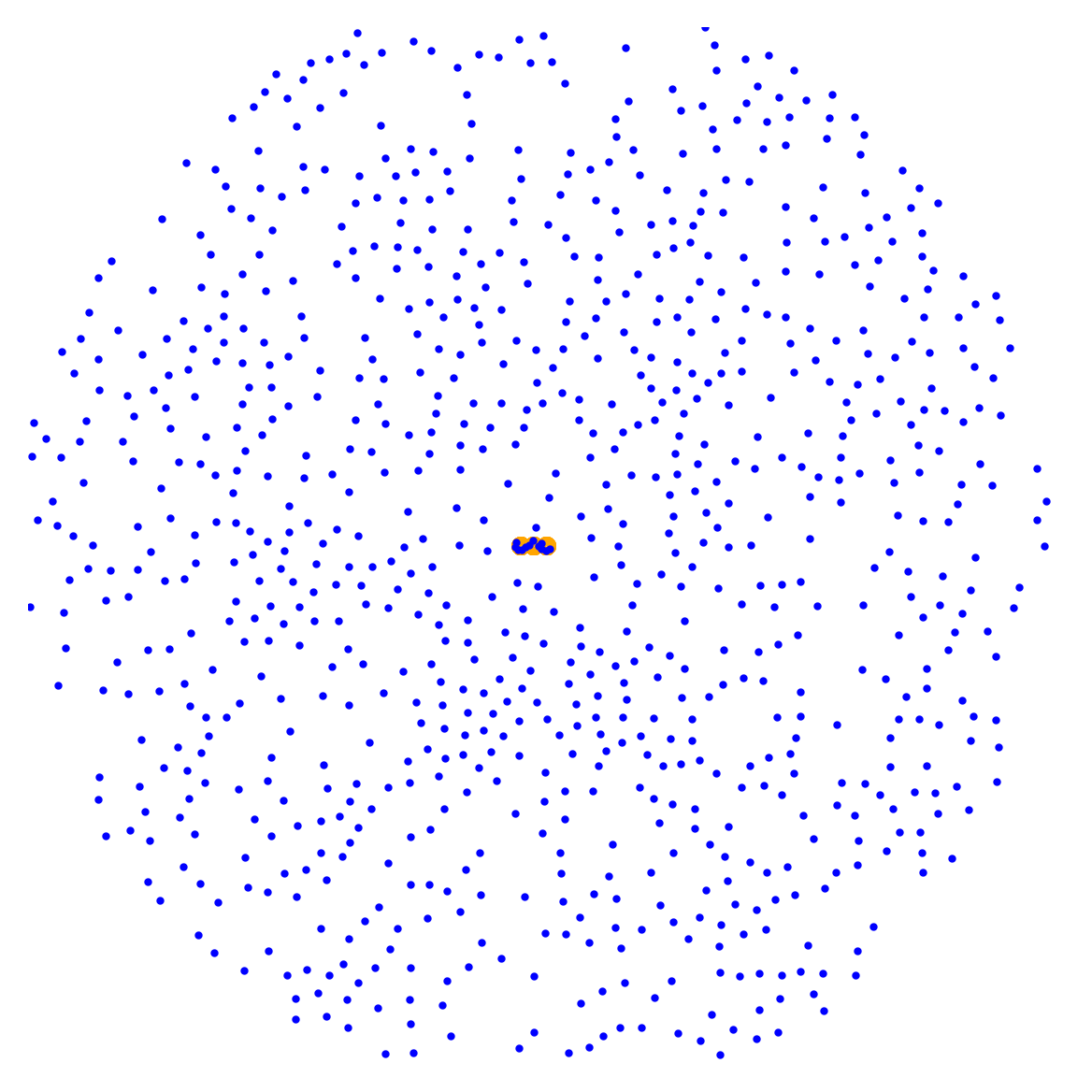}}
    \end{subfigure}%
    \begin{subfigure}[t]{.14\textwidth}
        \includegraphics[width=\linewidth, valign=c]{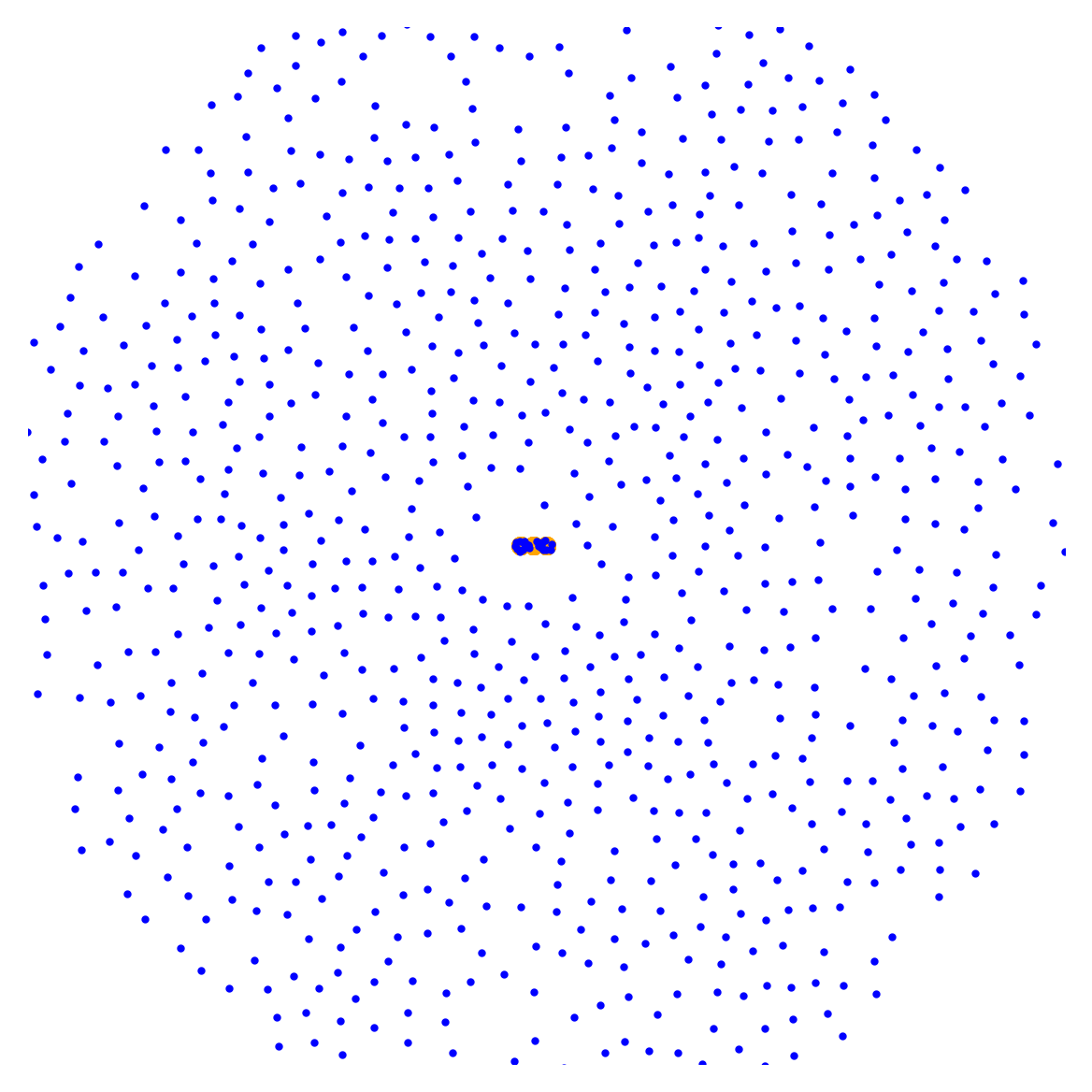}
    \end{subfigure}%
    \begin{subfigure}[t]{.14\textwidth}
        \includegraphics[width=\linewidth, valign=c]{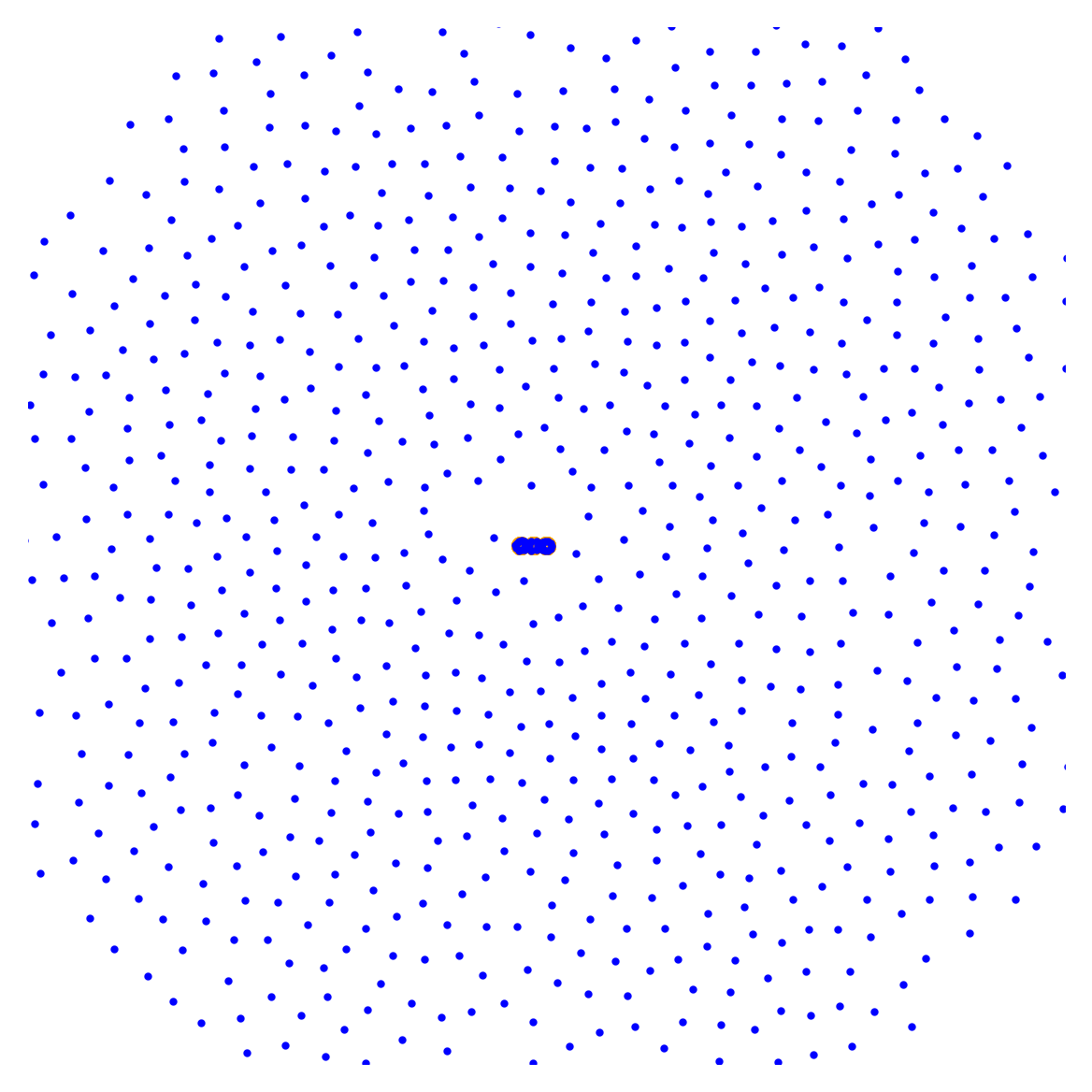}
    \end{subfigure} \\ %
    \rotatebox{90}{
        \begin{minipage}{.08\textwidth}
        \centering 
        \small $\num{e-2}$
    \end{minipage}} \hfill
    \begin{subfigure}[t]{.14\textwidth}
        \includegraphics[width=\linewidth]{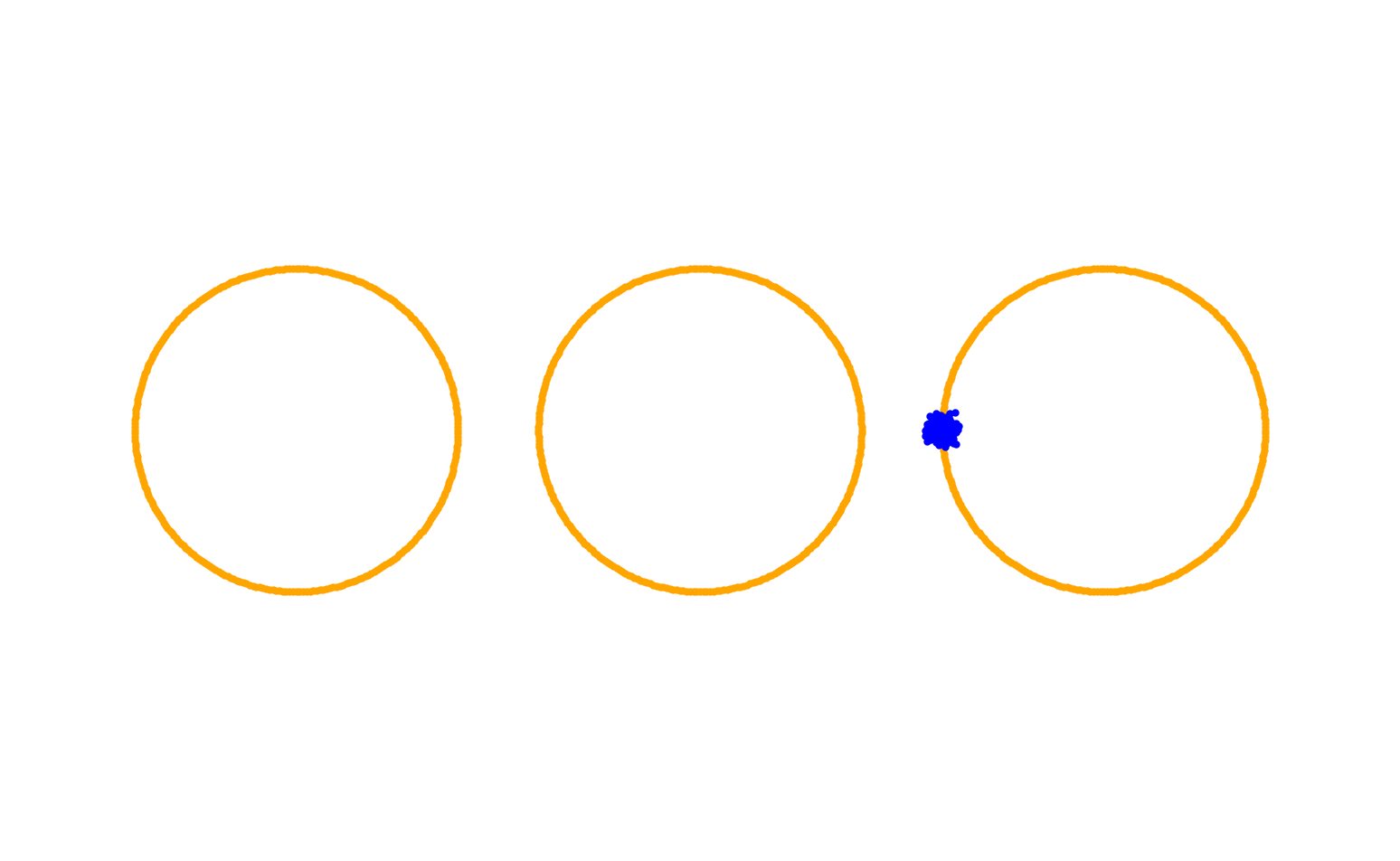}
    \end{subfigure}%
    \begin{subfigure}[t]{.14\textwidth}
        \includegraphics[width=\linewidth]{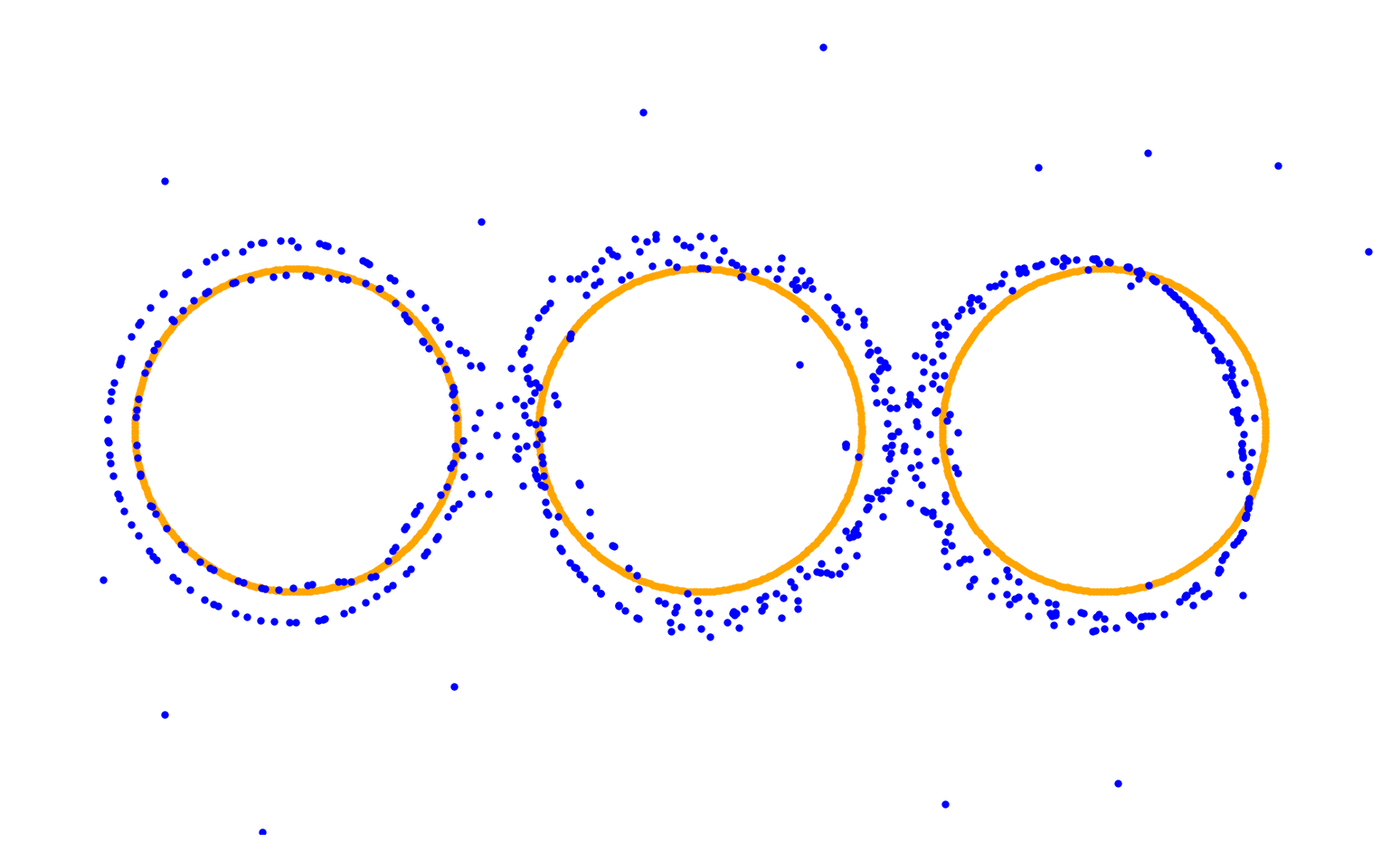}
    \end{subfigure}%
    \begin{subfigure}[t]{.14\textwidth}
        \includegraphics[width=\linewidth]{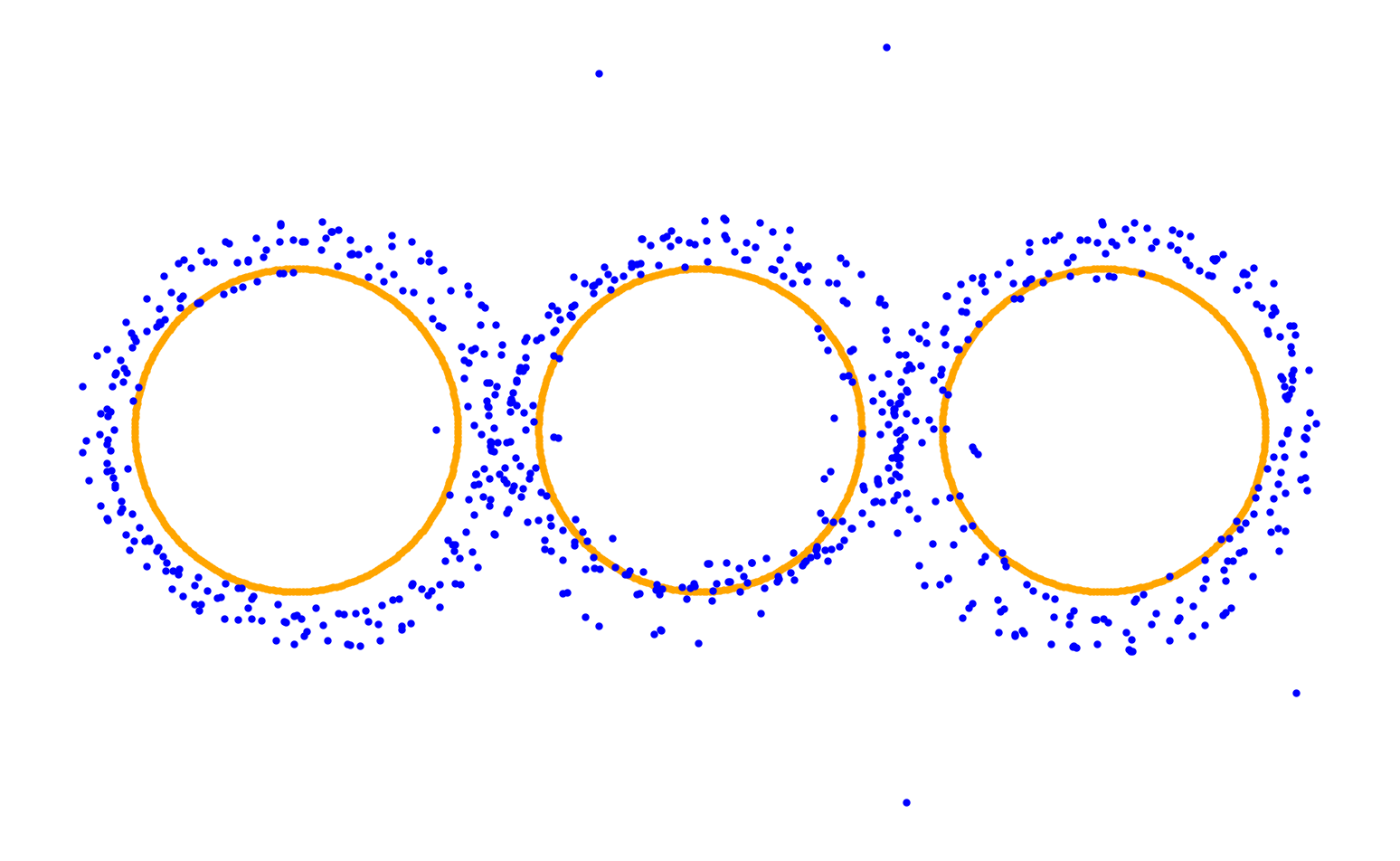}
    \end{subfigure}%
    \begin{subfigure}[t]{.14\textwidth}
        \includegraphics[width=\linewidth]{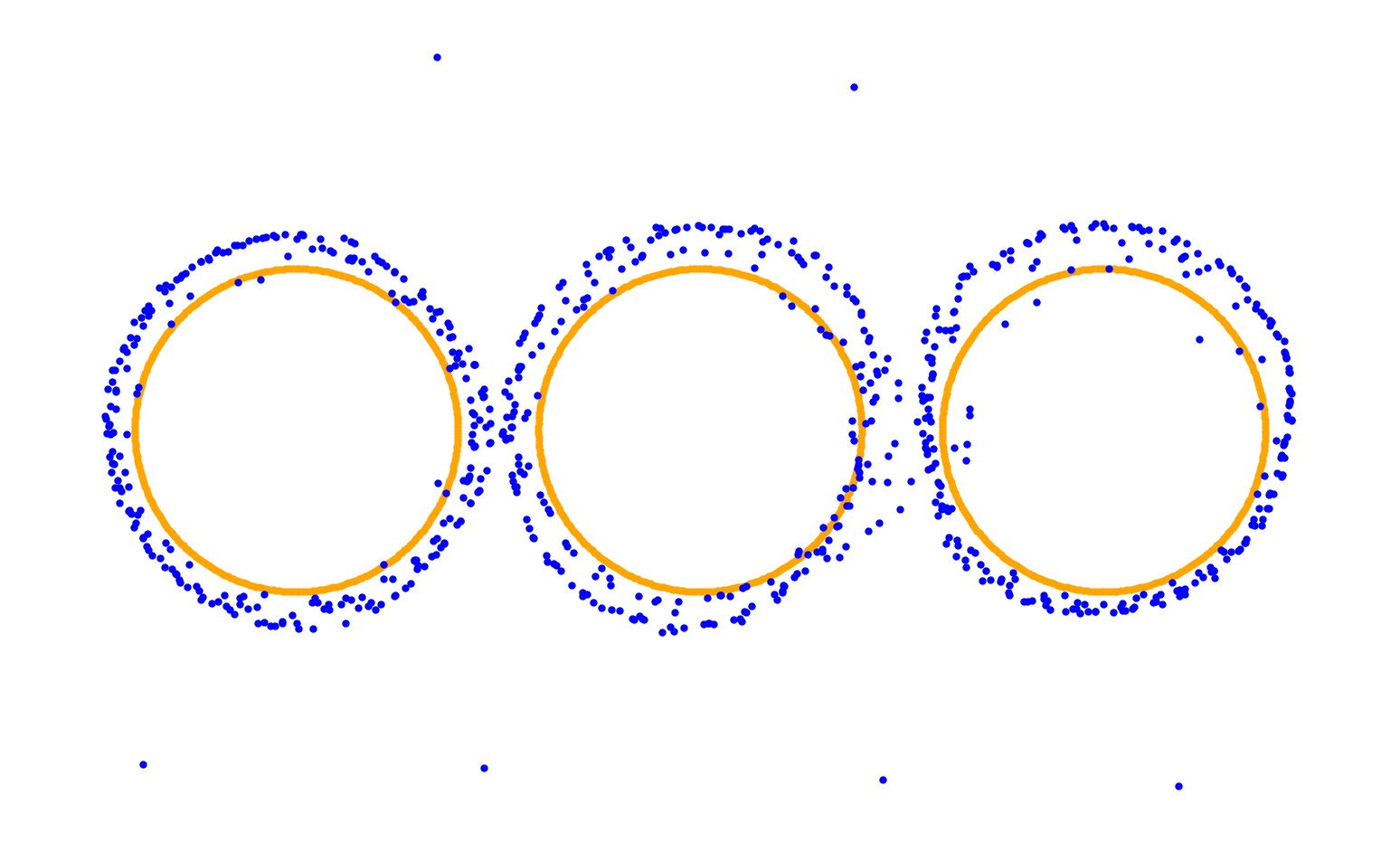}
    \end{subfigure}%
    \begin{subfigure}[t]{.14\textwidth}
        \includegraphics[width=\linewidth]{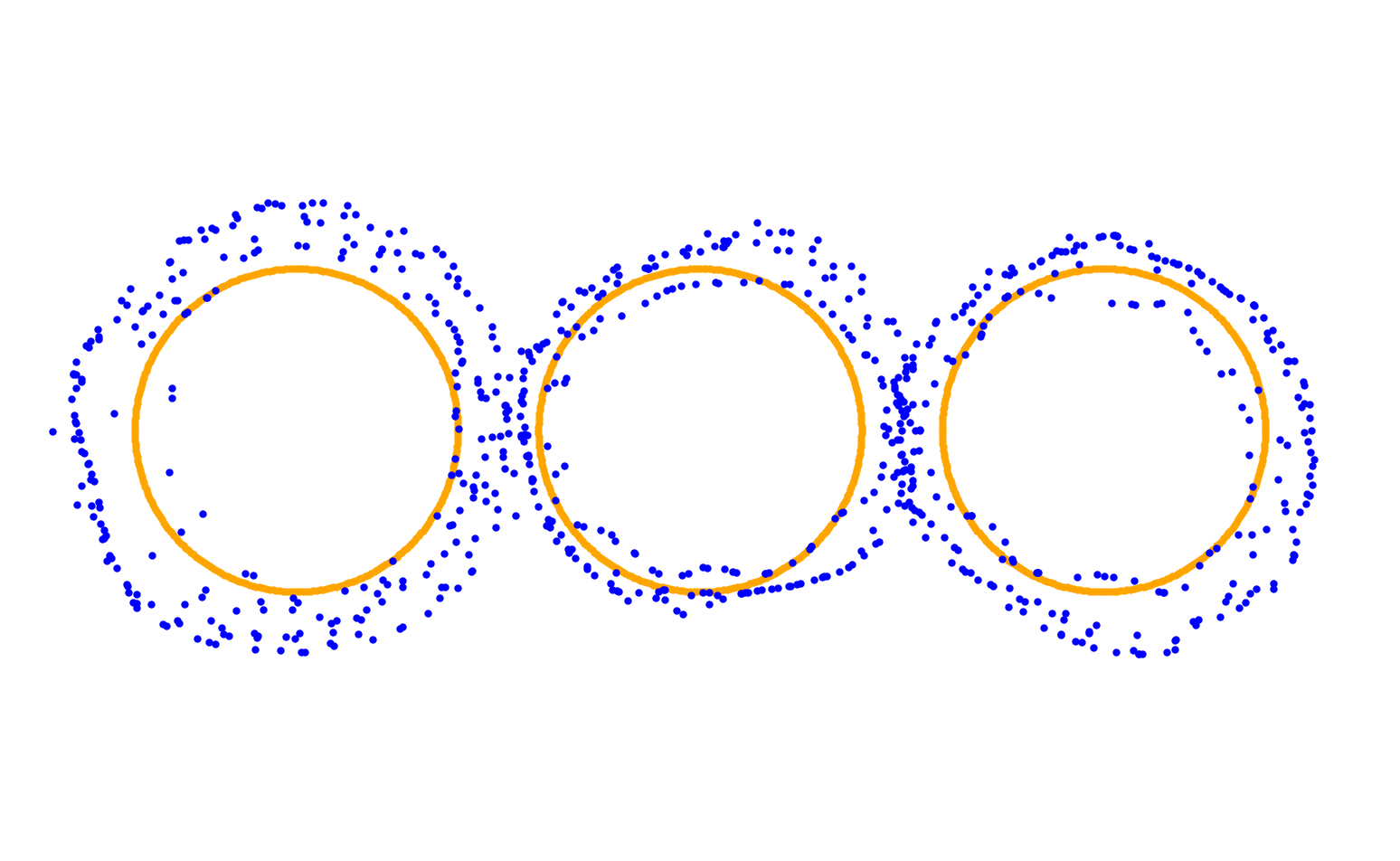}
    \end{subfigure}%
    \begin{subfigure}[t]{.14\textwidth}
        \includegraphics[width=\linewidth]{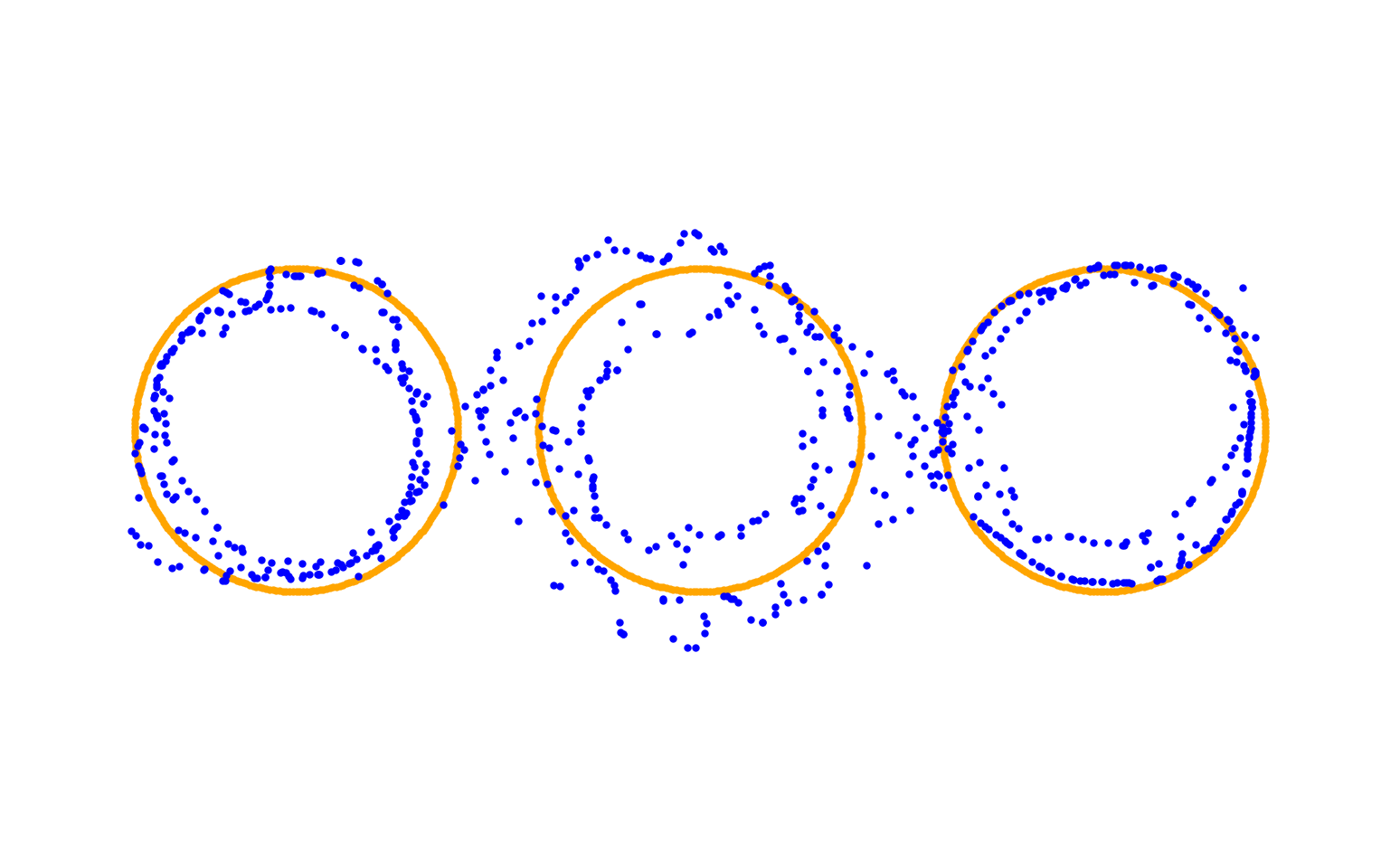}
    \end{subfigure} \\ %
    \rotatebox{90}{
        \begin{minipage}{.08\textwidth}
        \centering 
        \small $\num{e-1}$
    \end{minipage}} \hfill
    \begin{subfigure}[t]{.14\textwidth}
        \includegraphics[width=\linewidth]{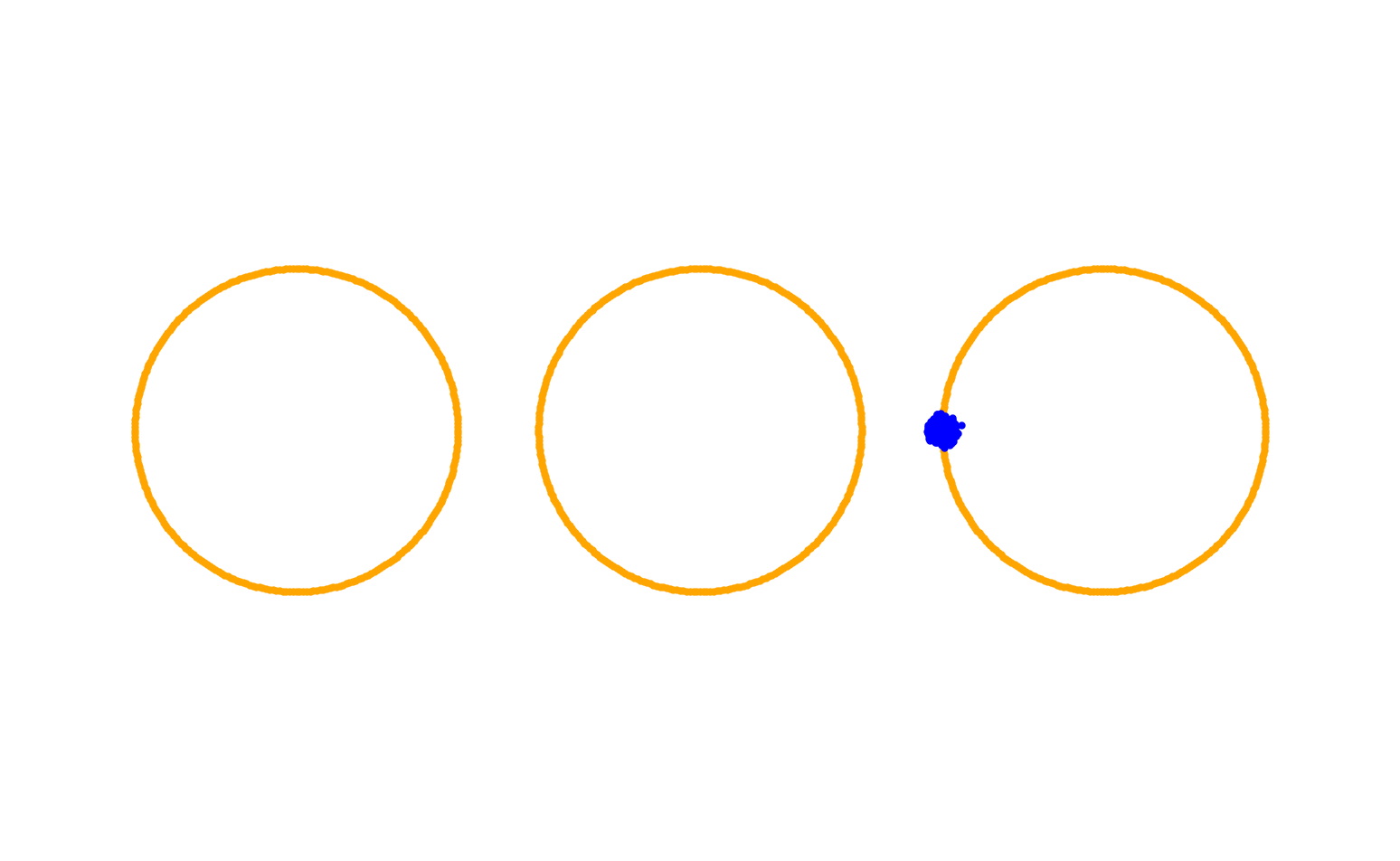}
    \end{subfigure}%
    \begin{subfigure}[t]{.14\textwidth}
        \includegraphics[width=\linewidth]{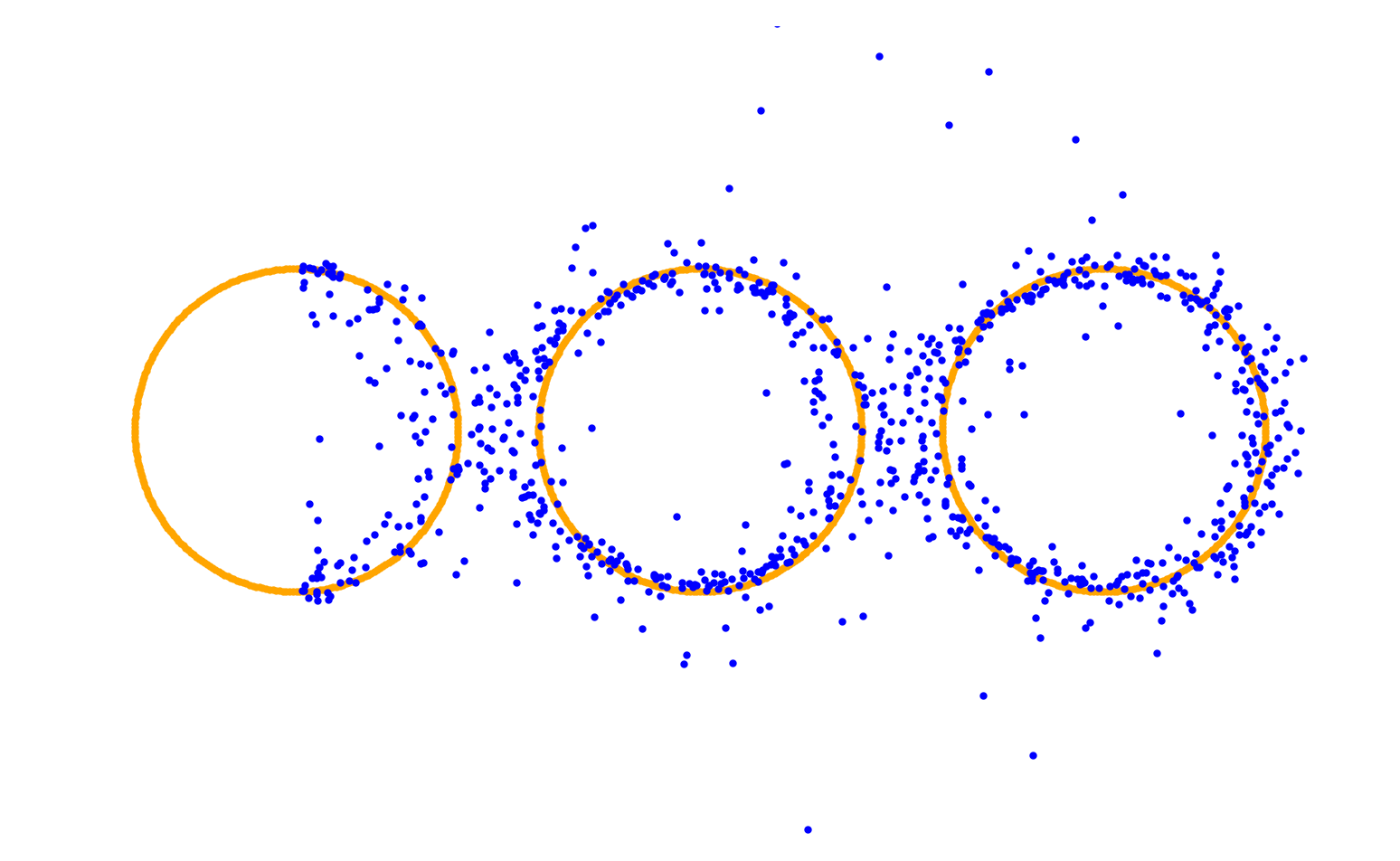}
    \end{subfigure}%
    \begin{subfigure}[t]{.14\textwidth}
        \includegraphics[width=\linewidth]{{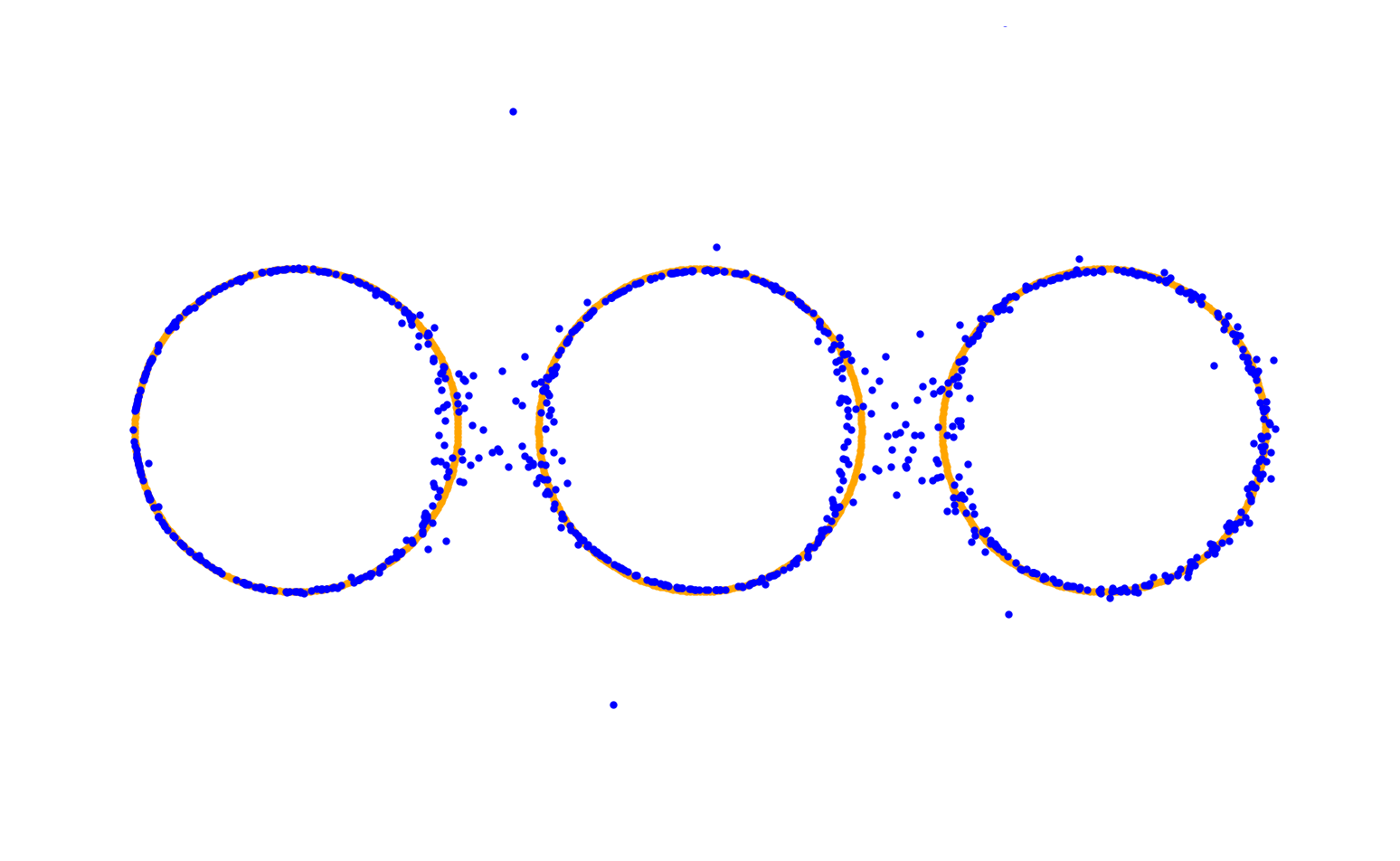}}
    \end{subfigure}%
    \begin{subfigure}[t]{.14\textwidth}
        \includegraphics[width=\linewidth]{{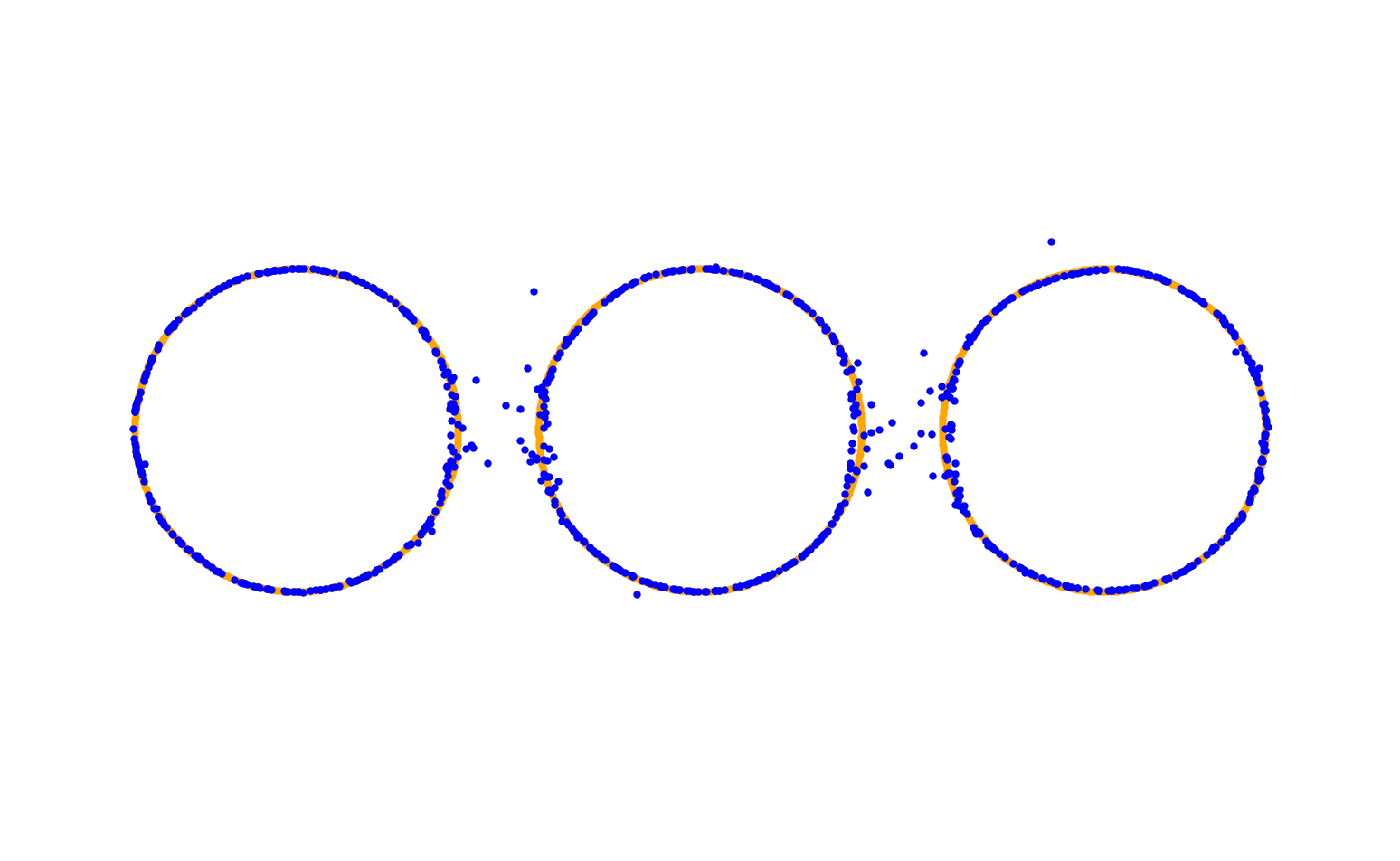}}
    \end{subfigure}%
    \begin{subfigure}[t]{.14\textwidth}
        \includegraphics[width=\linewidth]{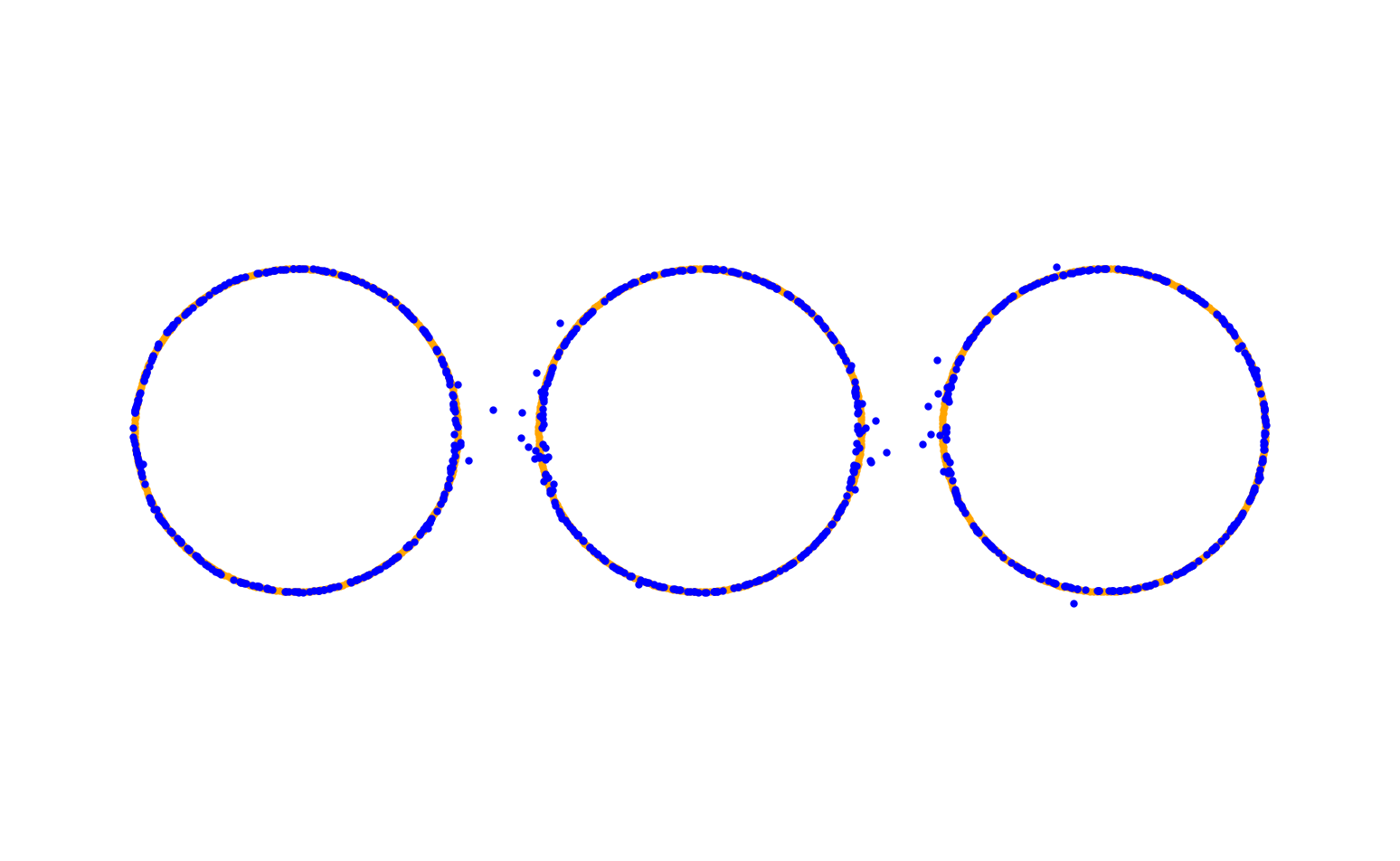}
    \end{subfigure}%
    \begin{subfigure}[t]{.14\textwidth}
        \includegraphics[width=\linewidth]{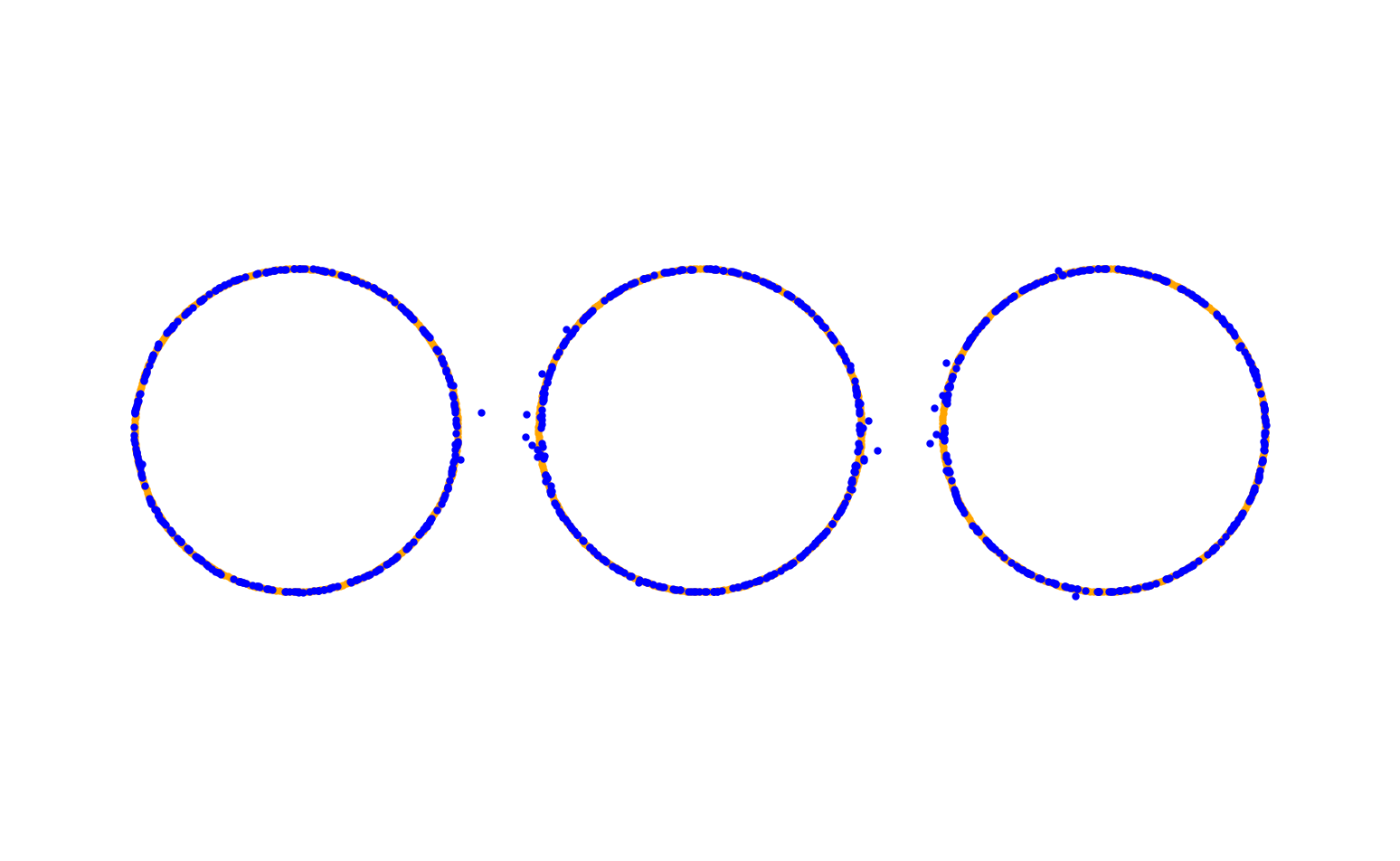}
    \end{subfigure} \\ %
     \rotatebox{90}{
        \begin{minipage}{.08\textwidth}
        \centering 
        \small $1$
    \end{minipage}} \hfill
    \begin{subfigure}[t]{.14\textwidth}
        \includegraphics[width=\linewidth]{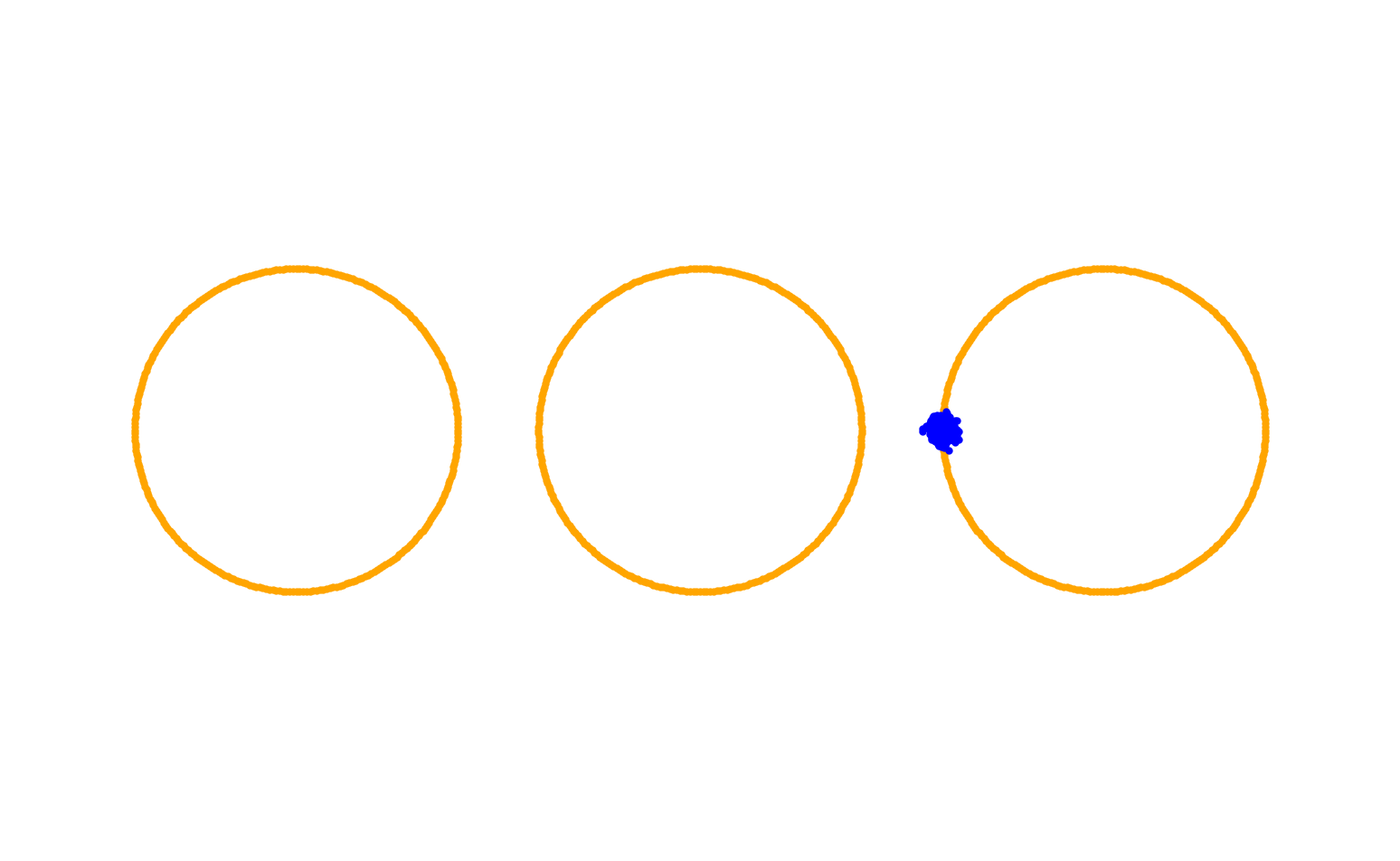}
    \end{subfigure}%
    \begin{subfigure}[t]{.14\textwidth}
        \includegraphics[width=\linewidth]{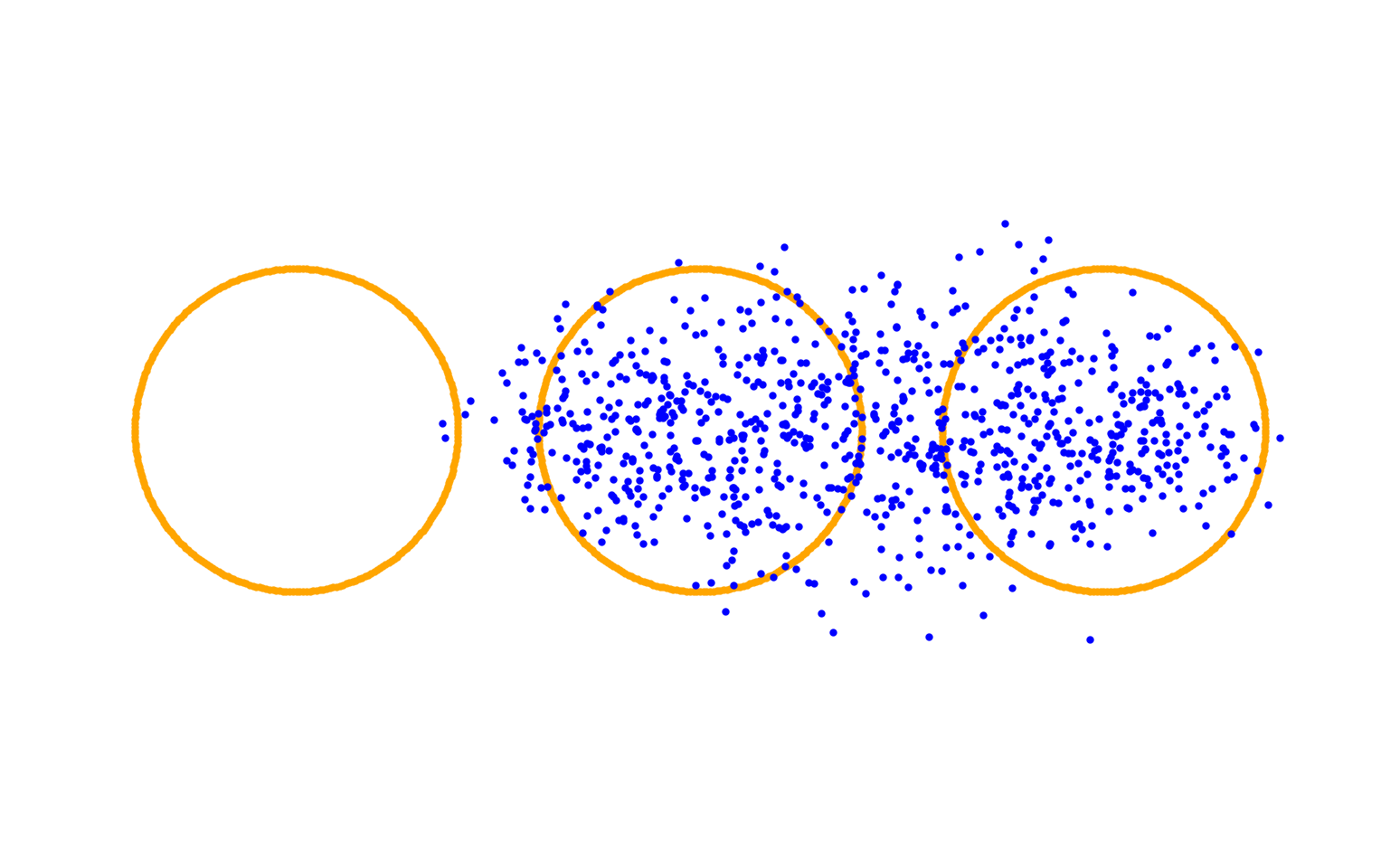}
    \end{subfigure}%
    \begin{subfigure}[t]{.14\textwidth}
        \includegraphics[width=\linewidth]{{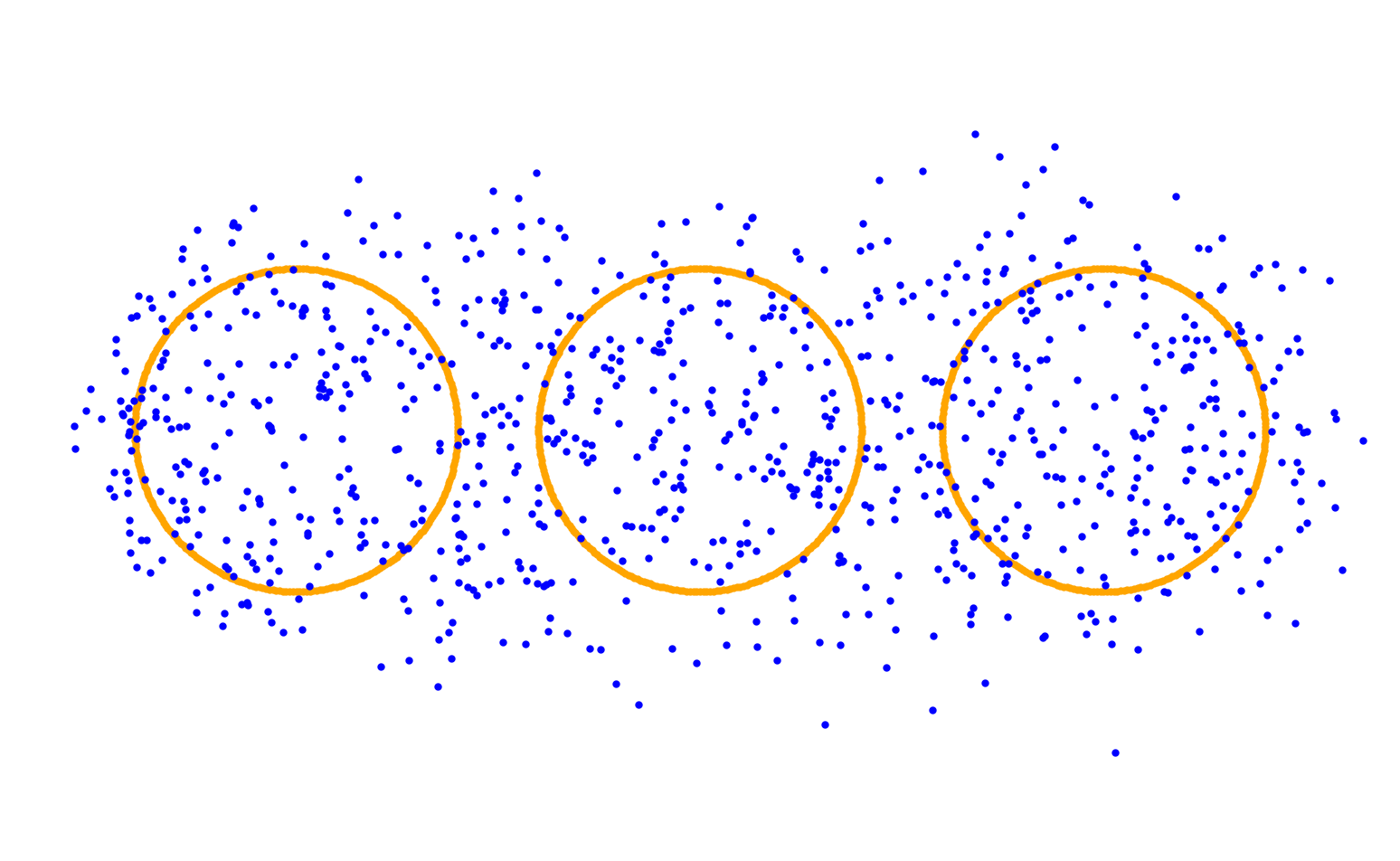}}
    \end{subfigure}%
    \begin{subfigure}[t]{.14\textwidth}
        \includegraphics[width=\linewidth]{{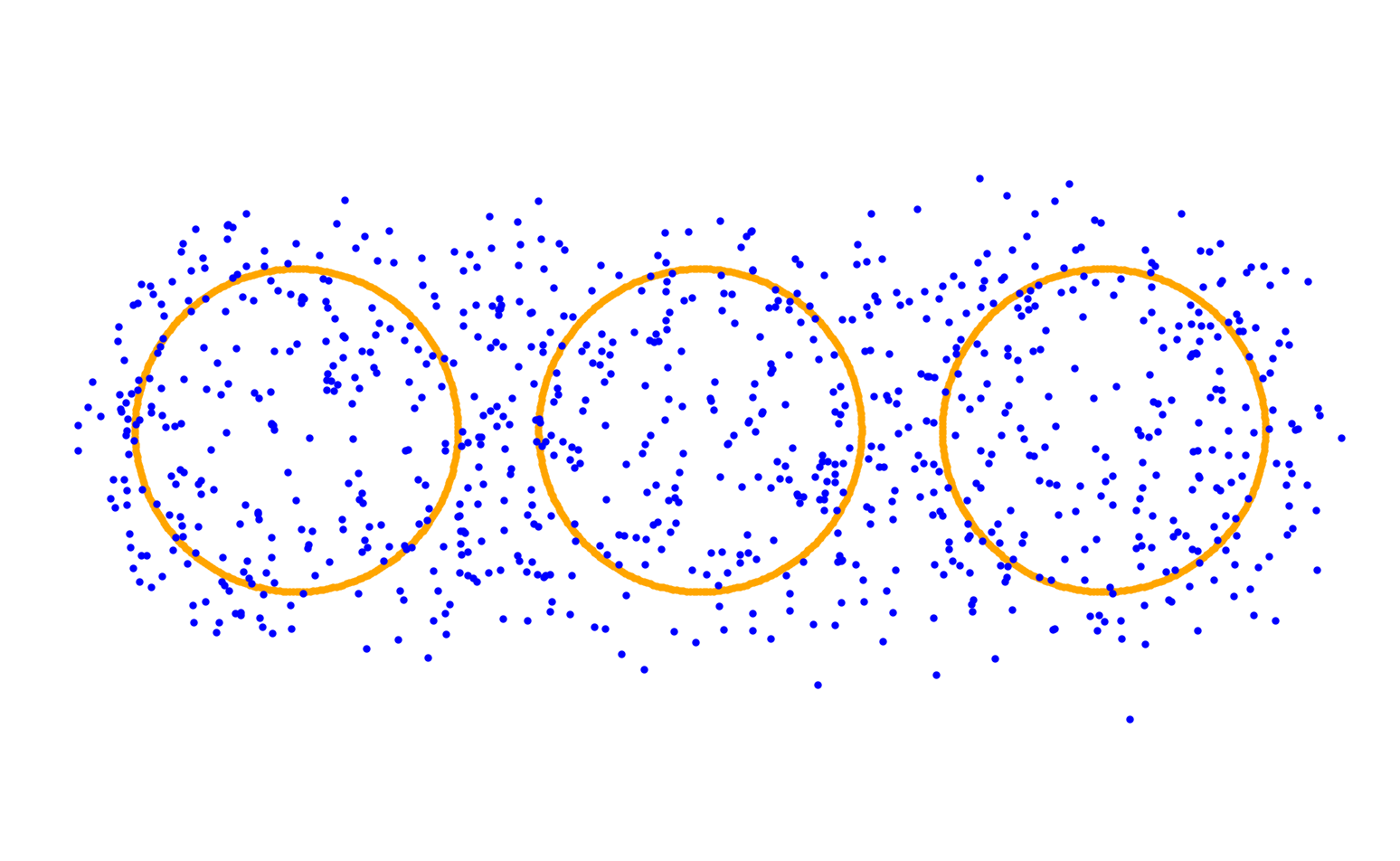}}
    \end{subfigure}%
    \begin{subfigure}[t]{.14\textwidth}
        \includegraphics[width=\linewidth]{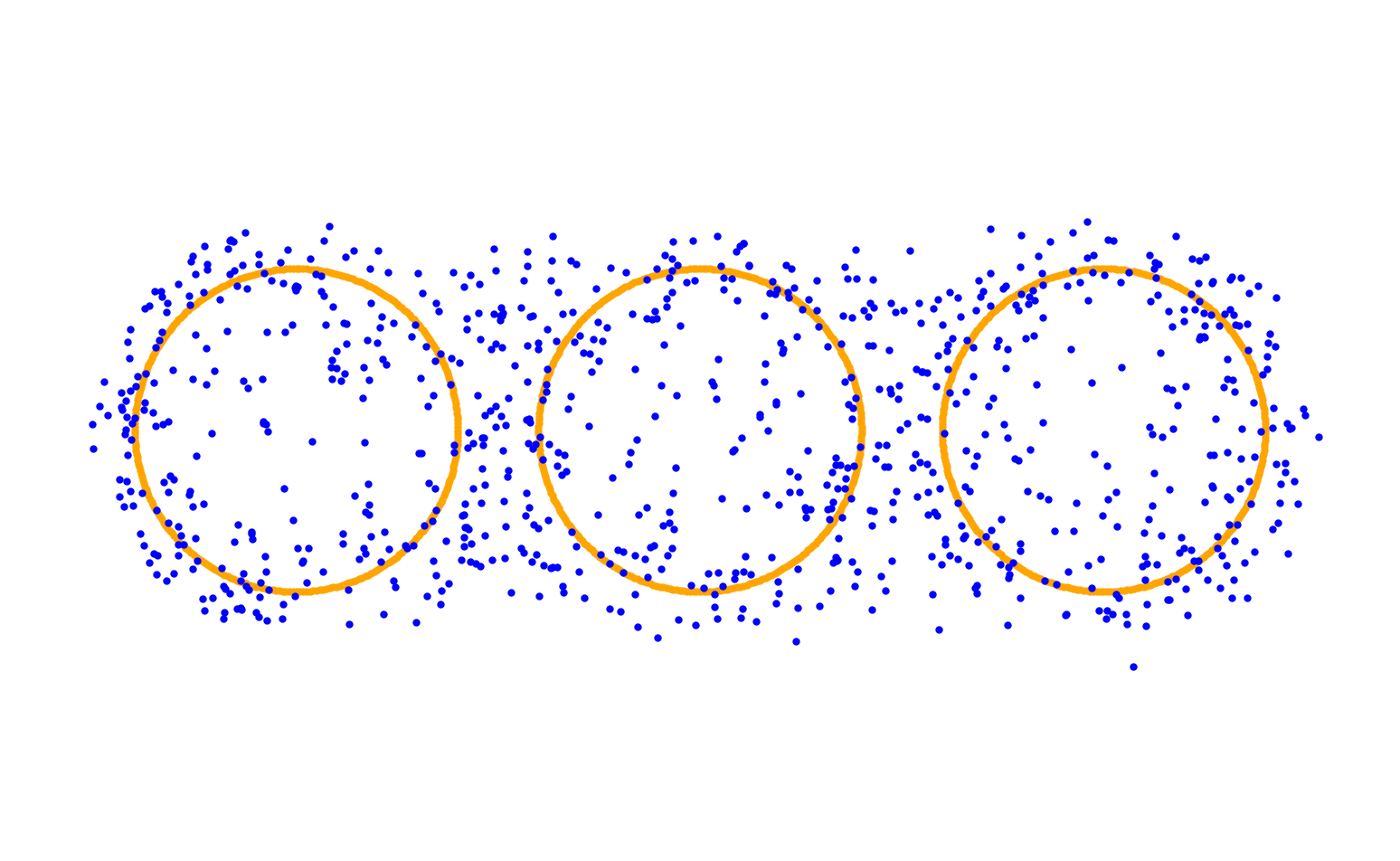}
    \end{subfigure}%
    \begin{subfigure}[t]{.14\textwidth}
        \includegraphics[width=\linewidth]{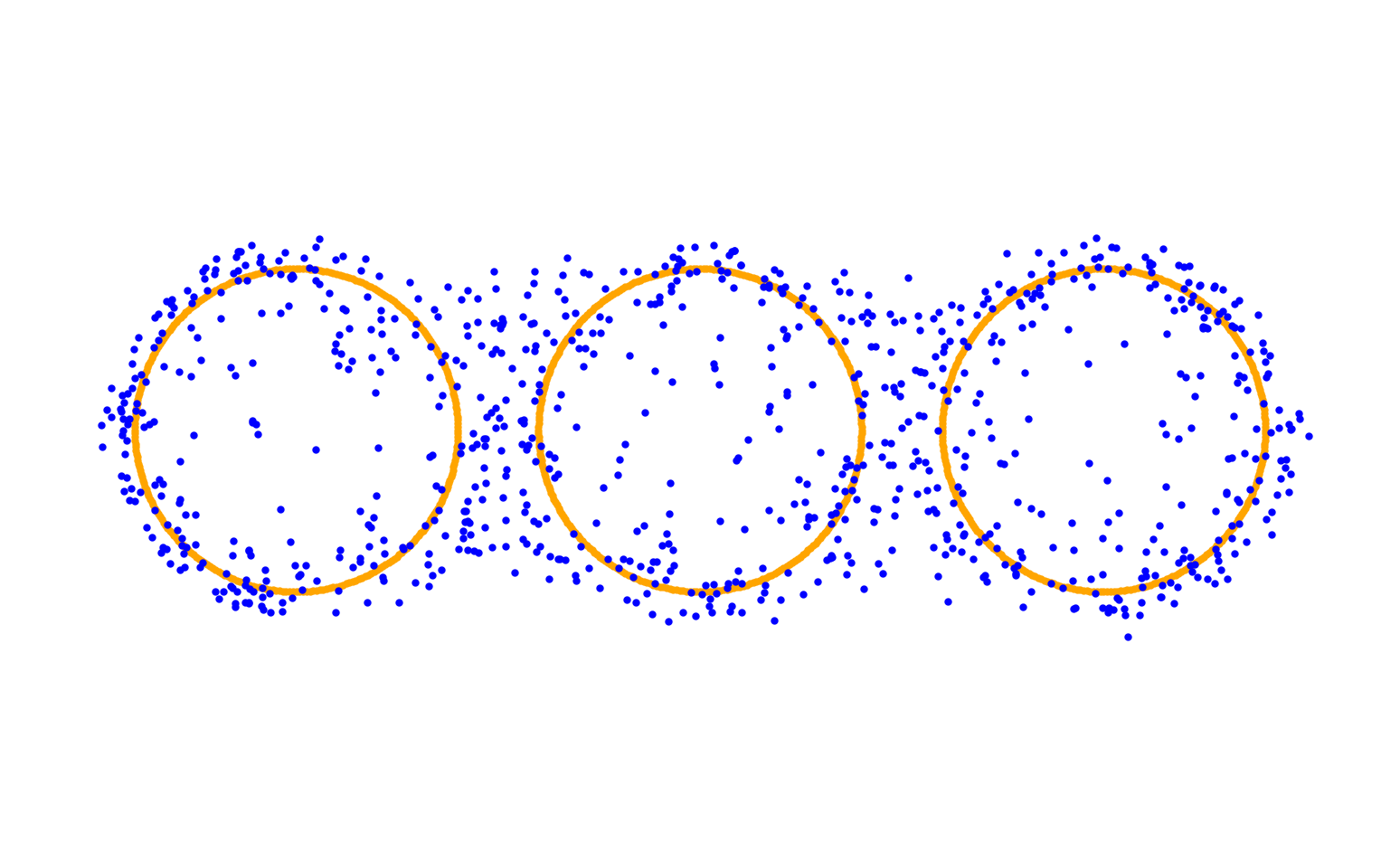}
    \end{subfigure} \\ %
    \rotatebox{90}{
        \begin{minipage}{.08\textwidth}
        \centering 
        \small $10$
    \end{minipage}} \hfill
    \begin{subfigure}[t]{.14\textwidth}
        \includegraphics[width=\linewidth]{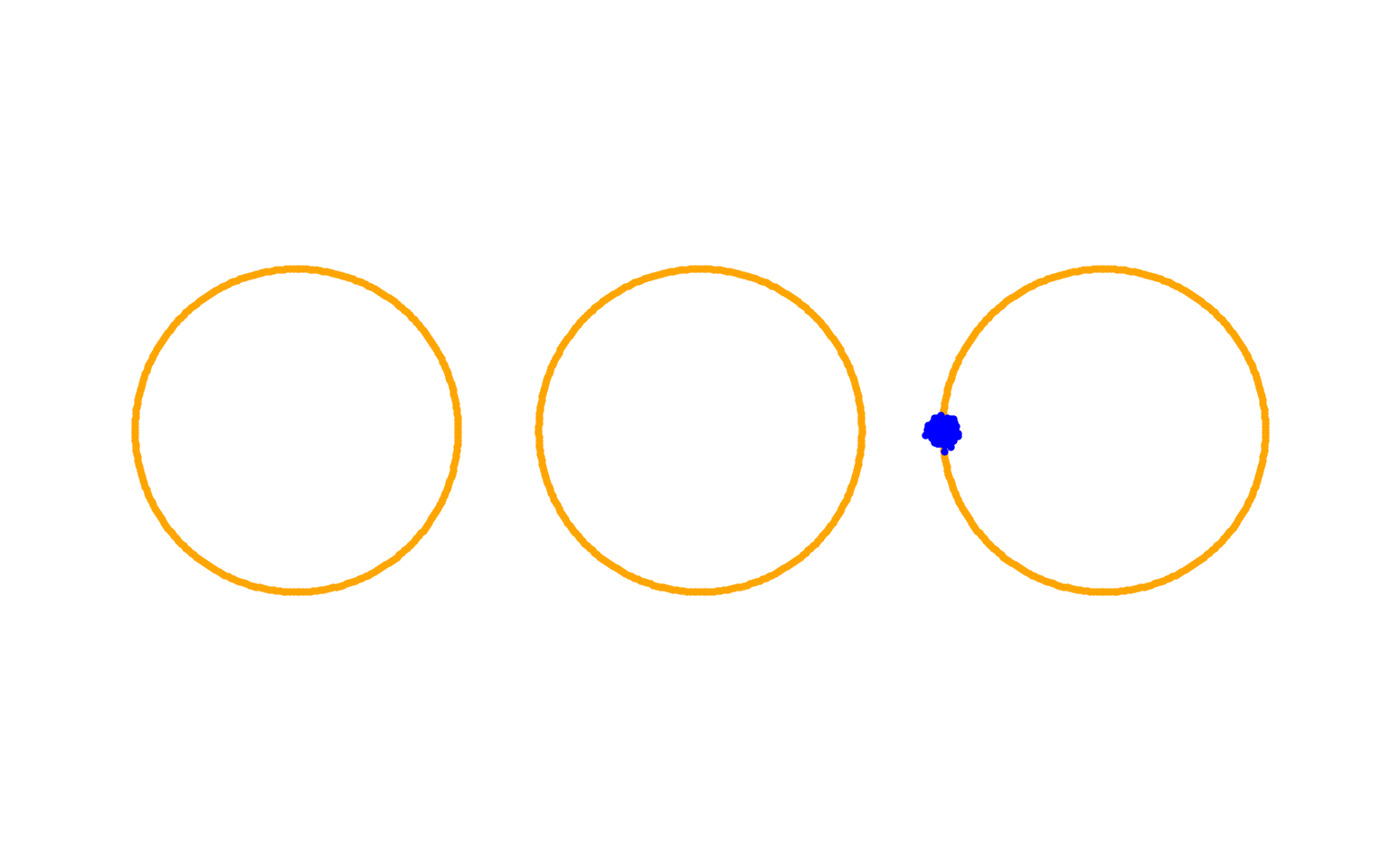}
    \caption*{t=0}
    \end{subfigure}%
    \begin{subfigure}[t]{.14\textwidth}
        \includegraphics[width=\linewidth]{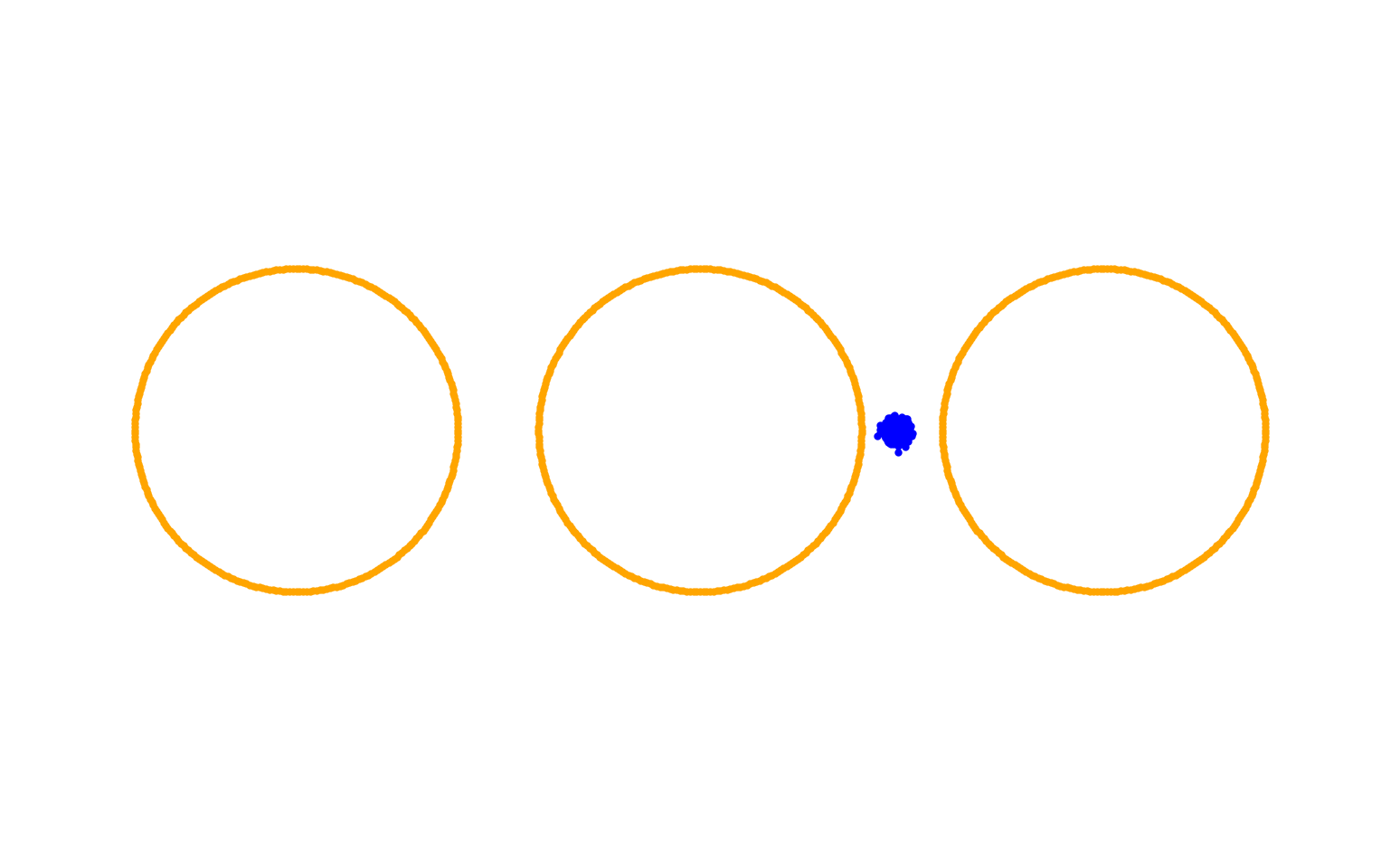}
    \caption*{t=0.1}
    \end{subfigure}%
    \begin{subfigure}[t]{.14\textwidth}
        \includegraphics[width=\linewidth]{{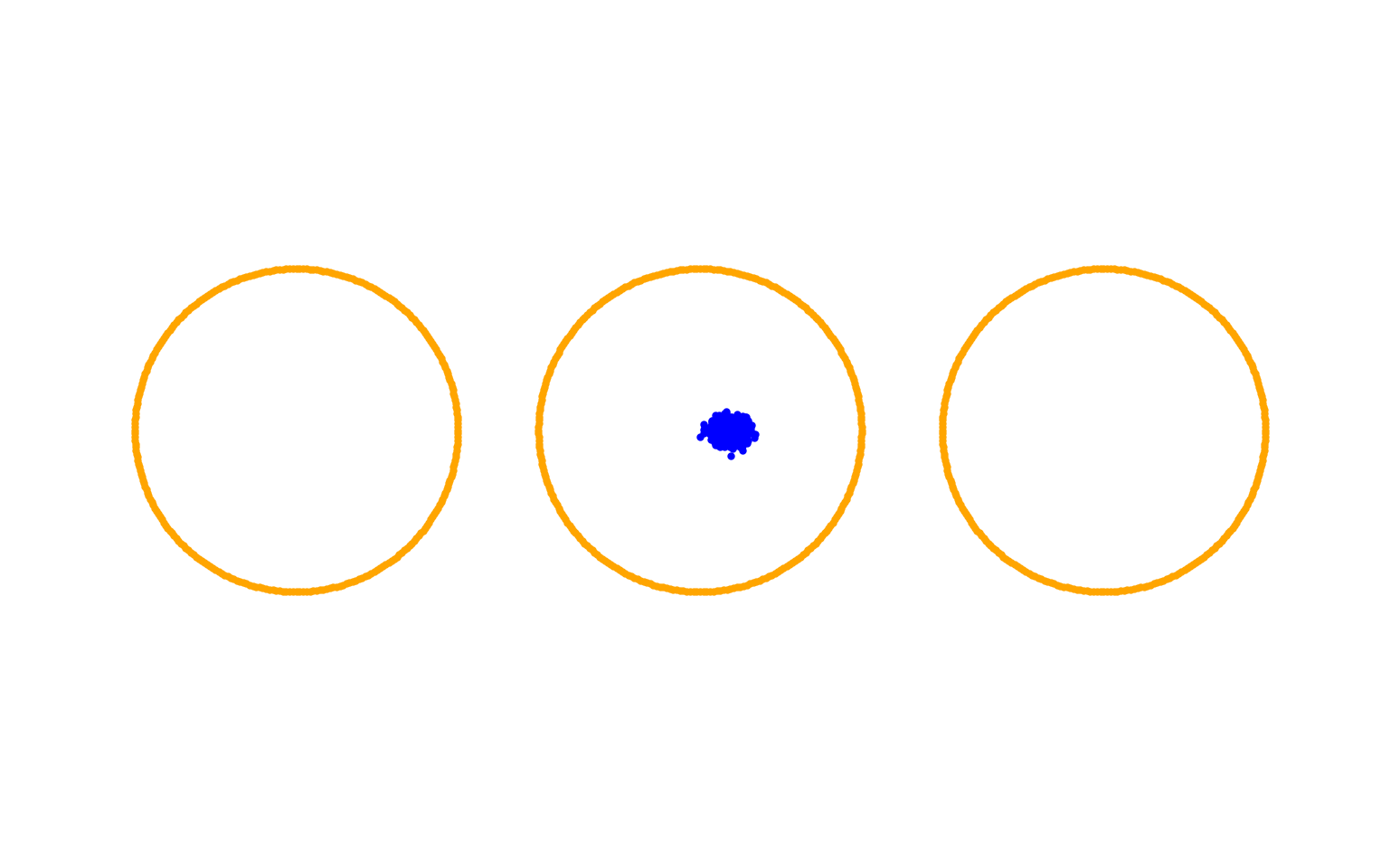}}
    \caption*{t=1}
    \end{subfigure}%
    \begin{subfigure}[t]{.14\textwidth}
        \includegraphics[width=\linewidth]{{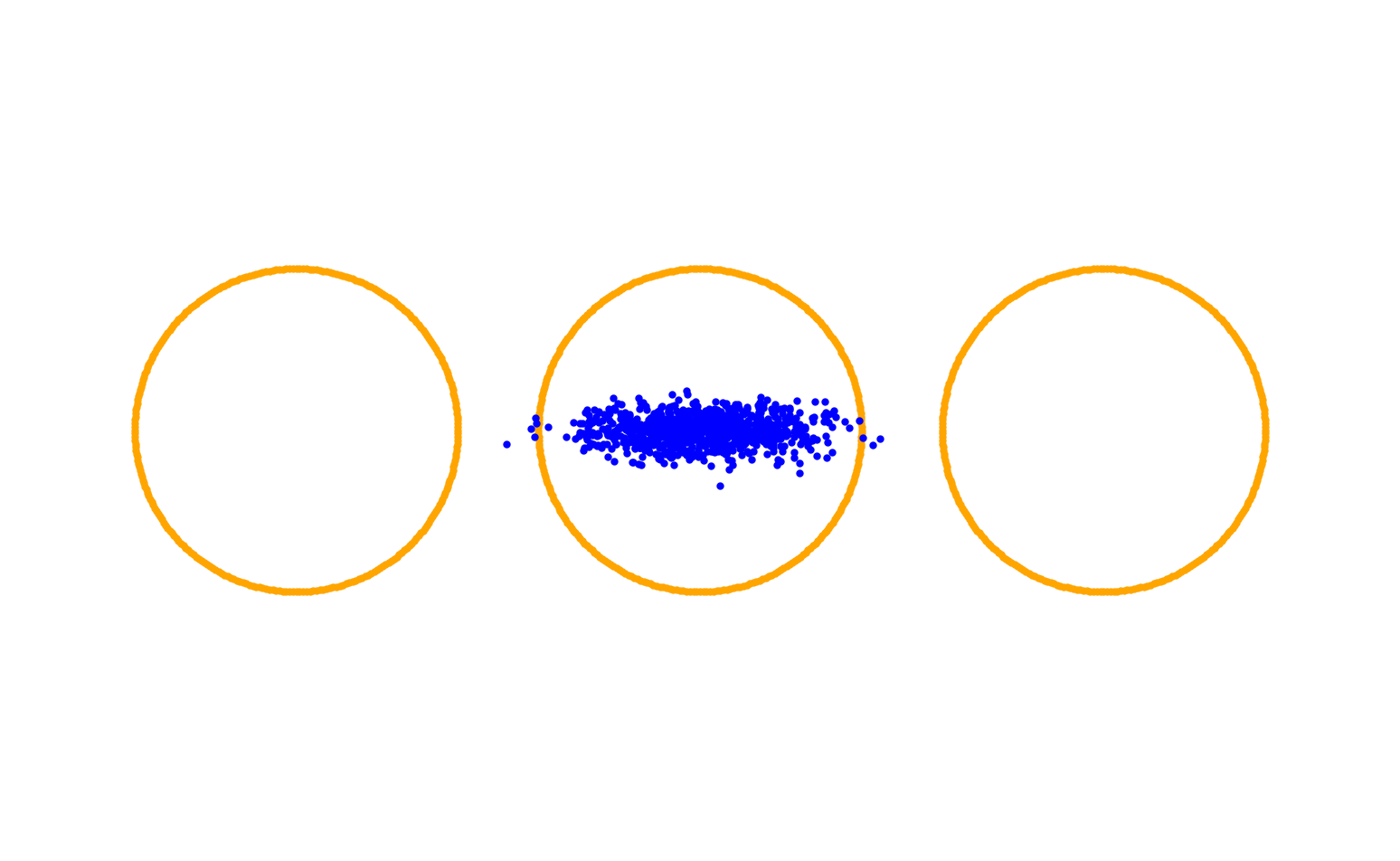}}
    \caption*{t=3}
    \end{subfigure}%
    \begin{subfigure}[t]{.14\textwidth}
        \includegraphics[width=\linewidth]{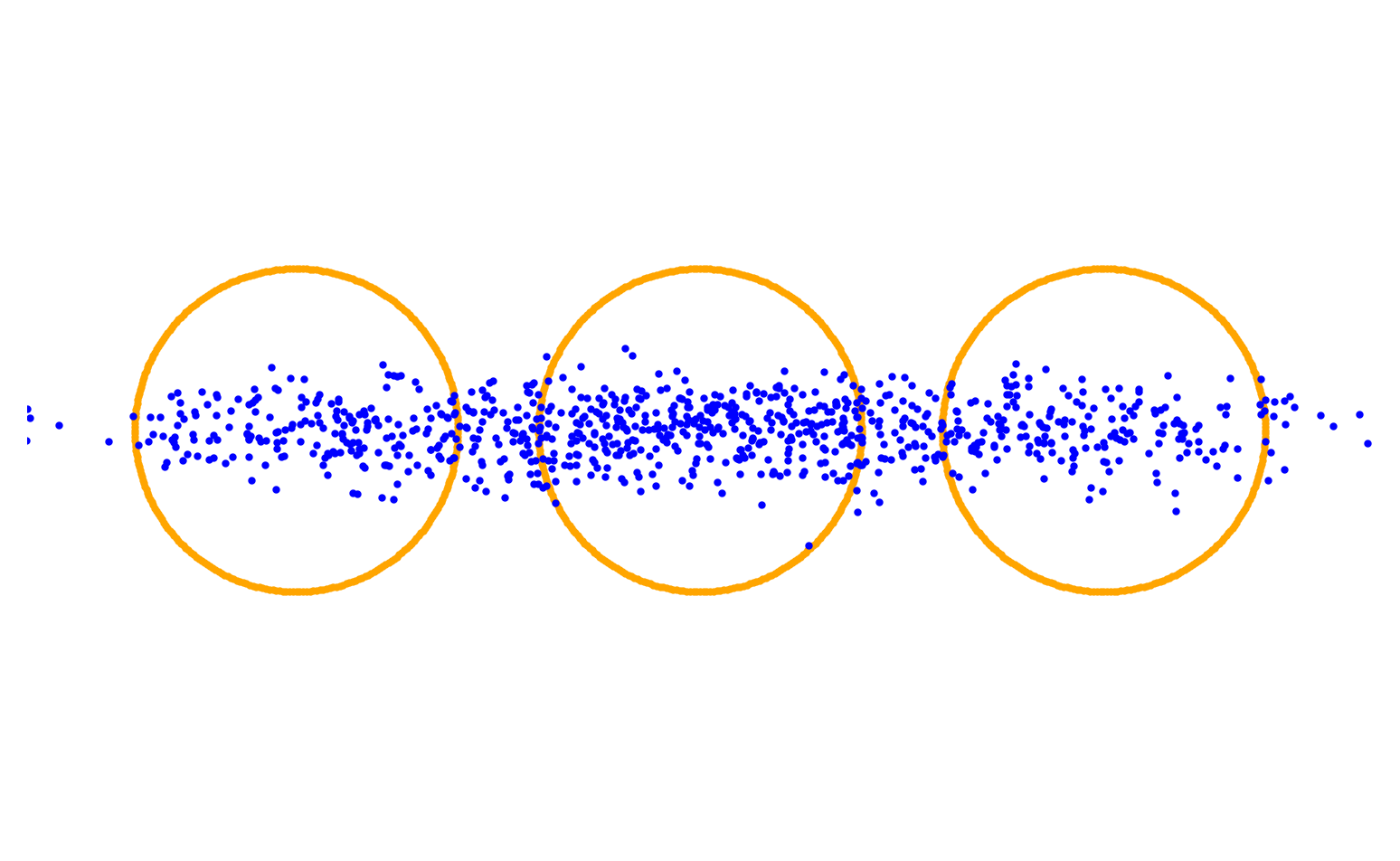}
    \caption*{t=10}
    \end{subfigure}%
    \begin{subfigure}[t]{.14\textwidth}
        \includegraphics[width=\linewidth]{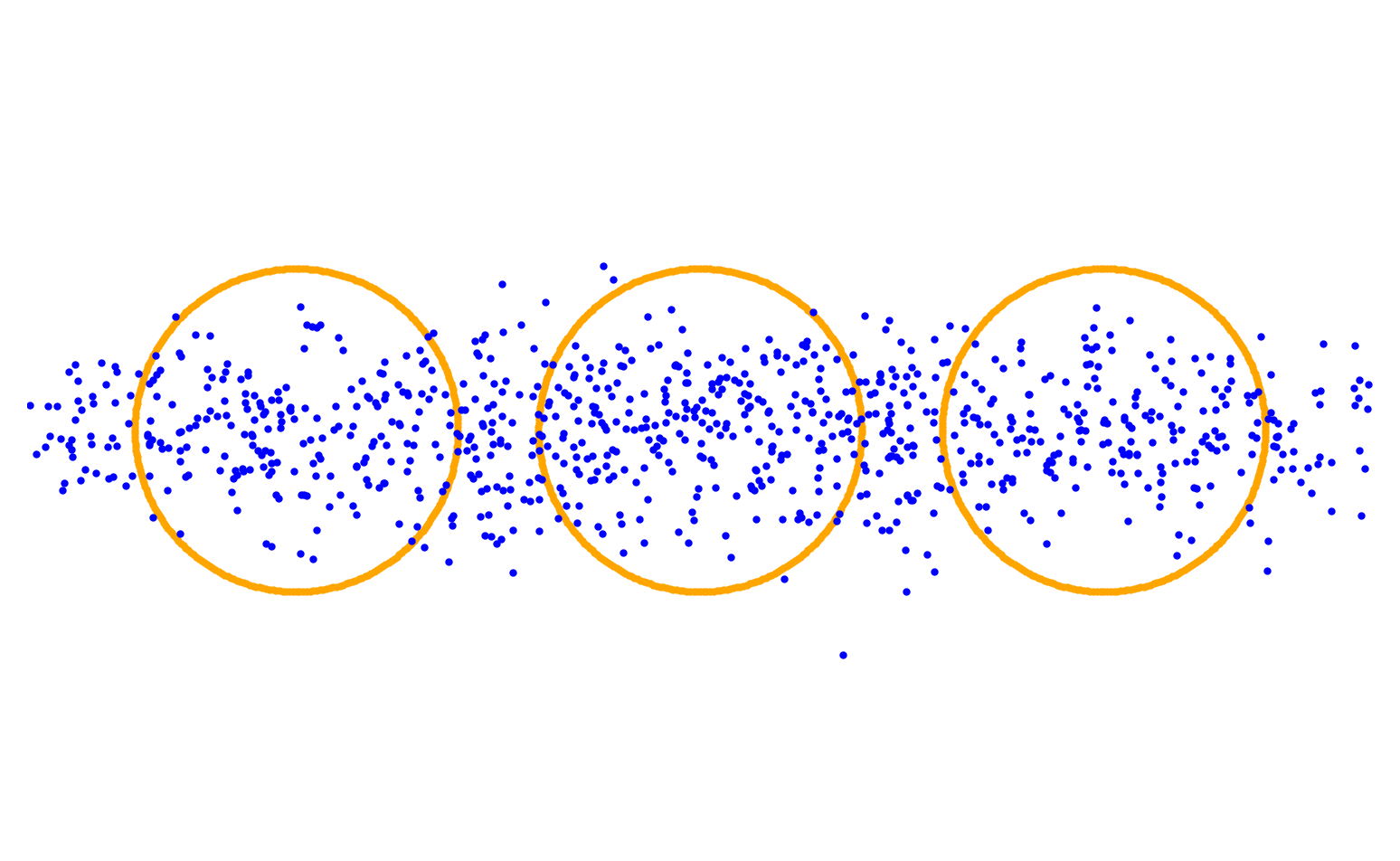}
    \caption*{t=25}
    \end{subfigure}%
    \caption{Ablation study for the parameter $\sigma^2$ of the inverse multiquadric kernel.}
    \label{fig:ablationCirclesIMQtau}
\end{figure}

The Matérn kernel with smoothness parameter $\frac{3}{2}$, see \cite[Subsec.~4.2.1]{RW2006}, reads
\begin{equation*}
    k_{\sigma^2, \frac{3}{2}}(x, y)
    = \left( 1 + \frac{\sqrt{3}}{\sigma^2} \| x - y \|_2\right) \exp\left( - \frac{\sqrt{3}}{\sigma^2} \| x - y \|_2\right),
\end{equation*}
The shape of the flow for different kernel widths $\sigma^2$ is quite different.
In Figure~\ref{fig:ablationCirclesMaterntau}, we choose $\alpha = 3$, $\lambda = \num{e-2}$, $\tau = \num{e-3}$ and $N = 900$.
We can see that when the width $\sigma^2 = \num{e-3}$ is too small, then the particles barely move.
The particles only spread on the two rightmost rings for $\sigma^2 = \num{e-2}$.
For $\sigma^2 = \num{e-1}$, the particles initially spread only onto two rings but also match the third one after some time.
The width $\sigma^2 = 1$ performs best, and there are no outliers.
For $\sigma^2 = 10$, the behavior is very similar to the case $\sigma^2 = 10$ in Figure~\ref{fig:ablationCirclesIMQtau}.
For $\sigma^2 = \num{e2}$, we observe an extreme case of mode-seeking behavior: the particles do not spread and only move towards the mean of the target distribution.
\begin{figure}[pt]
    \centering
    \rotatebox{90}{
        \begin{minipage}{.05\textwidth}
        \centering 
        \small $\num{e-3}$
    \end{minipage}} \hfill
    \begin{subfigure}[t]{.14\textwidth}
        \includegraphics[width=\linewidth, valign=c]{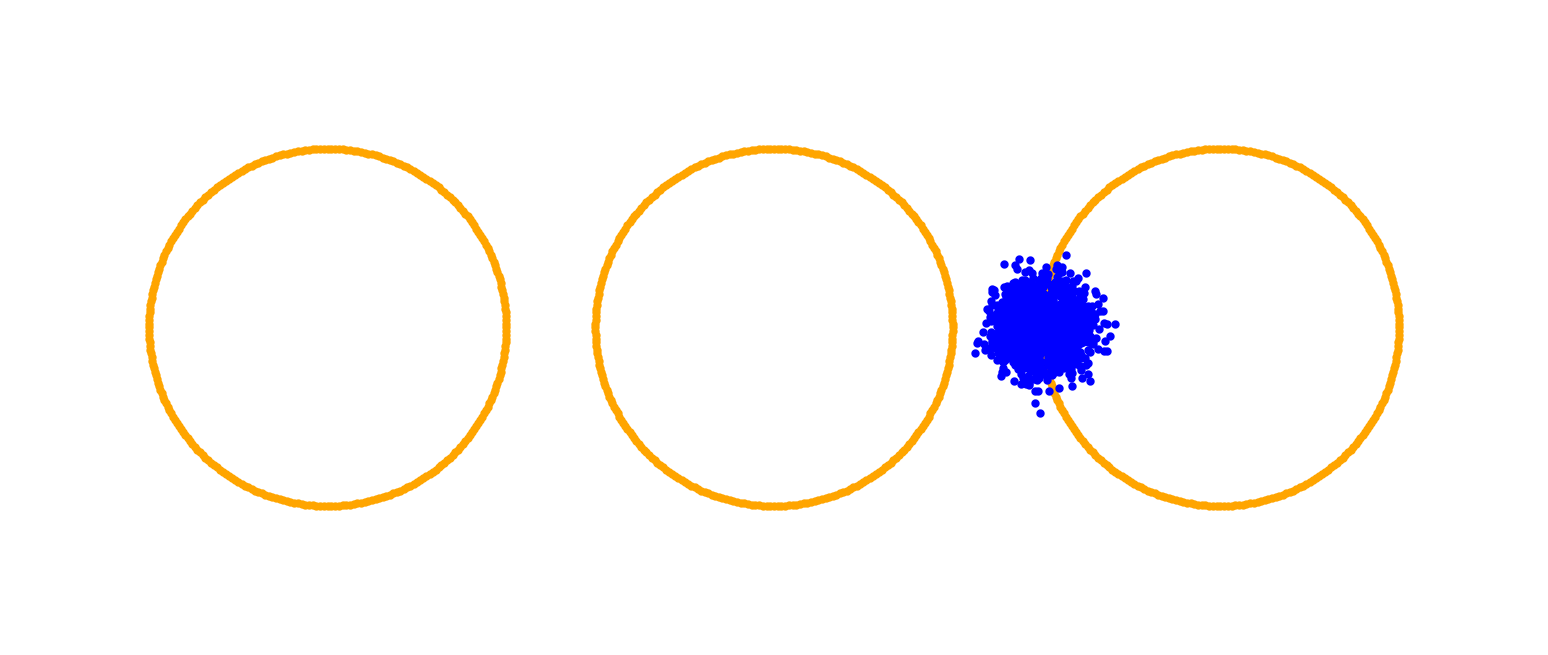}
    \end{subfigure}%
    \begin{subfigure}[t]{.14\textwidth}
        \includegraphics[width=\linewidth, valign=c]{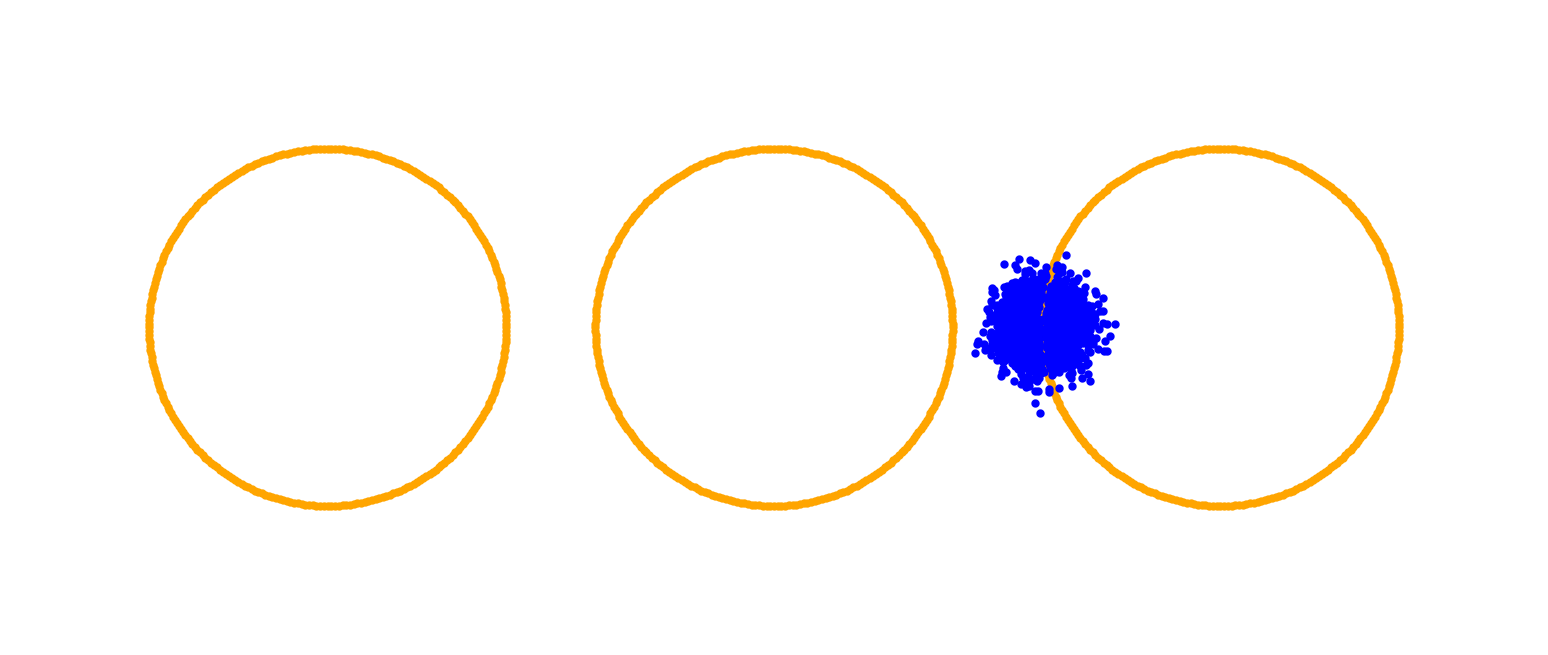}
    \end{subfigure}%
    \begin{subfigure}[t]{.14\textwidth}
        \includegraphics[width=\linewidth, valign=c]{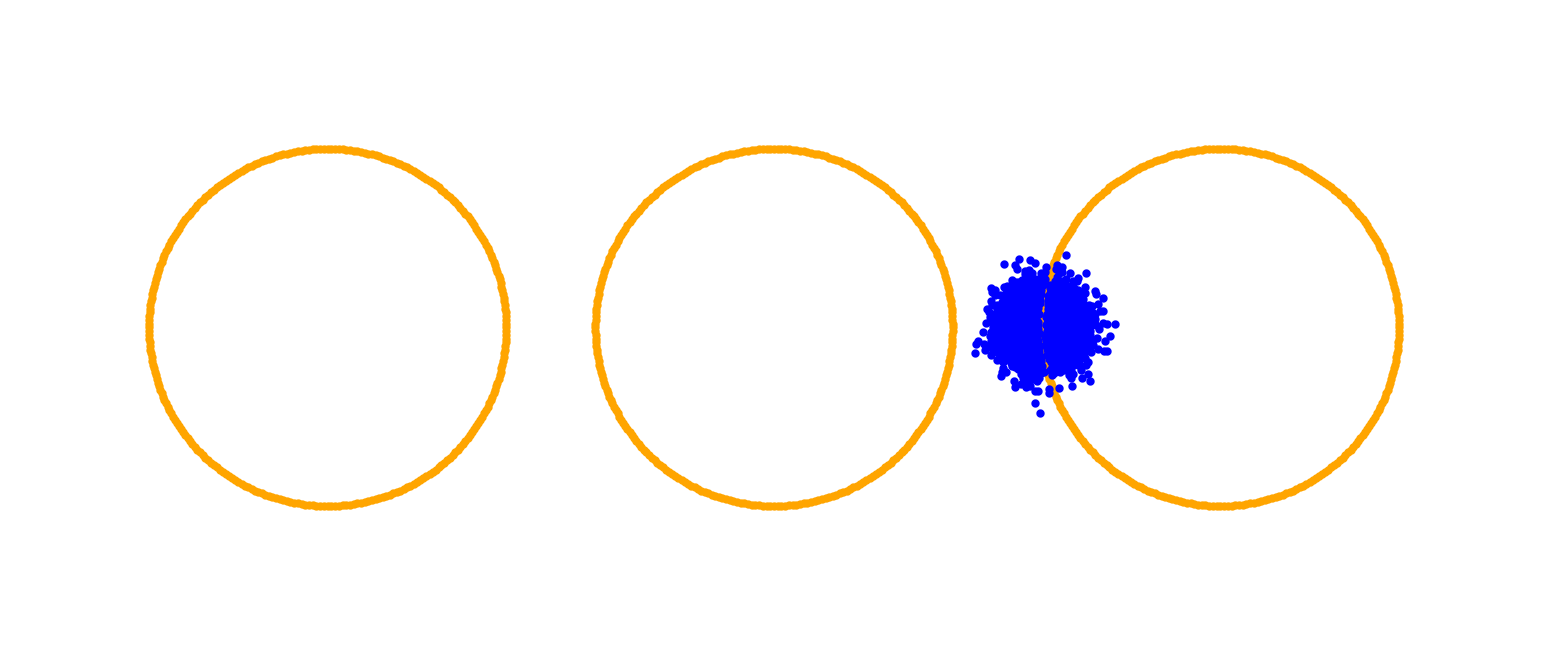}
    \end{subfigure}%
    \begin{subfigure}[t]{.14\textwidth}
        \includegraphics[width=\linewidth, valign=c]{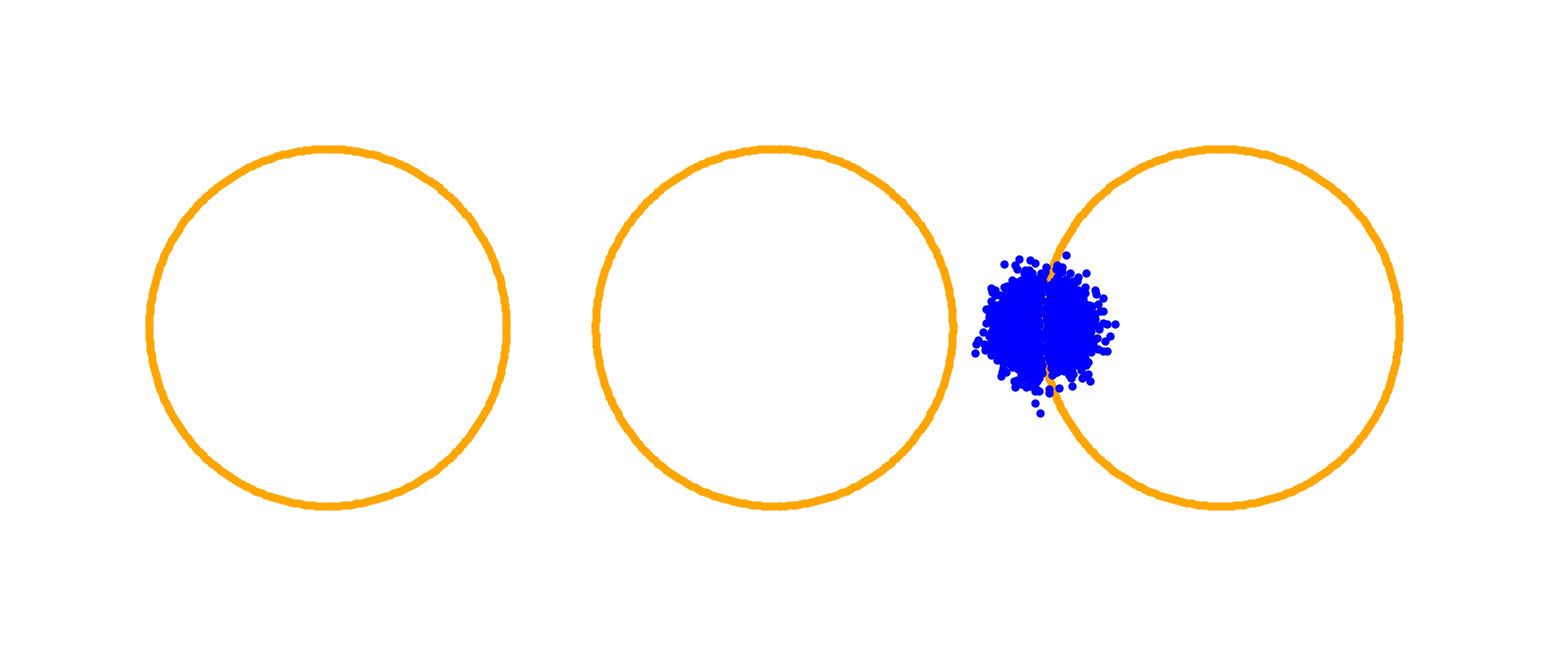}
    \end{subfigure}%
    \begin{subfigure}[t]{.14\textwidth}
        \includegraphics[width=\linewidth, valign=c]{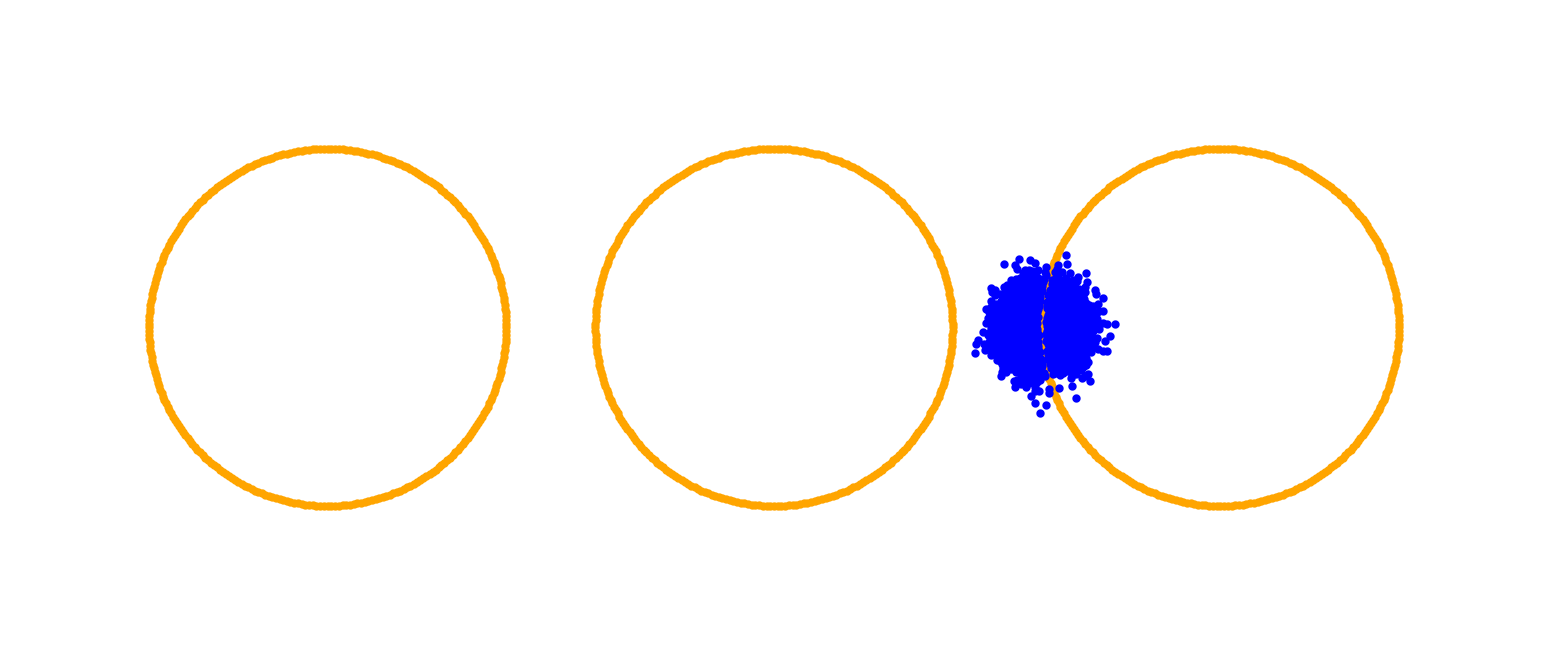}
    \end{subfigure}%
    \begin{subfigure}[t]{.14\textwidth}
        \includegraphics[width=\linewidth, valign=c]{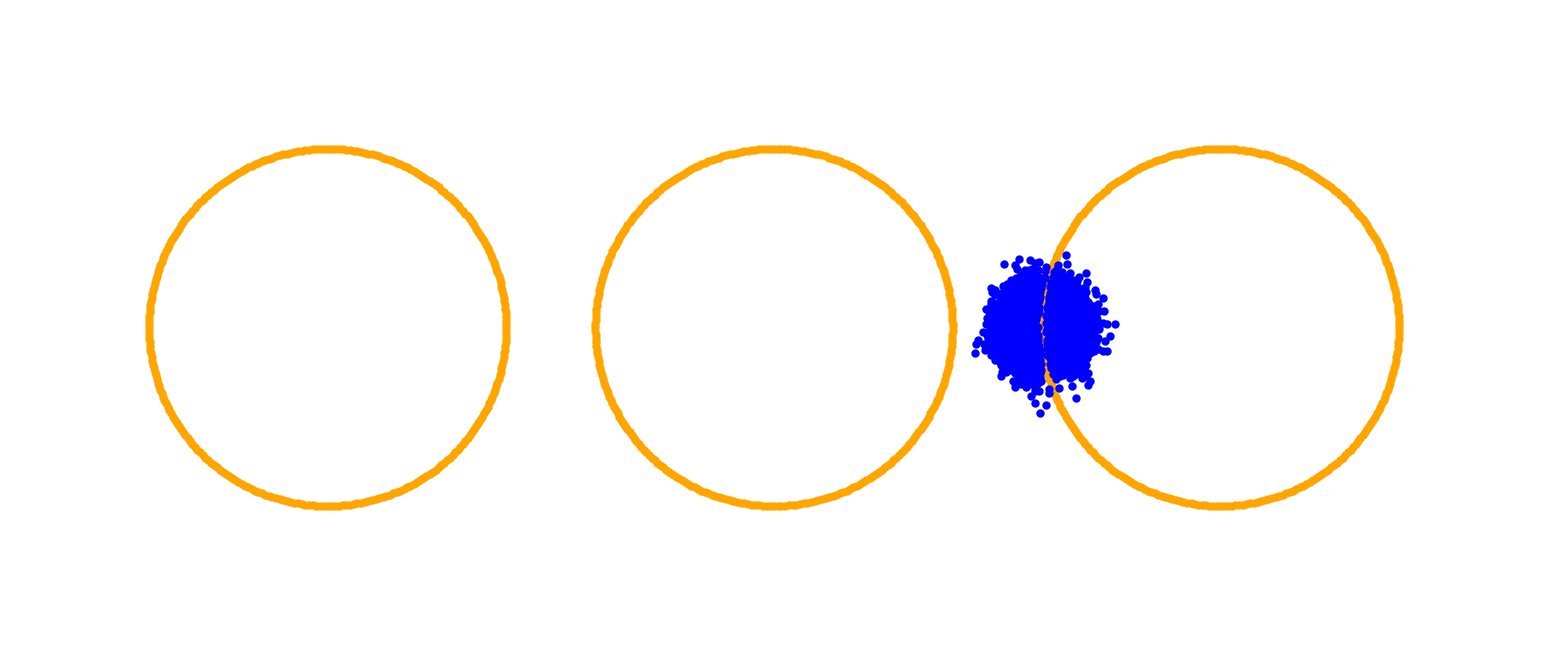}
    \end{subfigure} \\ %
    \rotatebox{90}{
        \begin{minipage}{.07\textwidth}
        \centering 
        \small $\num{e-2}$
    \end{minipage}} \hfill
    \begin{subfigure}[t]{.14\textwidth}
        \includegraphics[width=\linewidth]{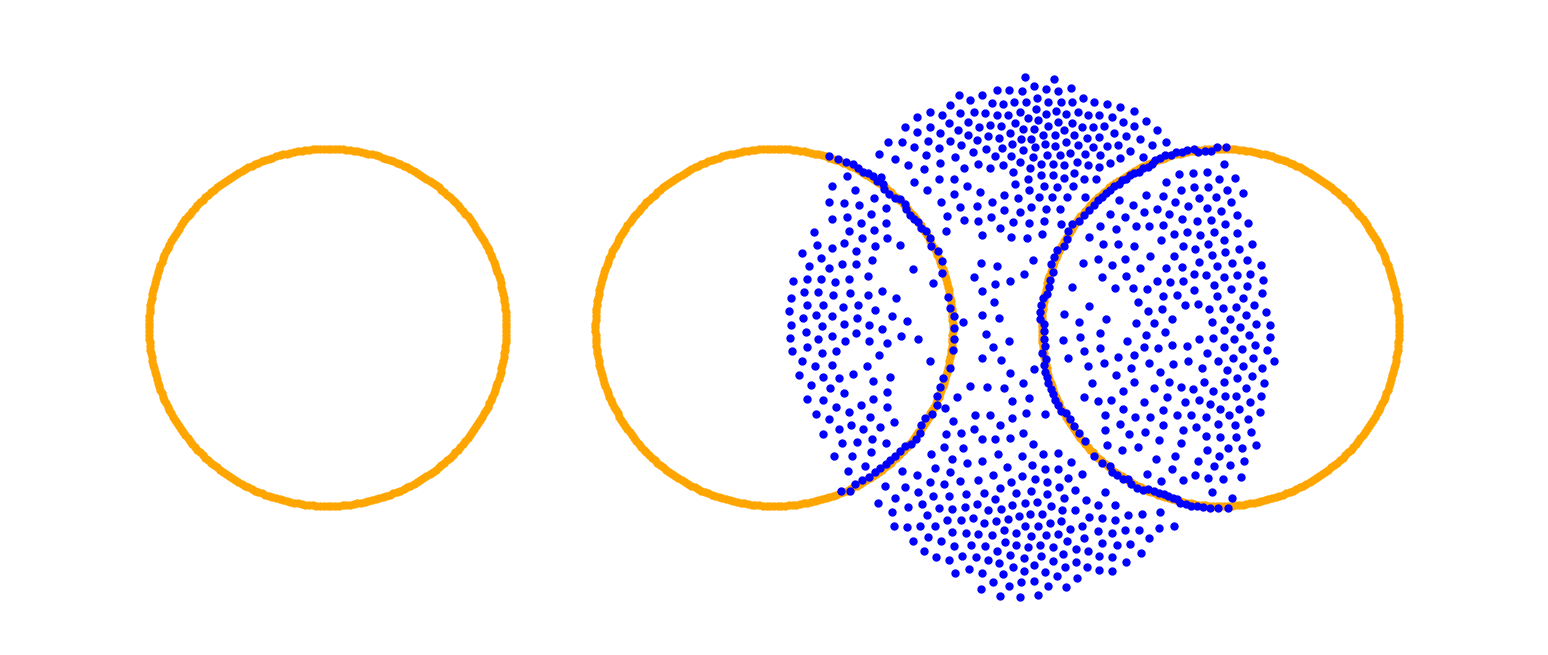}
    \end{subfigure}%
    \begin{subfigure}[t]{.14\textwidth}
        \includegraphics[width=\linewidth]{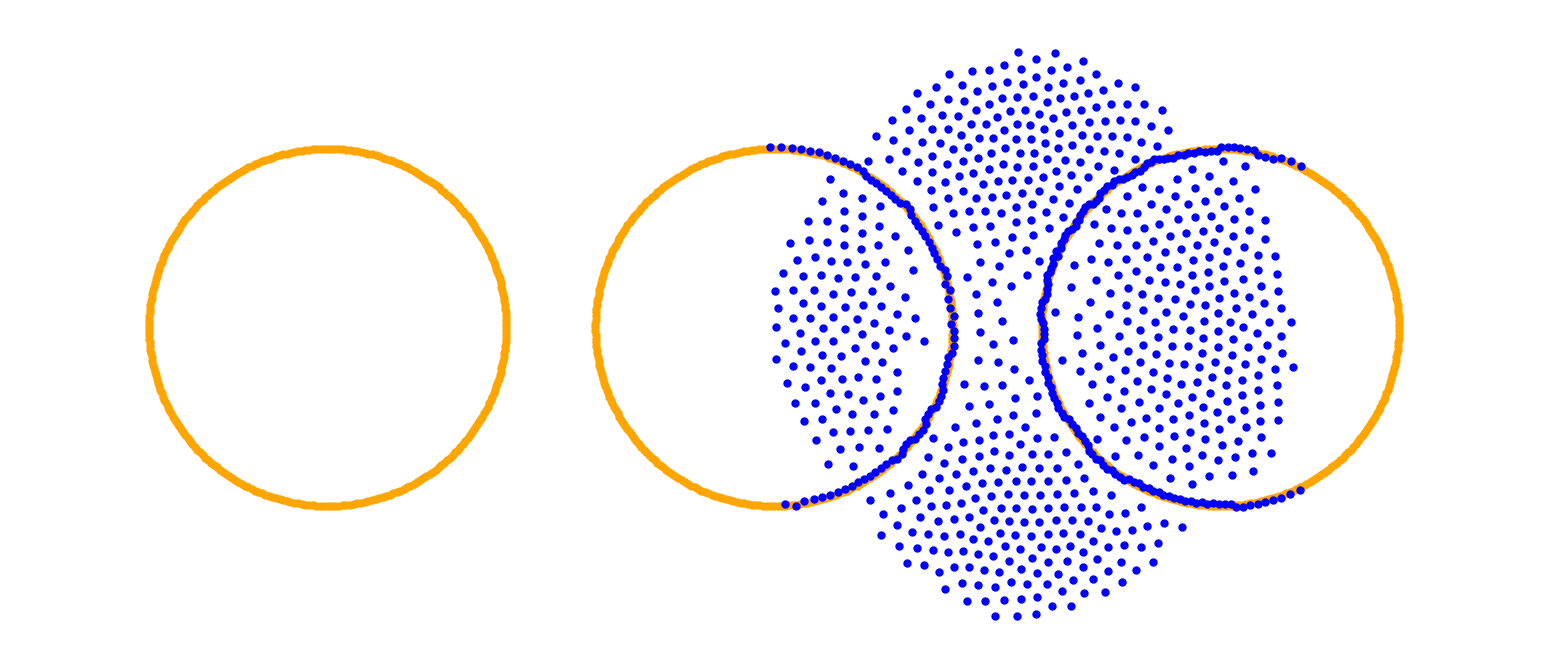}
    \end{subfigure}%
    \begin{subfigure}[t]{.14\textwidth}
        \includegraphics[width=\linewidth]{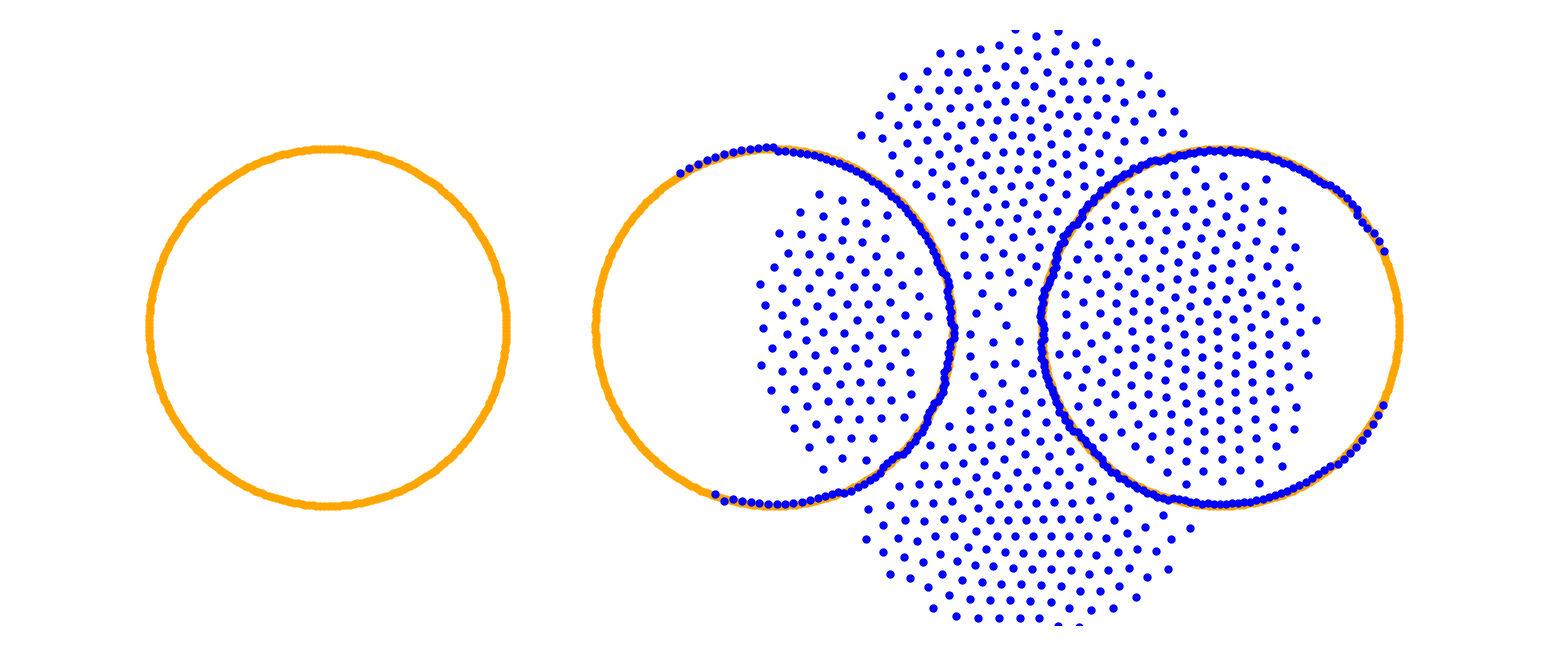}
    \end{subfigure}%
    \begin{subfigure}[t]{.14\textwidth}
        \includegraphics[width=\linewidth]{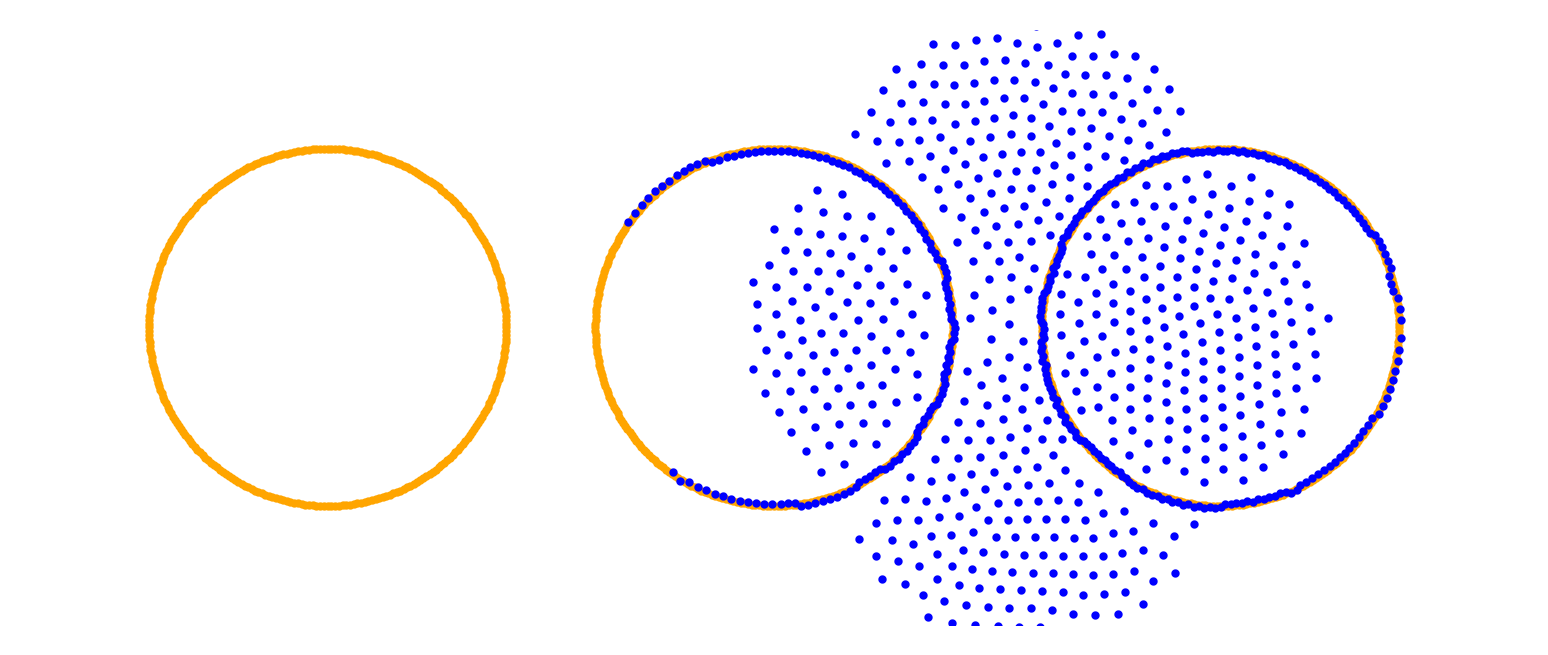}
    \end{subfigure}%
    \begin{subfigure}[t]{.14\textwidth}
        \includegraphics[width=\linewidth]{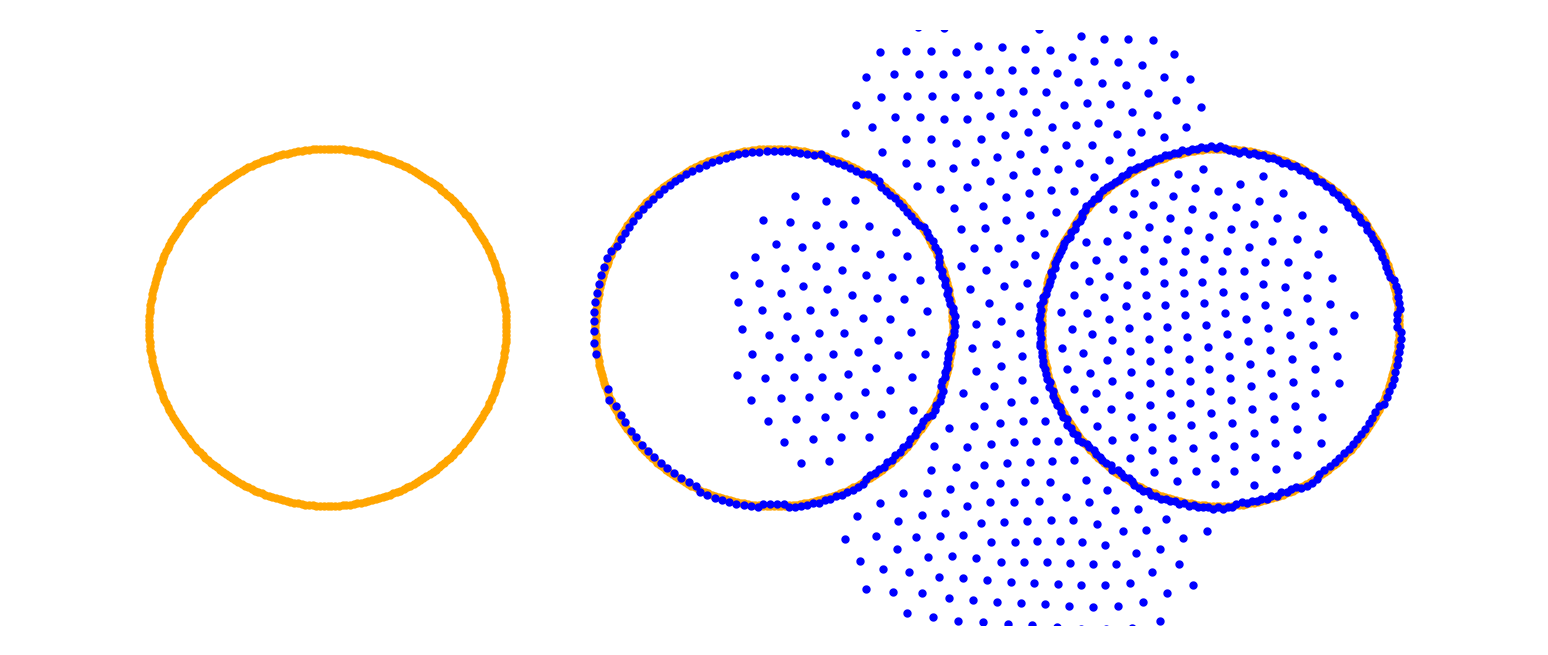}
    \end{subfigure}%
    \begin{subfigure}[t]{.14\textwidth}
        \includegraphics[width=\linewidth]{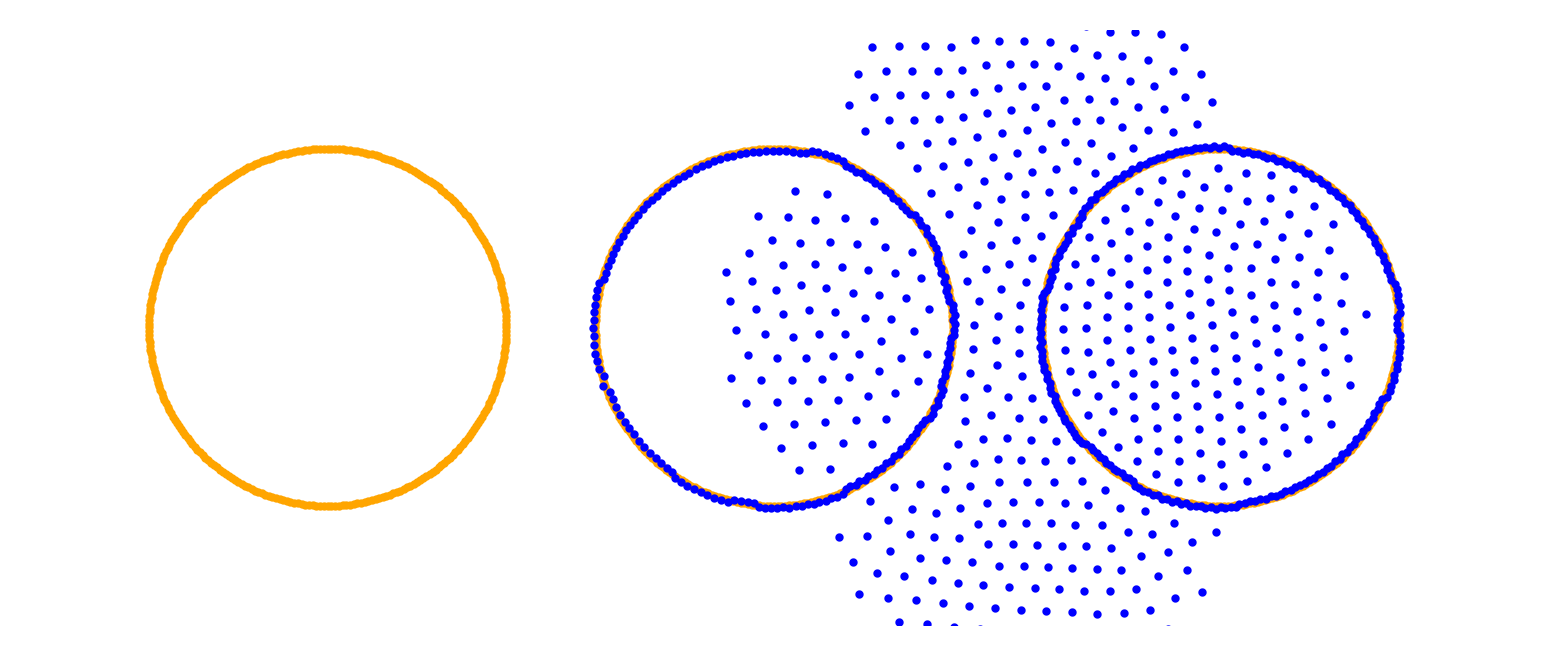}
    \end{subfigure} \\ %
    \rotatebox{90}{
        \begin{minipage}{.07\textwidth}
        \centering 
        \small $\num{e-1}$
    \end{minipage}} \hfill
    \begin{subfigure}[t]{.14\textwidth}
        \includegraphics[width=\linewidth]{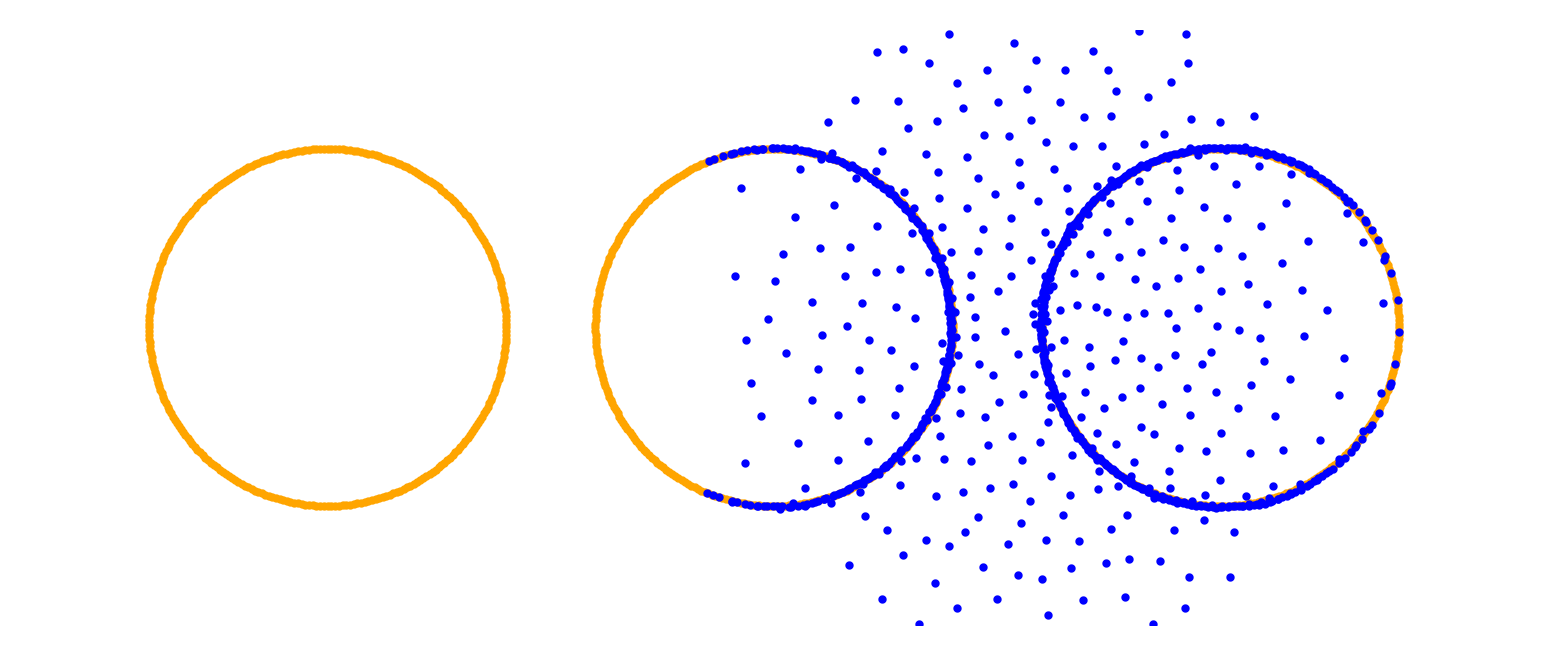}
    \end{subfigure}%
    \begin{subfigure}[t]{.14\textwidth}
        \includegraphics[width=\linewidth]{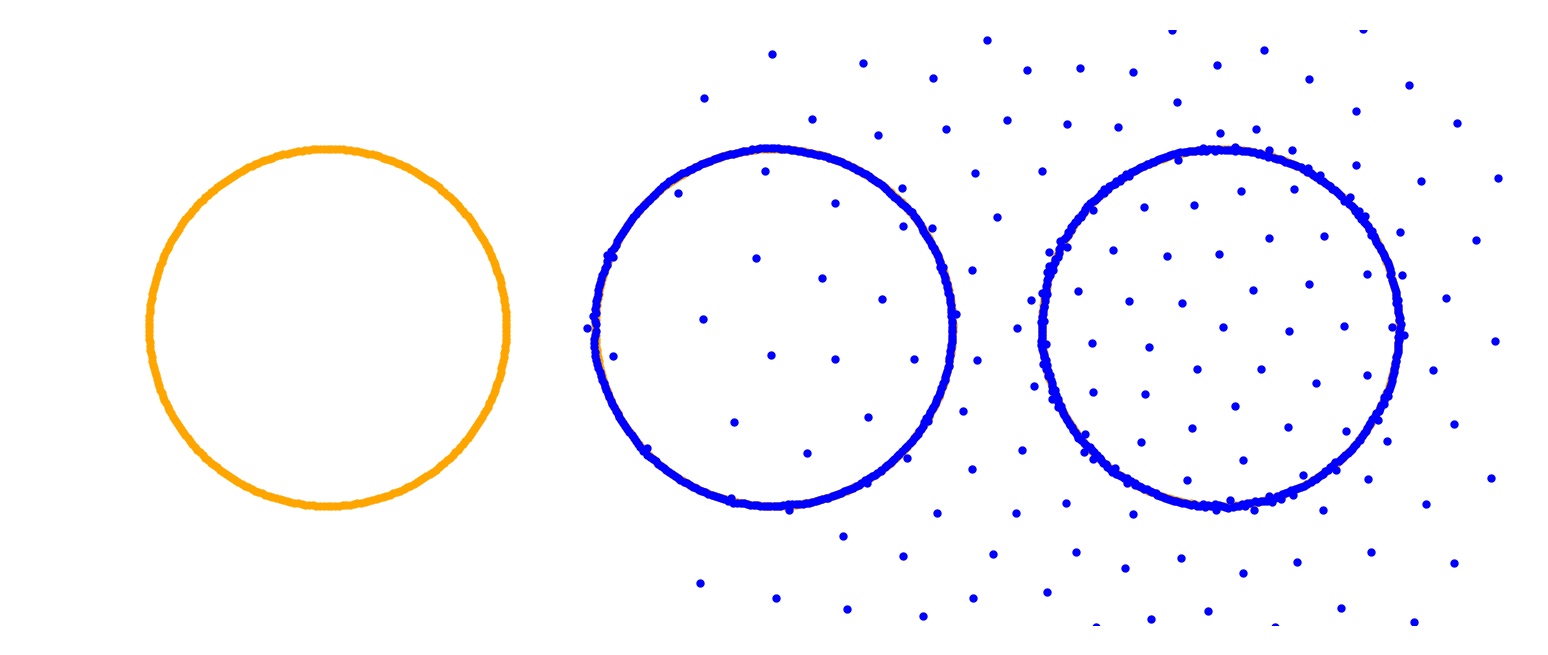}
    \end{subfigure}%
    \begin{subfigure}[t]{.14\textwidth}
        \includegraphics[width=\linewidth]{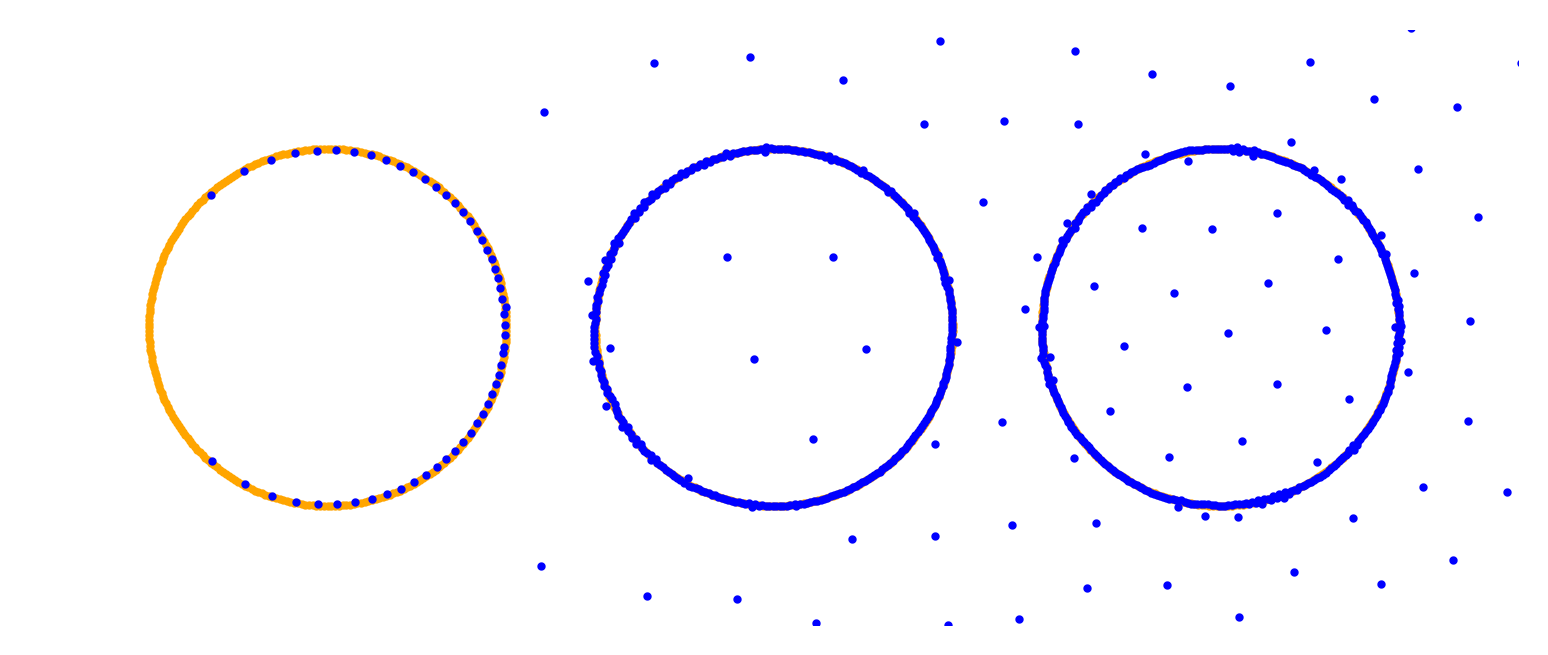}
    \end{subfigure}%
    \begin{subfigure}[t]{.14\textwidth}
        \includegraphics[width=\linewidth]{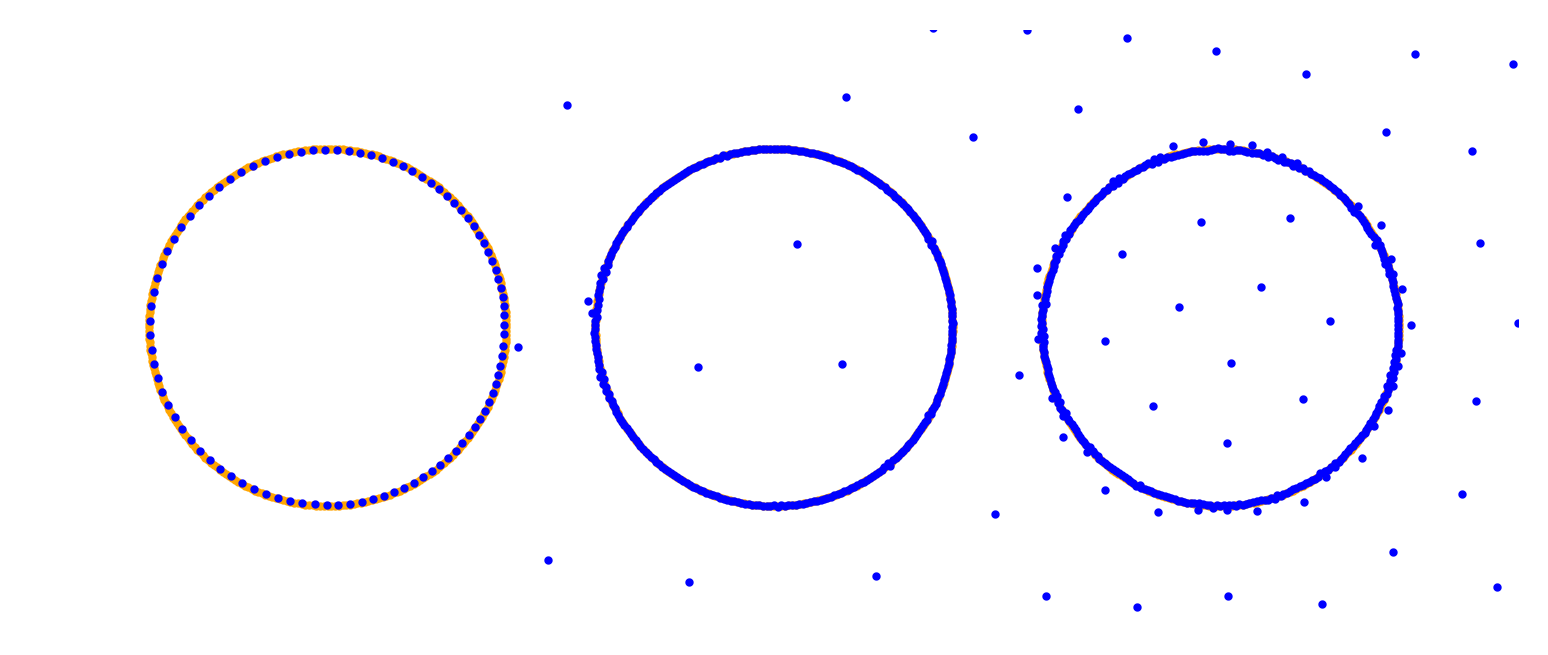}
    \end{subfigure}%
    \begin{subfigure}[t]{.14\textwidth}
        \includegraphics[width=\linewidth]{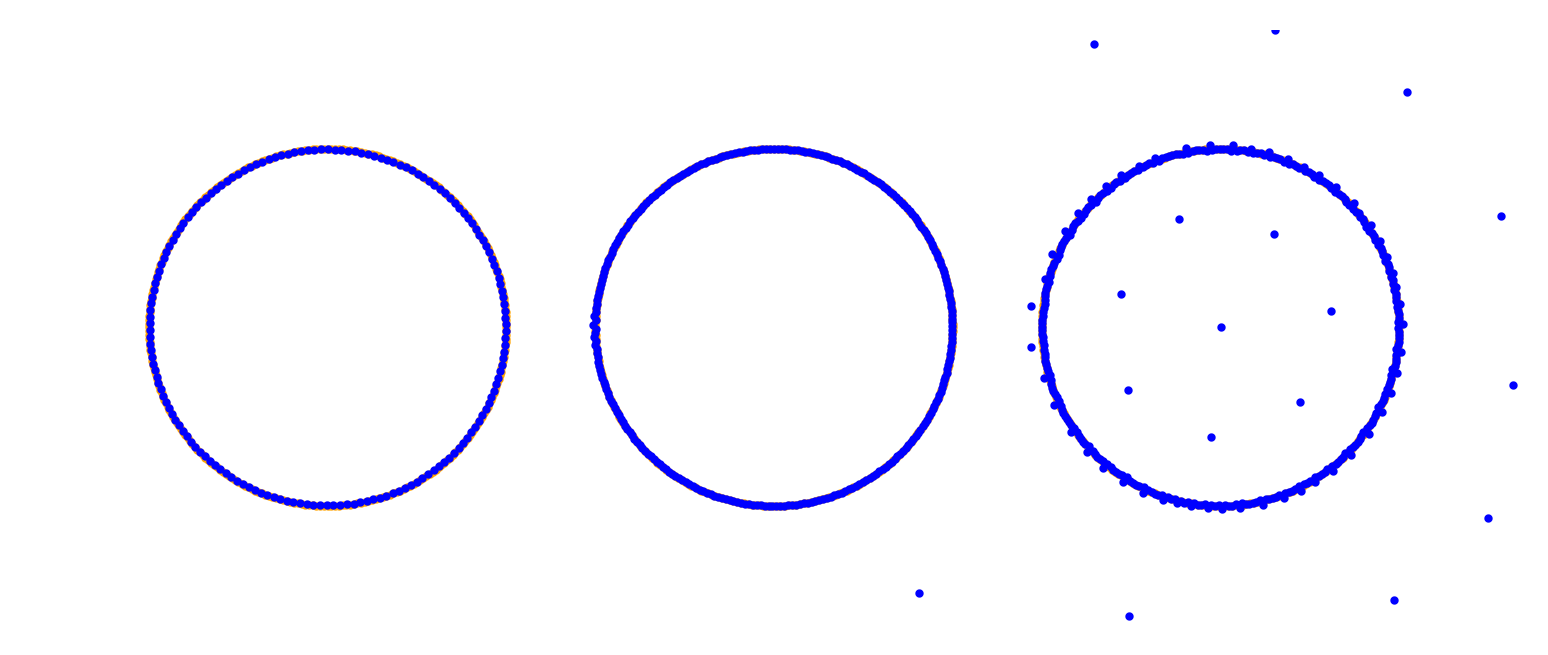}
    \end{subfigure}%
    \begin{subfigure}[t]{.14\textwidth}
        \includegraphics[width=\linewidth]{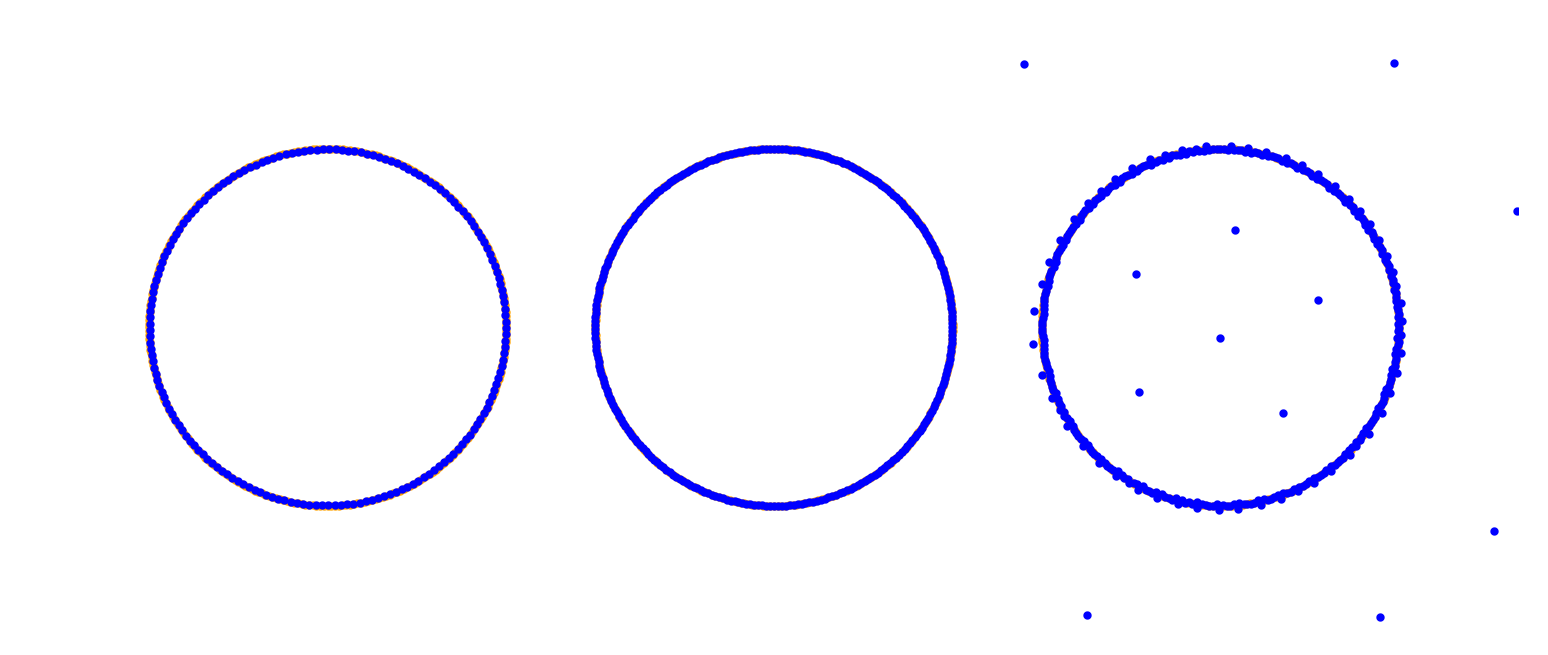}
    \end{subfigure} \\ %
     \rotatebox{90}{
        \begin{minipage}{.07\textwidth}
        \centering 
        \small $1$
    \end{minipage}} \hfill
    \begin{subfigure}[t]{.14\textwidth}
        \includegraphics[width=\linewidth]{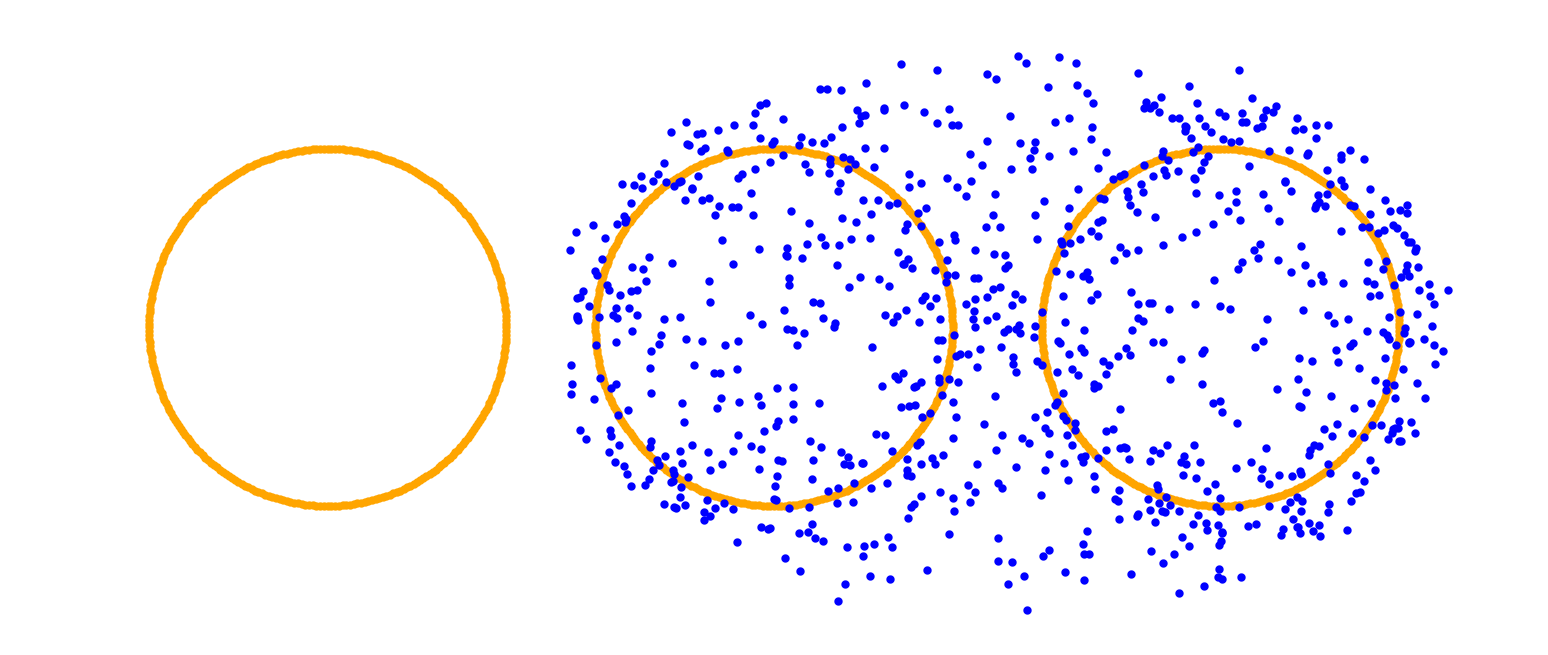}
    \end{subfigure}%
    \begin{subfigure}[t]{.14\textwidth}
        \includegraphics[width=\linewidth]{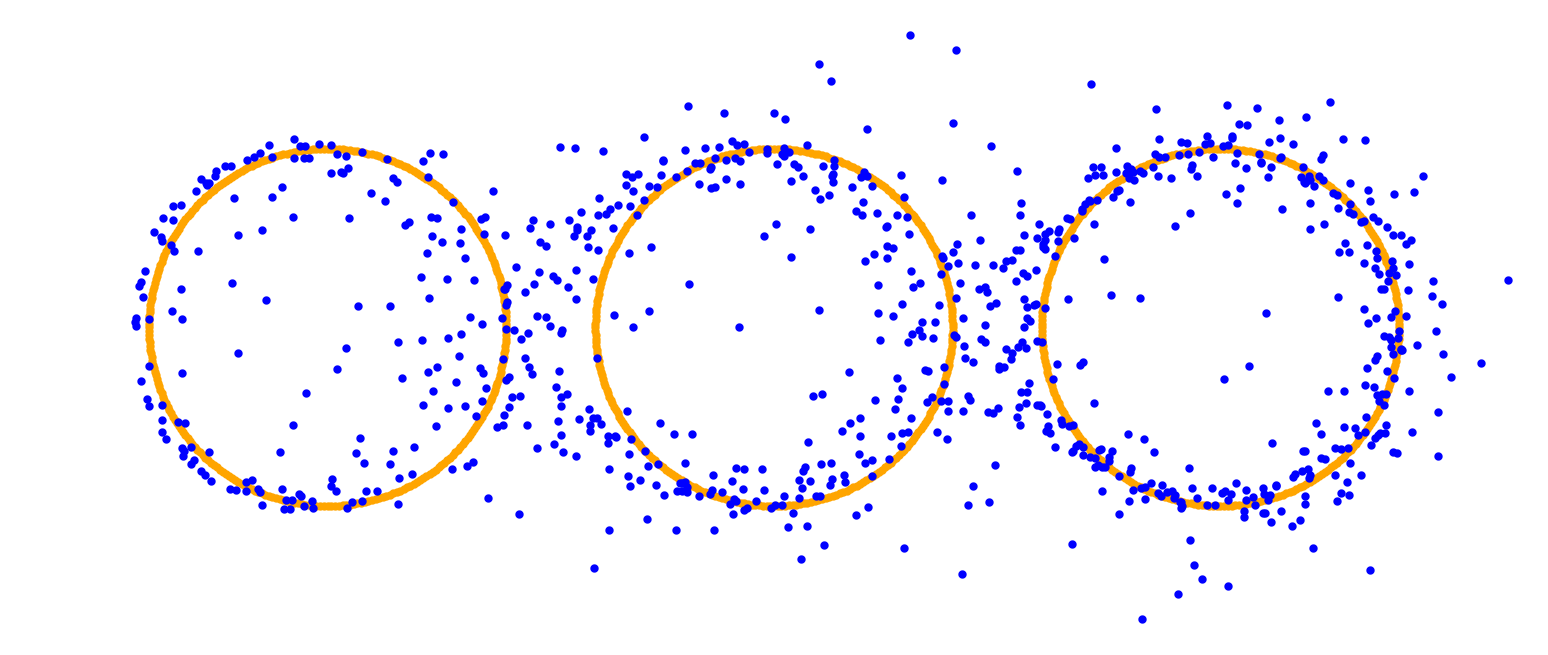}
    \end{subfigure}%
    \begin{subfigure}[t]{.14\textwidth}
        \includegraphics[width=\linewidth]{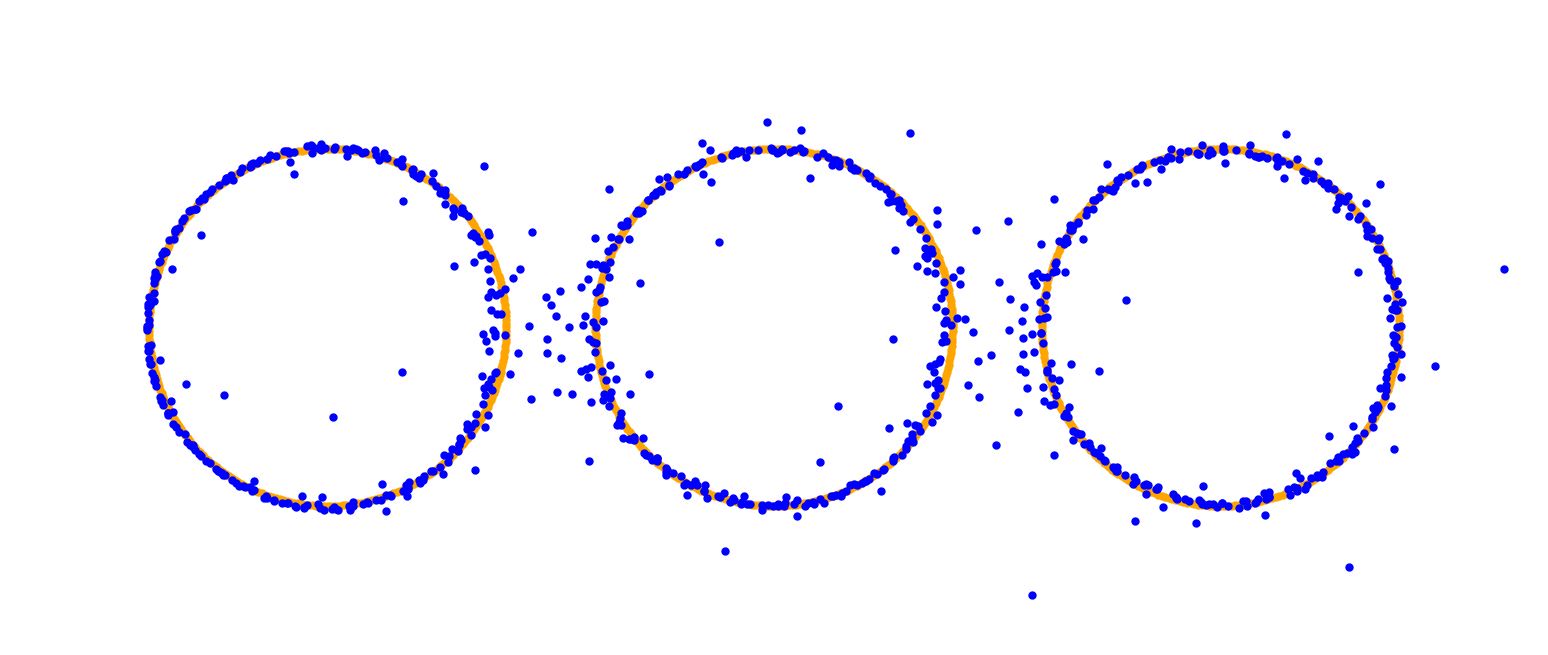}
    \end{subfigure}%
    \begin{subfigure}[t]{.14\textwidth}
        \includegraphics[width=\linewidth]{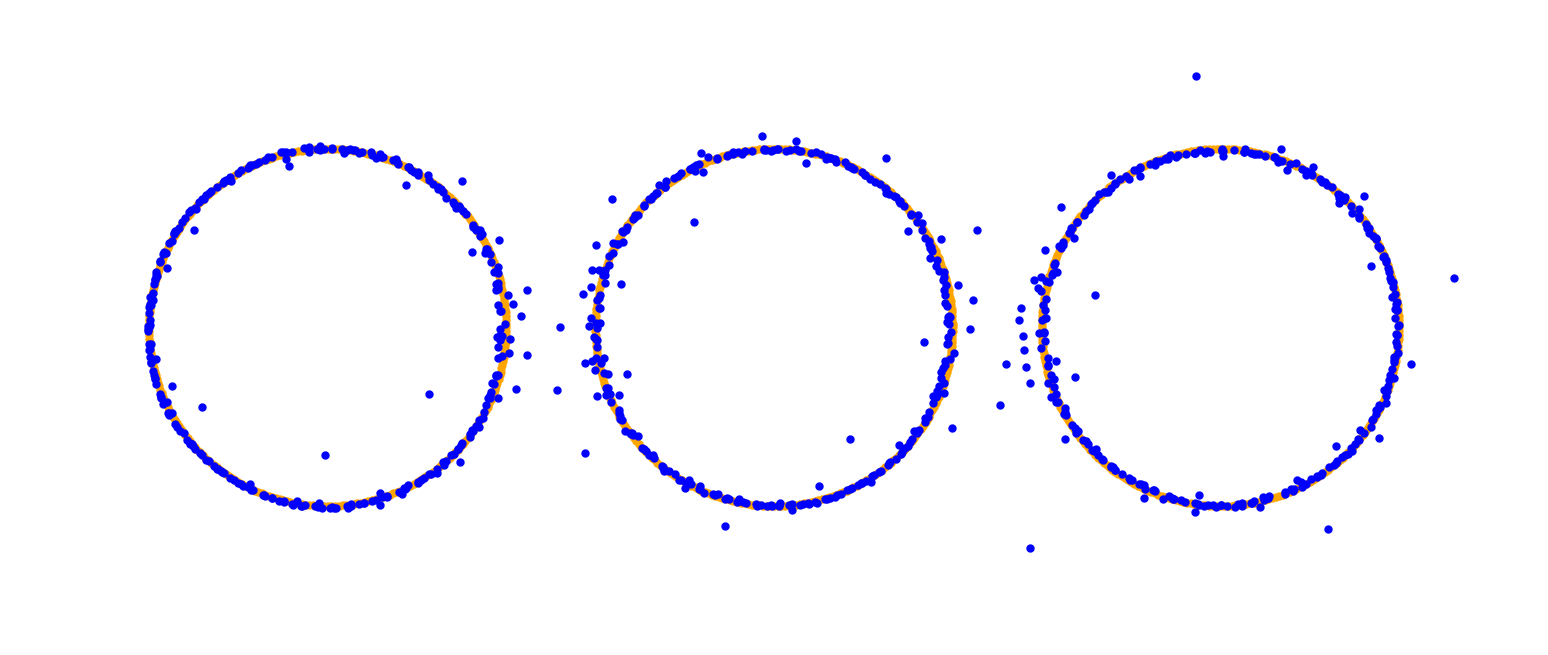}
    \end{subfigure}%
    \begin{subfigure}[t]{.14\textwidth}
        \includegraphics[width=\linewidth]{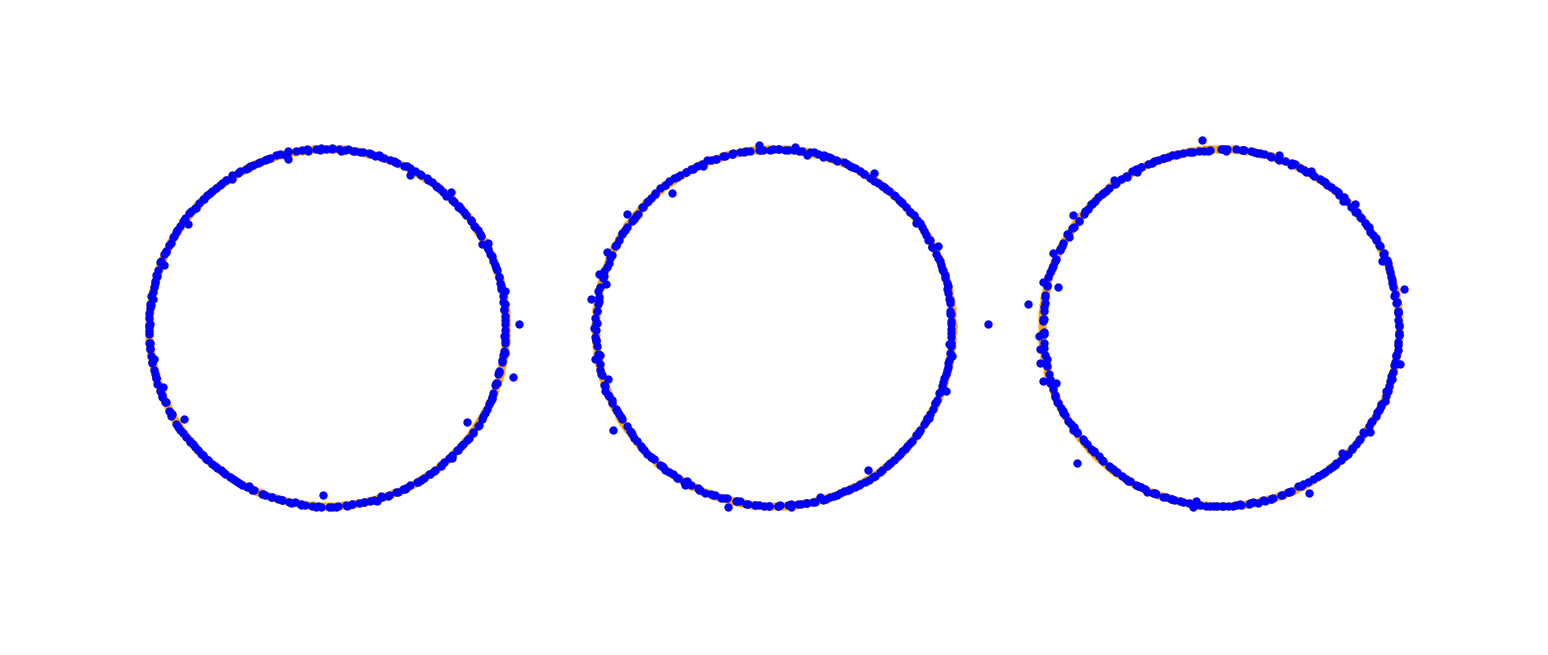}
    \end{subfigure}%
    \begin{subfigure}[t]{.14\textwidth}
        \includegraphics[width=\linewidth]{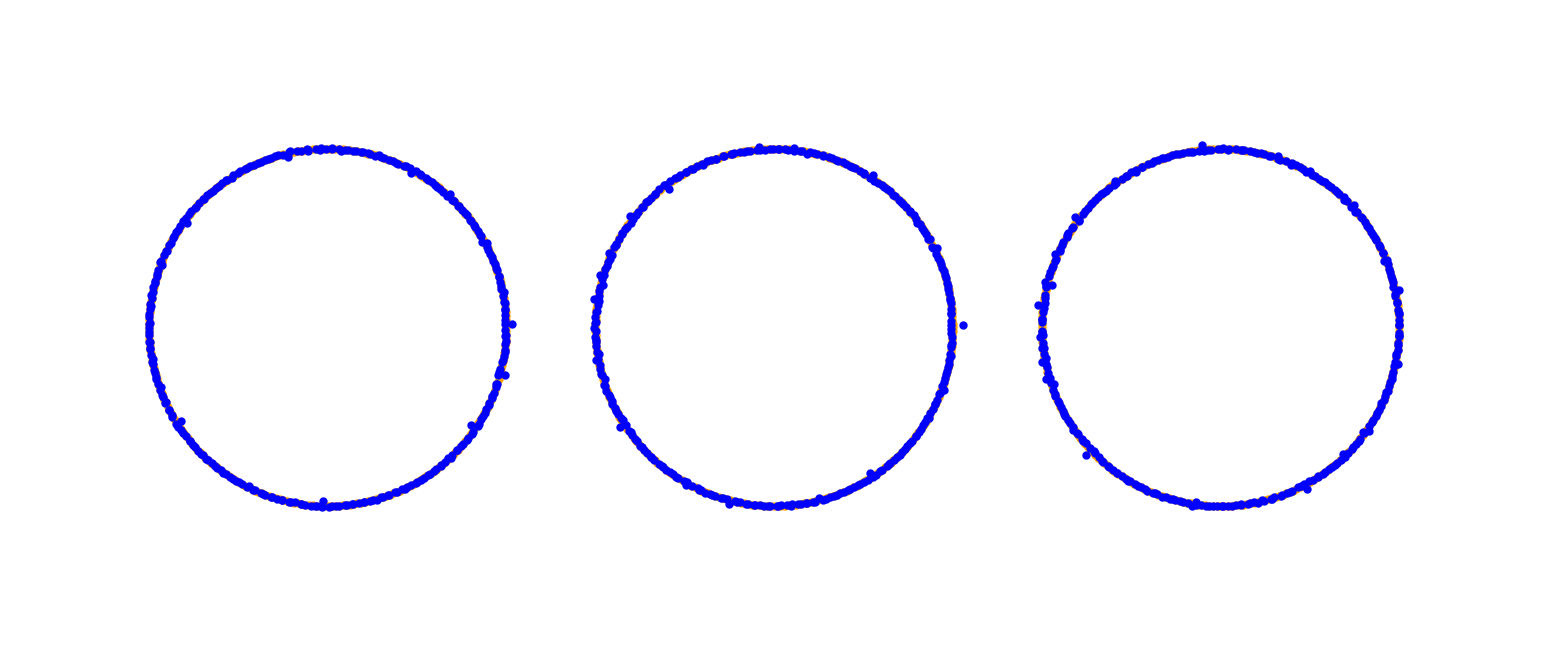}
    \end{subfigure} \\ %
    \rotatebox{90}{
        \begin{minipage}{.07\textwidth}
        \centering 
        \small $10$
    \end{minipage}} \hfill
    \begin{subfigure}[t]{.14\textwidth}
        \includegraphics[width=\linewidth]{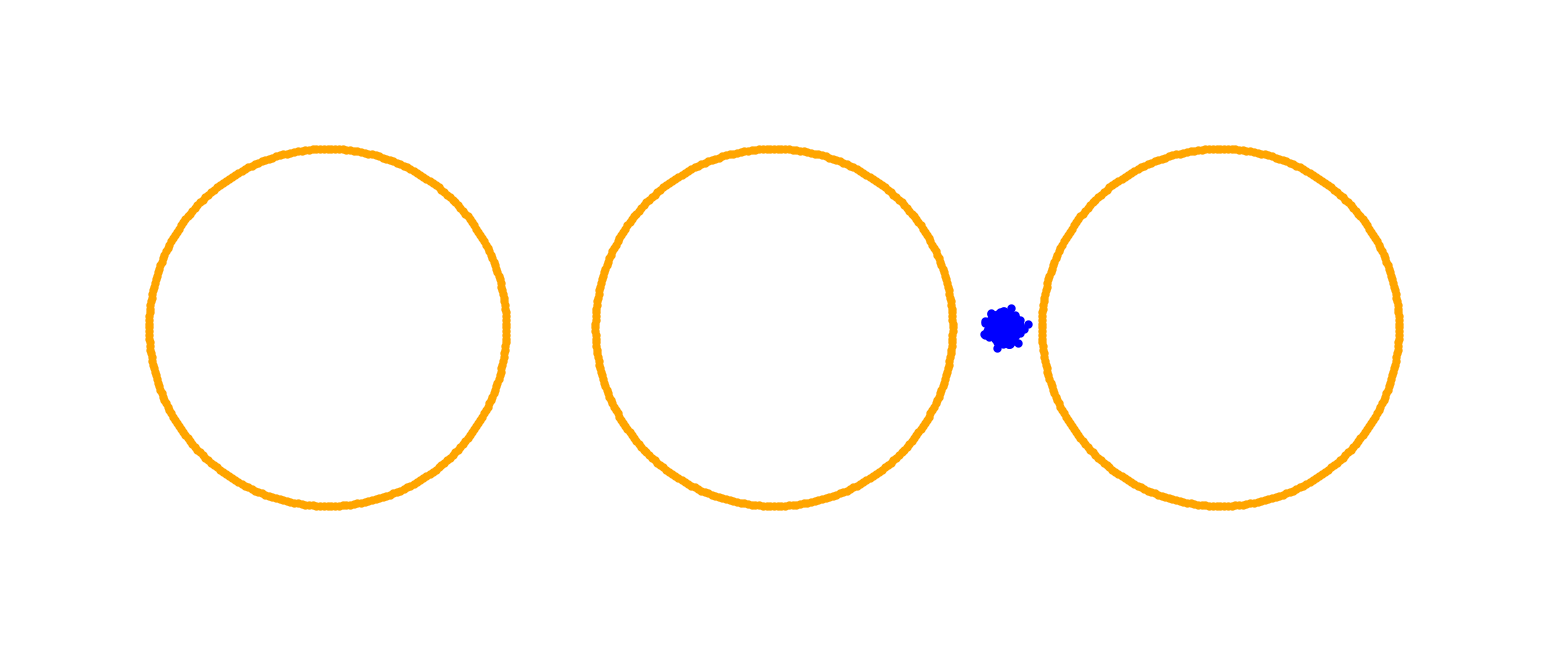}
    \end{subfigure}%
    \begin{subfigure}[t]{.14\textwidth}
        \includegraphics[width=\linewidth]{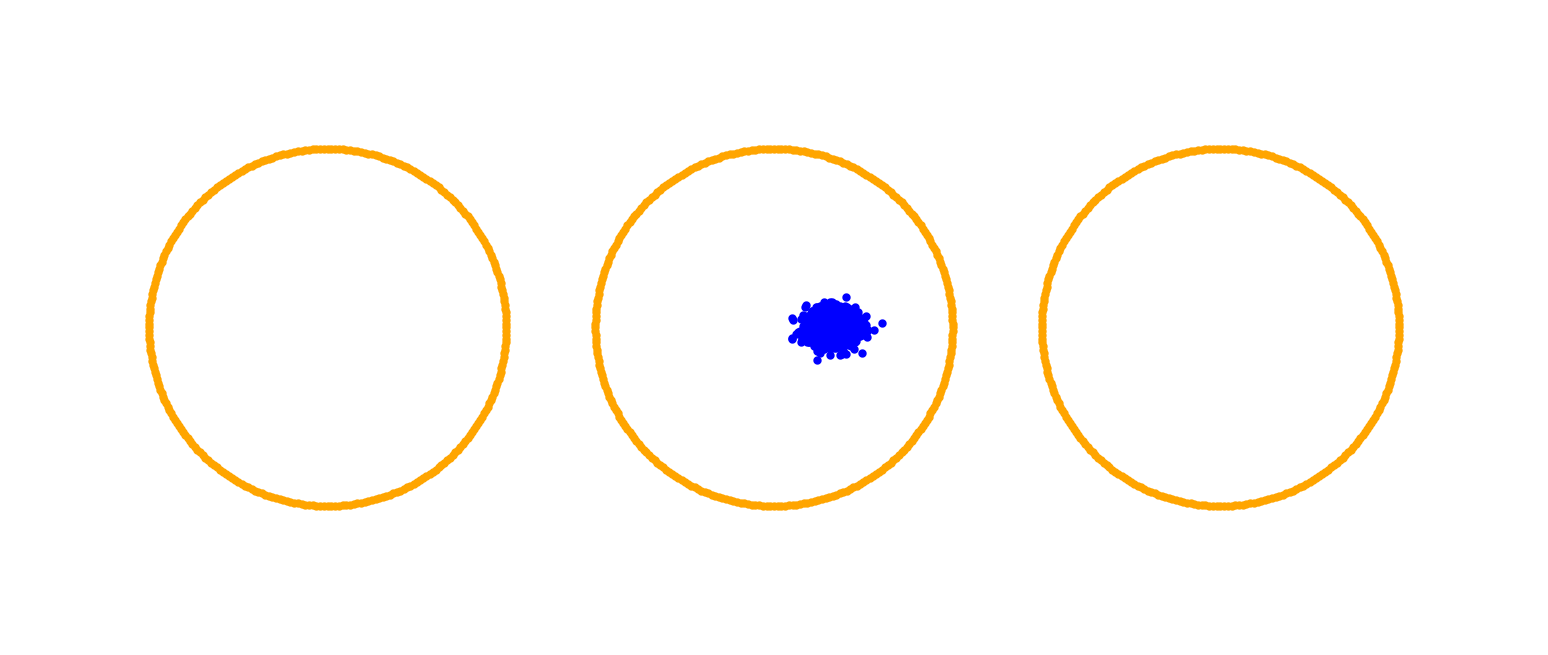}
    \end{subfigure}%
    \begin{subfigure}[t]{.14\textwidth}
        \includegraphics[width=\linewidth]{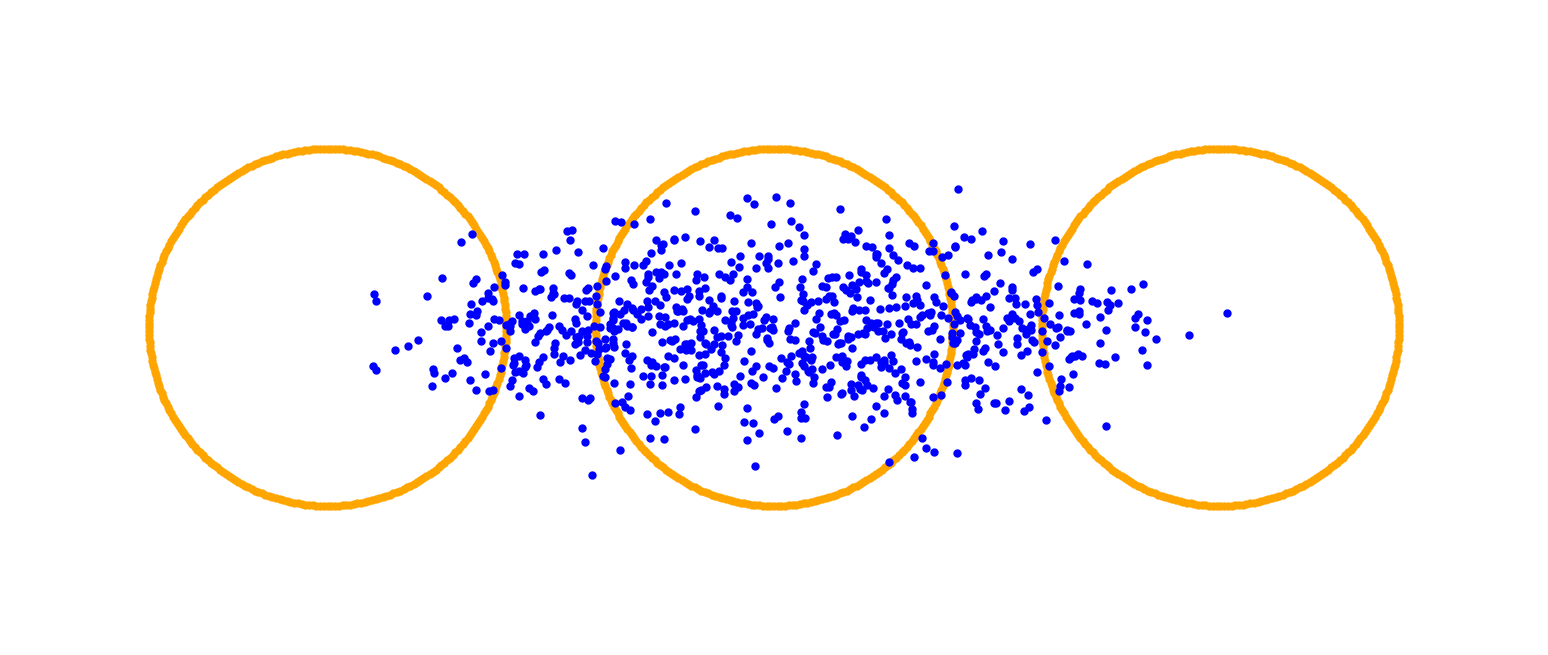}
    \end{subfigure}%
    \begin{subfigure}[t]{.14\textwidth}
        \includegraphics[width=\linewidth]{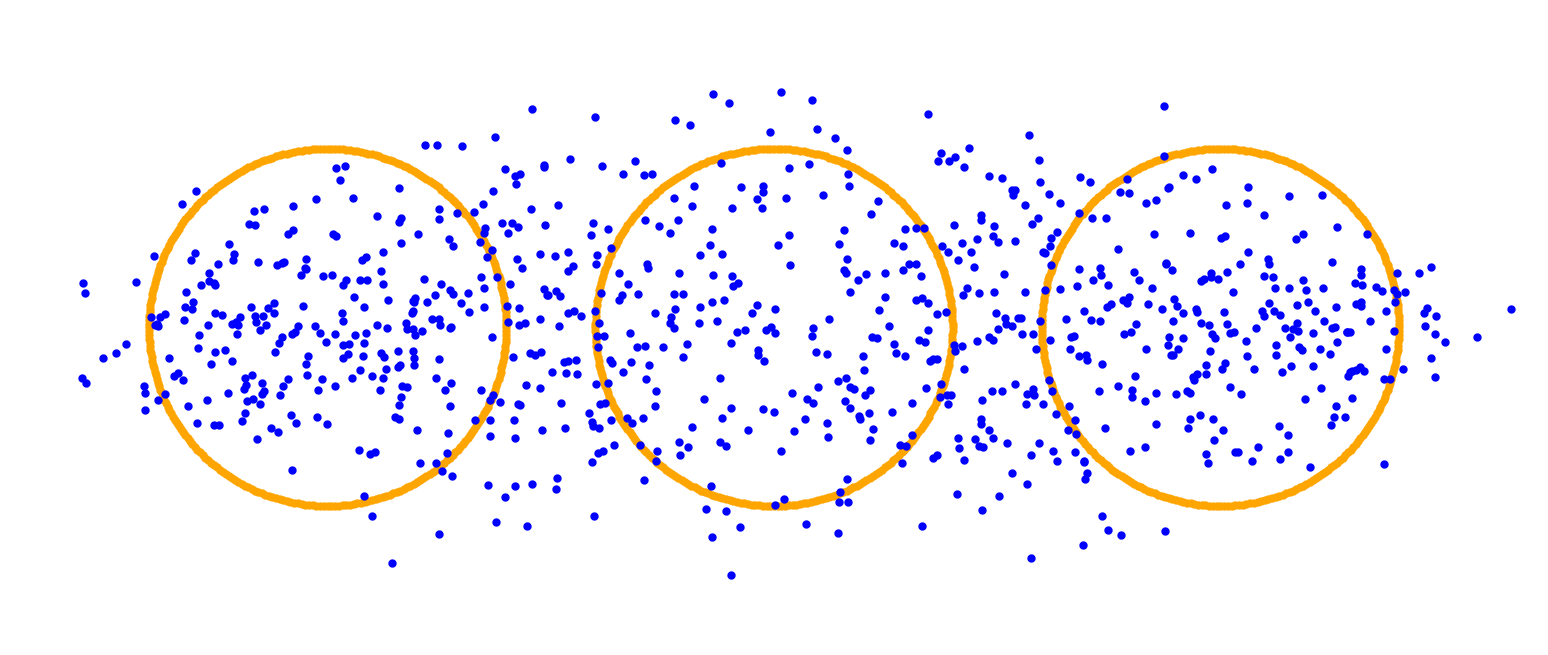}
    \end{subfigure}%
    \begin{subfigure}[t]{.14\textwidth}
        \includegraphics[width=\linewidth]{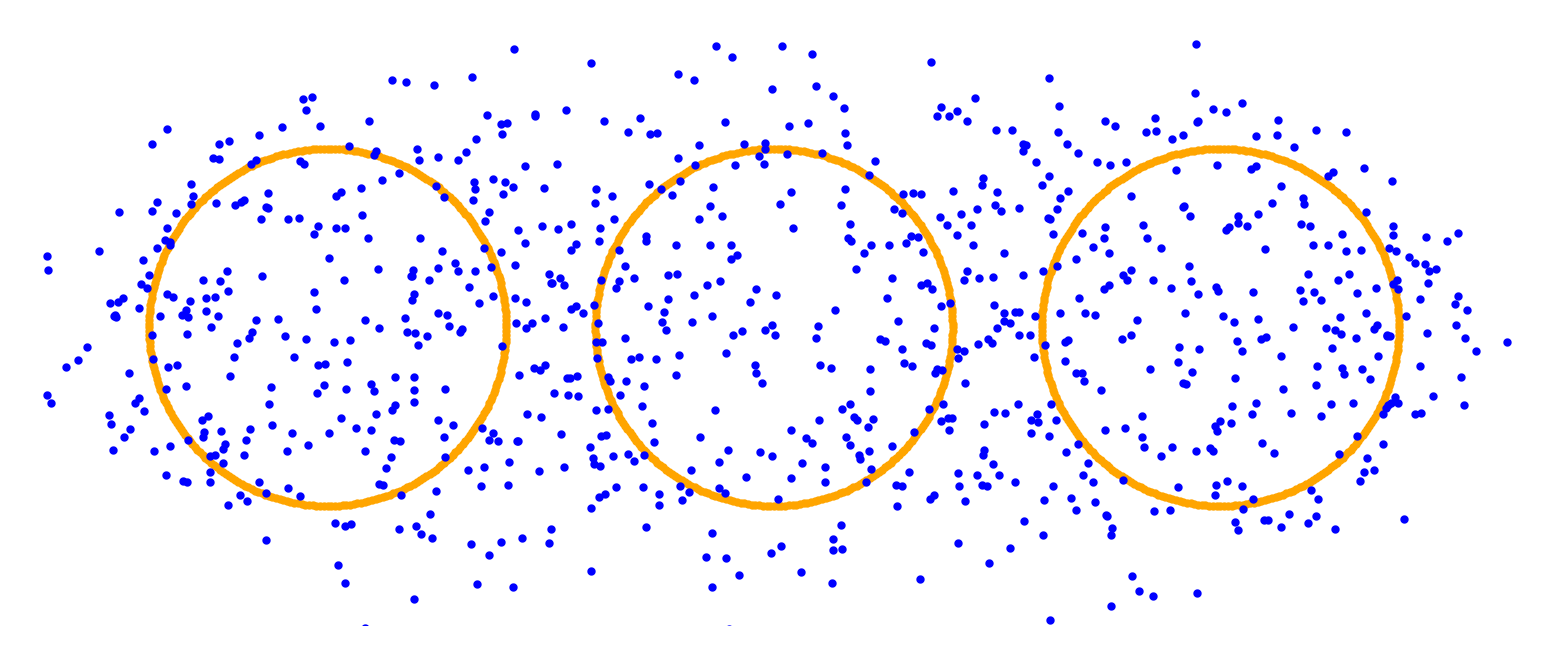}
    \end{subfigure}%
    \begin{subfigure}[t]{.14\textwidth}
        \includegraphics[width=\linewidth]{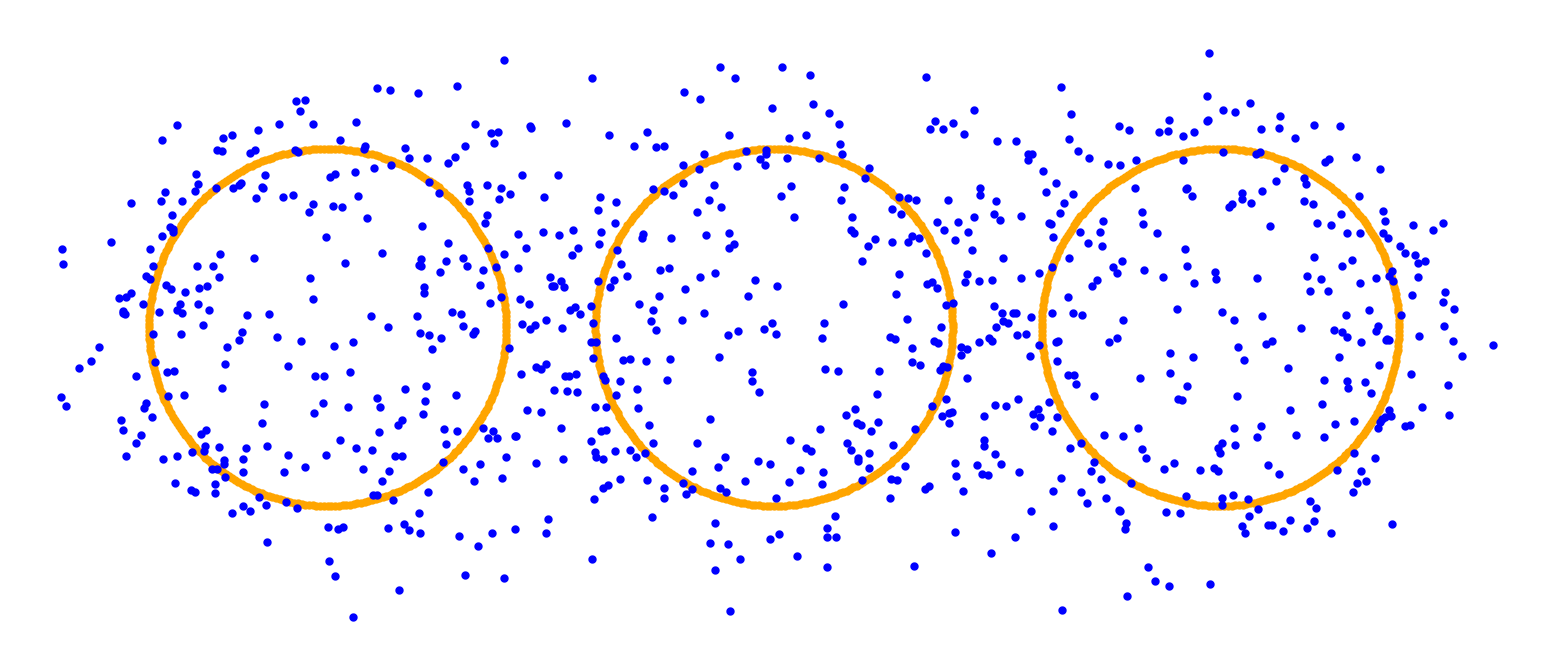}
    \end{subfigure} \\ %
    \rotatebox{90}{
        \begin{minipage}{.07\textwidth}
        \centering 
        \small $100$
    \end{minipage}} \hfill
    \begin{subfigure}[t]{.14\textwidth}
        \includegraphics[width=\linewidth]{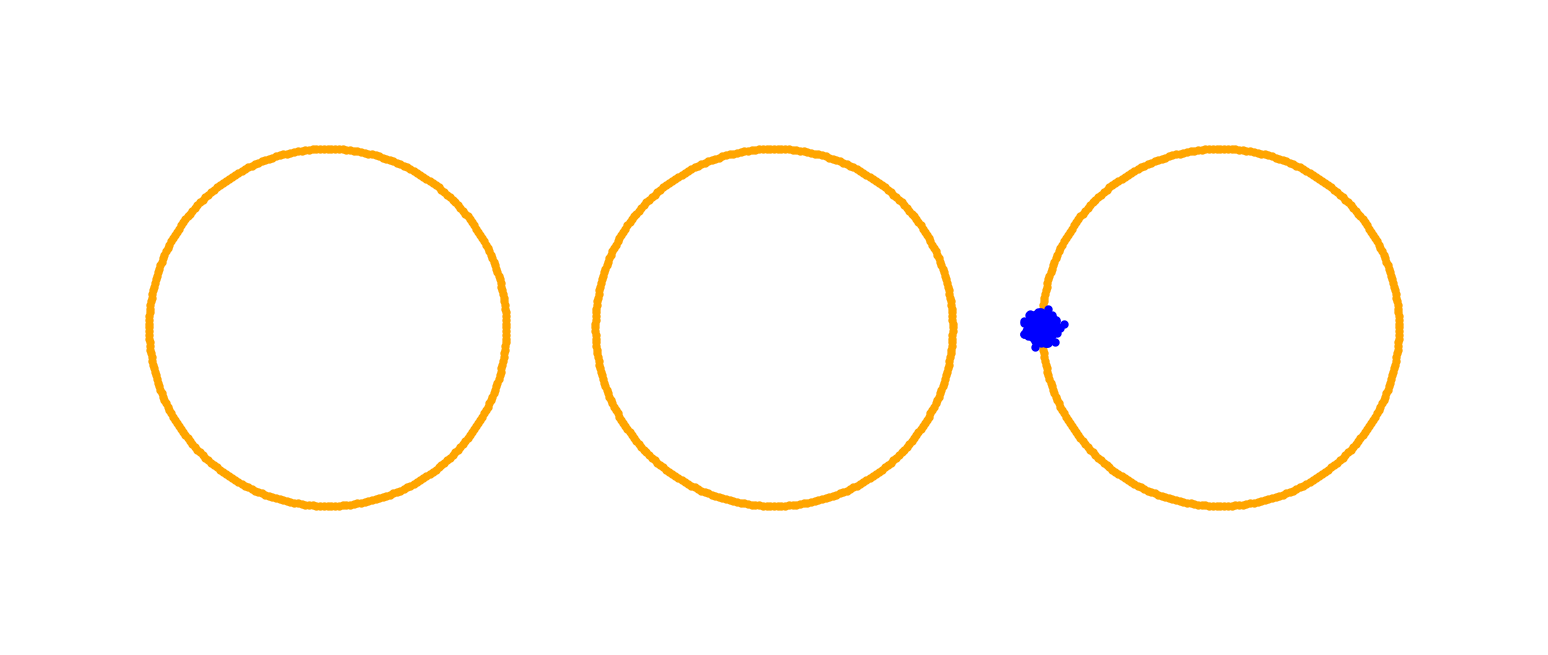}
    \caption*{t=0.1}
    \end{subfigure}%
    \begin{subfigure}[t]{.14\textwidth}
        \includegraphics[width=\linewidth]{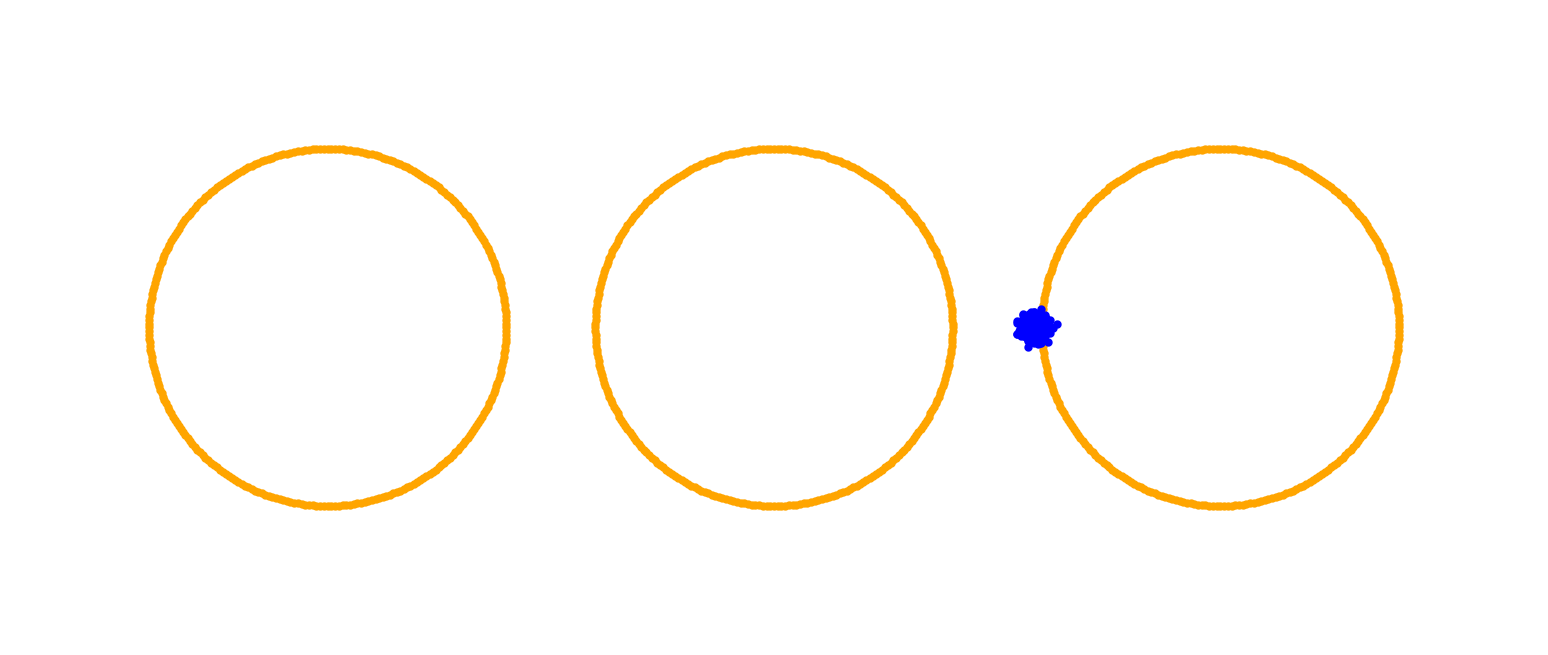}
    \caption*{t=1}
    \end{subfigure}%
    \begin{subfigure}[t]{.14\textwidth}
        \includegraphics[width=\linewidth]{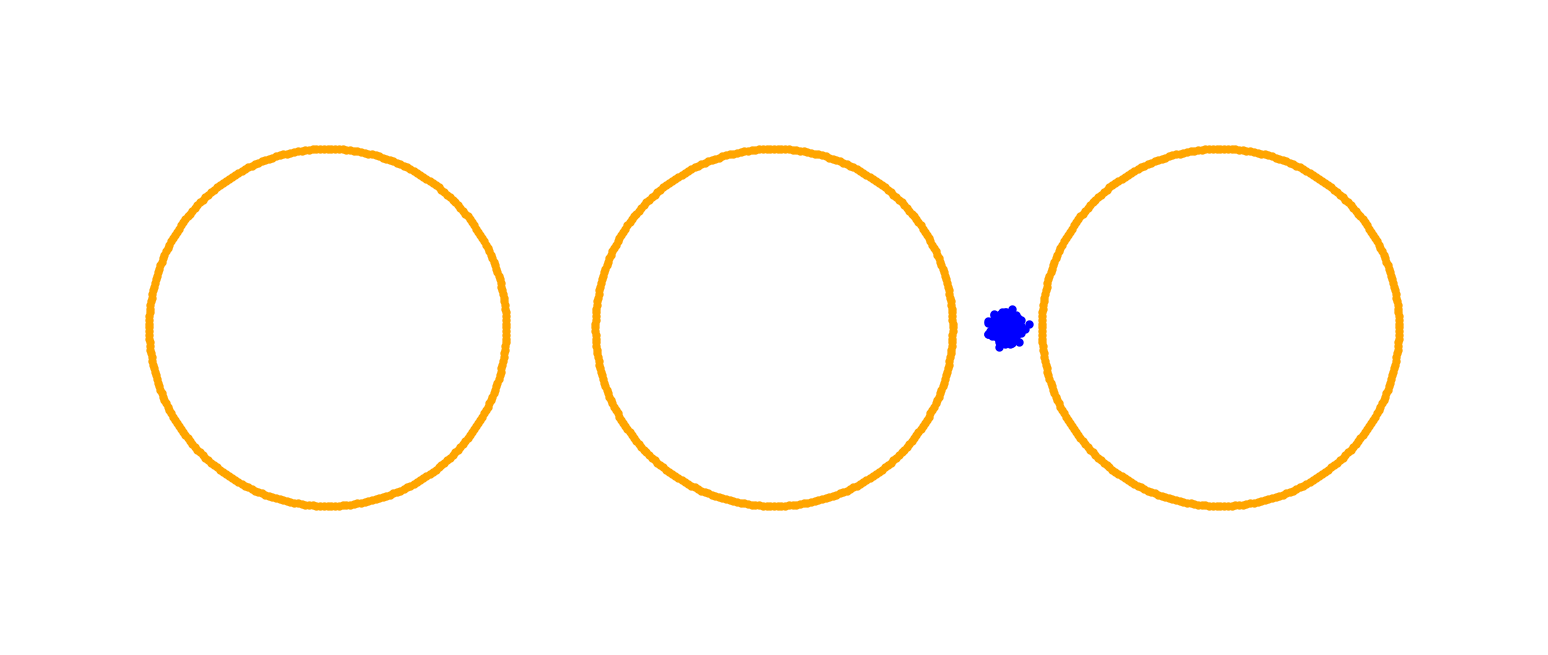}
    \caption*{t=5}
    \end{subfigure}%
    \begin{subfigure}[t]{.14\textwidth}
        \includegraphics[width=\linewidth]{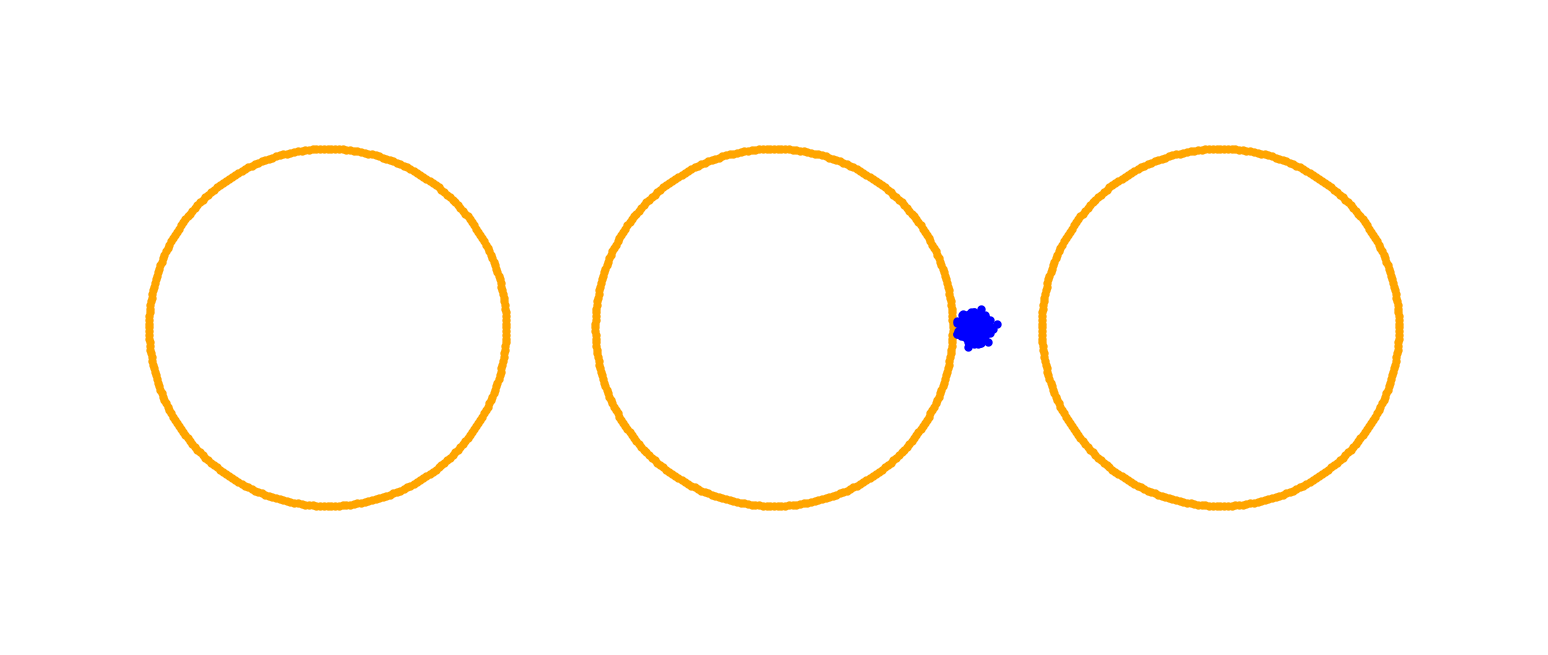}
    \caption*{t=10}
    \end{subfigure}%
    \begin{subfigure}[t]{.14\textwidth}
        \includegraphics[width=\linewidth]{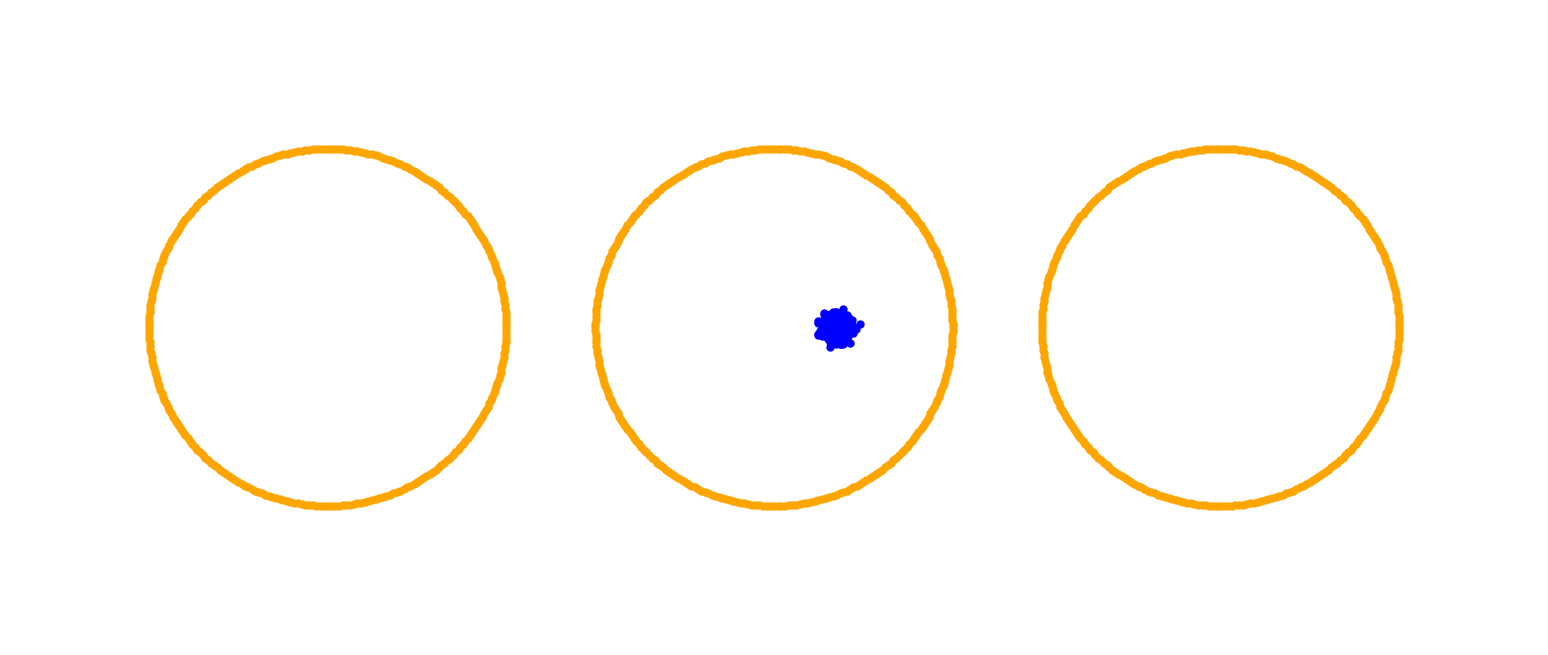}
    \caption*{t=50}
    \end{subfigure}%
    \begin{subfigure}[t]{.14\textwidth}
        \includegraphics[width=\linewidth]{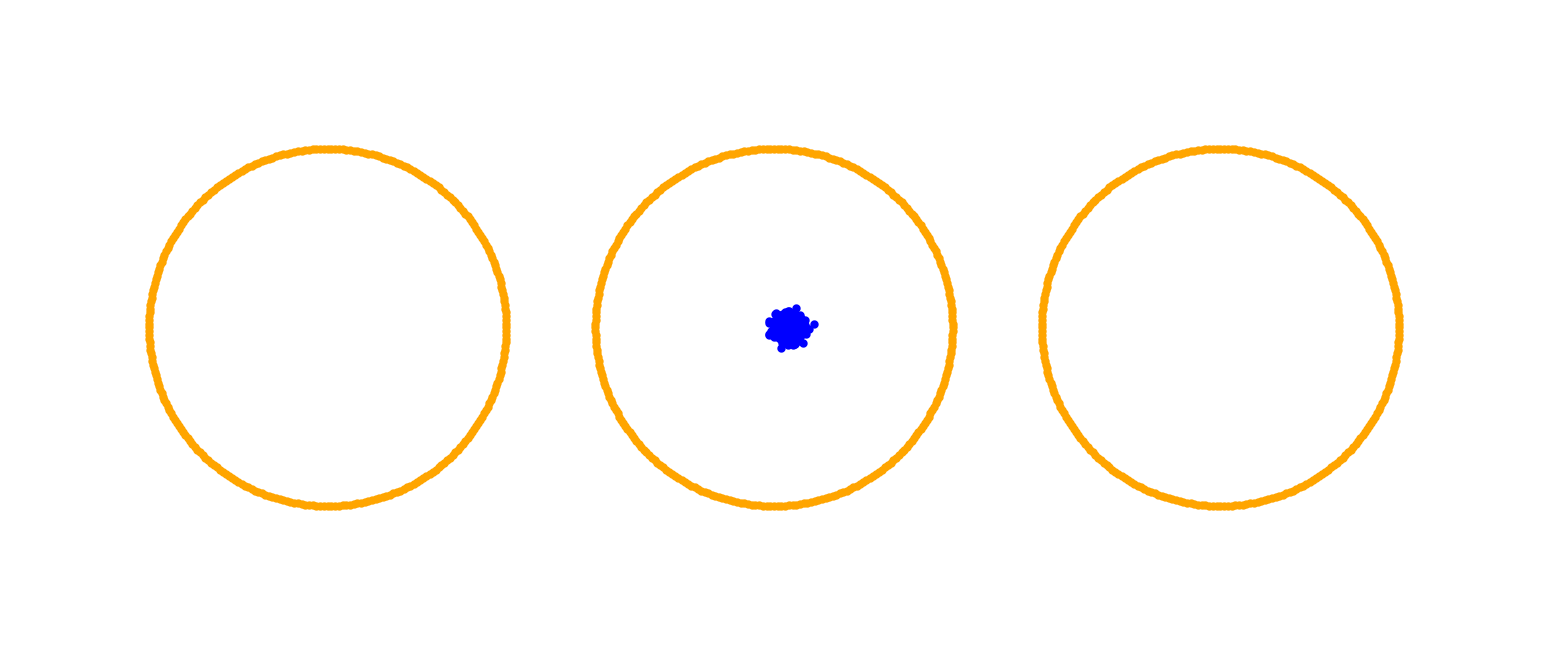}
    \caption*{t=100}
    \end{subfigure}%
    \caption{Ablation study for the parameter $\sigma^2$ of the Matern kernel with $\nu = \frac{3}{2}$.}
    \label{fig:ablationCirclesMaterntau}
\end{figure}

\subsection{The Regularization Parameter \texorpdfstring{$\lambda$}{lambda}}
First, we investigate the behavior of the flow if only the regularization parameter $\lambda$ varies.
For this purpose, we choose the Jeffreys divergence.
In our scheme, we choose $\tau = \num{e-3}$ and $N = 900$ and the inverse multiquadric kernel with width $\sigma^2 = \num{5e-1}$.
Even though the Jeffreys entropy function is infinite at $x = 0$, the flow still behaves reasonably well, see Figure \ref{fig:ablationCirclesCompact2Lambda}.
If $\lambda$ is very small, many particles get stuck in the middle and do not spread towards the funnels.
If, however, $\lambda$ is much larger than the step size $\tau$, then the particles only spread out slowly and in a spherical shape.
This behavior is also reflected in the corresponding distances to the target measure $\nu$, see Figure \ref{fig:ablationCirclesJeffreyLambdaGraphs}.
In our second example, we use the compactly supported \enquote{spline} kernel $k(x, y) \coloneqq (1 - \| x - y \|_2)_+^3 (3 \| x - y \| + 1)$, the 3-Tsallis divergence, $\tau = \num{e-3}$, $N = 900$ and the three circles target from before, see Figure~\ref{fig:ablationCirclesCompact2LambdaGraphs}.
Overall, the behavior is similar to before.
\begin{figure}[t]
    \centering
    \rotatebox{90}{
        \begin{minipage}{.05\textwidth}
        \centering 
        \small $\num{e-3}$
    \end{minipage}} \hfill
    \begin{subfigure}[t]{.14\textwidth}
        \includegraphics[width=\linewidth, valign=c]{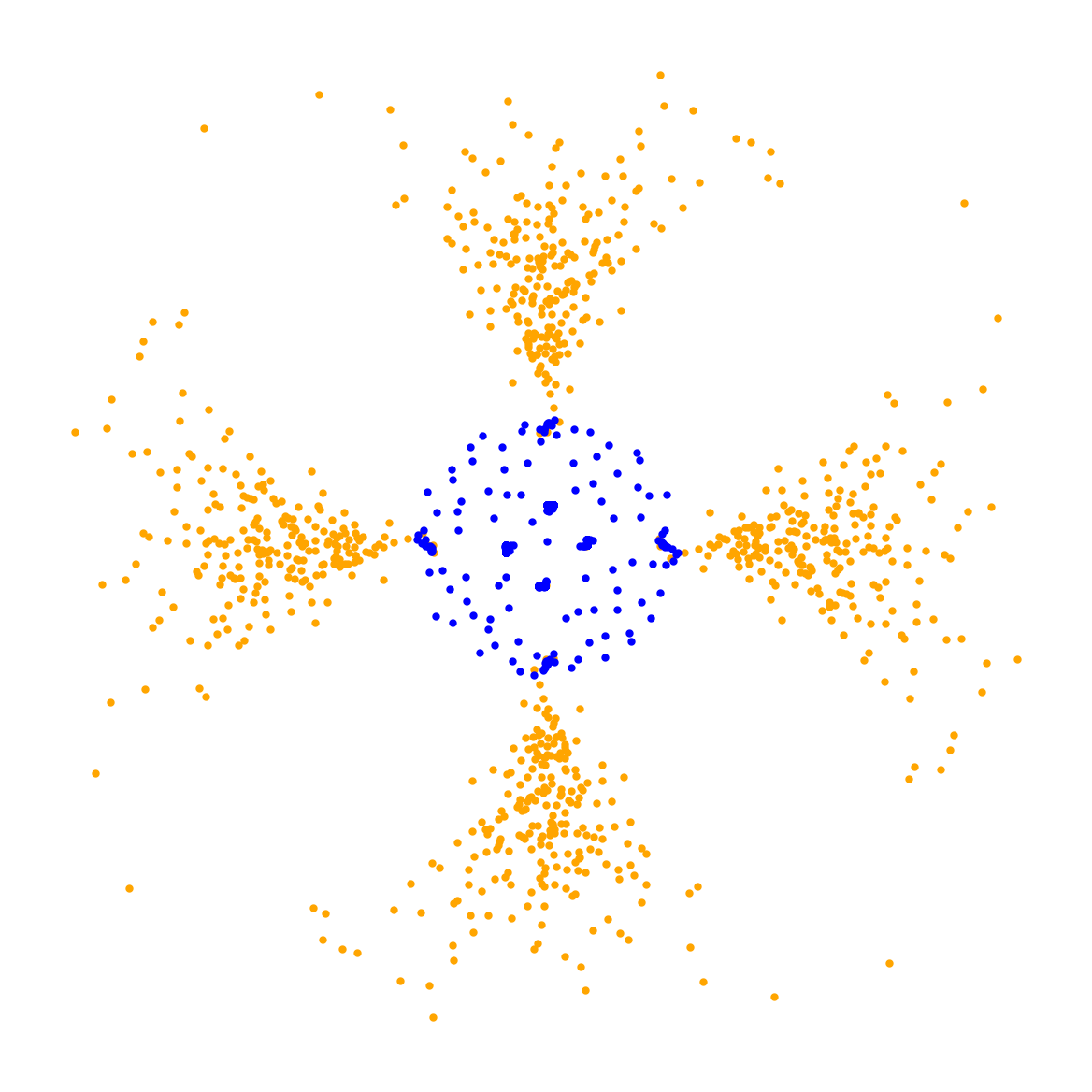}
    \end{subfigure}%
    \begin{subfigure}[t]{.14\textwidth}
        \includegraphics[width=\linewidth, valign=c]{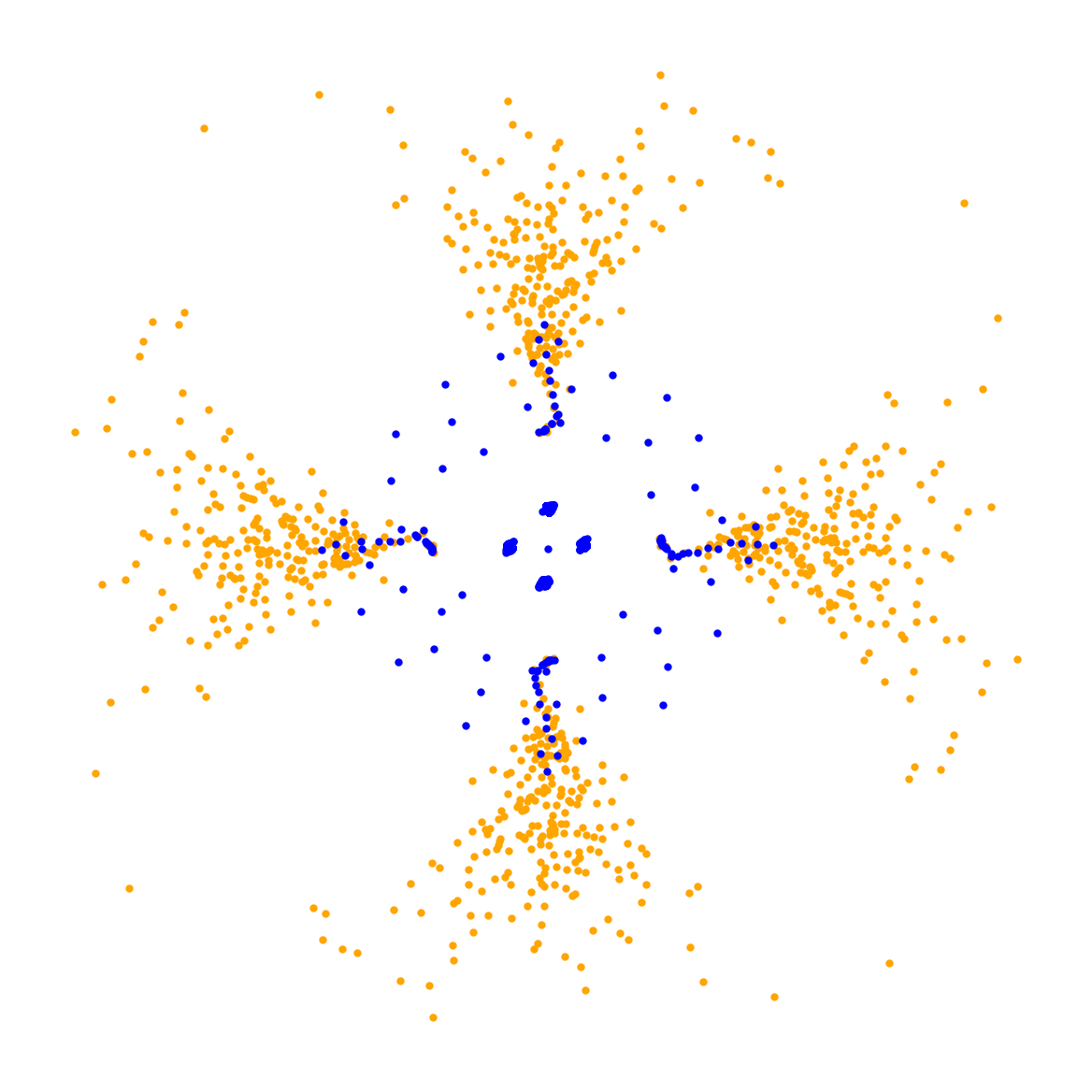}
    \end{subfigure}%
    \begin{subfigure}[t]{.14\textwidth}
        \includegraphics[width=\linewidth, valign=c]{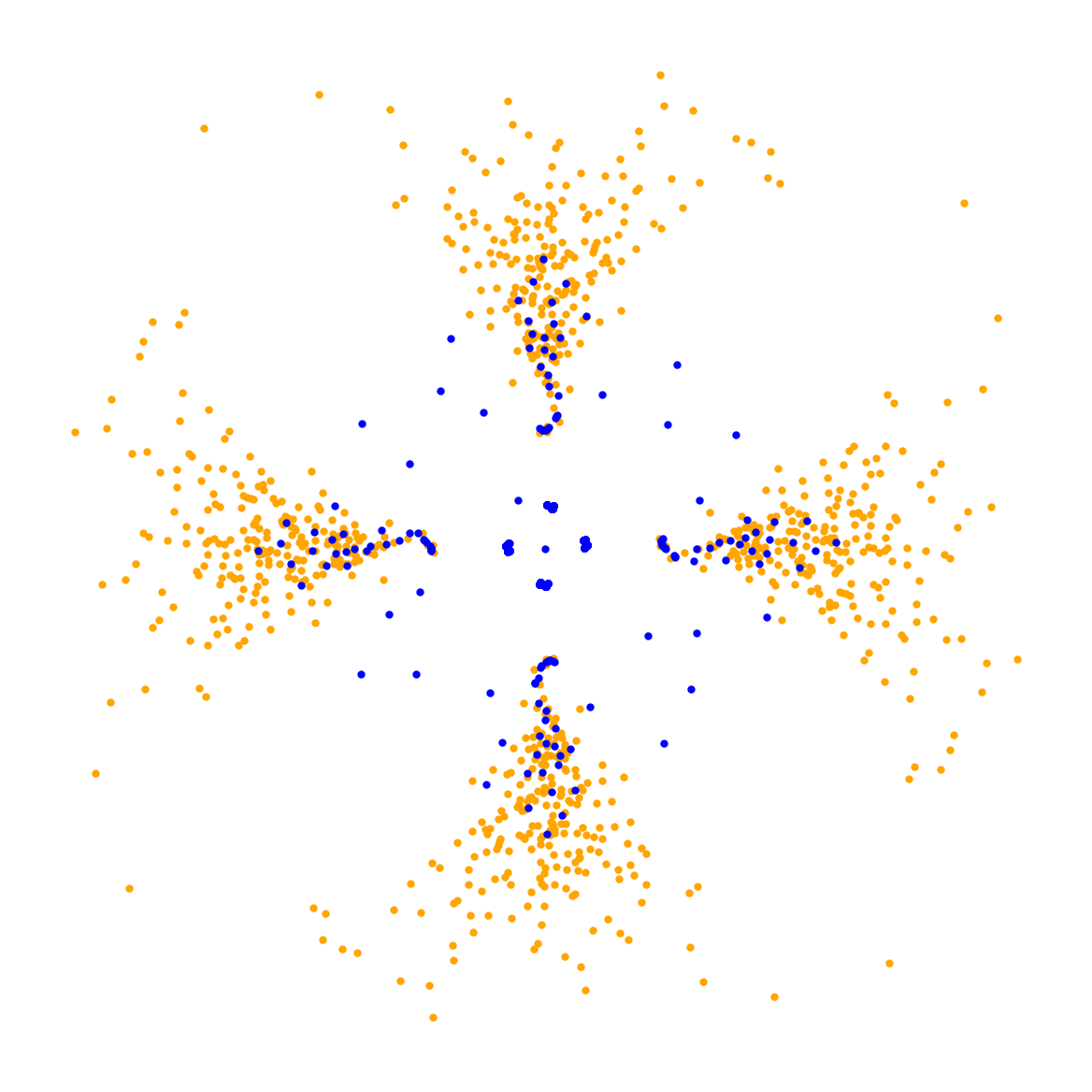}
    \end{subfigure}%
    \begin{subfigure}[t]{.14\textwidth}
        \includegraphics[width=\linewidth, valign=c]{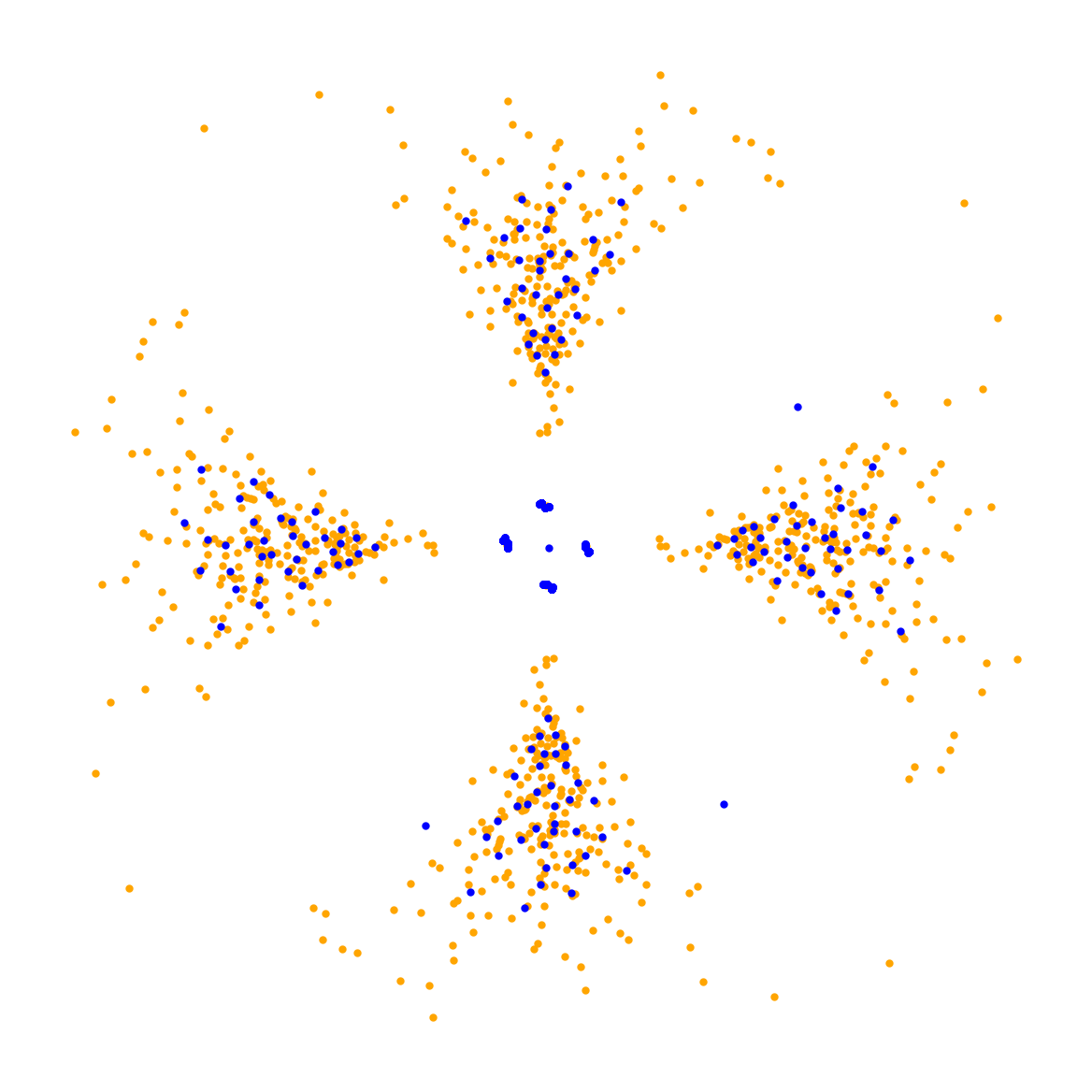}
    \end{subfigure}%
    \begin{subfigure}[t]{.14\textwidth}
        \includegraphics[width=\linewidth, valign=c]{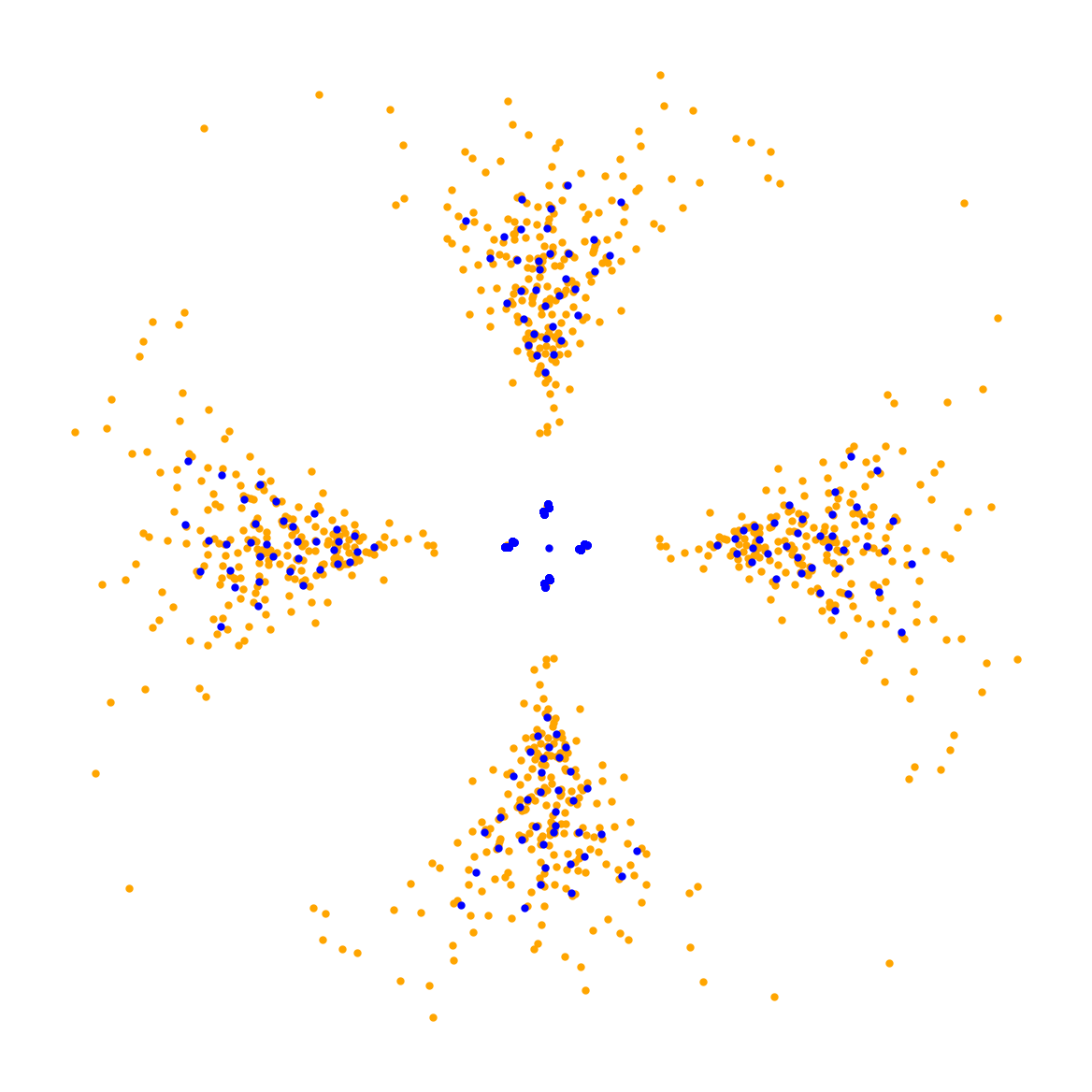}
    \end{subfigure}%
    \begin{subfigure}[t]{.14\textwidth}
        \includegraphics[width=\linewidth, valign=c]{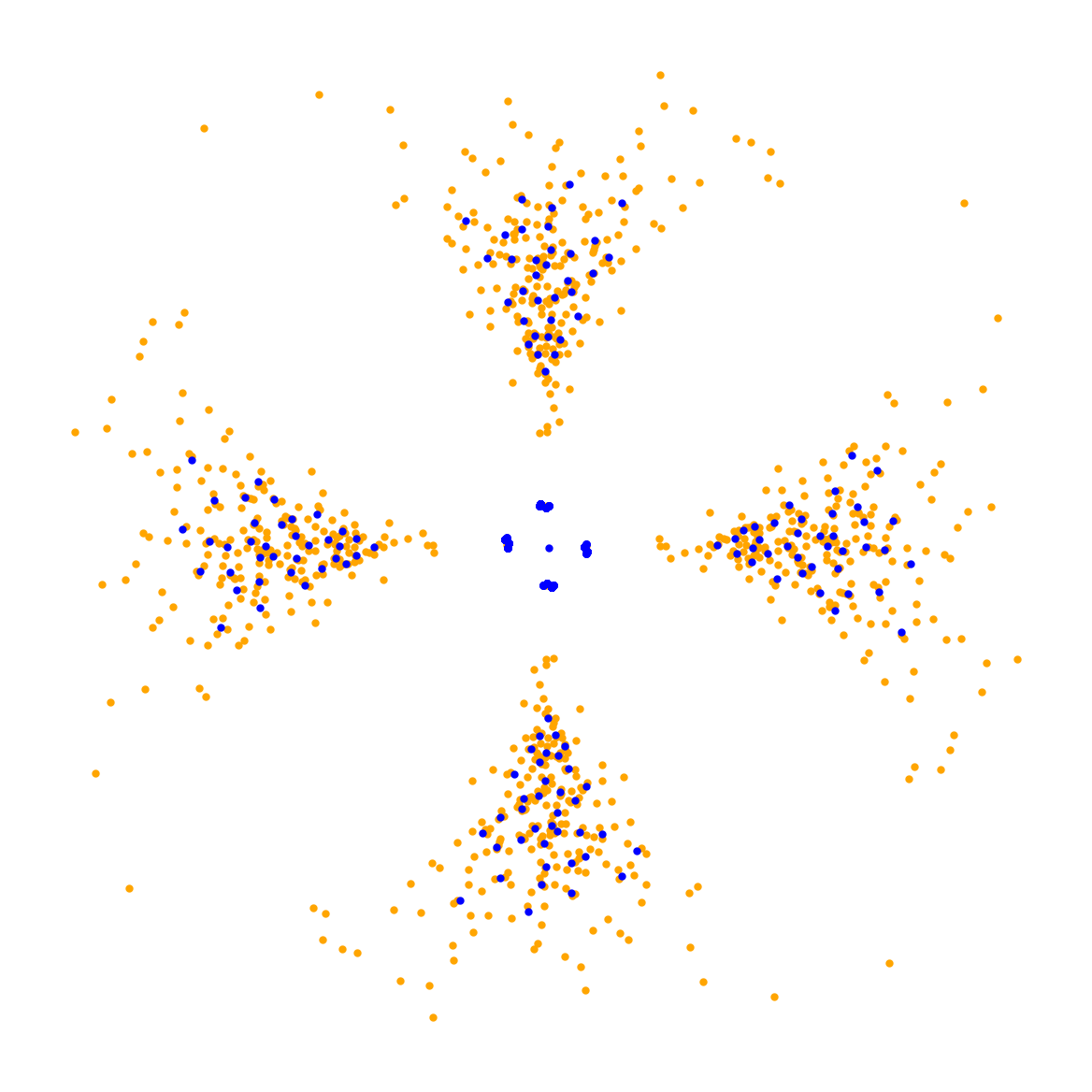}
    \end{subfigure} \\ %
    \rotatebox{90}{
        \begin{minipage}{.05\textwidth}
        \centering 
        \small $\num{e-2}$
    \end{minipage}} \hfill
    \begin{subfigure}[t]{.14\textwidth}
        \includegraphics[width=\linewidth, valign=c]{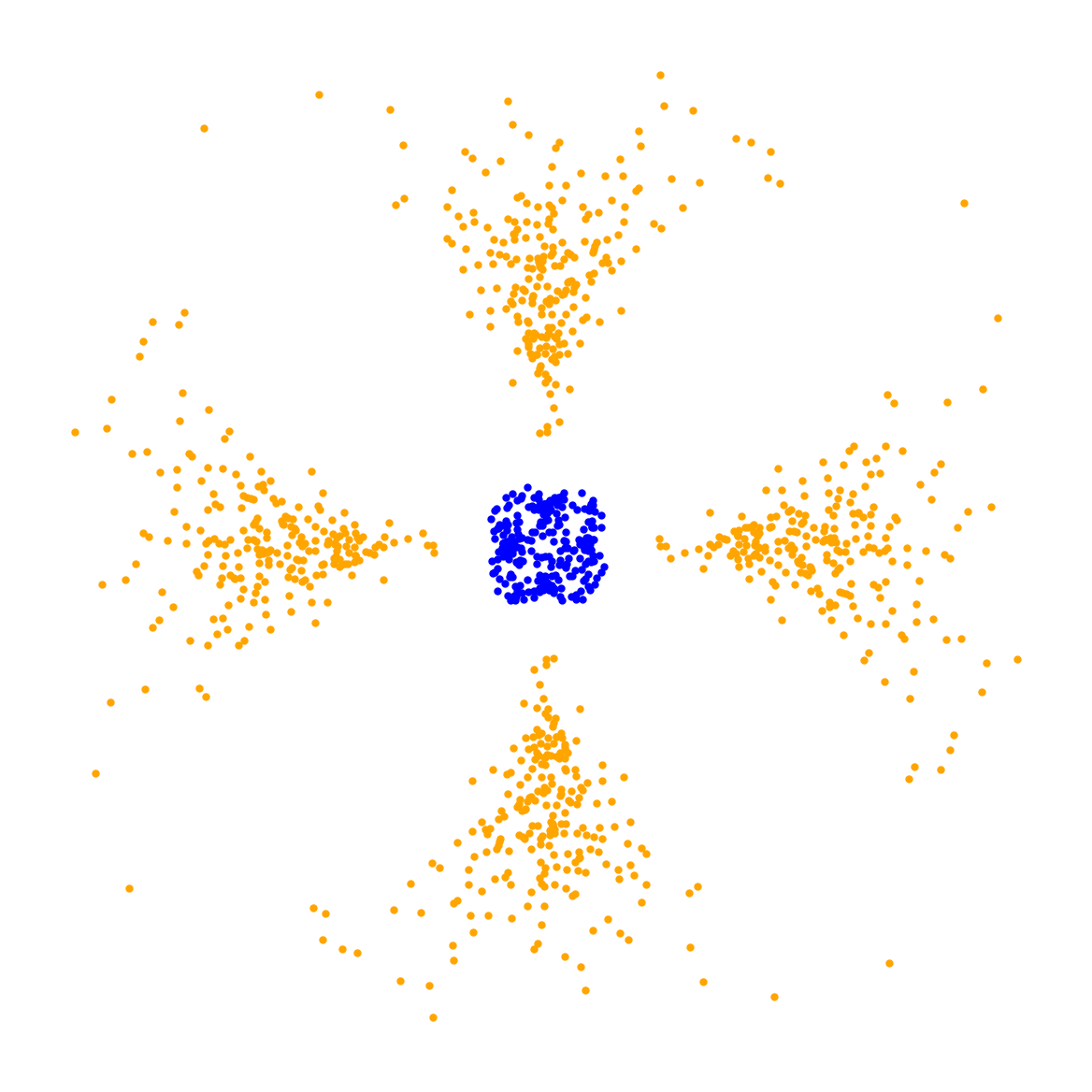}
    \end{subfigure}%
    \begin{subfigure}[t]{.14\textwidth}
        \includegraphics[width=\linewidth, valign=c]{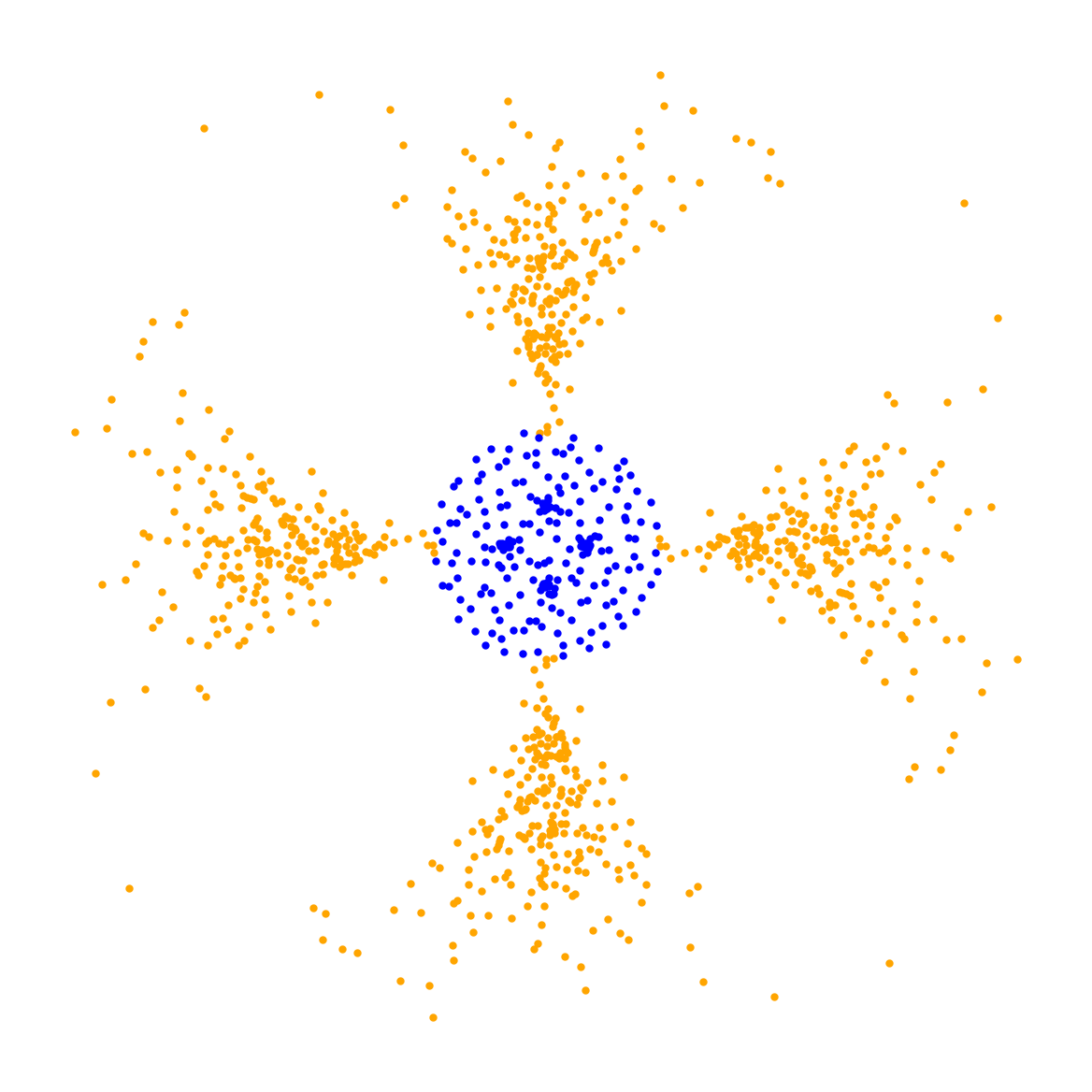}
    \end{subfigure}%
    \begin{subfigure}[t]{.14\textwidth}
        \includegraphics[width=\linewidth, valign=c]{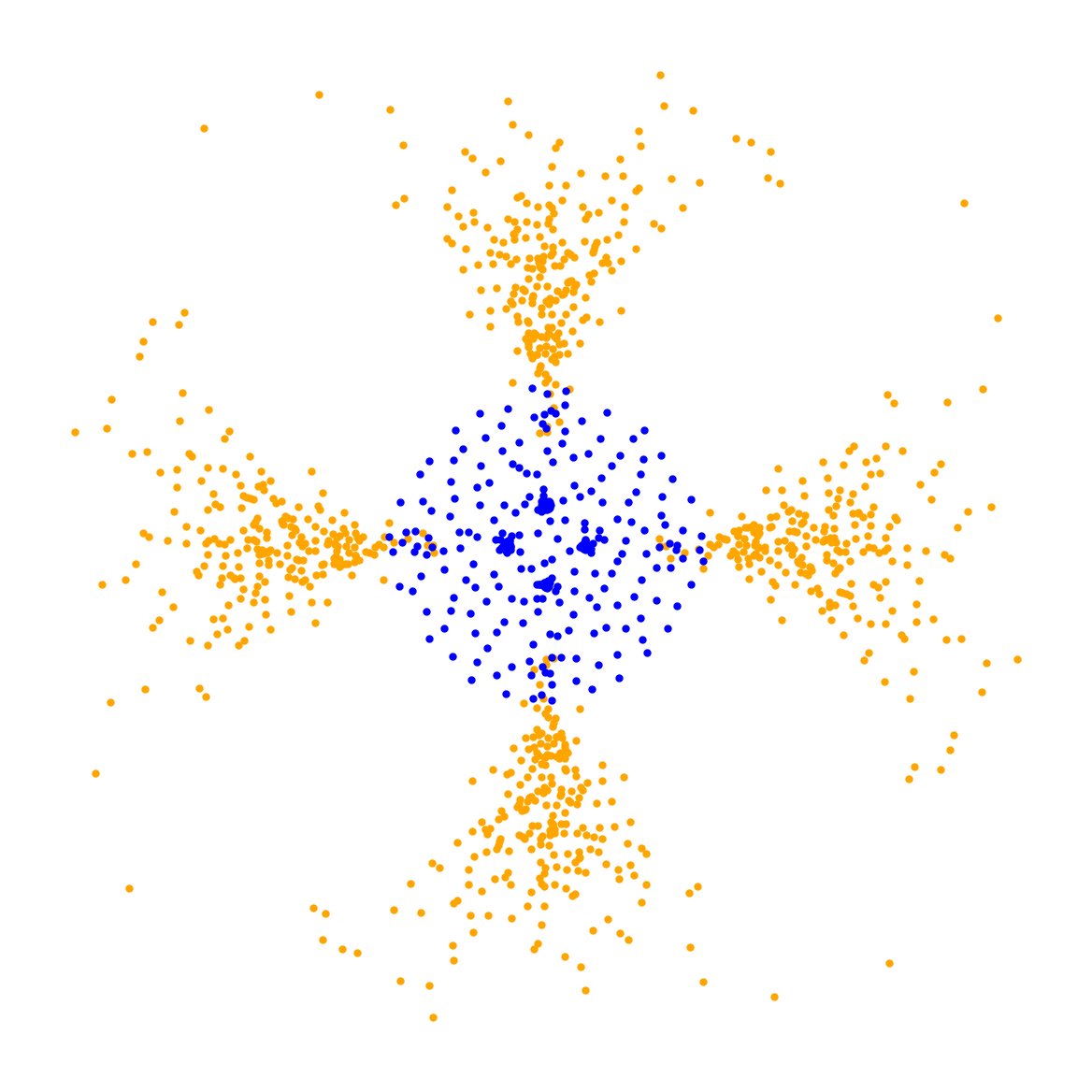}
    \end{subfigure}%
    \begin{subfigure}[t]{.14\textwidth}
        \includegraphics[width=\linewidth, valign=c]{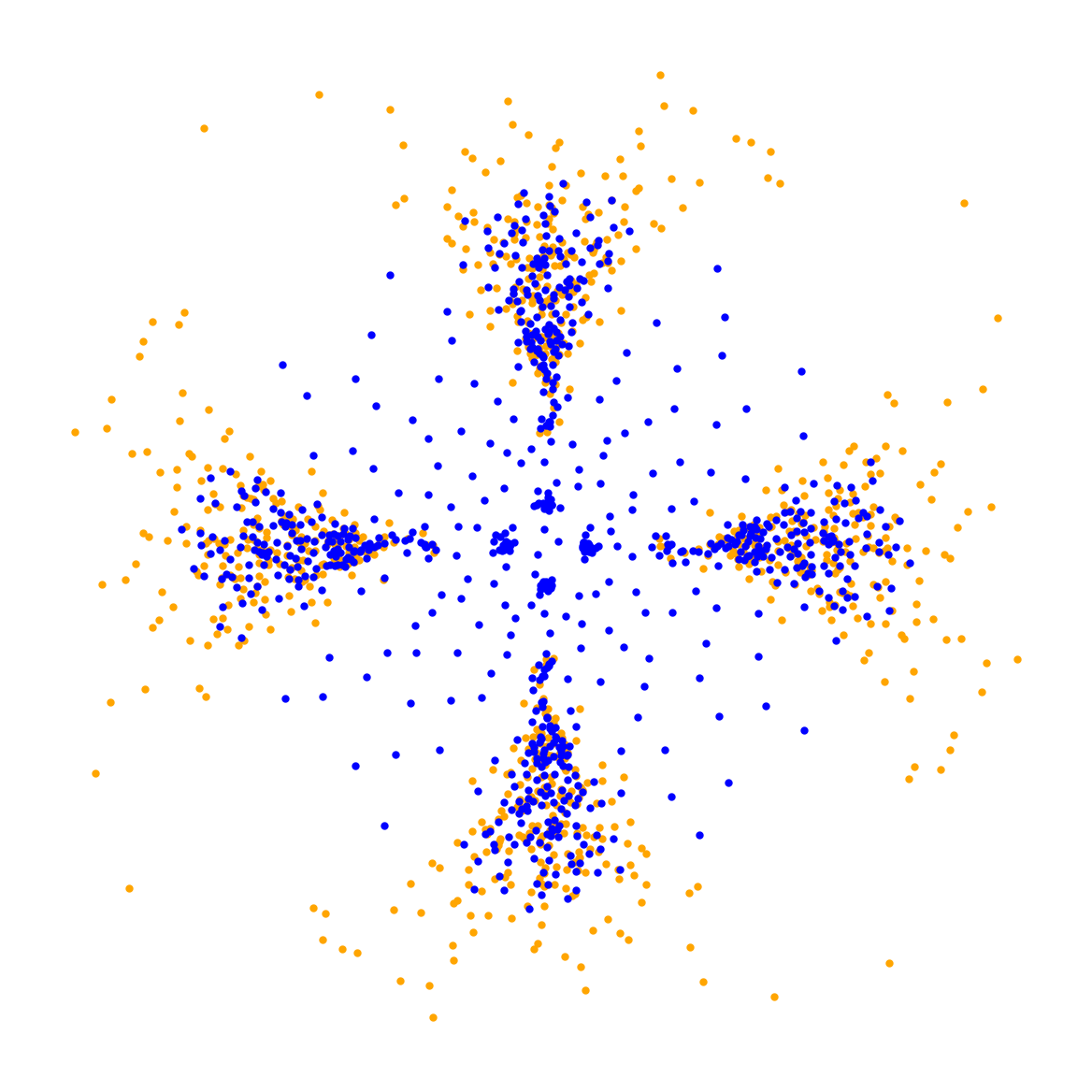}
    \end{subfigure}%
    \begin{subfigure}[t]{.14\textwidth}
        \includegraphics[width=\linewidth, valign=c]{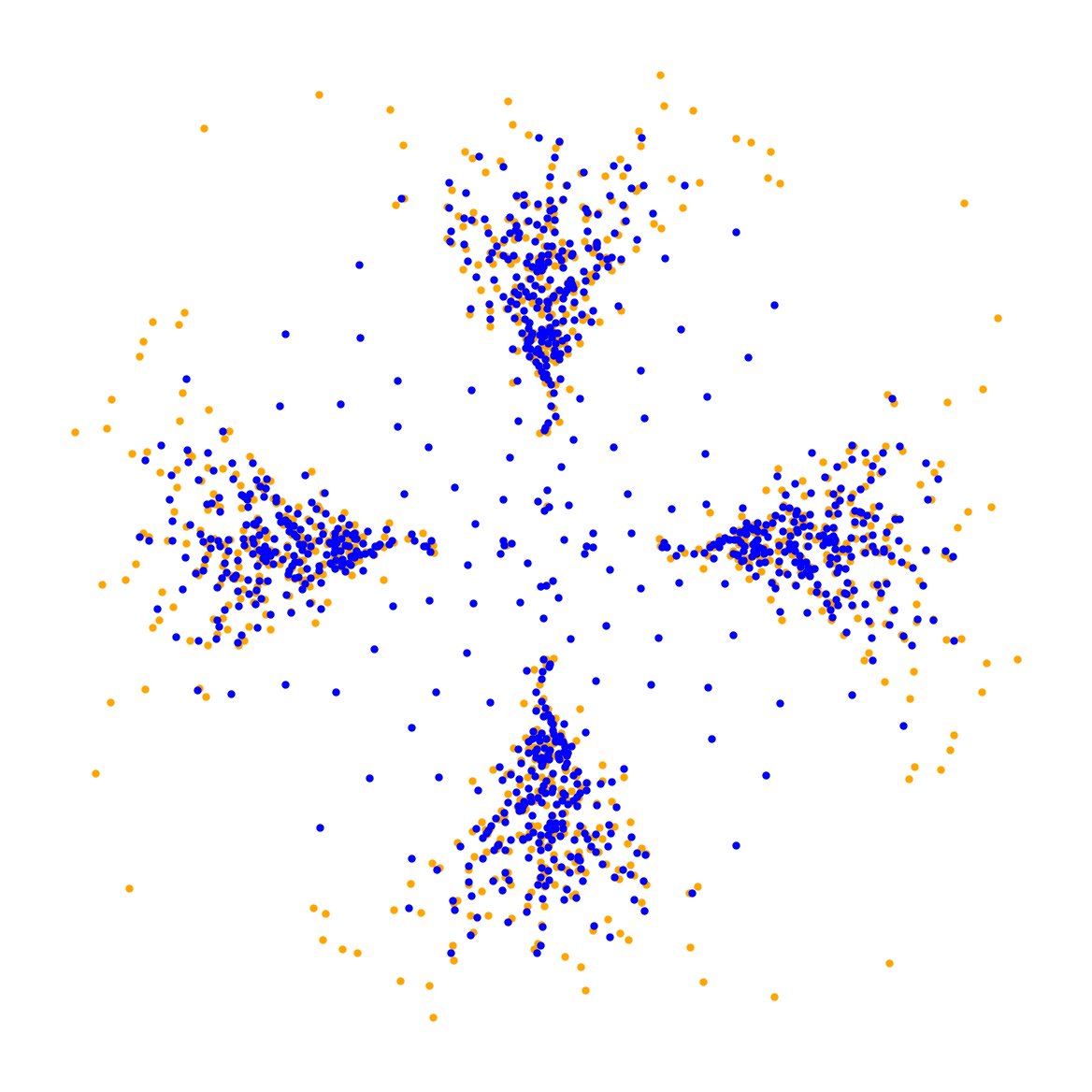}
    \end{subfigure}%
    \begin{subfigure}[t]{.14\textwidth}
        \includegraphics[width=\linewidth, valign=c]{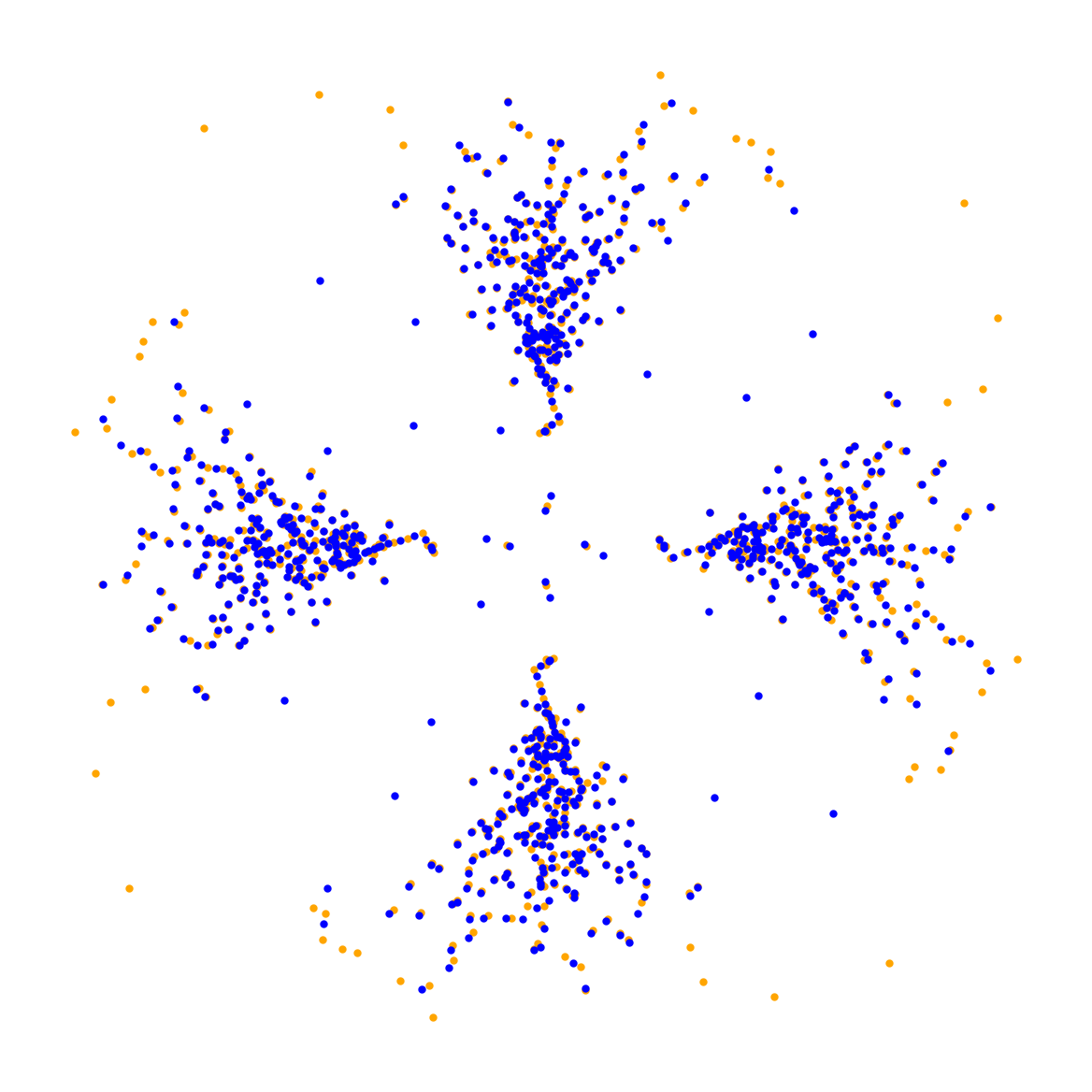}
    \end{subfigure} \\ %
    \rotatebox{90}{
        \begin{minipage}{.05\textwidth}
        \centering 
        \small $\num{e-1}$
    \end{minipage}} \hfill
    \begin{subfigure}[t]{.14\textwidth}
        \includegraphics[width=\linewidth, valign=c]{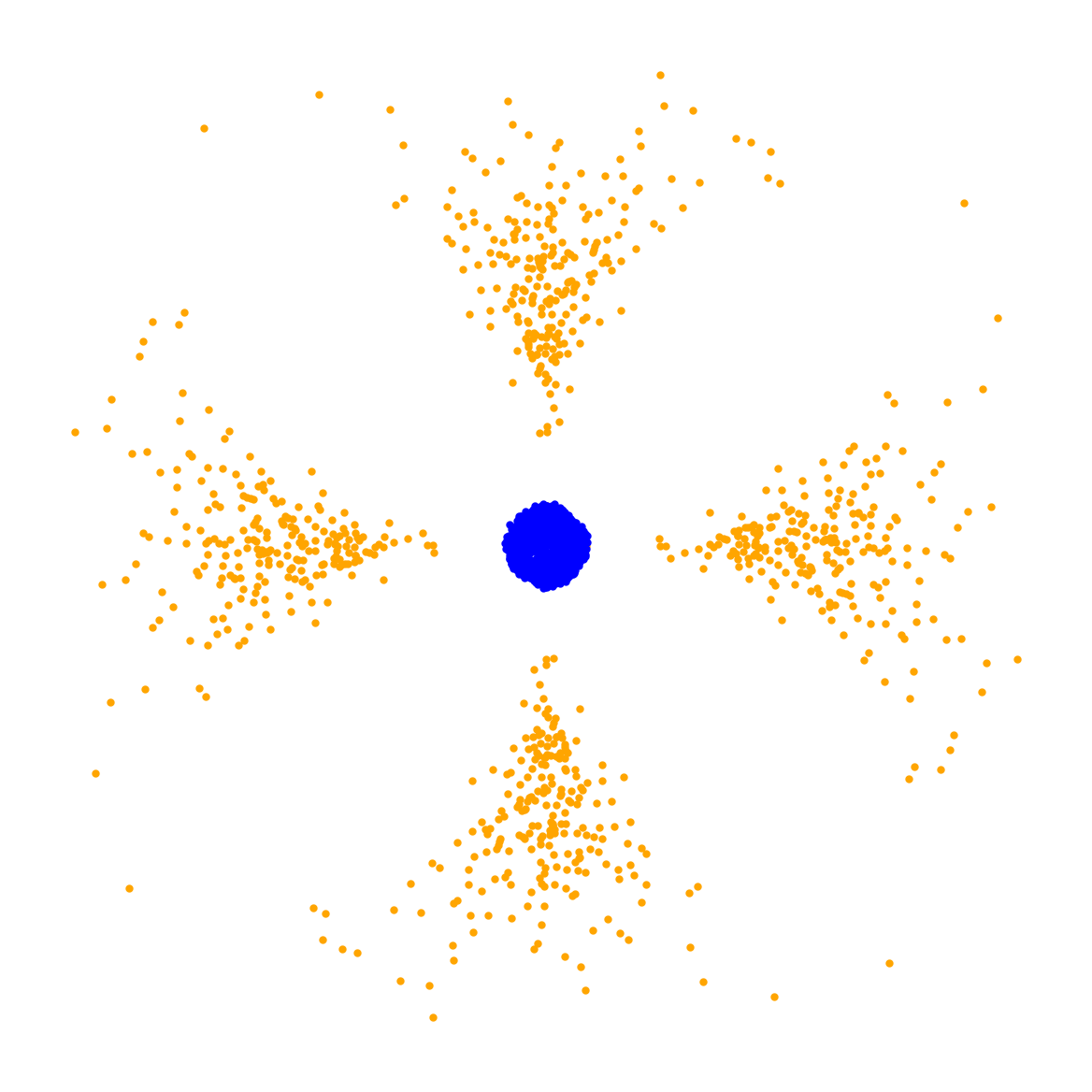}
    \end{subfigure}%
    \begin{subfigure}[t]{.14\textwidth}
        \includegraphics[width=\linewidth, valign=c]{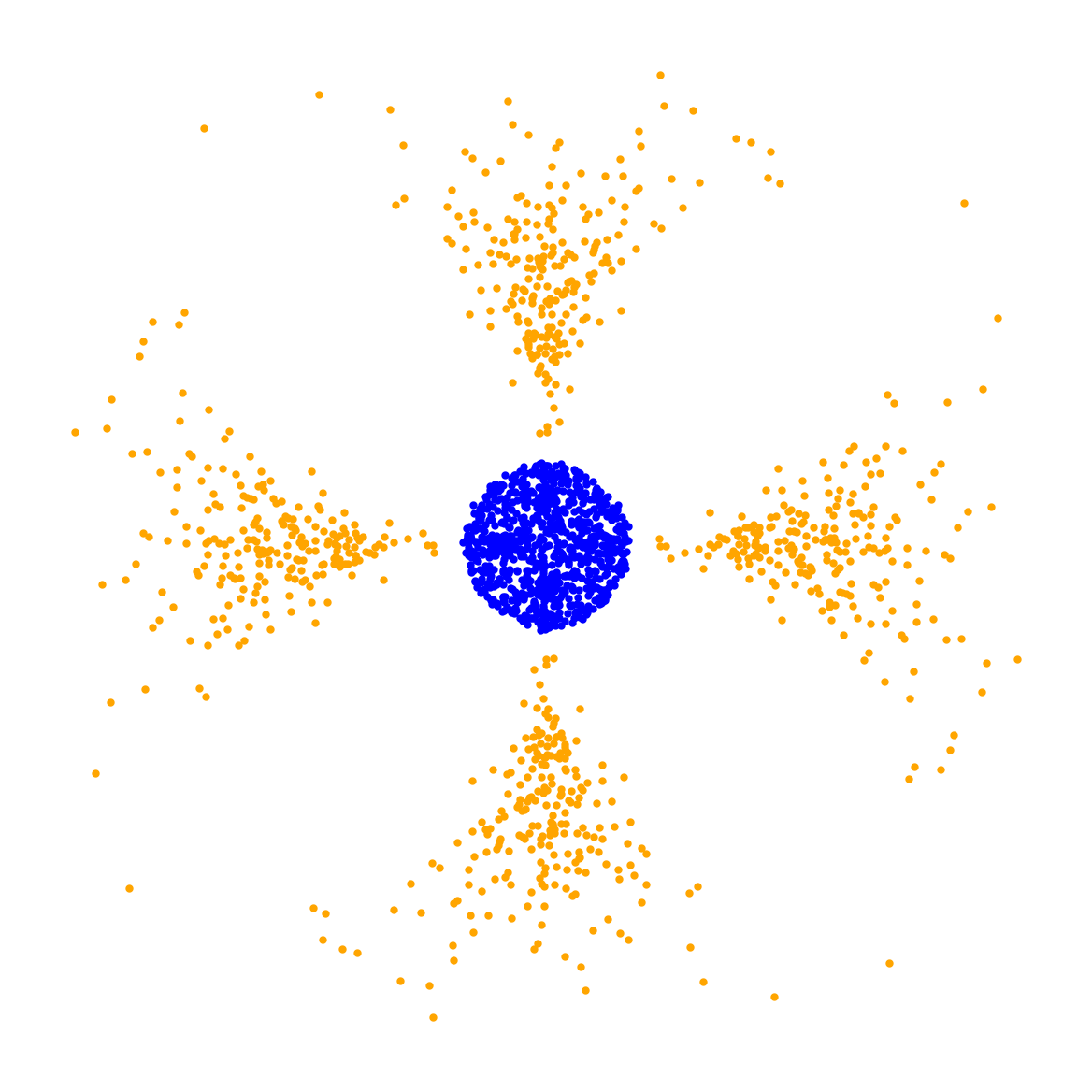}
    \end{subfigure}%
    \begin{subfigure}[t]{.14\textwidth}
        \includegraphics[width=\linewidth, valign=c]{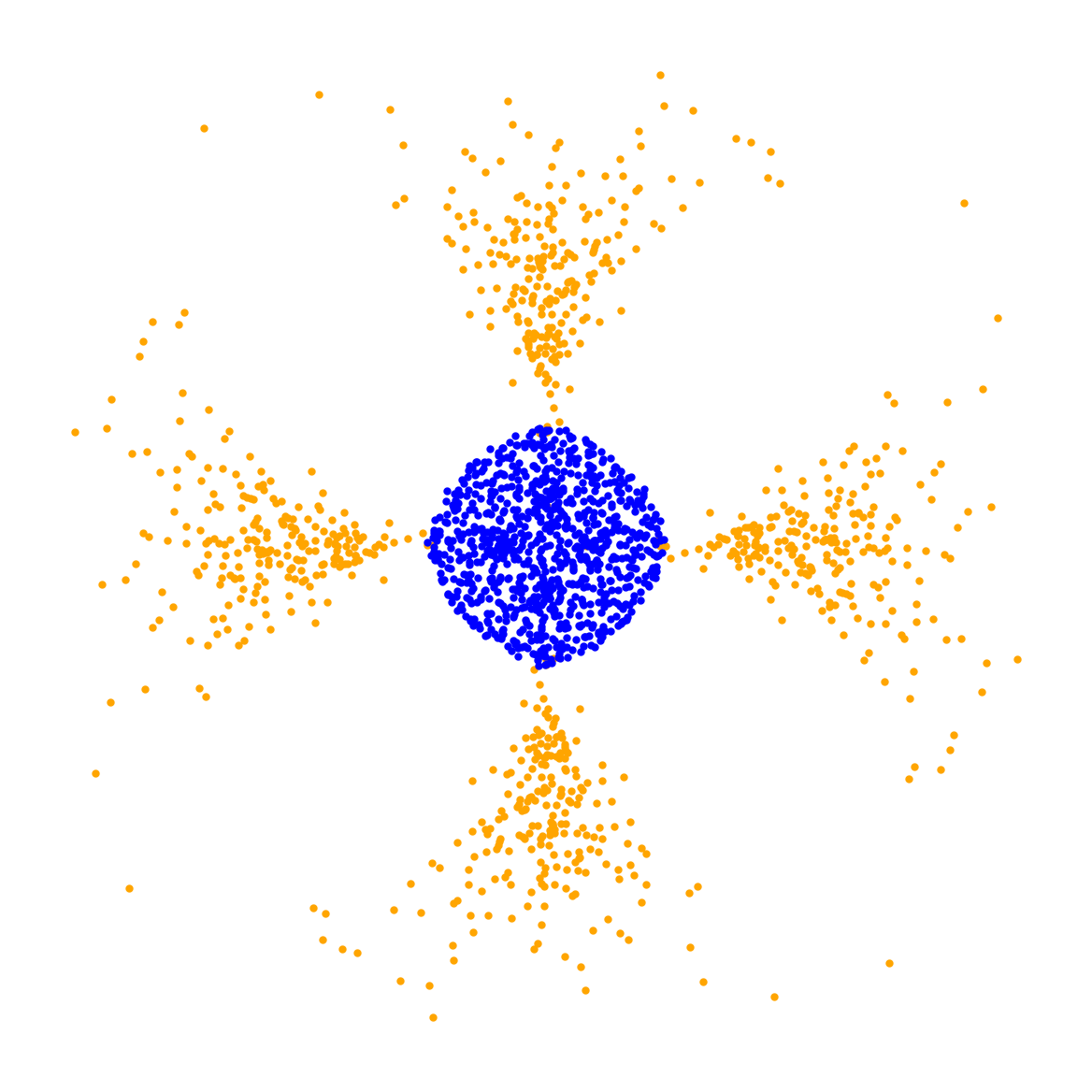}
    \end{subfigure}%
    \begin{subfigure}[t]{.14\textwidth}
        \includegraphics[width=\linewidth, valign=c]{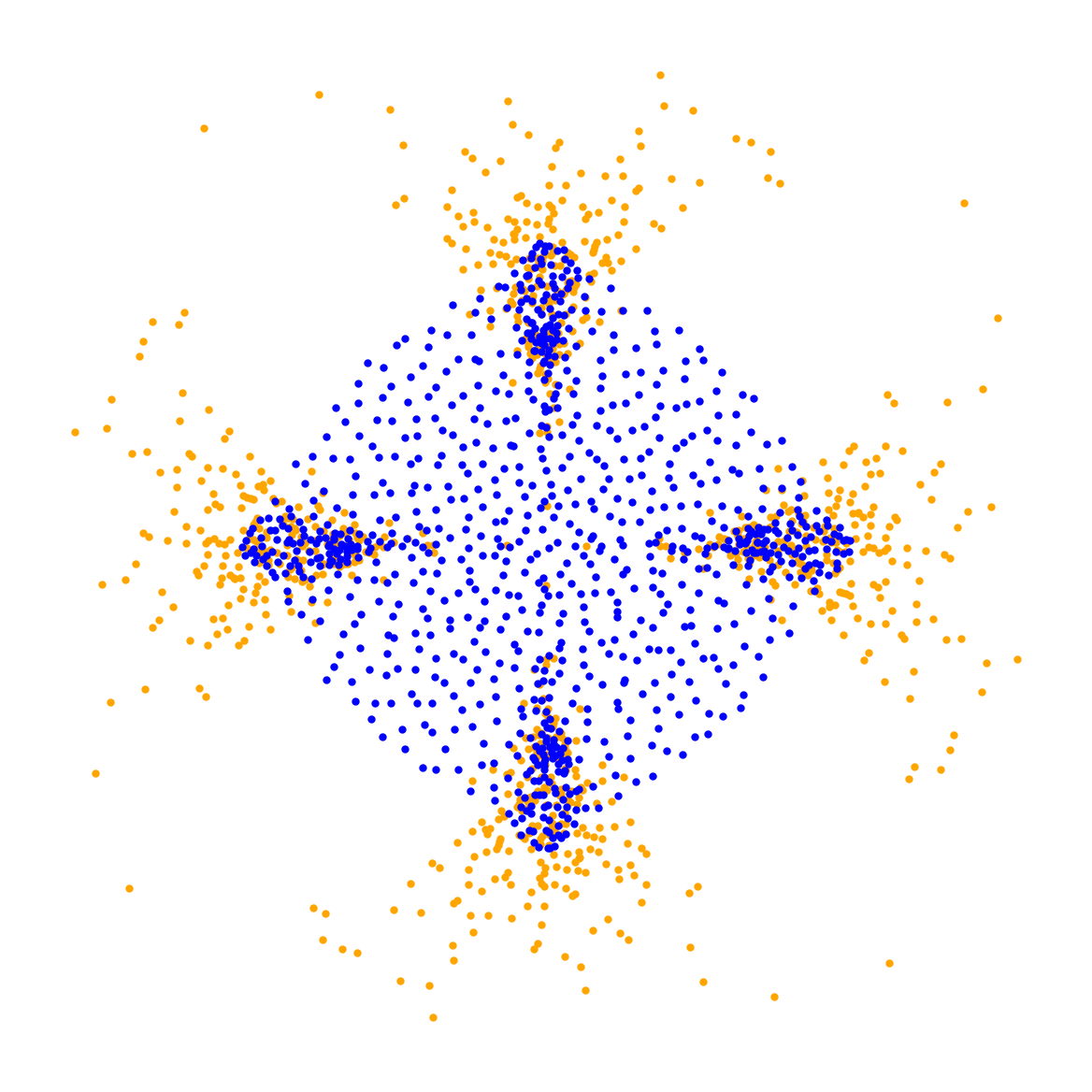}
    \end{subfigure}%
    \begin{subfigure}[t]{.14\textwidth}
        \includegraphics[width=\linewidth, valign=c]{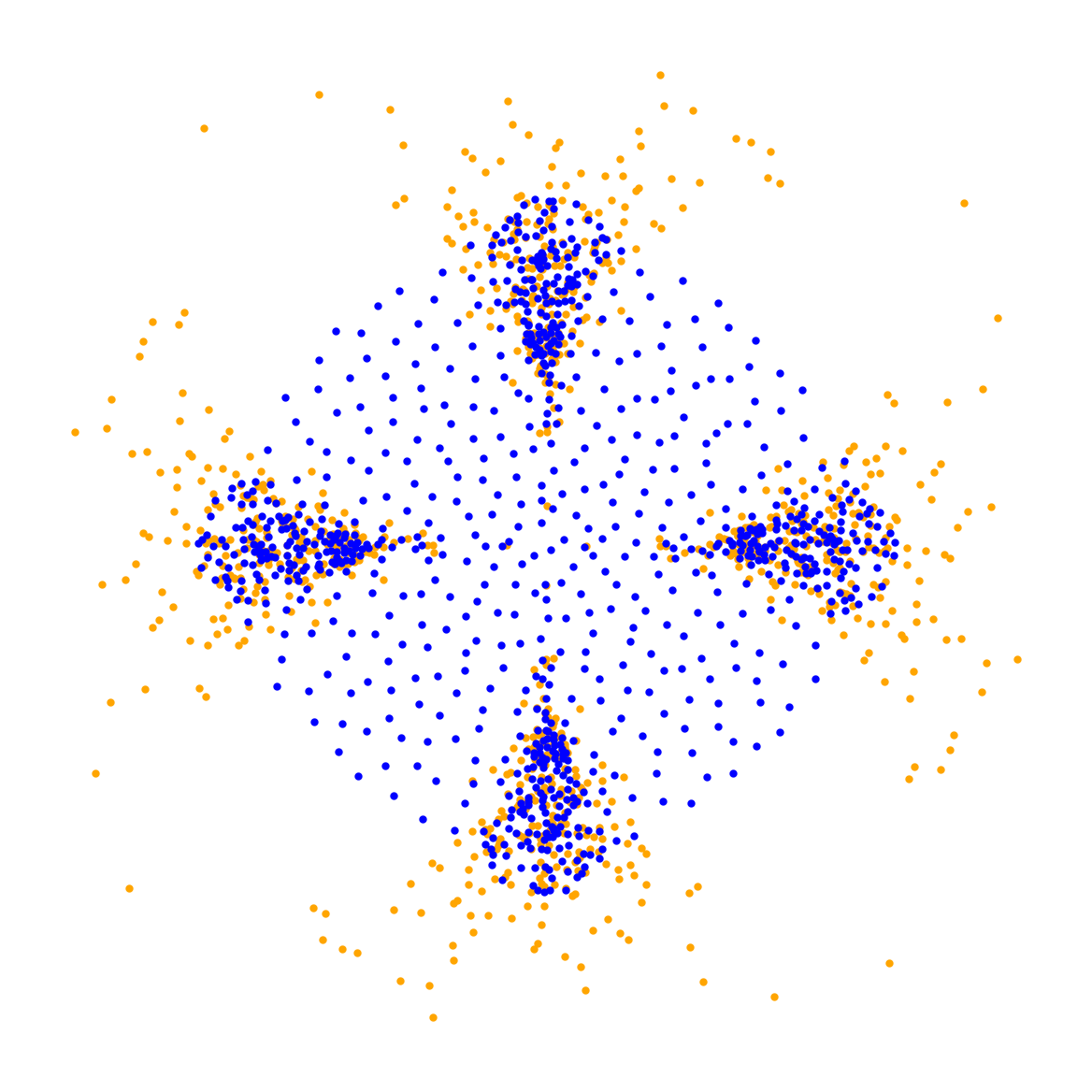}
    \end{subfigure}%
    \begin{subfigure}[t]{.14\textwidth}
        \includegraphics[width=\linewidth, valign=c]{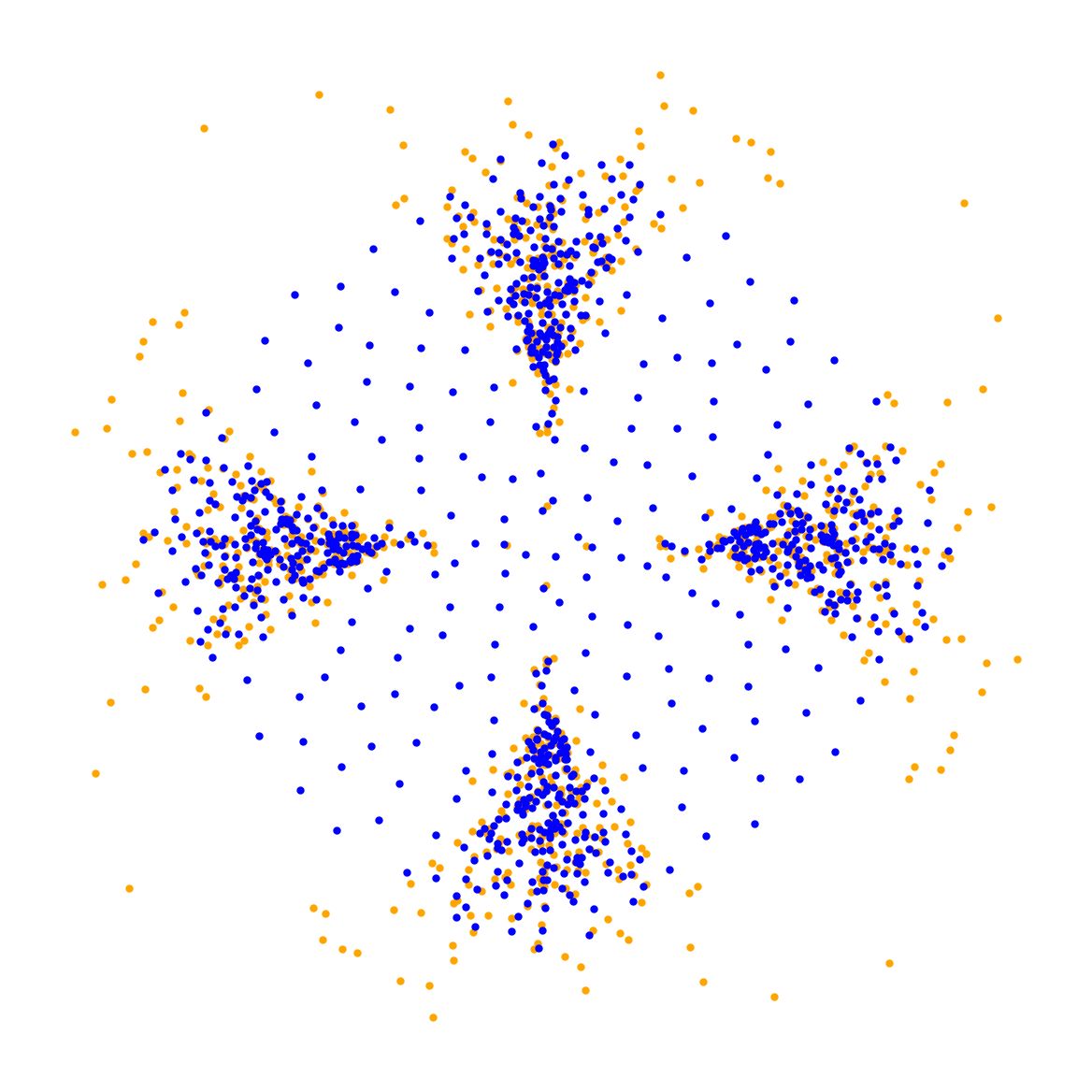}
    \end{subfigure} \\ %
    \rotatebox{90}{
        \begin{minipage}{.05\textwidth}
        \centering 
        \small $1$
    \end{minipage}} \hfill
    \begin{subfigure}[t]{.14\textwidth}
        \includegraphics[width=\linewidth, valign=c]{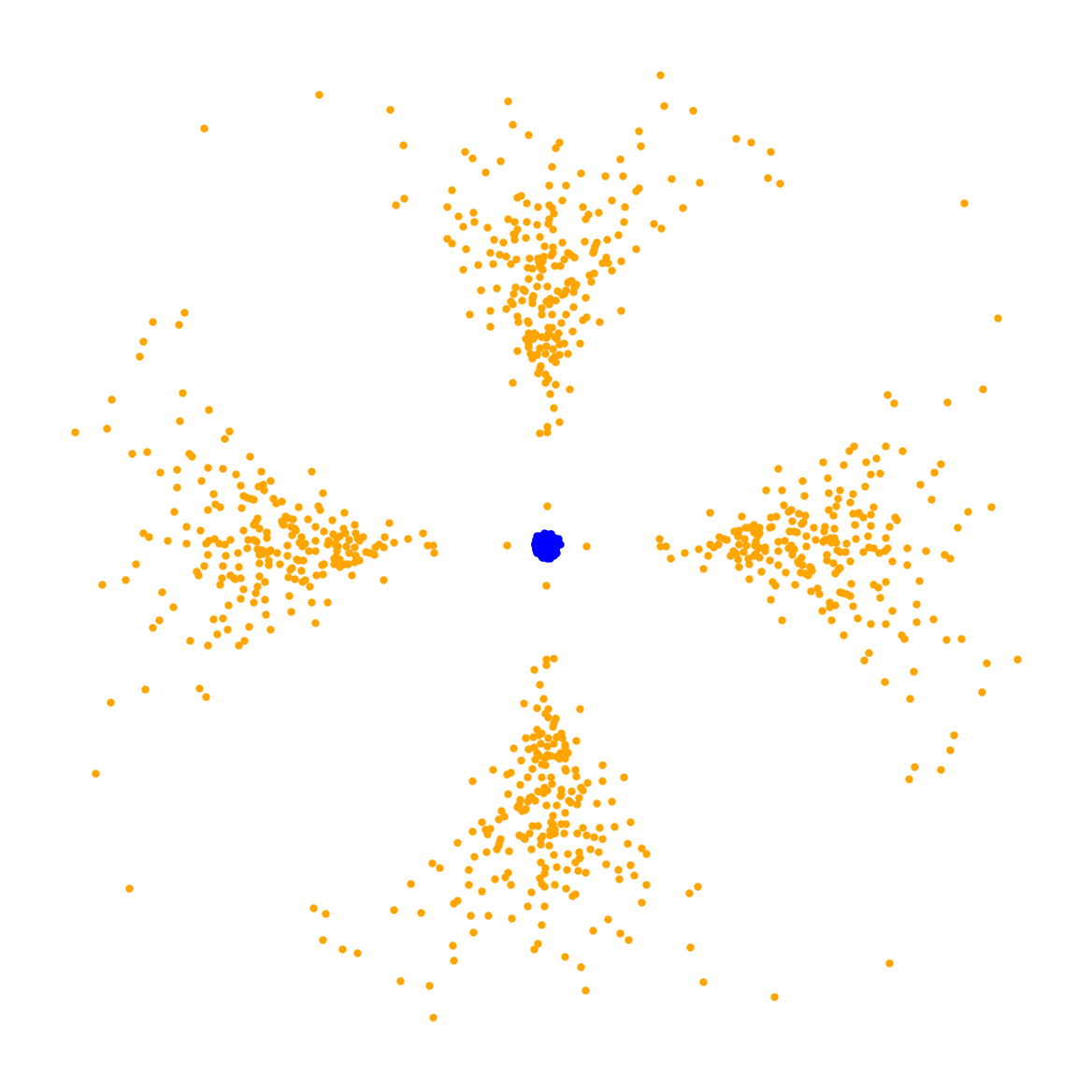}
        \caption*{$t = 0.1$}
    \end{subfigure}%
    \begin{subfigure}[t]{.14\textwidth}
        \includegraphics[width=\linewidth, valign=c]{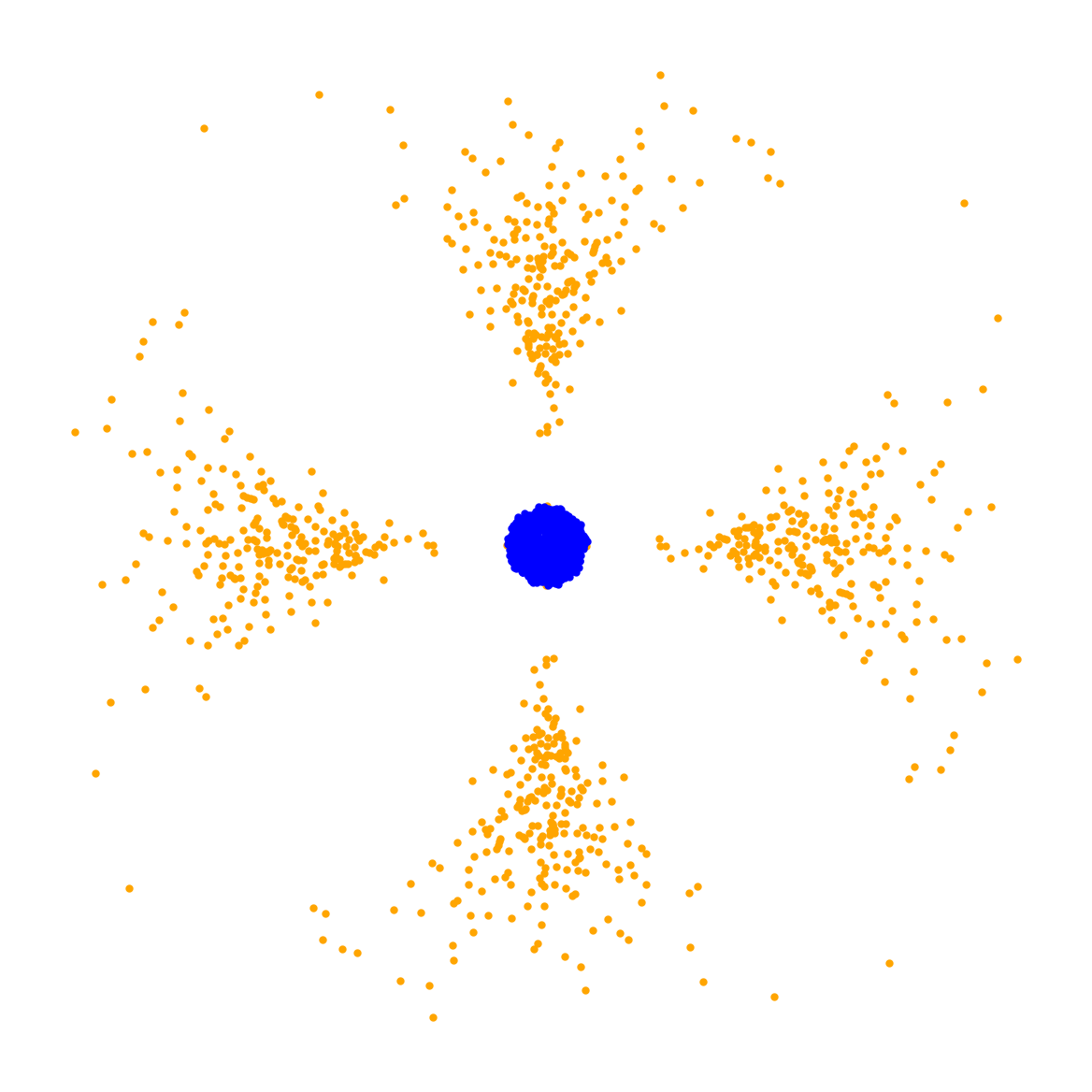}
        \caption*{$t = 0.5$}
    \end{subfigure}%
    \begin{subfigure}[t]{.14\textwidth}
        \includegraphics[width=\linewidth, valign=c]{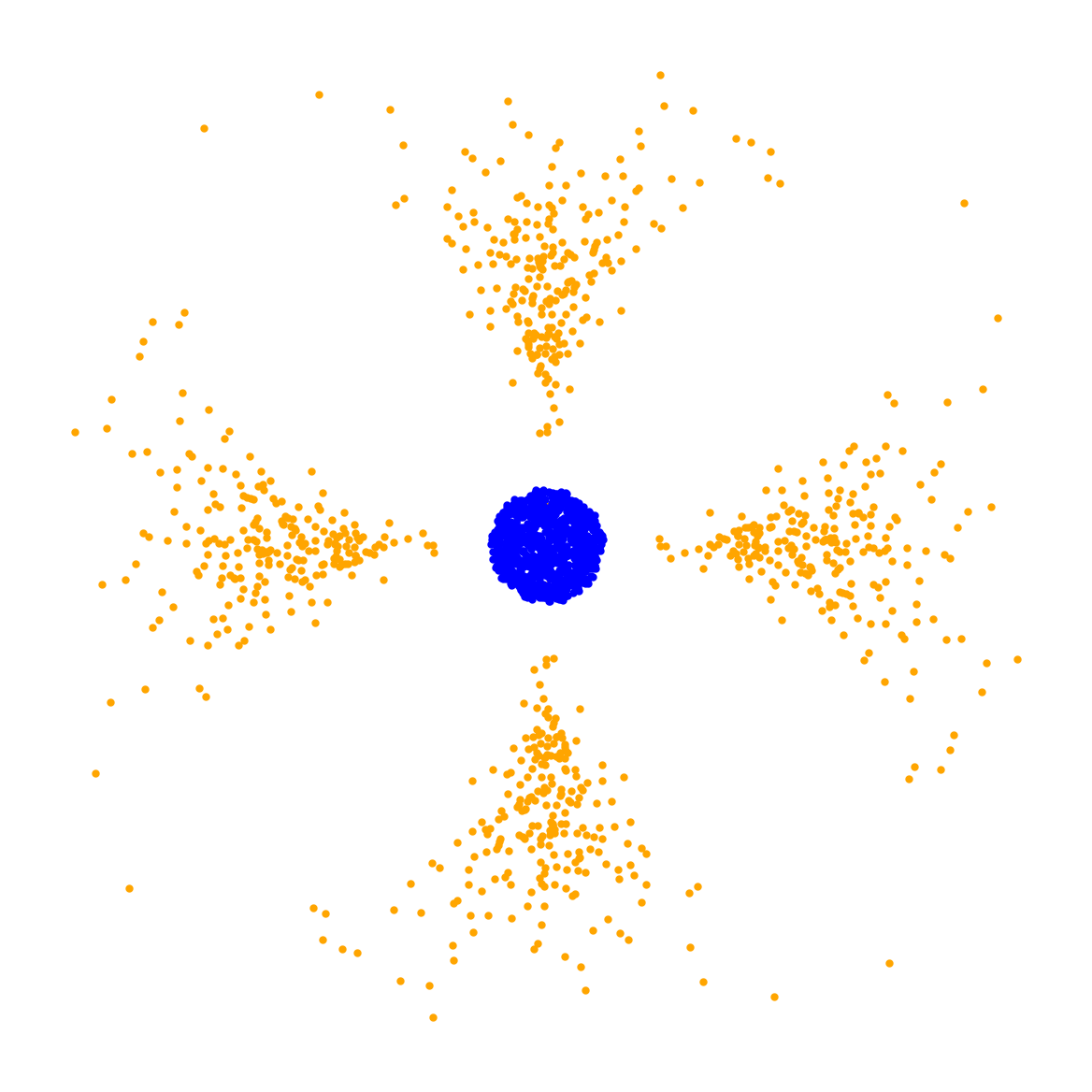}
        \caption*{$t = 1$}
    \end{subfigure}%
    \begin{subfigure}[t]{.14\textwidth}
        \includegraphics[width=\linewidth, valign=c]{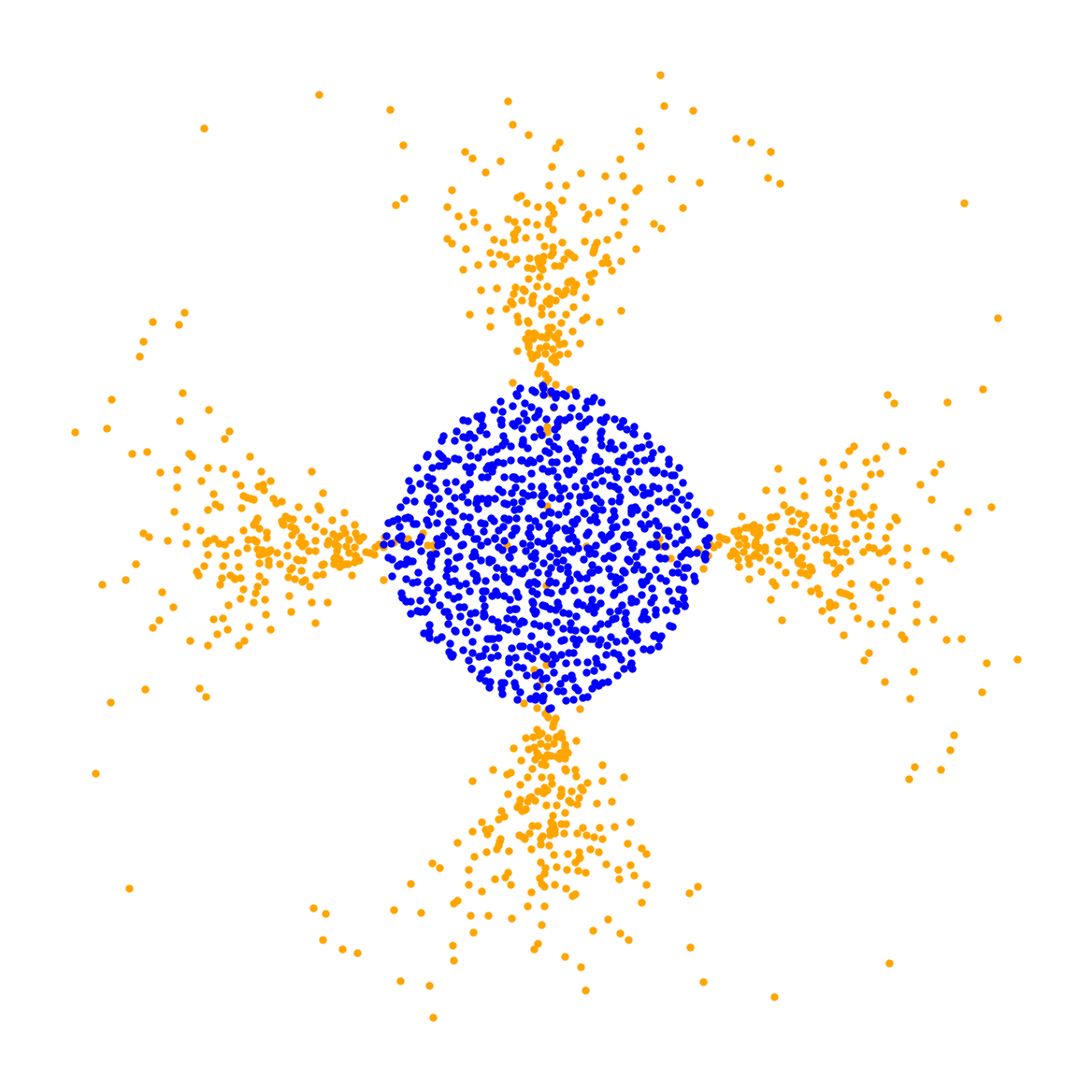}
        \caption*{$t = 10$}
    \end{subfigure}%
    \begin{subfigure}[t]{.14\textwidth}
        \includegraphics[width=\linewidth, valign=c]{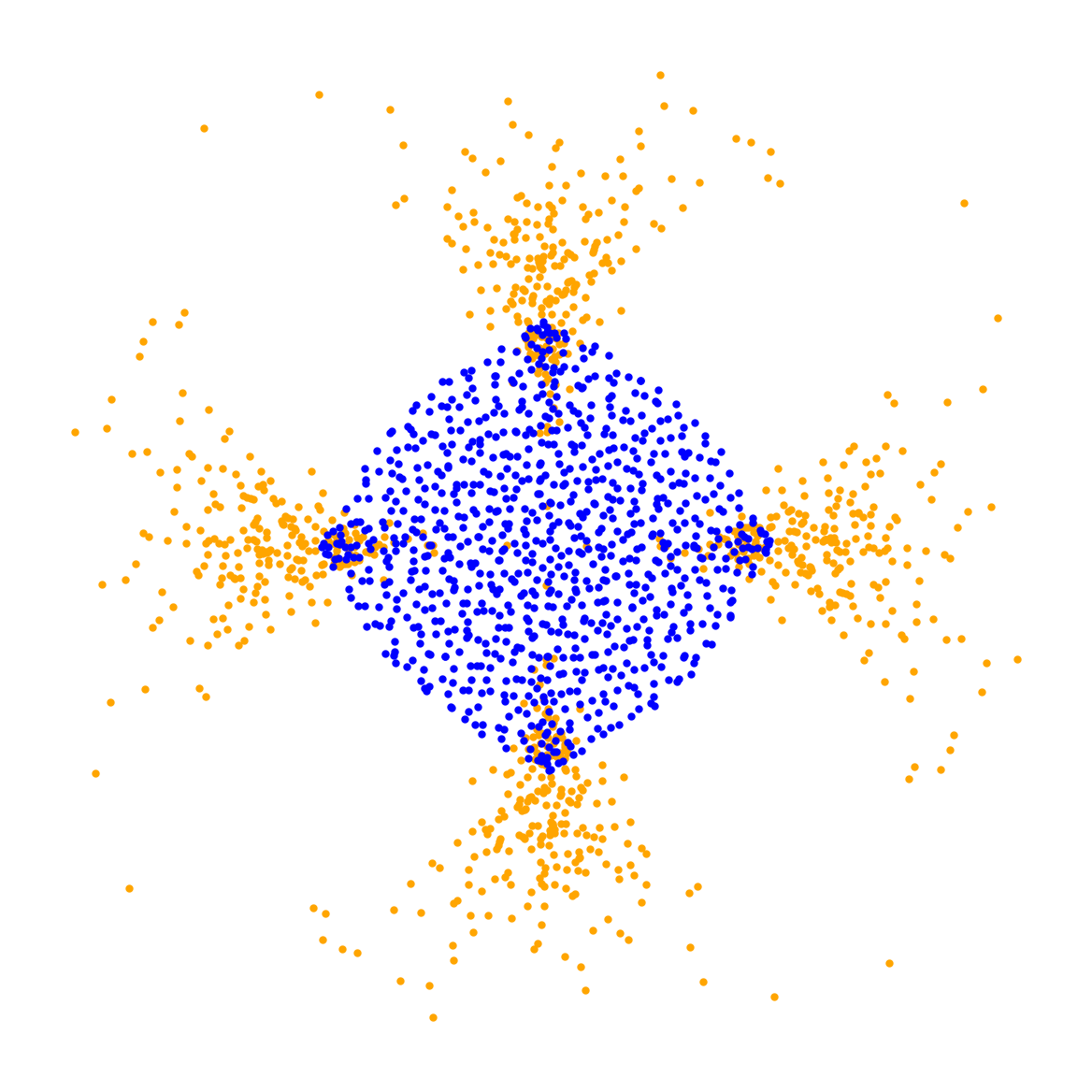}
        \caption*{$t = 20$}
    \end{subfigure}%
    \begin{subfigure}[t]{.14\textwidth}
        \includegraphics[width=\linewidth, valign=c]{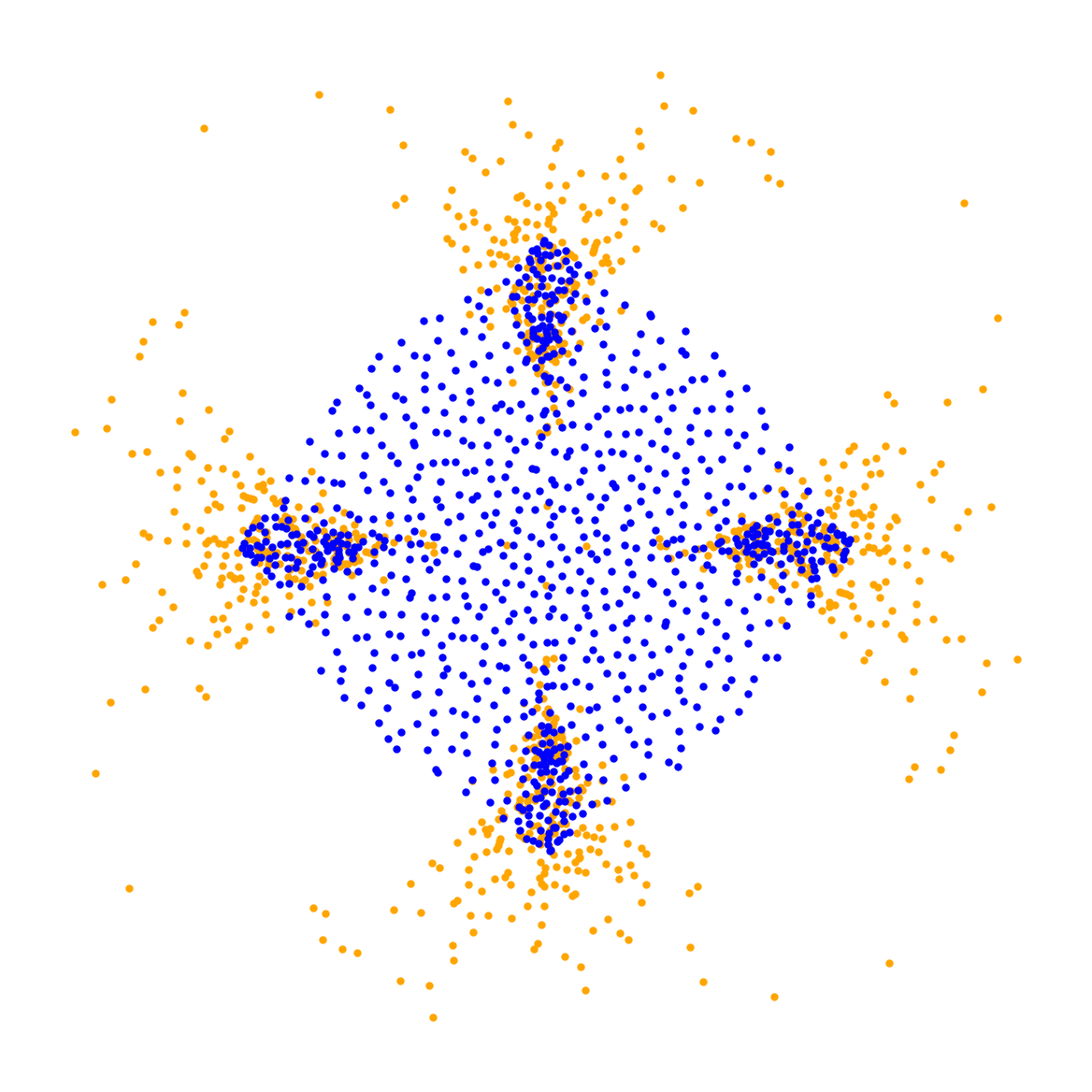}
        \caption*{$t = 50$}
    \end{subfigure} \\ %
    \caption{Ablation study for the parameter $\lambda$, illustrated for the Jeffreys divergence.}
    \label{fig:ablationCirclesCompact2Lambda}
\end{figure}
\begin{figure}[t]
    \centering
    \includegraphics[width=.45\textwidth]{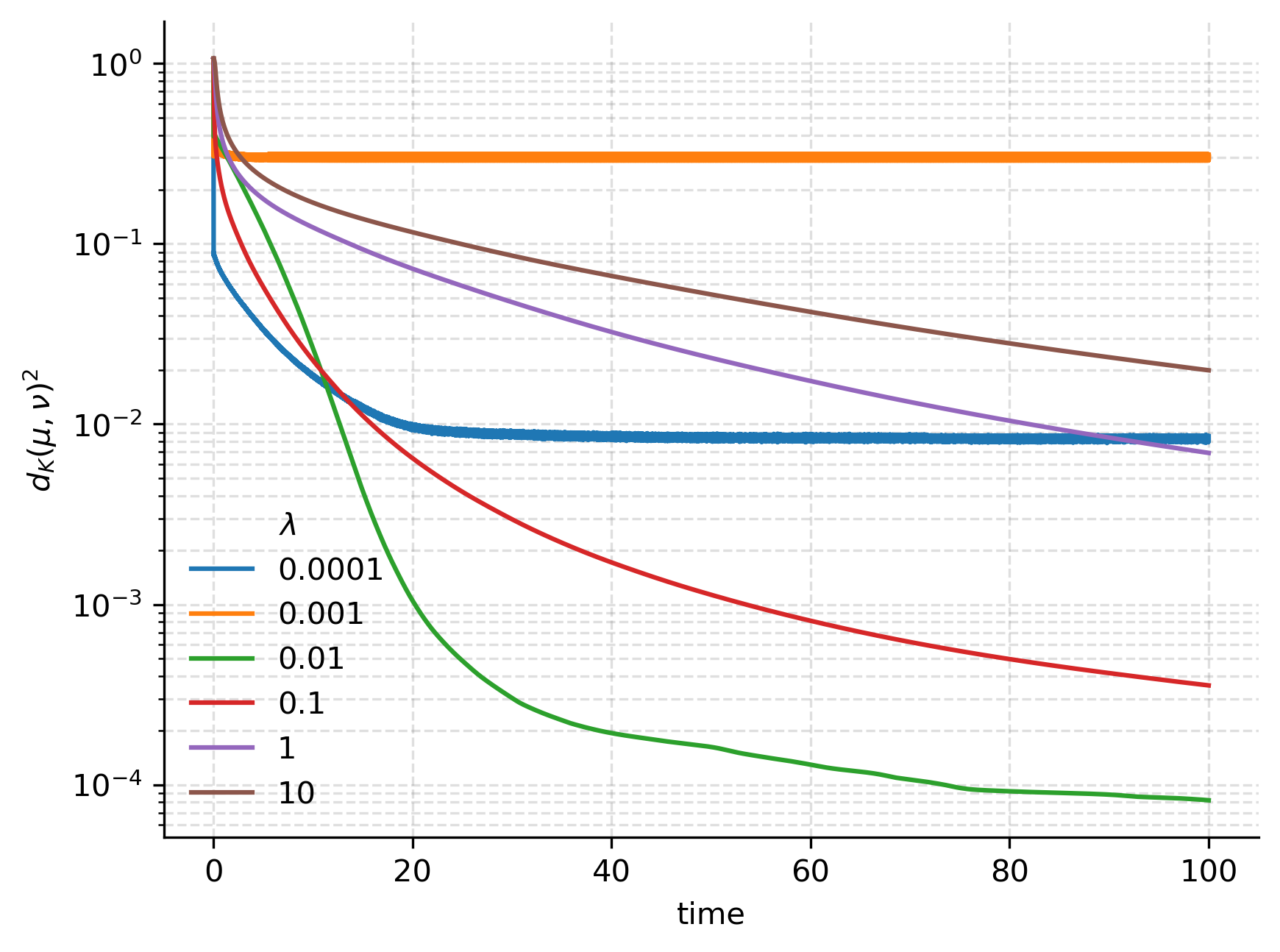}
    \includegraphics[width=.45\textwidth]{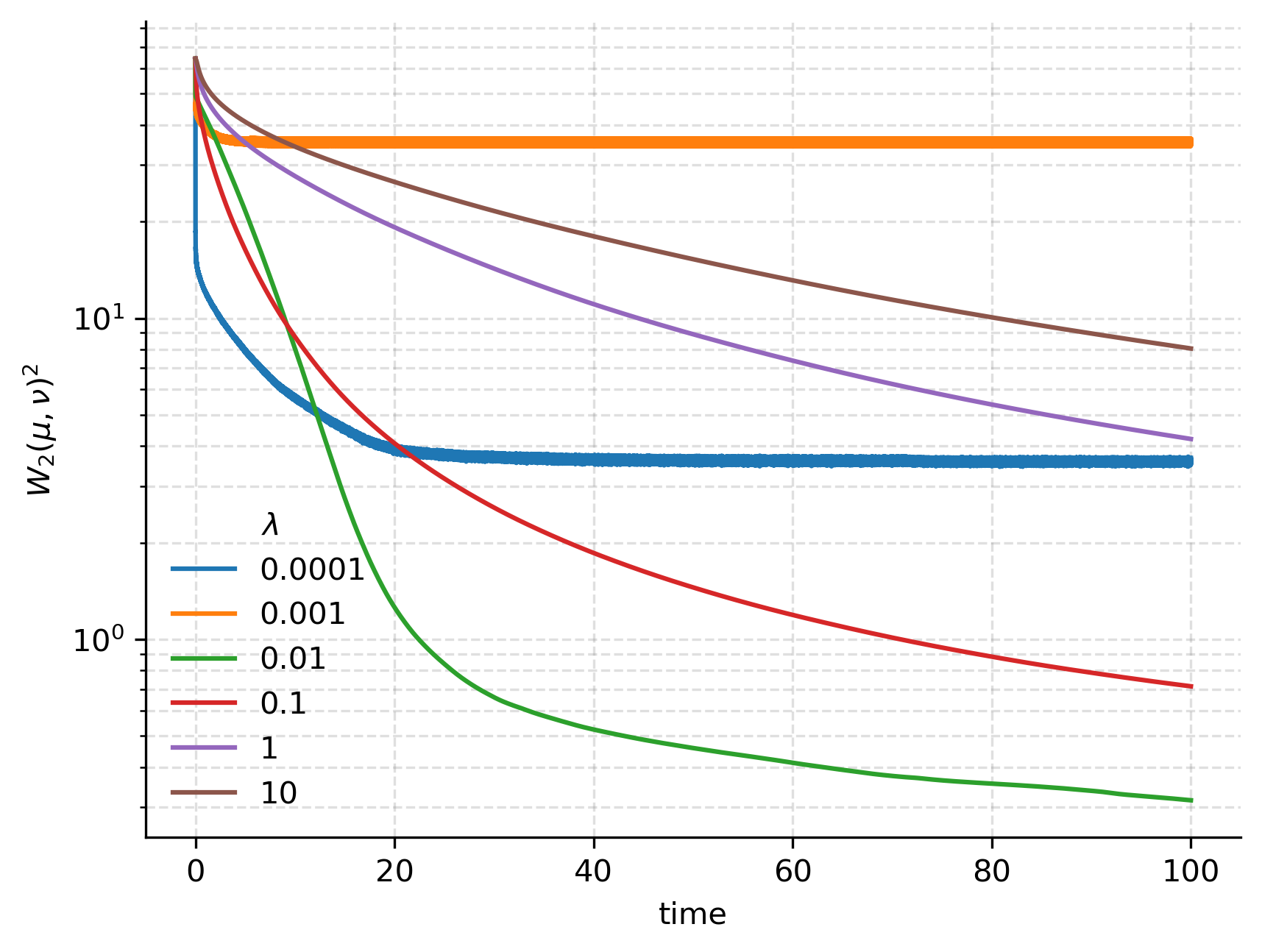}
    \caption{Comparison of the squared MMD (\textbf{left}) and $W_2$ distance (\textbf{right}) for Figure~\ref{fig:ablationCirclesCompact2Lambda}.}
    \label{fig:ablationCirclesJeffreyLambdaGraphs}
\end{figure}
\begin{figure}[t]
    \centering
    \includegraphics[width=.45\textwidth]{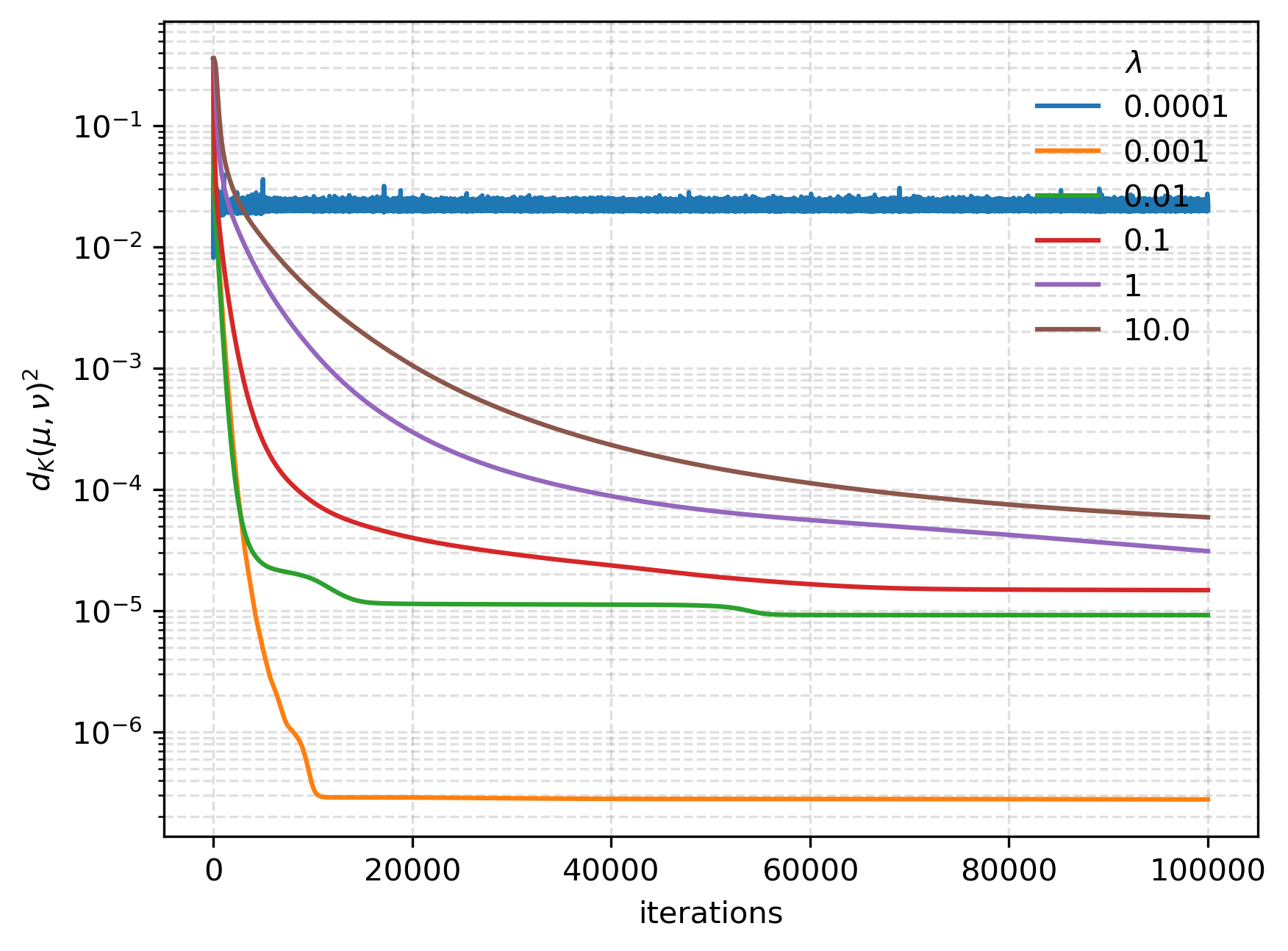}
    \includegraphics[width=.45\textwidth]{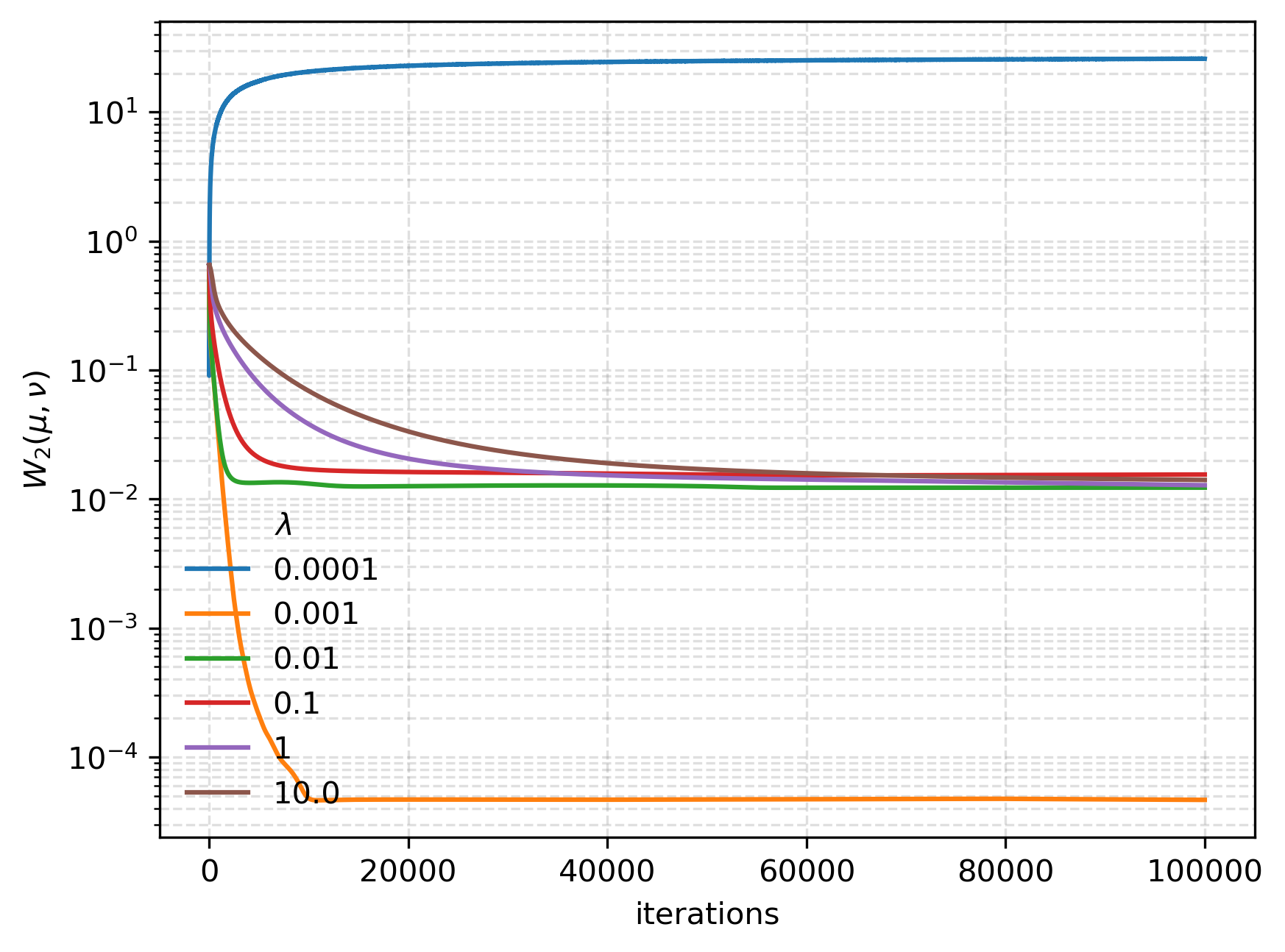}
    \caption{Comparing the squared MMD  (\textbf{left}) and $W_2$ distance (\textbf{right}) for different $\lambda$.}
    \label{fig:ablationCirclesCompact2LambdaGraphs}
    \vspace{.3cm}

    % \centering
    % \includegraphics[width=.45\textwidth]{img/chi/Reg_chiMMD_timeline,.001,900,IMQ,.05,100,circles.png}
    % \includegraphics[width=.45\textwidth]{img/chi/Reg_chiW2_timeline,.001,900,IMQ,.05,100,circles.png}
    % \caption{Comparison of the squared MMD  (\textbf{left}) and $W_2$ distance (\textbf{right}) for different values of the $\chi^{\alpha}$-divergence parameter $\alpha$.}
    % \label{fig:chi_Comparison}
    % \vspace{.3cm}

    % \centering
    % \includegraphics[width=.45\textwidth]{img/chi-alpha_generators.png}
    % \includegraphics[width=.45\textwidth]{img/Tsallis_generators.png}
    % \caption{\textbf{Left:} The $\chi^{\alpha}$ entropy functions for $\alpha \in [1, 3]$.
    % \textbf{Right:} The $\alpha$-Tsallis entropy functions for $\alpha \in (0, 2.5]$.}
    % \label{fig:chi_entropy_functions}
\end{figure}

To find the combinations of $\lambda$ and $\tau$ that give the best results, we use the two \enquote{two bananas} target and evaluate each flow after 100 iterations of the forward Euler scheme.
Note that $\frac{\tau}{\lambda}$ is the prefactor in front of the gradient term in the update step \eqref{eq:compute2}.
Hence, it influences how much gradient information is taken into consideration when updating the particles.
\begin{figure}[t]
    \centering
    \rotatebox{90}{
        \begin{minipage}{.1\textwidth}
        \centering 
        \small $\lambda = \num{e-2}$
    \end{minipage}} \hfill
    \begin{subfigure}[t]{.24\textwidth}
        \includegraphics[width=\linewidth, valign=c]{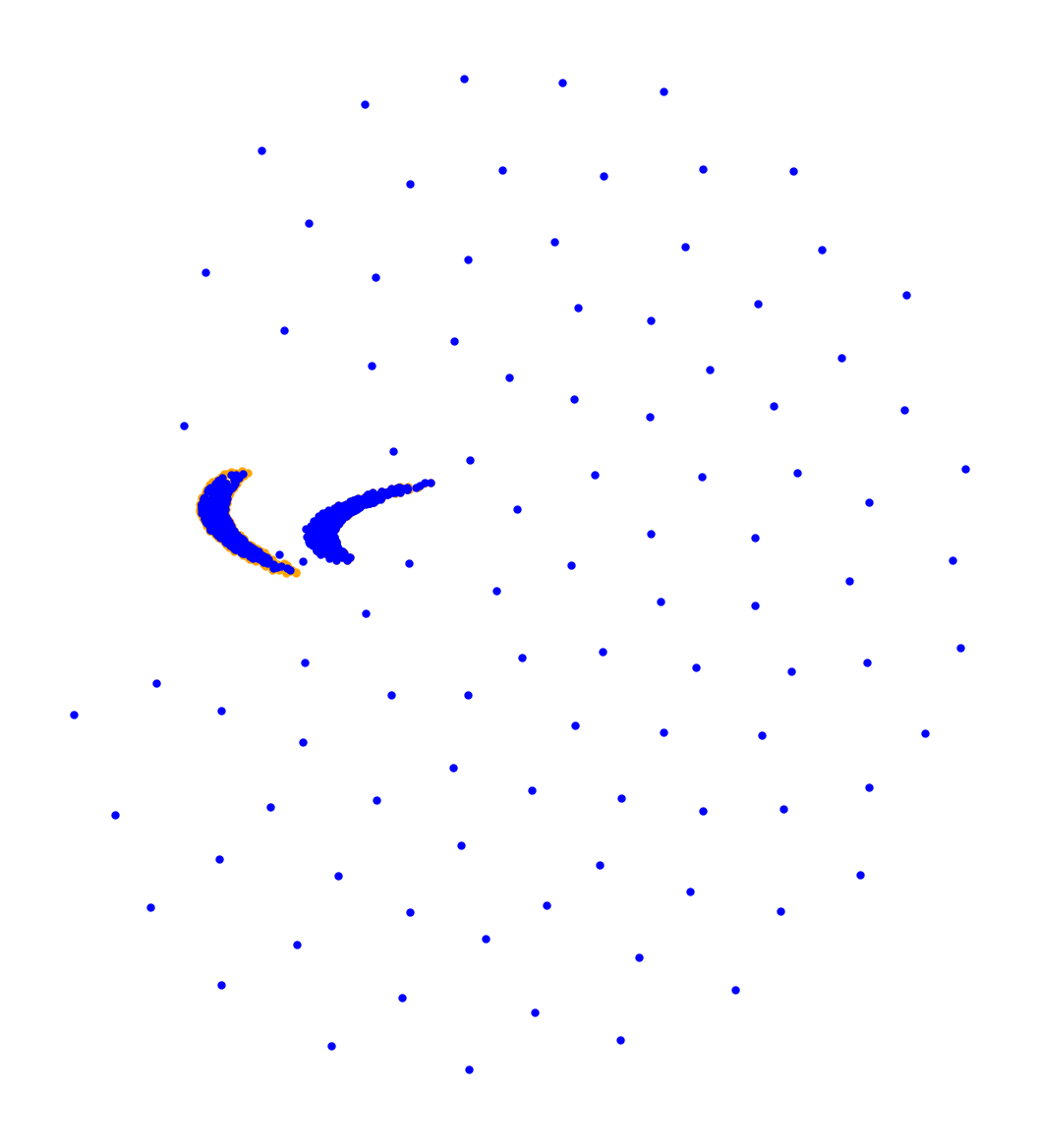}
    \end{subfigure}%
    \begin{subfigure}[t]{.24\textwidth}
        \includegraphics[width=\linewidth, valign=c]{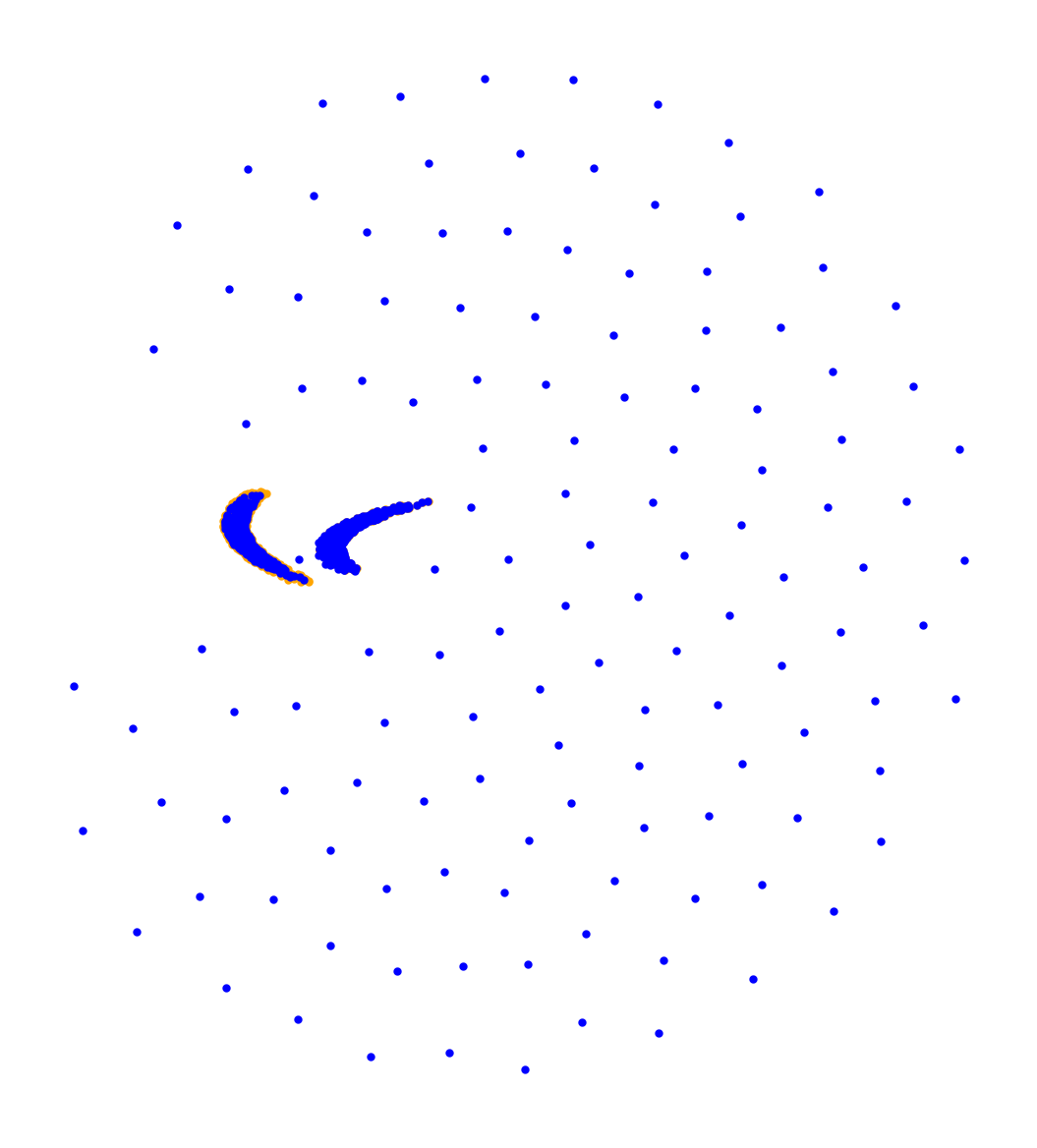}
    \end{subfigure}%
    \begin{subfigure}[t]{.24\textwidth}
        \includegraphics[width=\linewidth, valign=c]{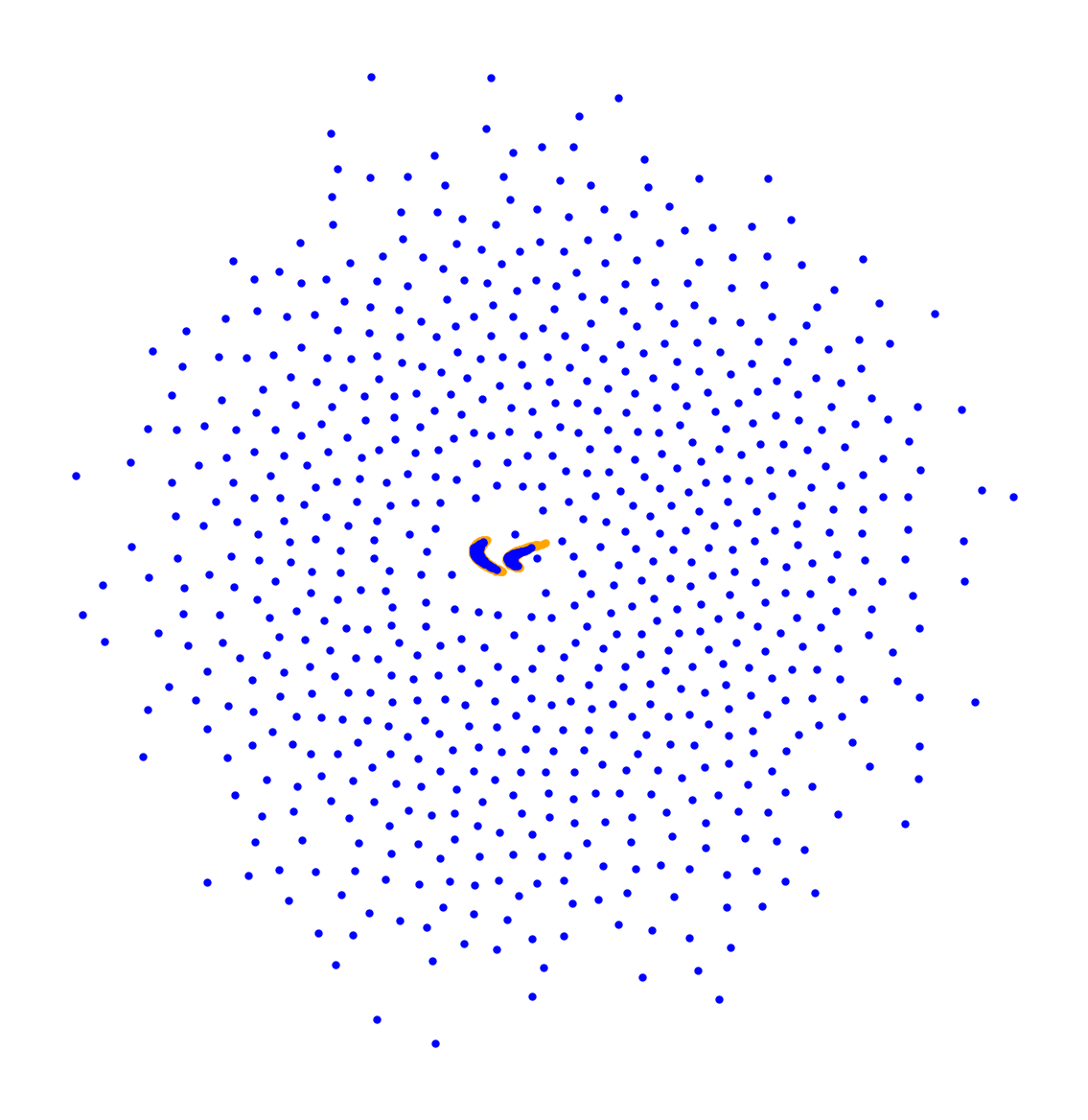}
    \end{subfigure}%
    \begin{subfigure}[t]{.24\textwidth}
        \includegraphics[width=\linewidth, valign=c]{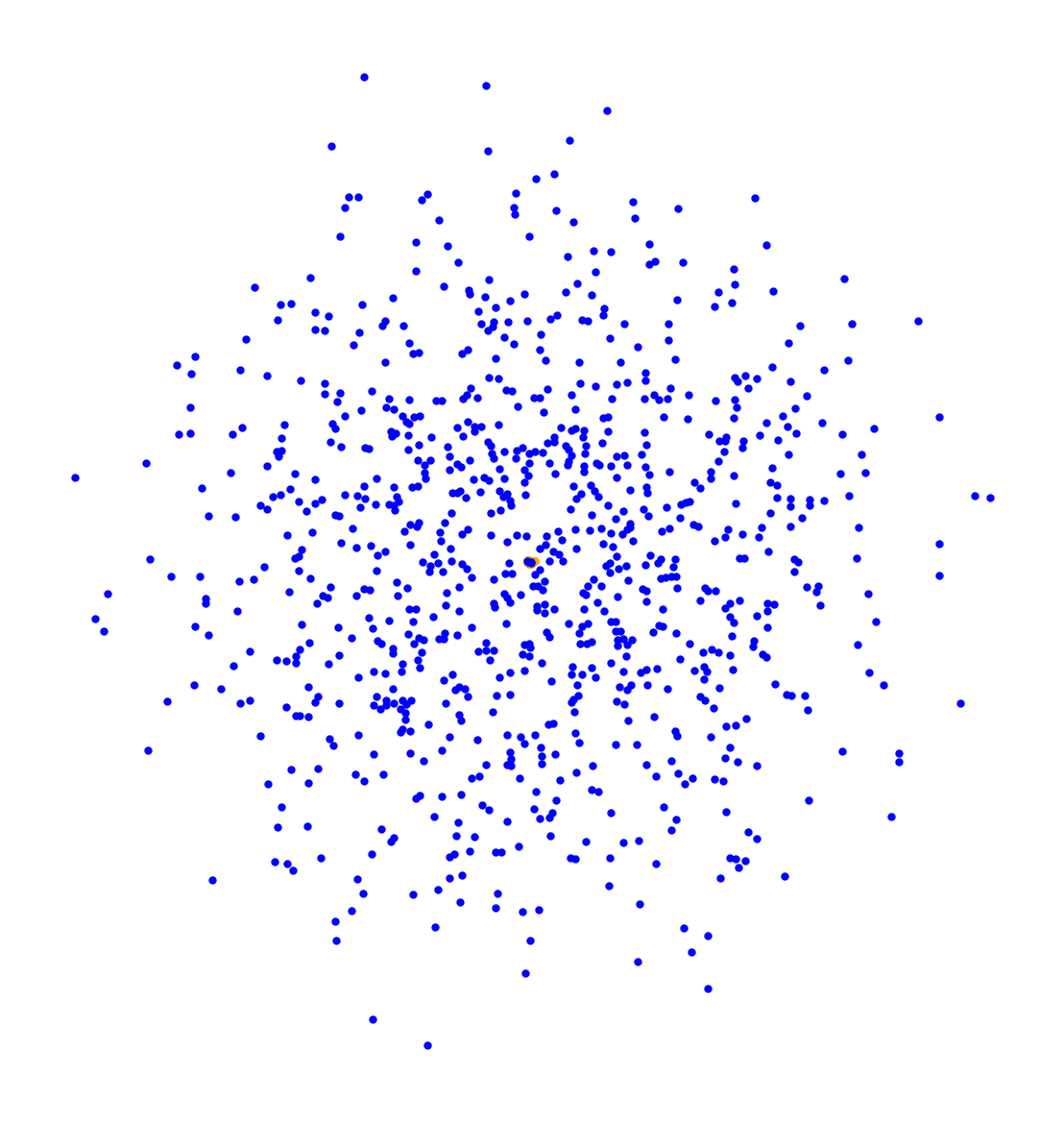}
    \end{subfigure} \\ %
    \rotatebox{90}{
        \begin{minipage}{.1\textwidth}
        \centering 
        \small $\lambda = \num{e-1}$
    \end{minipage}} \hfill
    \begin{subfigure}[t]{.24\textwidth}
        \includegraphics[width=\linewidth, valign=c]{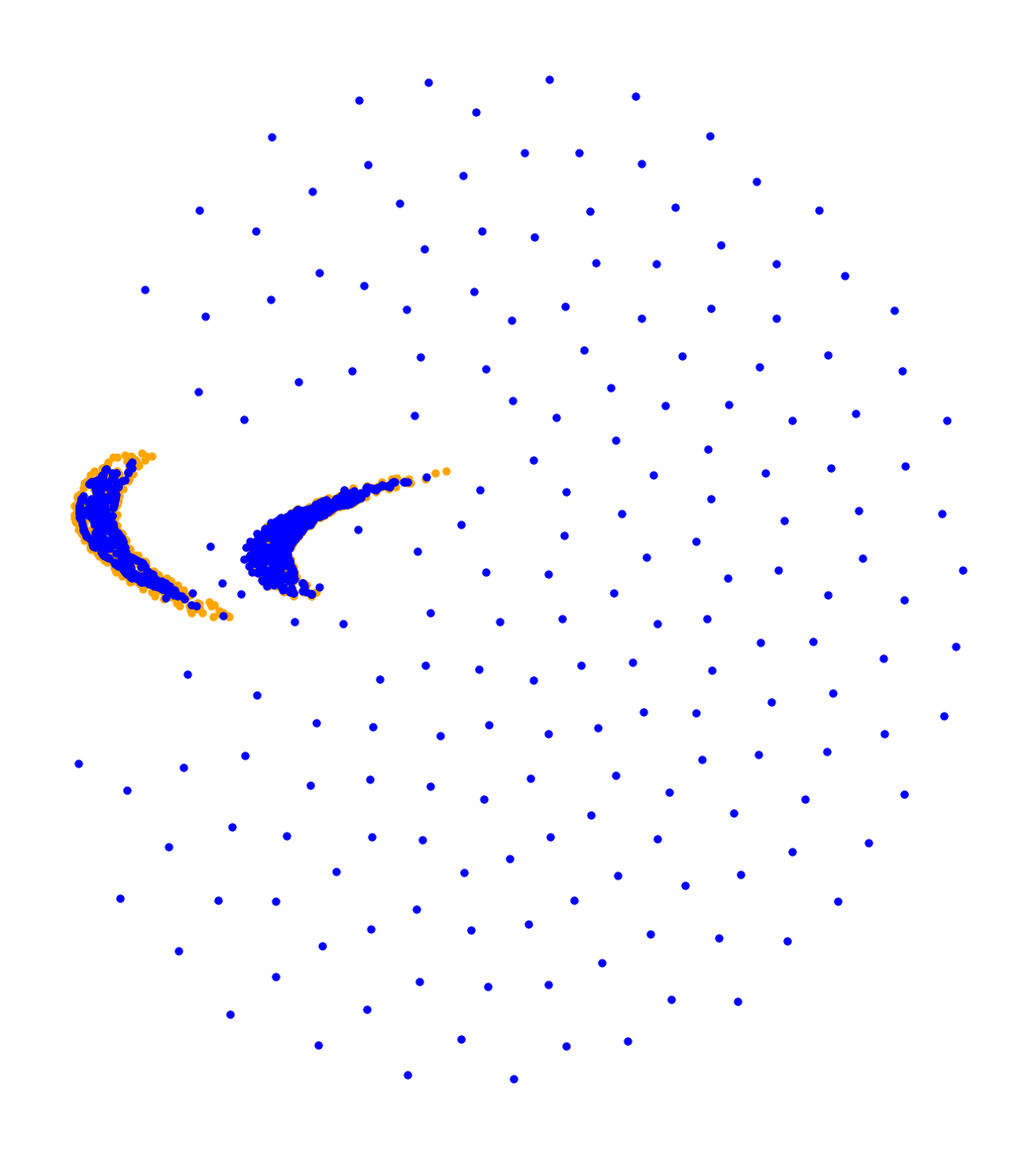}
    \end{subfigure}%
    \begin{subfigure}[t]{.24\textwidth}
        \includegraphics[width=\linewidth, valign=c]{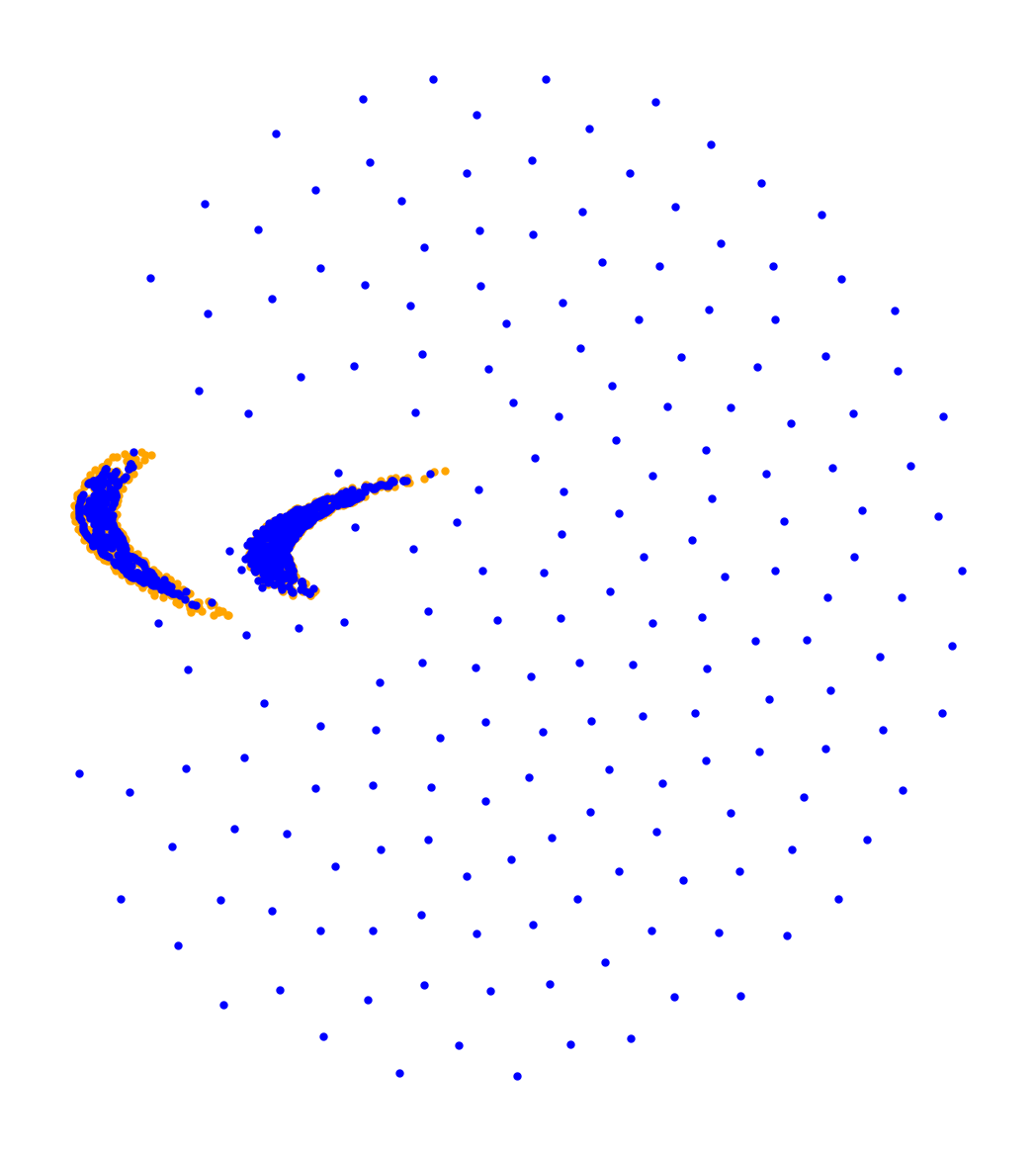}
    \end{subfigure}%
    \begin{subfigure}[t]{.24\textwidth}
        \includegraphics[width=\linewidth, valign=c]{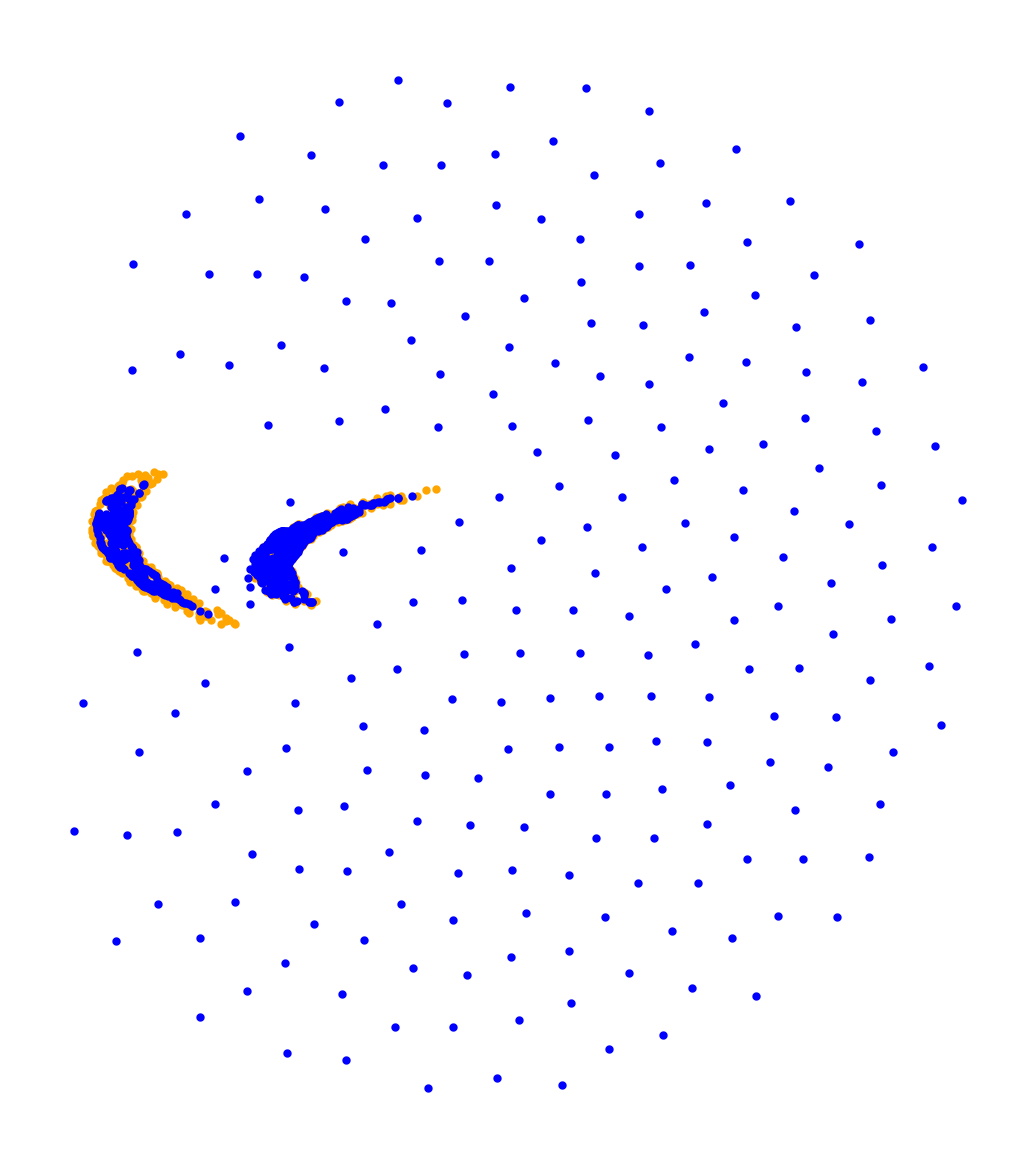}
    \end{subfigure}%
    \begin{subfigure}[t]{.24\textwidth}
        \includegraphics[width=\linewidth, valign=c]{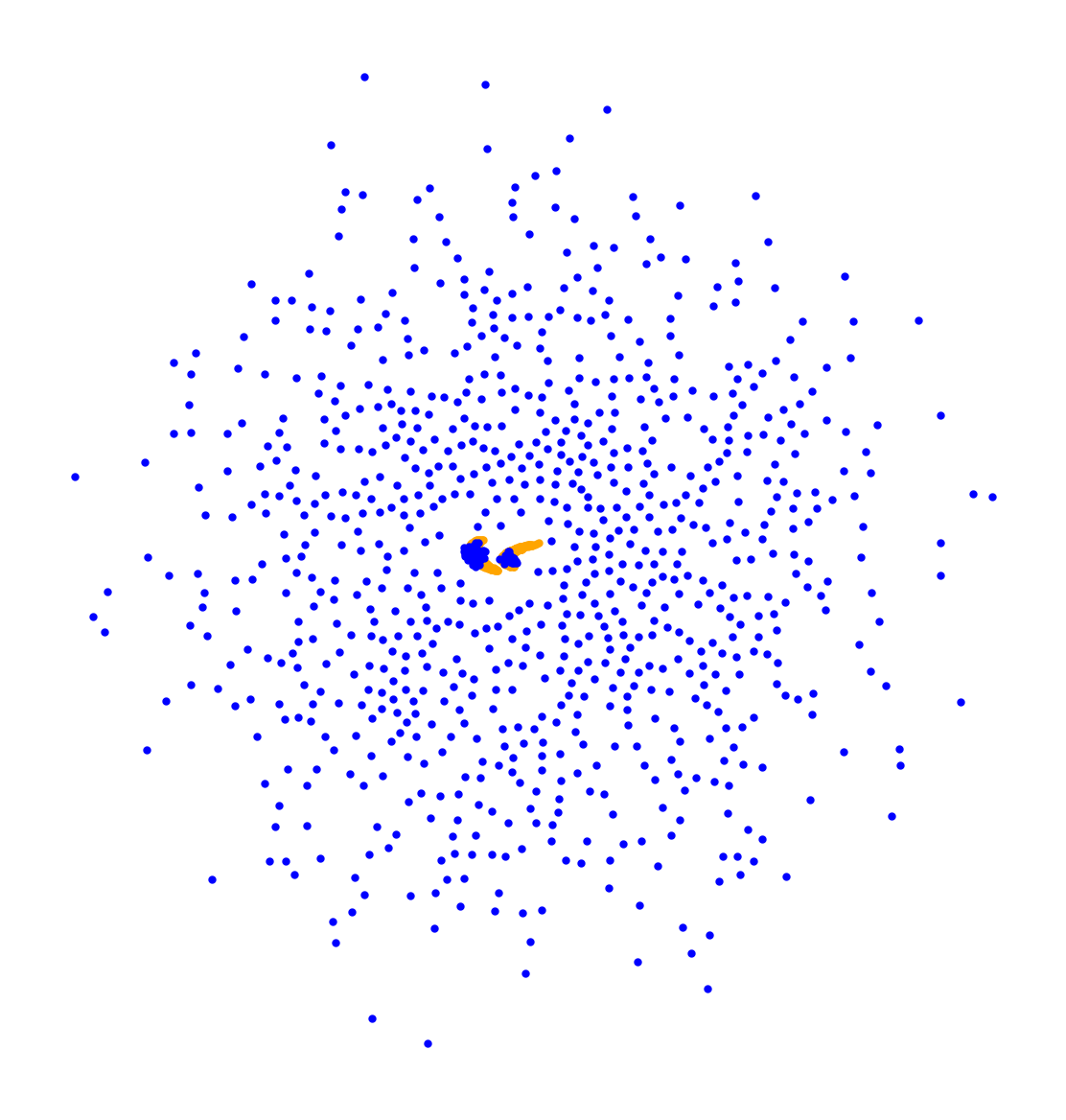}
    \end{subfigure} \\ %
    \rotatebox{90}{
        \begin{minipage}{.1\textwidth}
        \centering 
        \small $\lambda = 1$
    \end{minipage}} \hfill
    \begin{subfigure}[t]{.24\textwidth}
        \includegraphics[width=\linewidth, valign=c]{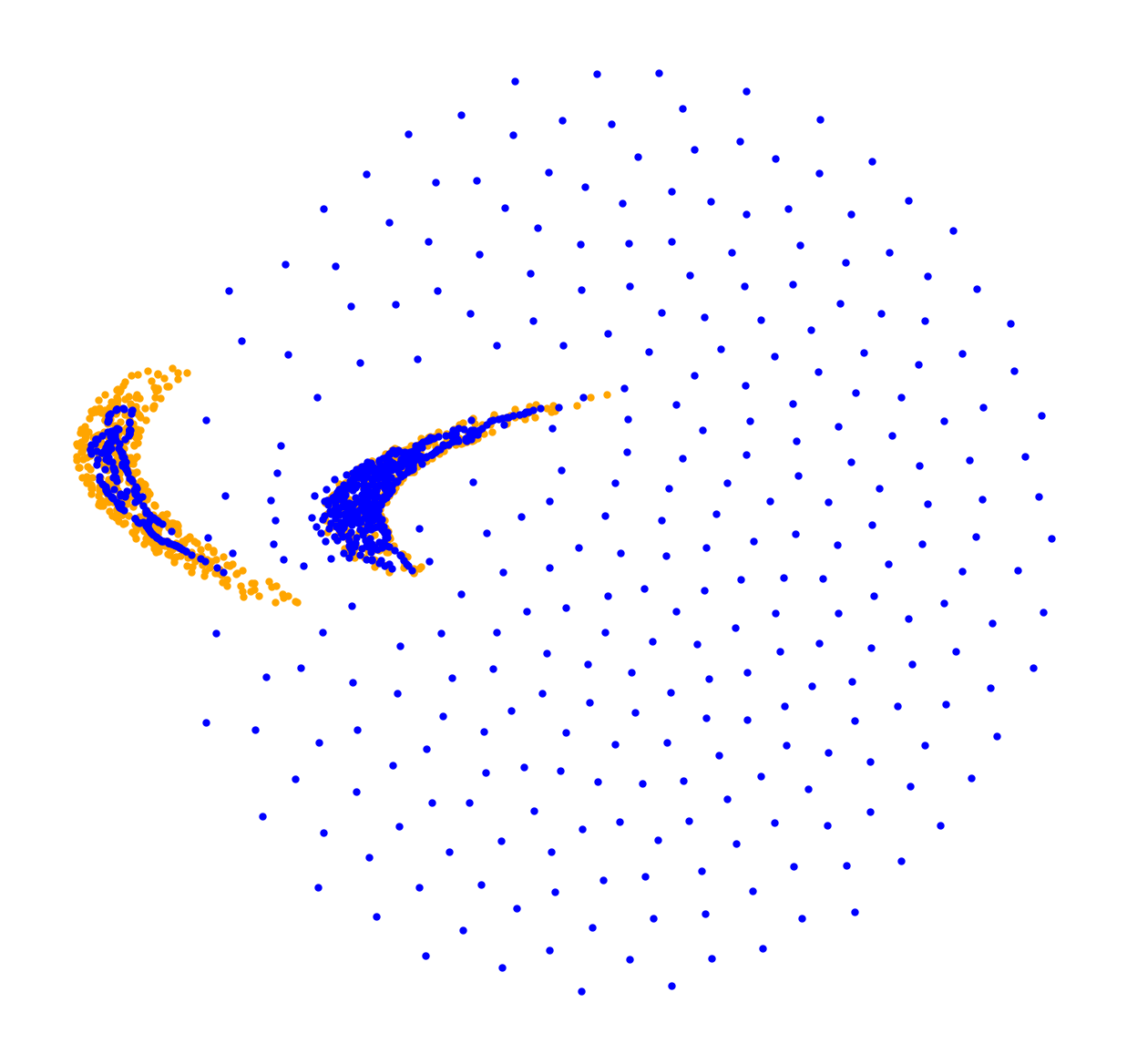}
    \end{subfigure}%
    \begin{subfigure}[t]{.24\textwidth}
        \includegraphics[width=\linewidth, valign=c]{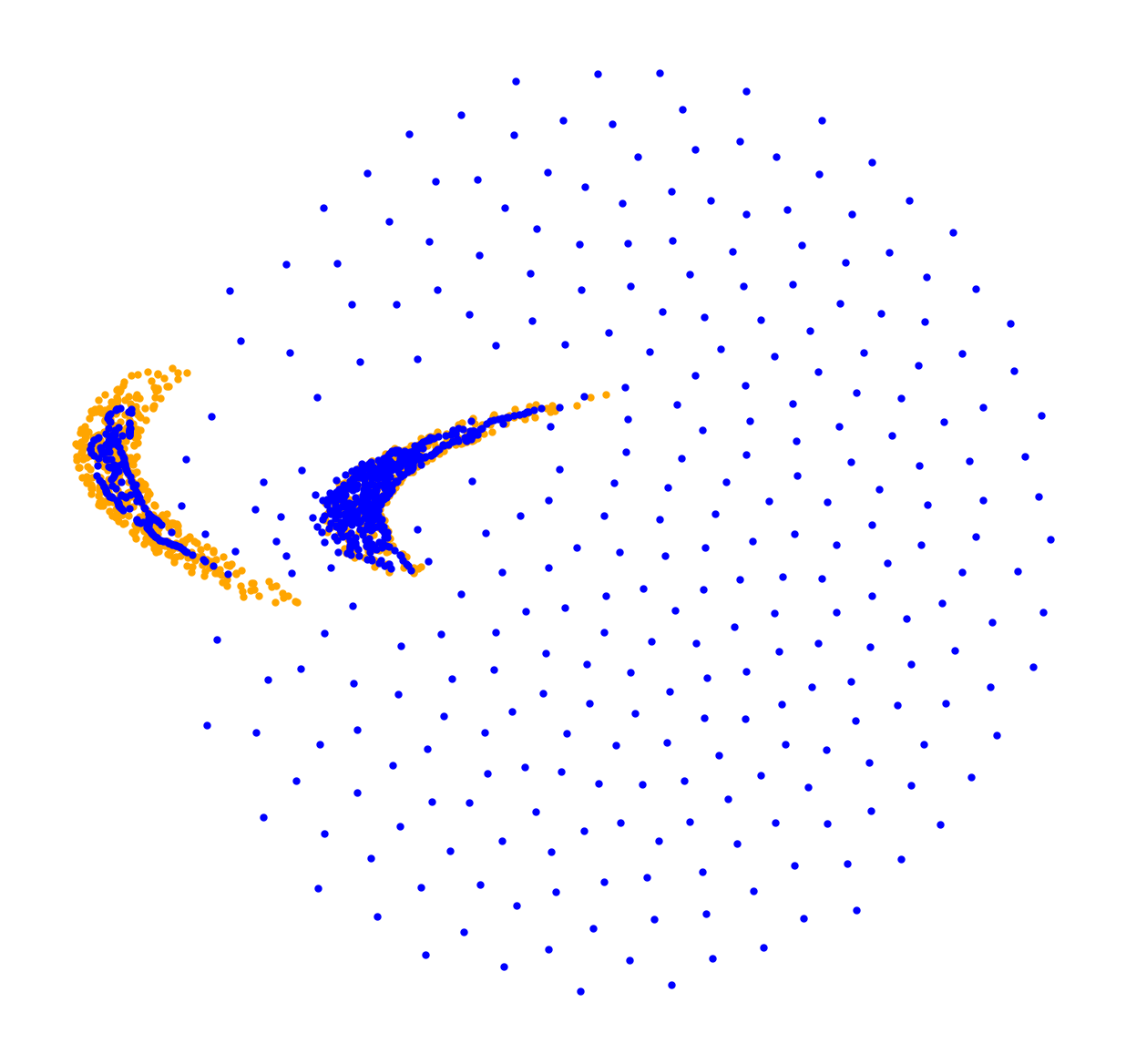}
    \end{subfigure}%
    \begin{subfigure}[t]{.24\textwidth}
        \includegraphics[width=\linewidth, valign=c]{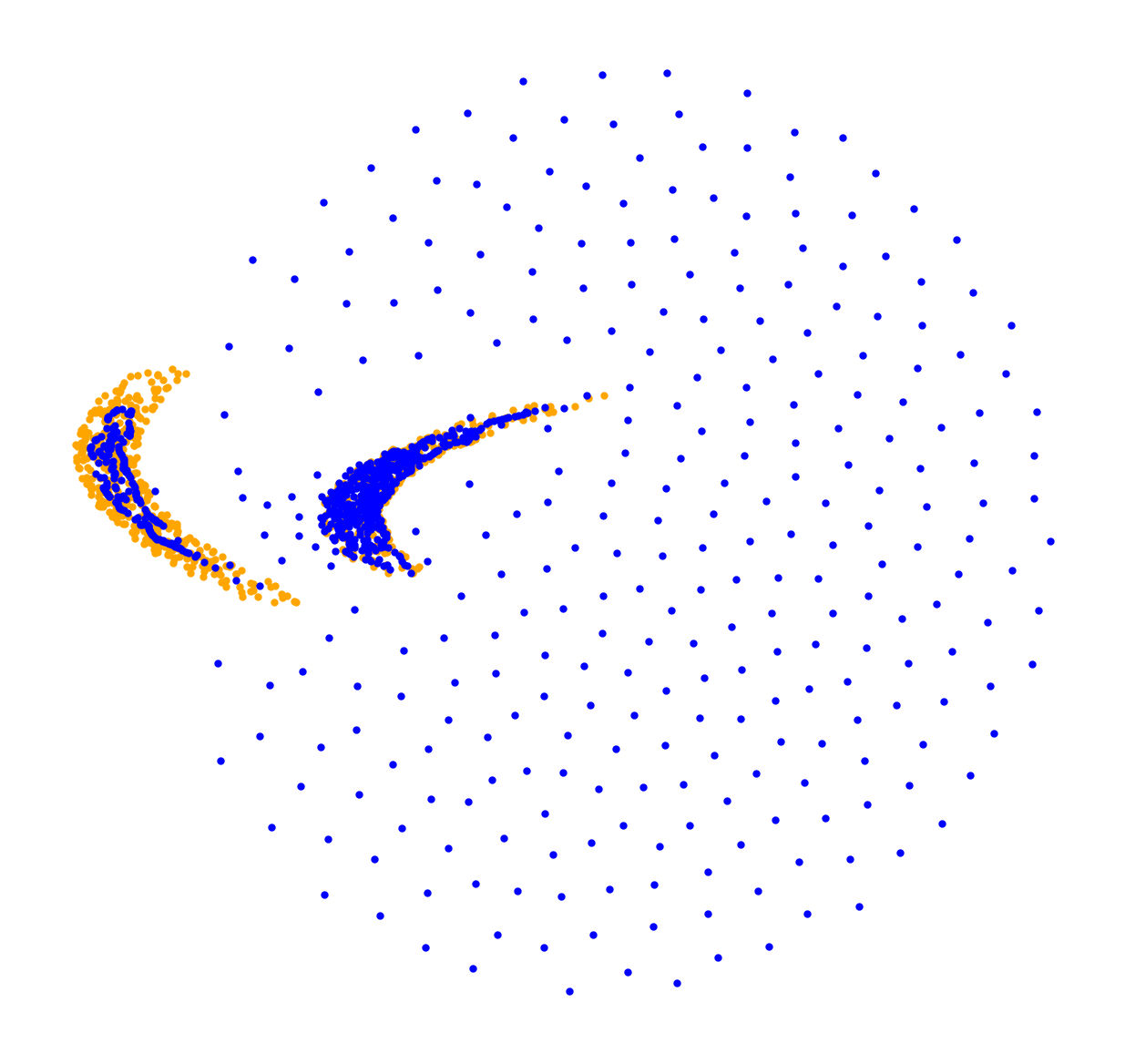}
    \end{subfigure}%
    \begin{subfigure}[t]{.24\textwidth}
        \includegraphics[width=\linewidth, valign=c]{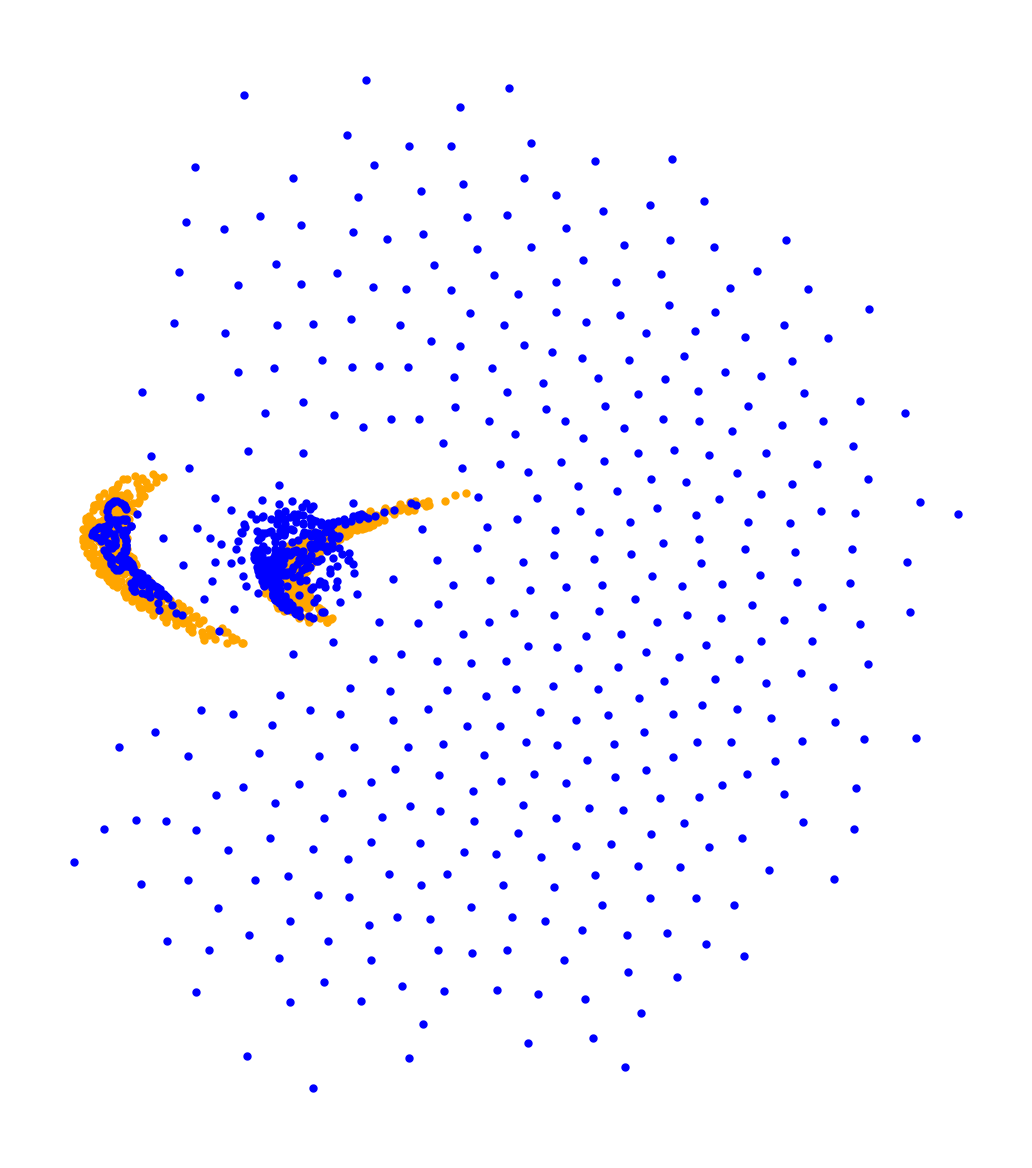}
    \end{subfigure} \\ %
    \rotatebox{90}{
        \begin{minipage}{.1\textwidth}
        \centering 
        \small $\lambda = 10$
    \end{minipage}} \hfill
    \begin{subfigure}[t]{.24\textwidth}
        \includegraphics[width=\linewidth, valign=c]{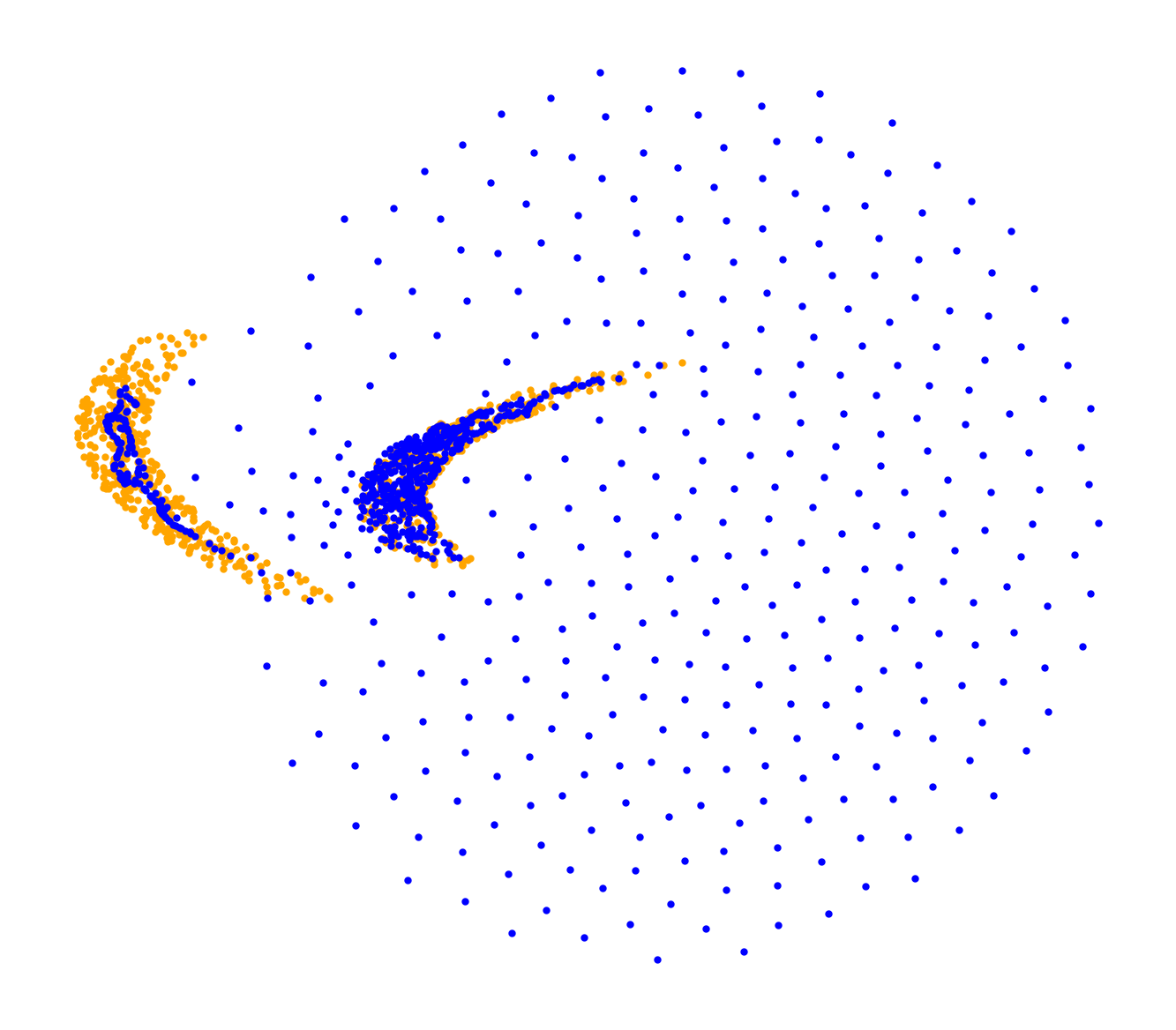}
        \caption*{$\tau = \num{e-3}$}
    \end{subfigure}%
    \begin{subfigure}[t]{.24\textwidth}
        \includegraphics[width=\linewidth, valign=c]{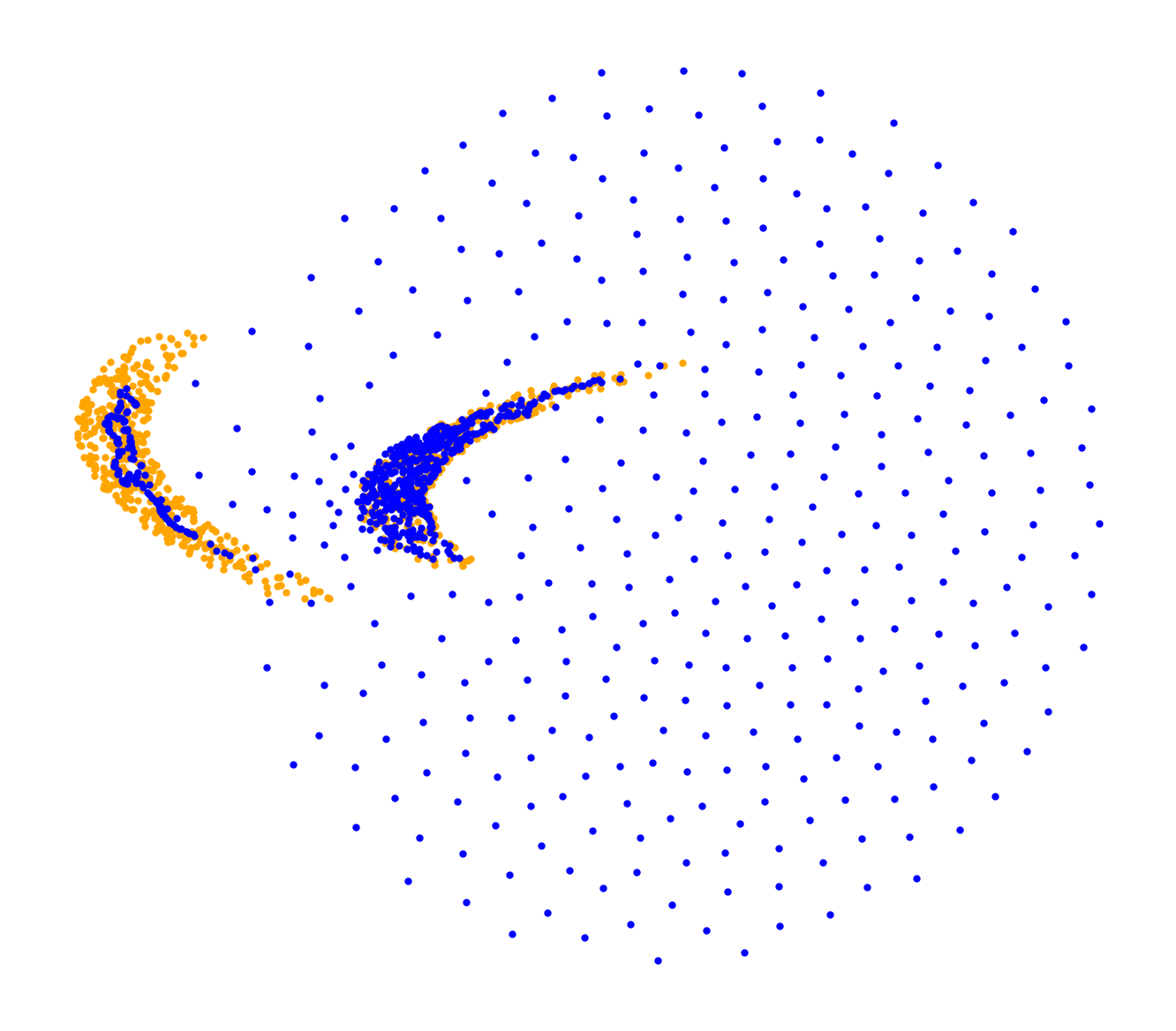}
        \caption*{$\tau = \num{e-2}$}
    \end{subfigure}%
    \begin{subfigure}[t]{.24\textwidth}
        \includegraphics[width=\linewidth, valign=c]{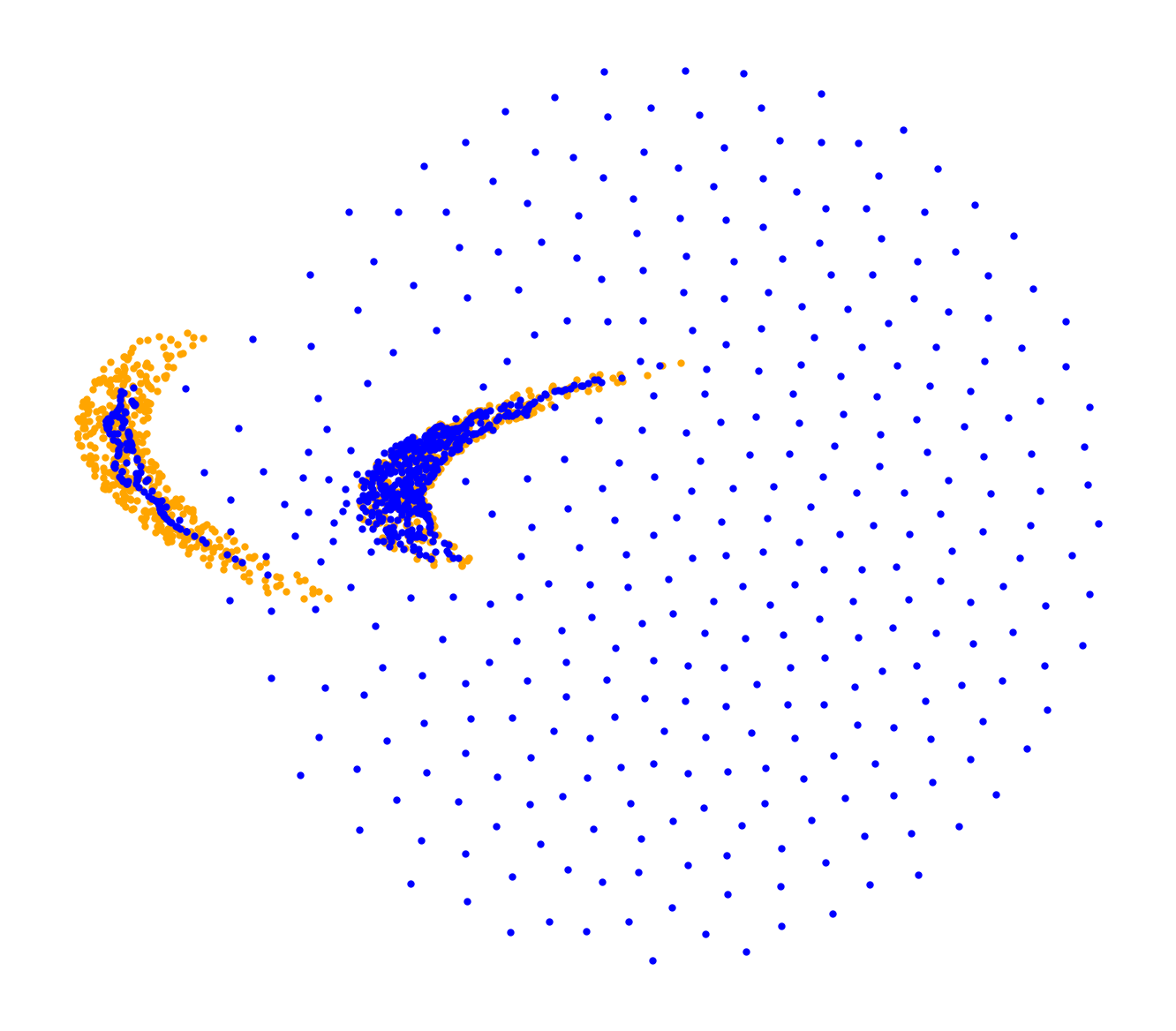}
        \caption*{$\tau = \num{e-1}$}
    \end{subfigure}%
    \begin{subfigure}[t]{.24\textwidth}
        \includegraphics[width=\linewidth, valign=c]{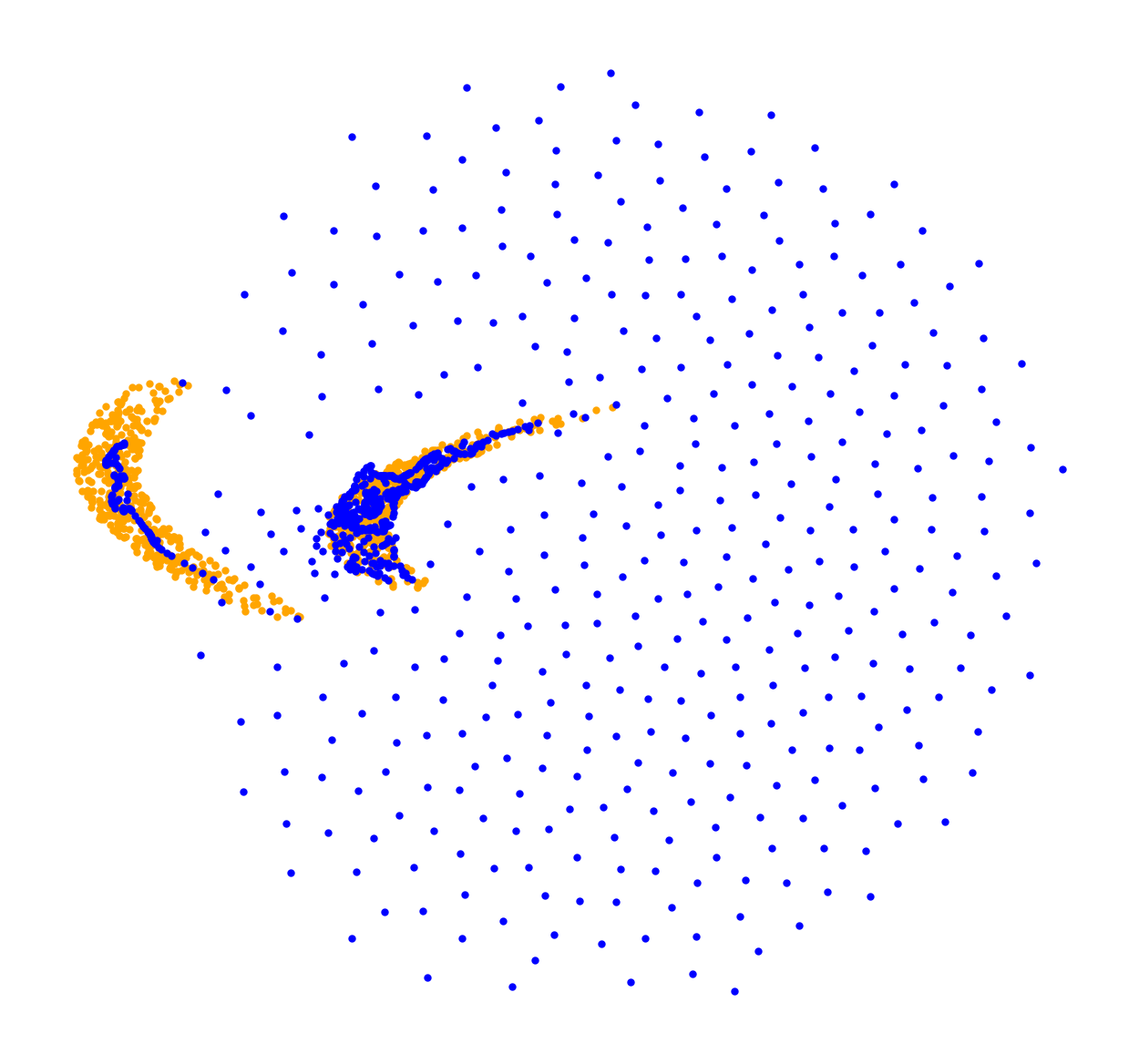}
        \caption*{$\tau = 1$}
    \end{subfigure}%
    \caption{
    Here, we use $\alpha = 2$, the inverse multiquadric kernel with width $\sigma^2 = \num{5e-2}$, and $N = 900$.
    We plot the configuration of the particles at $t = 100$.
    One should not choose $\tau$ to be larger than $\lambda$ by orders of magnitude.}
    \label{fig:lambdaVStau}
\end{figure}
The results are provided in Figure \ref{fig:lambdaVStau}.
We observed that the behavior for $\lambda \in \{ \num{e2}, \num{e3} \}$ is nearly identical to the behavior for $\lambda = 10$.

\end{document}